\documentclass[12pt]{article}
\usepackage{amsmath}
\usepackage{graphicx}
\usepackage{enumerate}
\usepackage{natbib}
\usepackage{booktabs} 
\usepackage{titlesec}
\usepackage{chngcntr}
\usepackage{amsmath, amsthm}
\usepackage{hyperref}
\usepackage{tocloft}
\usepackage{url} 
\usepackage{mathtools}
\usepackage{xcolor}
\usepackage{comment}
\usepackage{algorithm}
\usepackage{algpseudocode}
\usepackage{tikz}
\usetikzlibrary{arrows.meta, positioning}
\usepackage{graphicx}
\usepackage{setspace}
\usepackage{xr}
\usepackage{hyperref}
\usepackage{subcaption}
\usepackage{cleveref}
\usepackage{amssymb} 
\usepackage{float}
\usepackage{amsthm}
\usepackage{multirow}

\DeclareMathOperator{\tr}{tr}
\newtheorem{theorem}{Theorem}
\newtheorem{definition}{Definition}
\newtheorem{lemma}{Lemma}

\newtheorem{remark}{Remark}
\newtheorem{assumption}{Assumption}
\newtheorem{corollary}{Corollary}
\newcommand{\blind}{1}

\usetikzlibrary{positioning,arrows.meta}

\begin{document}

\def\spacingset#1{\renewcommand{\baselinestretch}%
{#1}\small\normalsize} \spacingset{1}


\if1\blind
{
  \title{\bf Decentralized Inference for Spatial Data Using Low-Rank Models}
  \author{Jianwei Shi, Sameh Abdulah, Ying Sun, and Marc G. Genton\\
  Statistics Program,\\
King Abdullah University of Science and Technology,\\ Thuwal, 23955, Saudi Arabia.}
  \maketitle
} \fi

\if0\blind
{
  \bigskip
  \bigskip
  \bigskip
  \begin{center}
    {\LARGE Decentralized Inference for Distributed Geospatial Data Using Low-Rank Models}
\end{center}
  \medskip
} \fi

\bigskip

\begin{abstract}
Advancements in information technology have enabled the creation of massive spatial datasets, driving the need for scalable and efficient computational methodologies. While offering viable solutions, centralized frameworks are limited by vulnerabilities such as single-point failures and communication bottlenecks. This paper presents a decentralized framework tailored for parameter inference in spatial low-rank models to address these challenges. A key obstacle arises from the spatial dependence among observations, which prevents the log-likelihood from being expressed as a summation—a critical requirement for decentralized optimization approaches. To overcome this challenge, we propose a novel objective function leveraging the evidence lower bound, which facilitates the use of decentralized optimization techniques. Our approach employs a block descent method integrated with multi-consensus and dynamic consensus averaging for effective parameter optimization. We prove the convexity of the new objective function in the vicinity of the true parameters, ensuring the convergence of the proposed method. Additionally, we present the first theoretical results establishing the consistency and asymptotic normality of the estimator within the context of spatial low-rank models. Extensive simulations and real-world data experiments corroborate these theoretical findings, showcasing the robustness and scalability of the framework.

\end{abstract}

\noindent%

\vfill

\newpage
\spacingset{1.9} 
\section{Introduction}
\label{sec:intro}
Advances in information technology, including geographic information systems (GIS) and global positioning systems (GPS), have significantly enhanced the ability to generate and manage large spatial datasets across various domains. For instance, as noted in \cite{eldawy2016era}, NASA's archive of satellite Earth images exceeds 500 TB and grows daily by 25 GB. Similarly, space telescopes produce up to 150 GB of spatial data weekly, while medical devices generate spatial images, such as X-rays, at an astonishing rate of 50 PB per year. These advancements underscore the critical role of innovation methods in addressing the increasing demand for efficient spatial data management and analysis.


Today, the availability of edge devices such as sensors and smartphones enables data processing and computation to take place closer to the source of data generation. This reduces latency, conserves bandwidth, and enhances real-time decision-making capabilities. When feasible, spatial modeling processes can be executed directly at individual endpoints, minimizing computational overhead and enabling faster response times. This is particularly important when data are initially dispersed across decentralized storage locations, as transferring massive datasets to a central computing node can be prohibitively expensive or challenging due to privacy concerns \citep{li2020federated}. Furthermore, even when data are already centralized, distributing the computational workload across multiple endpoint machines or devices can facilitate handling large spatial datasets and accelerate the modeling process. For example, divide-and-conquer algorithms can partition spatial data across multiple machines or computational units, significantly improving processing efficiency.

Coupling existing spatial modeling approximation algorithms with distributed computing can significantly reduce overall computational costs. Examples of such algorithms include covariance tapering \citep{furrer2006covariance}, Vecchia approximation \citep{katzfuss2021general,pan2024gpu}, and low-rank approximation \citep{banerjee2008gaussian,mondal2023tile}. In this study, we focus on spatial low-rank models, which can be represented through a linear combination of spatial basis functions with random coefficients.  Low-rank models are particularly well-suited for large spatial datasets, offering scalability to massive data sizes while often providing favorable predictive performance compared to other methods in certain scenarios \citep{bradley2016comparison}.  Although low-rank models have limitations in capturing fine-scale spatial variability \citep{stein2014limitations}, resulting in unsatisfactory performance in some scenarios if the rank is not sufficiently large, the computational efficiency provided by distributed computation allows for the use of larger ranks,  overcoming the limitations to some extent. Moreover, low-rank models form the foundational block of several refined approximation algorithms, such as hierarchical approaches \citep{huang2018hierarchical,abdulah2018parallel} and multi-resolution methods \citep{nychka2015multiresolution,katzfuss2017multi,chen2023linear}.

\cite{katzfuss2017parallel} proposed a centralized distributed framework for modeling large spatial datasets. A central machine is critical in orchestrating the local modeling processes across individual machines in such systems. However, this centralized architecture is susceptible to single-point failures and can experience communication bottlenecks when many machines transmit data to the central node \citep{gabrielli2023survey}. Additionally, the proposed framework relies on a Bayesian inference approach that requires numerous iterations of Markov Chain Monte Carlo (MCMC) sampling to achieve convergence, further exacerbating these challenges. Furthermore, there is a lack of corresponding theoretical development (such as parameter inference efficiency and MCMC convergence) addressing the efficiency of their framework.

In this study, we explore decentralized inference as an alternative to overcome the limitations of centralized frameworks. In a decentralized network, machines are interconnected through a distributed communication structure, which enhances resilience to single-machine failures \citep{li2010resilience} and mitigates communication congestion \citep{kermarrec2015want}. Furthermore, we employ maximum likelihood estimation (MLE) for parameter inference, improving computational efficiency and avoiding the resource-intensive demands of the Bayesian methods used in prior approaches. However, a significant challenge in performing decentralized inference for parameter estimation lies in optimizing the log-likelihood function. Due to the spatial dependence between observations, the log-likelihood function cannot be directly expressed as the sum of local and common functions, as is possible in cases with independent observations. In such independent scenarios, each local function is derived from data on an individual machine, while the common function is shared across all machines. This limitation poses a problem because existing decentralized optimization methods rely on this type of representation \citep{ryu2022large}.

To address this challenge, we borrow the idea from variational inference \citep{blei2017variational}. Instead of directly optimizing the log-likelihood, we optimize the evidence lower bound (ELBO), which, by definition, serves as a lower bound to the log-likelihood. In this case, it can be shown that maximizing the ELBO is equivalent to maximizing the log-likelihood. By conditioning on the latent variables, the observations become independent, allowing the ELBO to be expressed as the sum of local functions and a common function shared by all machines. This representation enables the application of decentralized optimization methods.  Specifically, we adopt the block descent method \citep{beck2013convergence} in combination with decentralized techniques such as dynamic consensus averaging and multi-consensus. The block descent approach is particularly well-suited for our problem, as optimizing parameters other than the covariance kernel parameters can be efficiently solved using explicit formulas.

We then establish the corresponding theory. First, we show that the negative ELBO as a function of the parameters is convex and has a positive Hessian matrix within a certain neighborhood of the true parameters. This property is then used to ensure the convergence of the proposed decentralized block descent method, assuming the initial points lie within this neighborhood. Such a condition is guaranteed when the local sample sizes are moderately large, with the initial point selected as a local minimizer. We also demonstrate that the minimizer, as the parameter estimator, is consistent and asymptotically normal when the model is correctly specified. To the best of our knowledge, this provides the first theory of consistency and asymptotic normality for the estimator under spatial low-rank models. Extensive simulations further validate our theoretical results.

The contributions of this paper are as follows: We propose a decentralized inference framework for parameter estimation in low-rank models, overcoming the limitations of centralized methods, such as susceptibility to single-machine failures and communication bottlenecks. By leveraging maximum likelihood estimation (MLE), our approach achieves greater computational efficiency than Bayesian inference. We provide the first theoretical analysis of consistency and asymptotic normality for the estimator in spatial low-rank models. Our results demonstrate that the evidence lower bound (ELBO) is convex with a positive Hessian matrix near the true parameters, hence guaranteeing the convergence of the decentralized block descent method. Furthermore, we introduce an efficient optimization strategy based on the ELBO, enabling decentralized processing of spatially dependent data. Finally, extensive simulations validate our theoretical findings, highlighting the effectiveness of the proposed method in terms of convergence and estimation accuracy.

The remainder of this paper is organized as follows: In Section \ref{sec:pre}, we introduce key preliminary concepts in decentralized optimization and discuss the challenges of applying existing decentralized methods to spatial low-rank models. Section \ref{sec:method} presents the development of our proposed decentralized method. In Section \ref{sec:theory}, we establish the theoretical framework supporting our approach. Finally, Sections \ref{sec:simu} and \ref{sec:realdata} provide extensive simulation results and a real data application to validate the effectiveness of the proposed method. Proof of theoretical results is provided in the Supplementary Material.


\section{Preliminaries}\label{sec:pre}
In this section, we introduce key preliminary concepts of decentralized optimization and discuss the challenges of applying existing decentralized optimization methods to spatial low-rank models.

\subsection{Decentralized Optimization}\label{sec:deop}
Decentralized optimization \citep{yang2019survey} involves multiple machines collaboratively solving an optimization problem without relying on a central authority or coordinator. The optimization problem is formulated as follows:  
\begin{equation}\label{eq:do}
   \min_{\boldsymbol{x}} \sum_{j=1}^J f_j(\boldsymbol{x}) + h(\boldsymbol{x}),
\end{equation}  
where \( f_j(\boldsymbol{x}) \) represents a local function specific to machine \( j \), and \( h(\boldsymbol{x}) \) is a common function shared by all machines.  The machines are organized into a decentralized network represented as \( (V, \mathcal{G}) \), where \( V \) denotes the set of vertices (machines) and \( \mathcal{G} \subset V \times V \) represents the edges (connections). Machines \( j_1 \) and \( j_2 \) are connected if \( (j_1, j_2) \in \mathcal{G} \), and direct communication is only possible between connected machines. For simplicity, we assume the network is undirected, meaning \( (j_1, j_2) \in \mathcal{G} \) if and only if \( (j_2, j_1) \in \mathcal{G} \).

\subsubsection{Weight Matrix}

In decentralized optimization, approximating averages is often essential. To enable this in a decentralized network, a weight matrix $\boldsymbol{W} = (w_{ij})$ is introduced and constructed to align with the structure of the network. This matrix supports decentralized averaging through the following process:
\begin{equation}
    y^t_j=\sum_i w_{ij} y^{t-1}_i \text{ with } y^0_j=a_j,
\end{equation}
where $w_{ij}\not =0$ only if machine $i$ and machine $j$ are connected.  
   Let $\boldsymbol{1}=(1,\ldots,1)^{\top}\in\mathbb{R}^J$, recall that a square matrix $\boldsymbol{W} \in \mathbb{R}^{J \times J}$ is doubly stochastic if $
\boldsymbol{W} \boldsymbol{1}=\boldsymbol{1} \text { and } \boldsymbol{1}^{\top} \boldsymbol{W}=\boldsymbol{1}^{\top}.
$ The following result shows that doubly stochasticity is necessary to achieve consensus.
\begin{lemma}[\cite{xiao2004fast}]\label{le:weight}
    Let $\boldsymbol{y}^{t}=\boldsymbol{W} \boldsymbol{y}^{t-1}$ with $\boldsymbol{y}^{0} \in \mathbb{R}^{J}$, then $\boldsymbol{y}^{t} \rightarrow$ $\frac{1}{J} \mathbf{1} \mathbf{1}^{\top} \boldsymbol{y}^{0}$ if and only if $\boldsymbol{W}$ is doubly stochastic and the mixing rate $\rho_{\boldsymbol{W} }$ satisfies \vspace{-15pt}
    \begin{equation}\label{eq:mixrate}
        \rho_{\boldsymbol{W} }:=\lambda_{\max }\left(\boldsymbol{W}-\frac{1}{J} \mathbf{1} \mathbf{1}^{\top}\right)<1. \vspace{-10pt}
    \end{equation}
\end{lemma}
Here, the mixing rate $\rho_{\boldsymbol{W} }$ is actually the per step convergence factor (linear factor of convergence) of the consensus error $\|\boldsymbol{y}^{t}-\frac{1}{J} \mathbf{1} \mathbf{1}^{\top} \boldsymbol{y}^{0}\|$, that is, $\|\boldsymbol{y}^{t}-\frac{1}{J} \mathbf{1} \mathbf{1}^{\top} \boldsymbol{y}^{0}\|\le \rho_{\boldsymbol{W} } \|\boldsymbol{y}^{t-1}-\frac{1}{J} \mathbf{1} \mathbf{1}^{\top} \boldsymbol{y}^{0}\|$.

A corresponding doubly stochastic matrix $\boldsymbol{W}$ with $\rho_{\boldsymbol{W} }<1$  exists and is not unique for any connected network. Commonly used doubly stochastic matrices with  $\rho_{\boldsymbol{W} }<1$ can be constructed without requiring full knowledge of the network's topology \citep{xiao2004fast}. Moreover, with additional effort, the optimal $\boldsymbol{W}$, minimizing $\rho_{\boldsymbol{W} }$, can be computed in a decentralized manner as described in \cite{rokade2022distributed}.

\subsubsection{Dynamic Consensus Averaging and Multi-Consensus}\label{sec:dac}
When local quantities $a_j=a_j(x_j^t)$ depending on other variables $x_j^t$ that vary across iterations, the dynamic consensus averaging technique, combined with multi-consensus \citep{ye2023multi}, is employed to reduce approximation errors.  In Section 
\ref{sec:supdac} of Supplement Material, we present a brief analysis demonstrating that dynamic consensus averaging achieves smaller approximation errors than simple weighting. 

Specifically, to approximate the average $\frac{1}{J}\sum_j a_j(x_j^t)$, dynamic consensus averaging leverages historical information as follows: $y^t_j=\sum_i w_{ij}[ y^{t-1}_i +a_i(x_i^t)-a_i(x_i^{t-1})]\text{ with } y^0_j=a_j(x_i^0).$ 
To further minimize approximation error, multi-consensus uses multiple rounds of weighted averaging within each iteration. This process is represented as: $y^{t}_j=y^{t-1,S}_j$ with $y^{t-1,s}_j=\sum_i w_{ij}[ y^{t-1,s-1}_i +a_i(x_i^t)-a_i(x_i^{t-1})]\text{ and } y^{t-1,0}_j=y^{t-1}_j,s=1,\ldots,K,$
which can be equivalently written as
\vspace{-15pt}
\begin{equation}\label{eq:dac}
    y^t_j=\sum_i w^{[K]}_{ij}[ y^{t-1}_i +a_i(x_i^t)-a_i(x_i^{t-1})],
    \vspace{-10pt}
\end{equation}
where $w^{[K]}_{ij}$ represents the $(i ,j)$-th elements of the matrix $\boldsymbol{W}^K$, i.e., $w^{[K]}_{ij}:=[\boldsymbol{W}^K]_{ij}$.

\subsection{Challenges in Spatial Low-Rank Models}\label{sec:low_rank}

We consider the low-rank model of distributed spatial data as follows:
\begin{equation}\label{eq: model}
    z (\mathbf{s}_{j i}) = \boldsymbol{x}_{j i}^{\top} \boldsymbol{\gamma} +
   \boldsymbol{b} (\mathbf{s}_{j i};\boldsymbol{\theta})^{\top} \boldsymbol{\eta} \boldsymbol{} +
   \varepsilon (\mathbf{s}_{j i}), i = 1, \ldots n_j, j = 1, \ldots, J,
\end{equation}
Here, \( j \) indexes the machines, and \( i \) indexes the data points within each machine. The location of each observation is represented by \( \mathbf{s}_{j i} \), and the observed measurement at \( \mathbf{s}_{j i} \) is \( z (\mathbf{s}_{j i}) \). The covariate vector \( \boldsymbol{x}_{j i} \) is associated with an unknown coefficient vector \( \boldsymbol{\gamma} \), while \( \boldsymbol{b}(\cdot;\boldsymbol{\theta}) \) is a known \( m \)-dimensional vector-valued function defined for each parameter set \( \boldsymbol{\theta} \). The random vector \( \boldsymbol{\eta} \sim \mathcal{N}_{m} (\boldsymbol{0}_m, \boldsymbol{K}(\boldsymbol{\theta})) \) follows a Gaussian distribution with covariance matrix \( \boldsymbol{K}(\boldsymbol{\theta}) \), and \( \varepsilon (\mathbf{s}_{j i}) \sim \mathcal{N}(0, \tau^2) \), known as the nugget or measurement error, is a normal random variable with an unknown variance parameter \( \tau^2 \). The objective is to infer the parameters \( (\boldsymbol{\gamma}, \tau, \boldsymbol{\theta}) \) and to make predictions at new locations with given covariates, leveraging data distributed across multiple machines.

The low-rank model is of paramount importance in spatial statistics due to its ability to balance computational efficiency and statistical accuracy. The reduced-rank structure allows for scalable computation while preserving the essential characteristics of spatial dependence. Significant advances in low-rank modeling have been made in recent years.  One prominent example is the Gaussian predictive process (GPP) model  \citep{banerjee2008gaussian}. Let $w(\cdot)\sim GP(0,c)$ be a zero mean Gaussian process with covariance function $c(\mathbf{s},\mathbf{s}';\boldsymbol{\theta})$ and $S^*=\{\mathbf{s}_1,\mathbf{s}_2,\ldots,\mathbf{s}_m\}$ represent a set of locations, referred to as knots. In the Gaussian predictive process model, we define  $\boldsymbol{\eta}=\boldsymbol{w}(S^*)$ and $\boldsymbol{b} (\mathbf{s}_{j i};\boldsymbol{\theta})^{\top}=\boldsymbol{c}(\mathbf{s},S^*;\boldsymbol{\theta})\boldsymbol{C}(S^*,S^*;\boldsymbol{\theta})^{-1}$. Here, $\boldsymbol{w}(S^*)=(w(\mathbf{s}_1),w(\mathbf{s}_2),\ldots,w(\mathbf{s}_m))^{\top}$ represents the values of the process at the knots,
$\boldsymbol{C}(S^*,S^*;\boldsymbol{\theta})=(c(\mathbf{s}_i,\mathbf{s}_j))_{{i,j}\in\{1,2,\ldots,m\}}$ is the covariance matrix between the knots, and
$\boldsymbol{c}(\mathbf{s},S^*;\boldsymbol{\theta})=(c(\mathbf{s},\mathbf{s}_1;\boldsymbol{\theta}),c(\mathbf{s},\mathbf{s}_2;\boldsymbol{\theta}),\ldots,c(\mathbf{s},\mathbf{s}_m;\boldsymbol{\theta}))$ is the vector of covariances between the location $\mathbf{s}$ and the knots. Other notable developments include multi-resolution approaches \citep{nychka2015multiresolution} and hierarchical models \citep{huang2018hierarchical}, which extend low-rank techniques to address varying spatial scales and complexities.

For convenience, let us define $\boldsymbol{z}_j = (z (\mathbf{s}_{j
1}), \ldots, z (\mathbf{s}_{j n_j}))^\top$, $\boldsymbol{X}_j = (\boldsymbol{x}_{j 1} , \ldots,
\boldsymbol{x}_{j n_j} )^{\top}$, $\boldsymbol{B}_j(\boldsymbol{\theta}) = (\boldsymbol{b} (\mathbf{s}_{j 1};\boldsymbol{\theta}), \ldots,
\boldsymbol{b} (\mathbf{s}_{j n_j};\boldsymbol{\theta}))^{\top}$, $\boldsymbol{\varepsilon}_j =
(\varepsilon (\mathbf{s}_{j 1}), \ldots, \varepsilon (\mathbf{s}_{j
n_j}))^\top$, and aggregate them as $\boldsymbol{z}  = (\boldsymbol{z}_1^\top, \ldots, \boldsymbol{z}_J^\top)^\top$, $ \boldsymbol{X}  =
( \boldsymbol{X}_1^{\top}, \ldots, \boldsymbol{X}_J^{\top})^\top$, $\boldsymbol{B}(\boldsymbol{\theta})  = (\boldsymbol{B}_1^\top(\boldsymbol{\theta}), \ldots \boldsymbol{B}_J^\top(\boldsymbol{\theta}))^\top$, $\boldsymbol{\varepsilon} 
= (\boldsymbol{\varepsilon}_1^\top, \ldots, \boldsymbol{\varepsilon}_J^\top)^\top$.
Thus, we can rewrite the model in Equation \eqref{eq: model} as $\boldsymbol{z} =  \boldsymbol{X} \boldsymbol{\gamma} + \boldsymbol{B}(\boldsymbol{\theta}) \boldsymbol{\eta} +
   \boldsymbol{\varepsilon} .$
  The negative log-likelihood function is given by \vspace{-15pt}
\begin{equation}\label{eq:log_lik}
\begin{split}
   \mathcal{L} \left( \boldsymbol{\gamma},
     \delta, \boldsymbol{\theta} \right):= \log p(\boldsymbol{z}) = & \frac{N}{2} \log 2\pi 
    +\frac{1}{2} \log \left|\delta^{-1} \boldsymbol{I} + \boldsymbol{B}(\boldsymbol{\theta}) \boldsymbol{K}(\boldsymbol{\theta}) \boldsymbol{B}^{\top}\boldsymbol{\theta})\right| \\
    & +\frac{1}{2}(\boldsymbol{z} - \boldsymbol{X} \boldsymbol{\gamma})^{\top} 
    \left(\delta^{-1}  \boldsymbol{I} + \boldsymbol{B}(\boldsymbol{\theta}) \boldsymbol{K}(\boldsymbol{\theta}) \boldsymbol{B}^{\top}(\boldsymbol{\theta})\right)^{-1}  (\boldsymbol{z} -  \boldsymbol{X} \boldsymbol{\gamma}).
\end{split}\vspace{-15pt}
\end{equation}
Here,  $\delta=\tau^{-2}$ , $N = \sum_j n_j$ represents the total sample size. Due to the spatial correlation among the observations, this form of negative log-likelihood $\mathcal{L} \left( \boldsymbol{\gamma},
     \delta, \boldsymbol{\theta} \right)$ cannot be expressed as Equation \eqref{eq:do}. Consequently, existing decentralized optimization methods and their associated convergence theories are not directly applicable to $\mathcal{L} \left( \boldsymbol{\gamma},
     \delta, \boldsymbol{\theta} \right)$.

\section{The Proposed Decentralized Method}\label{sec:method}

As discussed in Section \ref{sec:low_rank}, existing decentralized methods cannot be directly applied because the negative log-likelihood function $\mathcal{L} \left( \boldsymbol{\gamma},
     \delta, \boldsymbol{\theta} \right)$ does not conform to the structure of Equation \eqref{eq:do}. To address this limitation, we adopt an approach inspired by variational inference.

Instead of optimizing the negative log-likelihood $\mathcal{L} \left( \boldsymbol{\gamma},
     \delta, \boldsymbol{\theta} \right)$ directly, we optimize the evidence lower bound (ELBO), equivalent to maximizing the log-likelihood function,  as demonstrated below. Notably, the ELBO can be decomposed into the sum of local functions and a common function shared among all machines. This decomposition enables the use of existing decentralized optimization techniques effectively.

Let \( q(\cdot) : \mathbb{R}^m \to \mathbb{R} \) be a density function. The 
ELBO, as a function of $q$, is defined as
\[  \text{ELBO} (q) := \int q (\boldsymbol{\eta}) \log \frac{p (\boldsymbol{z}
   | \boldsymbol{\eta}  ) p (\boldsymbol{} \boldsymbol{\eta})}{q
   (\boldsymbol{\eta})} \mathrm{d} \boldsymbol{\eta} =\mathbb{E}_q \log p (\boldsymbol{z} |
   \boldsymbol{\eta}  ) -  \text{KL} (q(\boldsymbol{\eta}) \| p (\boldsymbol{\eta})). \]
According to Jensen's inequality, $ \text{ELBO}$ is a lower bound of the 
log-likelihood, that is, $\text{ELBO} \leqslant \log p(\boldsymbol{z})$, with equality achieved if and only if $q (\boldsymbol{\eta}) = p (\boldsymbol{}
  \boldsymbol{\eta} | \boldsymbol{z})$. Since $p (
\boldsymbol{\eta} | \boldsymbol{z})$ is a Gaussian density, it can be parameterized by a mean vector $\boldsymbol{\mu}$ and a covariance matrix $\boldsymbol{\Sigma}$. Additionally, with  the conditional independence of $\boldsymbol{z}_1,\ldots,\boldsymbol{z}_J$ given $\boldsymbol{\eta}$, we have  $\log p (\boldsymbol{z} |
\boldsymbol{\eta}  ) = \sum_j \log p (\boldsymbol{z}_j | \boldsymbol{\eta}
)$. This leads to
\begin{equation}\label{eq:ob1}
    {\max_{\boldsymbol{\gamma}, \delta, \boldsymbol{\theta}}}  \log p(\boldsymbol{z}) = \max_{\boldsymbol{\gamma},  \delta, \boldsymbol{\theta},
   \boldsymbol{\mu}, \boldsymbol{\Sigma}} \left\{ \sum_j \mathbb{E}_{\mathcal{N} (\boldsymbol{\eta} |
   \boldsymbol{\mu}, \boldsymbol{\Sigma}  )} \left\{ \log p (\boldsymbol{z}_j |
   \boldsymbol{\eta}  )\right\}-  \text{KL} (\mathcal{N} (\boldsymbol{\eta} |
   \boldsymbol{\mu}, \boldsymbol{\Sigma}  ) \|  p (\boldsymbol{\eta})) \right\},
\end{equation}
where $\delta=\tau^{-2}$.  The negative ELBO on the right-hand side of Equation \eqref{eq:ob1} becomes \vspace{-10pt}
\begin{equation}\label{eq:ob2}
f(\boldsymbol{\mu},\boldsymbol{\Sigma},\boldsymbol{\gamma},\delta,\boldsymbol{\theta}):= \sum_j f_j(\boldsymbol{\mu},\boldsymbol{\Sigma},\boldsymbol{\gamma},\delta,\boldsymbol{\theta})+h(\boldsymbol{\mu},\boldsymbol{\Sigma},\boldsymbol{\theta}),
\vspace{-30pt}
\end{equation}
where 
\vspace{-20pt}
\begin{align}
  f_j:=&-\mathbb{E}_{\mathcal{N} (\boldsymbol{\eta} |
  \boldsymbol{\mu}, \boldsymbol{\Sigma}  )} \left\{  \log p (\boldsymbol{z}_j |
  \boldsymbol{\eta}  )  \right\} \nonumber\\
  =&-\frac{n_j}{2} \log(2\pi \delta) +\frac{1}{2} \delta
\tr(\boldsymbol{B}_j(\boldsymbol{\theta})^{\top}\boldsymbol{B}_j(\boldsymbol{\theta})(\boldsymbol{\Sigma}+\boldsymbol{\mu\mu^{\top}}))\nonumber\\&-\delta(\boldsymbol{z}_j-\boldsymbol{X}_j\boldsymbol{\gamma})^{\top}\boldsymbol{B}_j(\boldsymbol{\theta})\boldsymbol{\mu}+\frac{1}{2}\delta(\boldsymbol{z}_j-\boldsymbol{X}_j\boldsymbol{\gamma})^{\top}(\boldsymbol{z}_j-\boldsymbol{X}_j\boldsymbol{\gamma}),\nonumber\\
h:=&\text{KL} (\mathcal{N} (\boldsymbol{\eta} |
  \boldsymbol{\mu}, \boldsymbol{\Sigma}  ) \|  p (\boldsymbol{\eta}))
   =\frac{1}{2}\left[\boldsymbol{\mu}^{\top}\boldsymbol{K}^{-1}(\boldsymbol{\theta})\boldsymbol{\mu}+\tr(\boldsymbol{K}^{-1}(\boldsymbol{\theta})\boldsymbol{\Sigma})-
   \log\frac{\det(\boldsymbol{\Sigma})}{\det(\boldsymbol{K}(\boldsymbol{\theta}))}-m\right].\nonumber
    \vspace{-30pt}
\end{align}

Consequently, since the new objective function in Equation~\eqref{eq:ob2} now conforms to the structure of Equation~\eqref{eq:do}, existing decentralized optimization methods~\citep{ryu2022large} can be applied. Furthermore, apart from the covariance kernel parameter vector $\boldsymbol{\theta}$, the optimization of other parameters can be explicitly solved when the remaining parameters are held fixed. This insight motivates the adoption of a decentralized block coordinate descent method for optimization, as outlined below.

At the $t$-th iteration, suppose each machine $j\in\{1,\ldots,J\}$ holds the parameters $\boldsymbol{\mu}^{t}_j,\boldsymbol{\Sigma}^{t}_j,\boldsymbol{\gamma}^{t}_j,\delta^{t}_j,\boldsymbol{\theta}^{t}_j$. The parameters for the $(t+1)$-th iteration are then updated using a decentralized block coordinate descent approach, where the following subproblems are solved sequentially: \vspace{-15pt}
\begin{align}
  &\text{Update $\boldsymbol{\mu}$ and $\boldsymbol{\Sigma}:$  solving}\ \min _{\boldsymbol{\mu}, \boldsymbol{\Sigma}} f\left(\boldsymbol{\mu}, \boldsymbol{\Sigma}, \boldsymbol{\gamma}^t_j, \delta^t_j, \boldsymbol{\theta}^t_j\right)\ \text{to obtain}\ \boldsymbol{\mu}^{t+1}_j, \boldsymbol{\Sigma}^{t+1}_j;\label{eq:mu}\\
  &\text{Update  $\boldsymbol{\gamma}$: solving}\ \min _{\boldsymbol{\gamma}} f\left(\boldsymbol{\mu}^{t+1}_j, \boldsymbol{\Sigma}^{t+1}_j, \boldsymbol{\gamma}, \delta^t_j, \boldsymbol{\theta}^t_j\right) \text{to obtain}\ \boldsymbol{\gamma}^{t+1}_j\label{eq:beta};\\
  &\text{Update $\delta$: solving}\ \min _\delta f\left(\boldsymbol{\mu}^{t+1}_j, \boldsymbol{\Sigma}^{t+1}_j, \boldsymbol{\gamma}^{t+1}_j, \delta, \boldsymbol{\theta}^t_j\right) \text{to obtain}\  \delta^{t+1}_j;\label{eq:tau}\\
  &\text{Update $\boldsymbol{\theta}$: solving}\ \min _{\boldsymbol{\theta}} f\left(\boldsymbol{\mu}^{t+1}_j, \boldsymbol{\Sigma}^{t+1}_j, \boldsymbol{\gamma}^{t+1}_j, \delta^{t+1}_j, \boldsymbol{\theta}\right) \text{to obtain}\ \boldsymbol{\theta}^{t+1}_j. \label{eq:theta} \vspace{-15pt}
\end{align}

The exact solution to the optimization problem  in \eqref{eq:mu}  is
\begin{align}
  \vspace{-15pt}
\boldsymbol{\Sigma}^{*,t+1}_j=\left[J\delta^t_jJ^{-1}\sum_{i=1}^J\boldsymbol{B}_i(\boldsymbol{\theta}^t_j)^{\top}\boldsymbol{B}_i(\boldsymbol{\theta}^t_j)+\boldsymbol{K}^{-1}(\boldsymbol{\theta}^t_j)\right]^{-1}\nonumber,\\
  \boldsymbol{\mu}^{*,t+1}_j= \boldsymbol{\Sigma}^{*,t+1}_j\left(J\delta^t_j J^{-1}\sum_{i=1}^J\boldsymbol{B}_i\left(\boldsymbol{\theta}_j^{t}\right)^{\top}\left(\boldsymbol{z}_i-\boldsymbol{X}_i \boldsymbol{\gamma}_j^{t}\right)\right)\nonumber. \vspace{-10pt}
\end{align}
However, directly computing $J^{-1}\sum_{i=1}^J\boldsymbol{B}_i(\boldsymbol{\theta}^t_j)^{\top}\boldsymbol{B}_i(\boldsymbol{\theta}^t_j)$ and $J^{-1}\sum_{i=1}^J\boldsymbol{B}_i\left(\boldsymbol{\theta}_j^{t}\right)^{\top}\left(\boldsymbol{z}_i-\boldsymbol{X}_i \boldsymbol{\gamma}_j^{t}\right)$ is not feasible in a decentralized network. These terms involve averages across all machines, and the parameters $\boldsymbol{\theta}^t_j$ change at each iteration $t$.

Instead, we employ the dynamic consensus averaging technique combined with multi-consensus, introduced in Equation \eqref{eq:dac},  Section \ref{sec:dac}, to enable decentralized averaging without direct computation. Specifically, we use $\boldsymbol{Y}_{\boldsymbol{\Sigma}, j}^{t}$ and $\boldsymbol{y}_{\boldsymbol{\mu}, j}^{t}$, defined in Equation~\eqref{eq:ymu_Sigma}, to dynamically approximate the averages. Recall that $w^{[K]}_{ij}$ in the equation represents the $(i ,j)$-th elements of the matrix $\boldsymbol{W}^K$. In practice, the number of multi-consensus rounds, $K$, can be selected such that the difference between two consecutive iterations falls below a specified threshold. Simulations and theoretical analysis in subsequent sections consistently demonstrate that 
$K$ remains small. Specifically, we have: \vspace{-10pt}
\begin{equation}\label{eq:ymu_Sigma}
  \begin{split}
    \boldsymbol{Y}_{\boldsymbol{\Sigma}, j}^{t}&=\sum_i w^{[K]}_{i j}\left(\boldsymbol{Y}_{\boldsymbol{\Sigma}, i}^{t-1}+\boldsymbol{B}_i\left(\boldsymbol{\theta}_i^{t}\right)^{\top} \boldsymbol{B}_i\left(\boldsymbol{\theta}_i^{t}\right)-\boldsymbol{B}_i\left(\boldsymbol{\theta}_i^{t-1}\right)^{\top} \boldsymbol{B}_i\left(\boldsymbol{\theta}_i^{t-1}\right)\right),\\
     \boldsymbol{y}_{\boldsymbol{\mu}, j}^{t}&=\sum_i w^{[K]}_{i j}\left(\boldsymbol{y}_{\boldsymbol{\mu}, i}^{t-1}+\boldsymbol{B}_i\left(\boldsymbol{\theta}_i^{t}\right)^{\top}\left(\boldsymbol{z}_i-\boldsymbol{X}_i \boldsymbol{\gamma}_i^{t}\right)-\boldsymbol{B}_i\left(\boldsymbol{\theta}_i^{{t-1}}\right)^{\top}\left(\boldsymbol{z}_i-\boldsymbol{X}_i \boldsymbol{\gamma}_i^{t-1}\right)\right),
  \end{split}\vspace{-10pt}
\end{equation}
where $\boldsymbol{Y}_{\boldsymbol{\Sigma}, j}^0=\boldsymbol{B}_j\left(\boldsymbol{\theta}_j^0\right)^{\top} \boldsymbol{B}_j\left(\boldsymbol{\theta}_j^0\right)$ and $\boldsymbol{y}_{\boldsymbol{\mu}, j}^0=\boldsymbol{B}_j\left(\boldsymbol{\theta}_j^0\right)^{\top}\left(\boldsymbol{z}_j-\boldsymbol{X}_j \boldsymbol{\gamma}_j^0\right)$.
This leads to \vspace{-10pt}
\begin{equation*}
\boldsymbol{\mu}_j^{t+1}=\left[\delta^t_jJ\boldsymbol{Y}_{\boldsymbol{\Sigma}, j}^t+\sum_i w_{ij}\boldsymbol{K}^{-1}\left(\boldsymbol{\theta}^t_i\right)\right]^{-1} J\delta^t_j\boldsymbol{y}_{\boldsymbol{\mu}, j}^t,
\boldsymbol{\Sigma}_j^{t+1}=\left[\delta^t_jJ\boldsymbol{Y}_{\boldsymbol{\Sigma}, j}^t+\boldsymbol{K}^{-1}\left(\boldsymbol{\theta}_j^t\right)\right]^{-1}. \vspace{-10pt}
\end{equation*}

The exact optimizer of the problem in \eqref{eq:beta} is
\vspace{-10pt}
$$
\boldsymbol{\gamma}^{\ast,t+1}_j=\left(\frac{1}{J} \sum_i \boldsymbol{X}_i^{\top} \boldsymbol{X}_i\right)^{-1}\left(\frac{1}{J}\sum_i\left(\boldsymbol{X}_i^{\top} \boldsymbol{z}_i-\boldsymbol{X}_i^{\top} \boldsymbol{B}_i\left(\boldsymbol{\theta}_j^t\right) \boldsymbol{\mu}_j^{t+1}\right)\right).\vspace{-10pt}
$$
Similarly, directly computing the terms $\frac{1}{J} \sum_i \boldsymbol{X}_i^{\top} \boldsymbol{X}_i$ and $\frac{1}{J} \sum_i \left(\boldsymbol{X}_i^{\top} \boldsymbol{z}_i - \boldsymbol{X}_i^{\top} \boldsymbol{B}_i\left(\boldsymbol{\theta}_j^t\right) \boldsymbol{\mu}j^{t+1}\right)$ is infeasible in a decentralized network because it requires global averaging across all machines. To overcome this limitation, we approximate these terms using $\boldsymbol{Y}_{\boldsymbol{X}, j}^t$ and $\boldsymbol{y}_{\boldsymbol{\gamma}, j}^t$, as defined in Equation \eqref{eq:ybeta}. \vspace{-15pt}
\begin{equation}\label{eq:ybeta} 
  \begin{split}
    \boldsymbol{Y}_{\boldsymbol{X}, j}^t&=\sum_i w_{i j}^{[K]} \boldsymbol{Y}_{\boldsymbol{X}, i}^{t-1}, \boldsymbol{Y}_{\boldsymbol{X}, j}^0=\boldsymbol{X}_j^{\top} \boldsymbol{X}_j,\\
    \boldsymbol{y}_{\boldsymbol{\gamma}, j}^t&=\sum_i w^{[K]}_{i j}\left(\boldsymbol{y}_{\boldsymbol{\gamma}, i}^t-\boldsymbol{X}_i^{\top} \boldsymbol{B}_i\left(\boldsymbol{\theta}_i^t\right) \boldsymbol{\mu}_i^{t+1}+\boldsymbol{X}_i^{\top} \boldsymbol{B}_i\left(\boldsymbol{\theta}_i^{t-1}\right) \boldsymbol{\mu}_i^t\right),
    \end{split}
    \vspace{-15pt}
\end{equation}
where $\boldsymbol{y}_{\boldsymbol{\gamma}, j}^0=\left(\boldsymbol{X}_j^{\top} \boldsymbol{z}_j-\boldsymbol{X}_j^{\top} \boldsymbol{B}_j\left(\boldsymbol{\theta}_j^0\right) \boldsymbol{\mu}_j^1\right)$. Then, $\boldsymbol{\gamma}^{t+1}_j=\left(\boldsymbol{Y}_{
\boldsymbol{X}, j}^t\right)^{-1}\boldsymbol{y}_{\boldsymbol{\gamma}, j}^t.$
Note that $p$, the dimension of the parameter $\boldsymbol{\gamma}$, is typically small in spatial statistics. Consequently, transferring the $p \times p$ matrix is communication efficient in decentralized networks.

For the optimization problem in Equation \eqref{eq:tau}, the exact optimizer is  \vspace{-15pt}
\begin{equation}\label{eq:taut}
    \begin{aligned}
    \delta^{*,t+1}_j&=\frac{\sum_i n_i}{J}\left(\frac{\sum_i l_i\left(\boldsymbol{\mu}_j^{t+1}, \boldsymbol{\Sigma}_j^{t+1}, \boldsymbol{\gamma}_j^{t+1}, \boldsymbol{\theta}_j^{t}\right)}{J}\right)^{-1},\\
  \text{with\ } l_i(\boldsymbol{\mu}, \boldsymbol{\Sigma}, \boldsymbol{\gamma}, \boldsymbol{\theta})&:=\operatorname{tr}\left[\boldsymbol{B}_i(\boldsymbol{\theta})^{\top} \boldsymbol{B}_i(\boldsymbol{\theta})(\boldsymbol{\Sigma}+\boldsymbol{\mu} \boldsymbol{\mu}^{\top})\right]\\
  &\quad \quad-2\left(\boldsymbol{z}_i-\boldsymbol{X}_i \boldsymbol{\gamma}\right)^{\top} \boldsymbol{B}_i(\boldsymbol{\theta}) \boldsymbol{\mu}+\left(\boldsymbol{z}_i-\boldsymbol{X}_i\boldsymbol{\gamma}\right)^{\top}\left(\boldsymbol{z}_i-\boldsymbol{X}_i\boldsymbol{\gamma}\right).
\end{aligned}
\vspace{-15pt}
\end{equation}
Similarly, directly computing the averages in Equation \eqref{eq:taut} is not feasible. Instead, we use $y_{\delta, j}^t$ and $y_{n, j}^t$ to dynamically approximate them separately:
\vspace{-10pt}
\begin{equation}\label{eq:ydelta}
\begin{aligned}
     y_{\delta, j}^{t}&=\sum_i w^{[K]}_{i j}\left(y_{\delta, i}^{t-1}+l_i\left(\boldsymbol{\mu}_i^{t+1}, \boldsymbol{\Sigma}_i^{t+1}, \boldsymbol{\gamma}_i^{t+1}, \boldsymbol{\theta}_i^{t}\right)-l_i\left(\boldsymbol{\mu}_i^t, \boldsymbol{\Sigma}_i^t, \boldsymbol{\gamma}_i^t, \boldsymbol{\theta}_i^{t-1}\right)\right), \\ 
  y_{n,j}^t&=\sum_i  w_{i j}^{[K]}  y_{n,i}^{t-1},  
\end{aligned}
\vspace{-10pt}
\end{equation}
where  $y_{\delta, j}^0=l_j\left(\boldsymbol{\mu}_j^1, \boldsymbol{\Sigma}_j^1, \boldsymbol{\gamma}_j^1,\right. 
 \left.\boldsymbol{\theta}_j^0\right)$ and $y_{n,j}^{0}=n_j$.
Thus, $\delta^{t+1}_j=y_{n,j}^t(y_{\delta, j}^{t})^{-1}.$

 As for the problem in Equation \eqref{eq:theta}, we adopt the decentralized Newton-type method. Recall the definition of $f_j$ and $h$ in Equation \eqref{eq:ob1}. Let $\boldsymbol{G}_{f_i,j}^t(\boldsymbol{\theta})=\frac{\partial f_i\left(\boldsymbol{\mu}_j^{t+1}, \boldsymbol{\Sigma}_j^{t+1}, \boldsymbol{\gamma}_j^{t+1}, \delta_j^{t+1}, \boldsymbol{\theta}\right)}{\partial \boldsymbol{\theta}}, \boldsymbol{H}_{f_i,j}^t(\boldsymbol{\theta})=\frac{\partial \boldsymbol{G}_{f_i}^t(\boldsymbol{\theta})}{\partial \boldsymbol{\theta}^{\top}}$, $\boldsymbol{G}_{h,j}^t(\boldsymbol{\theta})=\frac{\partial h\left(\boldsymbol{\mu}_j^{t+1}, \boldsymbol{\Sigma}_j^{t+1}, \boldsymbol{\theta}\right)}{\partial \boldsymbol{\theta}}, \boldsymbol{H}_{h,j}^t(\boldsymbol{\theta})=\frac{\partial \boldsymbol{G}_h^t(\boldsymbol{\theta})}{\partial \boldsymbol{\theta}^{\top}}$.  Direct applying the Newton-Raphson method leads to 
 $\boldsymbol{\theta}_j^{\ast,t+1}=\boldsymbol{\theta}_j^{\ast,t, S}$ and, for $s=1, \ldots, S$:
$$
\boldsymbol{\theta}_j^{*, t, s}=\boldsymbol{\theta}_j^{*, t, s-1}-\left[\frac{1}{J} \sum_i \boldsymbol{H}_{f_i,j}^t\left(\boldsymbol{\theta}_j^{*, t, s}\right)+\frac{1}{J} \boldsymbol{H}_{h,j}^t\left(\boldsymbol{\theta}_j^{*, t, s}\right)\right]^{-1}\left[\frac{1}{J} \sum_i \boldsymbol{G}_{f_i,j}^t\left(\boldsymbol{\theta}_j^{*, t, s}\right)+\frac{1}{J}\boldsymbol{G}_{h,j}^t\left(\boldsymbol{\theta}_j^{*, t, s}\right)\right], 
$$
where $\boldsymbol{\theta}_j^{*,t, 0}=\boldsymbol{\theta}_j^t$ and $S$ is the number of iterations.

Directly computing the Hessian poses two significant challenges. First, the Hessian may fail to be positive definite due to the non-convexity of the function. Second, calculating the averages is impractical in a decentralized setting. These challenges necessitate the development of alternative techniques specifically designed for decentralized environments.
 
To address the first issue, we adopt the method proposed by \cite{paternain2019newton}, which adjusts the Hessian matrix to ensure positive definiteness. Specifically, negative eigenvalues are replaced with their absolute values, and eigenvalues that are too small are substituted with a predefined threshold value. For a  Hessian matrix  $\boldsymbol{H}$, it can be written as  $\boldsymbol{H} = \boldsymbol{Q} \boldsymbol{\Lambda} \boldsymbol{Q}^{\top},$ where  $\boldsymbol{\Lambda} = \text{diag}(\lambda_1, \lambda_2, \dots, \lambda_n)$
is the diagonal matrix of eigenvalues and $\boldsymbol{Q}$ is the orthogonal matrix, the modification of \( \boldsymbol{\Lambda} \) is performed as follows:
\vspace{-15pt}
\[
\lambda_i' = 
\begin{cases} 
|\lambda_i| & \text{if } \lambda_i < 0\ \text{and } |\lambda_i| \ge \epsilon, \\[-10pt]
\lambda_{\text{min}} & \text{if } |\lambda_i| < \epsilon, \\[-10pt]
\lambda_i & \text{otherwise}.
\end{cases}
\vspace{-15pt}
\]
Here, $\epsilon$ is a small positive threshold to determine whether an eigenvalue is "too small", and $\lambda_{\text{min}}$ is the minimum allowable eigenvalue to ensure numerical stability. The modified eigenvalue matrix is  $ \boldsymbol{\Lambda'}  = \text{diag}(\lambda_1', \lambda_2', \dots, \lambda_n')$ and the adjusted Hessian matrix  is computed as $ md( \boldsymbol{H} ) =  \boldsymbol{Q }  \boldsymbol{\Lambda'  } \boldsymbol{Q^{\top}}. $
This approach ensures that the Hessian matrix is positive definite while preserving its structure and mitigating instability caused by small or negative eigenvalues.

To address the second issue,  we apply the same technique to approximate the averages of the Hessian and gradient information. Specifically, we use $\boldsymbol{Y}_{\boldsymbol{H}_f, j}^{t, s}$ and  $\boldsymbol{y}_{\boldsymbol{G}_f, j}^{t, s}$ to approximate the average Hessian and gradient information at each machine. These updates are defined as:  \vspace{-10pt}
\begin{equation}
\begin{aligned}
    \boldsymbol{Y}_{\boldsymbol{H}_f, j}^{t, s}=\sum_i w^{[K]}_{i j}\left(\boldsymbol{Y}_{\boldsymbol{H}_f, i}^{t,s-1}+\boldsymbol{H}_{f_i,i}^t\left(\boldsymbol{\theta}_i^{t, s}\right)-\boldsymbol{H}_{f_i,i}^t\left(\boldsymbol{\theta}_i^{t, s-1}\right)\right), \\
\boldsymbol{y}_{\boldsymbol{G}_f, j}^{t, s}=\sum_i w^{[K]}_{i j}\left(\boldsymbol{y}_{\boldsymbol{G}_f, i}^{t,s-1}+\boldsymbol{G}_{f_i,i}^t\left(\boldsymbol{\theta}_i^{t, s}\right)-\boldsymbol{G}_{f_i,i}^t\left(\boldsymbol{\theta}_i^{t, s-1}\right)\right).
\end{aligned}
\vspace{-10pt}
\end{equation}
At last, we have $\boldsymbol{\theta}_j^{t+1}=\boldsymbol{\theta}_j^{t, S}$ and, for $s=1, \ldots, S$: \vspace{-10pt}
\begin{equation}
    \boldsymbol{\theta}_j^{t, s}=\boldsymbol{\theta}_j^{t, s-1}-\alpha_{t,s}\left\{m d\left[\boldsymbol{Y}_{\boldsymbol{H}_f, j}^{t, s}+\frac{1}{J}\sum w_{ij}^{[K]} \boldsymbol{H}_{h,i}^t\left(\boldsymbol{\theta}_i^{t, s}\right)\right]\right\}^{-1}\left[\boldsymbol{y}_{\boldsymbol{G}_f, j}^{t, s}+\frac{1}{J}\sum w_{ij}^{[K]}\boldsymbol{G}_{h,i}^t\left(\boldsymbol{\theta}_i^{t, s}\right)\right].
\vspace{-10pt}
\end{equation}
Here, $\boldsymbol{\theta}_j^{t, 0} = \boldsymbol{\theta}_j^t$, and $\alpha_{t,s} \leq 1$ is the step size. Due to the fast convergence of the Newton-Raphson method, the required number of iterations, $S$, is typically small. In subsequent simulations and real data analyses, $S$ consistently remains less than 6. In practice, $S$ can be chosen to ensure that consecutive parameter estimates show no significant change.

We summarize the above method in Algorithm \ref{al:dbcd}, included in the Supplementary Material due to the space limitations. The algorithm is further illustrated in Figure \ref{fig:dbcd}.

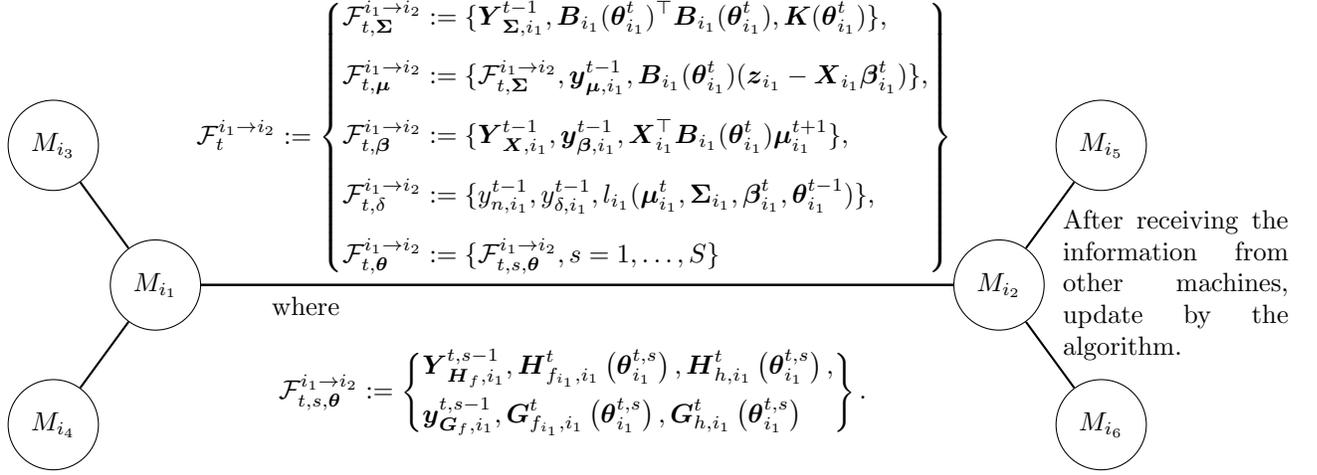
\begin{figure}[h!]
    \centering
\begin{tikzpicture}[
  machine/.style={circle, draw=black, fill=white, inner sep=2pt, minimum size=1.2cm, font=\footnotesize},
  arrow/.style={thick},
  label/.style={align=left, font=\footnotesize} 
]

\node[machine] (machine_{i_1}) {\(M_{i_1}\)};
\node[machine, right=10cm of machine_{i_1}] (machine_j) {\(M_{i_2}\)};
\node[right=0.1cm of machine_j, align=left, font=\footnotesize] (update_text) {\parbox{3cm}{After receiving the information from other machines, update by the algorithm.}};
\node[machine, above left=1cm and 0.5cm of machine_{i_1}] (neighbor_{i_1}1) {\(M_{i_3}\)};
\node[machine, below left=1cm and 0.5cm of machine_{i_1}] (neighbor_{i_1}2) {\(M_{i_4}\)};
\node[machine, above right=1cm and 0.5cm of machine_j] (neighbor_j1) {\(M_{i_5}\)};
\node[machine, below right=1cm and 0.5cm of machine_j] (neighbor_j2) {\(M_{i_6}\)};

\draw[arrow] (machine_{i_1}) -- (machine_j) node[midway, above, label=left] {
\begin{minipage}{10cm}
Information transferred from \(M_{i_1}\) to \(M_{i_2}\) when \(K=1\): \\
\[
\mathcal{F}_t^{i_1\rightarrow i_2} :=
 \left\{\begin{aligned}
&\mathcal{F}^{i_1\rightarrow i_2}_{t,\boldsymbol{\Sigma}}:= \{ \boldsymbol{Y}_{\boldsymbol{\Sigma},{i_1}}^{t-1}, \boldsymbol{B}_{i_1}(\boldsymbol{\theta}_{i_1}^t)^\top \boldsymbol{B}_{i_1}(\boldsymbol{\theta}_{i_1}^t), \boldsymbol{K}(\boldsymbol{\theta}_{i_1}^t)\}, \\[5pt]
& \mathcal{F}^{i_1\rightarrow i_2}_{t,\boldsymbol{\mu}}:= \{ \mathcal{F}^{i_1\rightarrow i_2}_{t,\boldsymbol{\Sigma}}, \boldsymbol{y}_{\boldsymbol{\mu},{i_1}}^{t-1}, \boldsymbol{B}_{i_1}(\boldsymbol{\theta}_{i_1}^t)(\boldsymbol{z}_{i_1} - \boldsymbol{X}_{i_1}\boldsymbol{\beta}_{i_1}^t) \}, \\[5pt]
& \mathcal{F}^{i_1\rightarrow i_2}_{ t,\boldsymbol{\beta}}: =\{ \boldsymbol{Y}_{\boldsymbol{X},{i_1}}^{t-1},  \boldsymbol{y}_{\boldsymbol{\beta},{i_1}}^{t-1},\boldsymbol{X}_{i_1}^\top \boldsymbol{B}_{i_1}(\boldsymbol{\theta}_{i_1}^t)\boldsymbol{\mu}_{i_1}^{t+1} \}, \\[5pt]
 & \mathcal{F}^{i_1\rightarrow i_2}_{t,\delta}:= \{ y_{n,{i_1}}^{t-1}, y_{\delta,{i_1}}^{t-1}, l_{i_1}(\boldsymbol{\mu}_{i_1}^t, \boldsymbol{\Sigma}_{i_1}, \boldsymbol{\beta}_{i_1}^t, \boldsymbol{\theta}_{i_1}^{t-1}) \}, \\[5pt]
& \mathcal{F}^{i_1\rightarrow i_2}_{ t,\boldsymbol{\theta}}:= \{ \mathcal{F}^{i_1\rightarrow i_2}_{ t,s,\boldsymbol{\theta}},s=1,\ldots,S \}
\end{aligned}\right\}
\]
\end{minipage}
};

\draw[arrow] (machine_{i_1}) -- (machine_j) node[midway, below, label=left] {
\begin{minipage}{8cm}
where \\
\[
\mathcal{F}^{i_1\rightarrow i_2}_{t,s,\boldsymbol{\theta}} := \left\{
\begin{aligned}
  & \boldsymbol{Y}_{\boldsymbol{H}_f,{i_1}}^{t,s-1}, \boldsymbol{H}_{f_{i_1},{i_1}}^t\left(\boldsymbol{\theta}_{i_1}^{t, s}\right), \boldsymbol{H}_{h,{i_1}}^t\left(\boldsymbol{\theta}_{i_1}^{t, s}\right), \\ 
  & \boldsymbol{y}_{\boldsymbol{G}_f,{i_1}}^{t,s-1}, \boldsymbol{G}_{f_{i_1},{i_1}}^t\left(\boldsymbol{\theta}_{i_1}^{t, s}\right), \boldsymbol{G}_{h,{i_1}}^t\left(\boldsymbol{\theta}_{i_1}^{t, s}\right)
\end{aligned}
\right\}.
\]
\end{minipage}
};
\draw[arrow] (machine_j) -- (machine_{i_1});

\draw[arrow] (neighbor_{i_1}1) -- (machine_{i_1});
\draw[arrow] (neighbor_{i_1}2) -- (machine_{i_1});

\draw[arrow] (neighbor_j1) -- (machine_j);
\draw[arrow] (neighbor_j2) -- (machine_j);

\end{tikzpicture}

    \caption{Decentralized block coordinate descent for spatial low-rank models. }
    \label{fig:dbcd}
\end{figure}








\begin{remark}[Initial points] The initial points are recommended to be the decentralized averages of the local minimizers. Specifically, let
\vspace{-10pt}
\begin{equation*}
  \widehat{\boldsymbol{\mu}}^{loc}_j, \widehat{\boldsymbol{\Sigma}}^{loc}_j,\widehat{\boldsymbol{\gamma}}^{loc}_j,\widehat{\delta}^{loc}_j,\widehat{\boldsymbol{\theta}}^{loc}_j=\arg\min \left\{f_j(\boldsymbol{\mu},\boldsymbol{\Sigma},\boldsymbol{\gamma},\delta,\boldsymbol{\theta})+h(\boldsymbol{\mu},\boldsymbol{\Sigma},\boldsymbol{\theta})\right\}, j=1,\ldots,J,
  \vspace{-10pt}
\end{equation*}
where $f_j,h$ are the functions in Equation \eqref{eq:ob2}. The initial points are then taken as the decentralized averages of these local minimizers.  From the theoretical results presented in a later section, local minimizers are shown to be close to the global minimizer, making this a well-founded choice.
  
\end{remark}

\begin{remark}[Prediction]
Recall that $\boldsymbol{\mu}$, $\boldsymbol{\Sigma}$ are the conditional mean and covariance matrix given the observations $\boldsymbol{z}$. Then, spatial predictions can be made easily after parameter inference, which is described in detail as follows. Suppose we want to make predictions at locations $\mathbf{s}_1^P, \mathbf{s}_2^P, \ldots, \mathbf{s}_{n_P}^P$, and the parameters $\boldsymbol{\mu}_j^T$, $\boldsymbol{\Sigma}_j^T$, $\boldsymbol{\gamma}_j^T$, $\delta_j^T$, and $\boldsymbol{\theta}_j^T$ are obtained in each machine $j$, with $T$ representing the number of iterations. It is reasonable to assume that none of the spatial prediction locations coincide with existing observation locations. Let $\boldsymbol{x}_1^P, \ldots, \boldsymbol{x}_{n_P}^P$ be the corresponding covariate vectors. Define
$\boldsymbol{z}^P = (z (\mathbf{s}_1^P), \ldots, z (\mathbf{s}_{n_P}^P))^{\top}$,
$\boldsymbol{X}^P = (\boldsymbol{x}_1^P, \ldots, \boldsymbol{x}_{n_P}^P)^{\top}$,
$\boldsymbol{B}^P (\boldsymbol{\theta}) = (\boldsymbol{b} (\mathbf{s}_1^P ;
\boldsymbol{\theta}), \ldots, \boldsymbol{b} (\mathbf{s}_{n_P}^P ;
\boldsymbol{\theta}))^{\top}$, $\boldsymbol{\varepsilon}_j = (\varepsilon
(\mathbf{s}_1^P), \ldots, \varepsilon (\mathbf{s}_{n_P}^P)){\top}$. 
Since the relationship is
\vspace{-10pt}
\[
\boldsymbol{z}^P = \boldsymbol{X}^P \boldsymbol{\gamma} + \boldsymbol{B}^P(\boldsymbol{\theta}) \boldsymbol{\eta} + \boldsymbol{\varepsilon}^P,
\vspace{-10pt}
\]the distribution of $\boldsymbol{z}^P$ conditional on the observed data is: \vspace{-10pt}
\[
\mathcal{N}_{n_P} \left( \boldsymbol{X}^P \boldsymbol{\gamma}^\ast + \boldsymbol{B}^P(\boldsymbol{\theta}^\ast) \boldsymbol{\mu}^\ast, \boldsymbol{B}^P(\boldsymbol{\theta}^\ast) \boldsymbol{\Sigma}^\ast \left[\boldsymbol{B}^P(\boldsymbol{\theta}^\ast)\right]^{\top} + \delta^\ast \boldsymbol{I} \right),
\vspace{-10pt}
\]
where $\boldsymbol{\gamma}^\ast, \boldsymbol{\mu}^\ast, \boldsymbol{\Sigma}^\ast, \delta^\ast$ are the true parameters. Thus, by replacing the true parameters with their decentralized estimators, the predicative distribution in each machine $j$ is given by:
\vspace{-10pt}
\[
\mathcal{N}_{n_P} \left( \boldsymbol{X}^P \boldsymbol{\gamma}_j^T + \boldsymbol{B}^P(\boldsymbol{\theta}_j^T) \boldsymbol{\mu}_j^T, \ \boldsymbol{B}^P(\boldsymbol{\theta}_j^T) \boldsymbol{\Sigma}_j^T \left[ \boldsymbol{B}^P(\boldsymbol{\theta}_j^T) \right]^{\top} + \delta_j^T \boldsymbol{I} \right), \quad j = 1, \ldots, J.
\vspace{-10pt}
\]
\end{remark}

\begin{remark}[Cost analysis]
In each iteration, the computation cost in each machine is $O(n m^2)$, while the communication cost for the $j$-th machine is $O(deg_j\times K m^2)$, where $deg_j$ is the degree of machine $j$, and $K$ is the number of multi-consensus rounds. Both $K$ and the total number of iterations are shown to be relatively small through subsequent simulations and theoretical analysis. In contrast, the computation cost of the traditional non-distributed method is $O(N m^2)$. Consequently, when the total sample size is large, the decentralized method can significantly outperform the non-distributed method in terms of speed.  In other words, with the same time budget, the number $m$ of knots in the distributed method can be increased, resulting in smaller approximation errors for the low-rank model.
\end{remark}
\begin{remark}[Extension to multi-resolution models]
As noted in the introduction, low-rank models have limitations in capturing fine-scale spatial variability \citep{stein2014limitations}. Multi-resolution models address these limitations by refining the low-rank structure by introducing finer details. In Section \ref{sec:extension} of the Supplementary Material, we demonstrate the feasibility of extending our method to the multi-resolution framework. A comprehensive exploration of this extension is currently underway.
\end{remark}

\vspace{-.5cm} 

\section{Theory}\label{sec:theory}
In this section, we first establish the properties of the objective function, demonstrating its local convexity within a specific neighborhood of the true parameters (Theorem \ref{thm:convexity}). We then investigate the asymptotic properties of the global minimizer, proving both its consistency and its asymptotic normality (Theorems \ref{thm:consistency} and  \ref{thm:asy_normality} ). Lastly, we analyze the convergence of the algorithm, showing that the proposed sequence of estimators converges to the global minimizer under suitable conditions (Theorem  \ref{thm:convergence} and Corollary \ref{cor:convergence}).

Before presenting the formal theory, we introduce some notations to enhance clarity and facilitate the demonstration. For a matrix $\boldsymbol{A}$, let $\| \boldsymbol{A} \|_{\text{op}}$ be its
operator norm. For a matrix function $\boldsymbol{A} (\boldsymbol{\theta}) {:
\mathbb{R}^{q_1}}  \rightarrow \mathbb{R}^{q_2\times q_3}$, denote $\boldsymbol{A}_{[k_1]}
(\boldsymbol{\theta}) = \frac{\partial \boldsymbol{A}
(\boldsymbol{\theta})}{\partial \boldsymbol{\theta}_k}, \boldsymbol{A}_{[k_1 k_2]}
(\boldsymbol{\theta}) = \frac{\partial^2 \boldsymbol{A}
(\boldsymbol{\theta})}{\partial \boldsymbol{\theta}_{k_1} \partial
\boldsymbol{\theta}_{k_2}}$, $\boldsymbol{A}_{[k_1 k_2 k_3]} (\boldsymbol{\theta}) =
\frac{\partial^3 \boldsymbol{A} (\boldsymbol{\theta})}{\partial
\boldsymbol{\theta}_{k_1} \partial \boldsymbol{\theta}_{k_2} \partial
\boldsymbol{\theta}_{k_3}}$ for $k_1, k_2, k_3 = 1, \ldots ., q_1$. Let
$\boldsymbol{X} = (\boldsymbol{x }_1, \ldots, \boldsymbol{x }_p) .$ Let
$\Gamma\subset \mathbb{R}^p$, $\Theta \subset \mathbb{R}^q$, $\mathbb{S}\subset \mathbb{R}$ be the parameter space of
$\boldsymbol{\gamma}, \boldsymbol{\theta}, \delta$, separately. Let $\boldsymbol{\gamma}^\ast, \boldsymbol{\theta}^\ast, \delta^\ast$ be the true values of $\boldsymbol{\gamma}, \boldsymbol{\theta}, \delta$, separately.

\begin{assumption}\label{as:compact}
The parameter spaces $\Gamma$, $\Theta$, $\mathbb{S}$ are compact and
$\mathbb{S} \subset [a, \infty)$ for some positive constant $a$.
\end{assumption}

\begin{assumption}\label{as:covariance} As $N\rightarrow\infty$,
$\| \boldsymbol{K }^{- 1} (\boldsymbol{\theta}) \|_{\text{op}} = O (1)$.
$\boldsymbol{B} (\boldsymbol{\theta})$ and $\boldsymbol{K} (\boldsymbol{\theta})$ are
three times differentiable in each position, and for each $k_1, k_2, k_3 = 1, 
\ldots, q$, uniformly for $\boldsymbol{\theta} \in
   \Theta$, $N\rightarrow\infty$, \vspace{-10pt}
\begin{equation*}
       \begin{aligned}
           \| \boldsymbol{B} (\boldsymbol{\theta}) \|_{\text{op}} = O (1), \|
   \boldsymbol{B}_{[k_1]} (\boldsymbol{\theta}) \|_{\text{op}} = O (1), \|
   \boldsymbol{B}_{[k_1 k_2]} (\boldsymbol{\theta}) \|_{\text{op}} = O (1), \|
   \boldsymbol{B}_{[k_1 k_2 k_3]} (\boldsymbol{\theta}) \|_{\text{op}} = O (1);\\
    \boldsymbol{} \boldsymbol{} \boldsymbol{} \| \boldsymbol{K} (\boldsymbol{\theta})
   \|_{\text{op}} = O (1), \| \boldsymbol{K}_{[k_1]}(\boldsymbol{\theta})
   \|_{\text{op}} = O (1), \| \boldsymbol{K}_{[k_1 k_2]} (\boldsymbol{\theta})
   \|_{\text{op}} = O (1) \text{, $\| \boldsymbol{K}_{[k_1 k_2 k_3]}
   (\boldsymbol{\theta}) \|_{\text{op}} = O (1)$ }.
       \end{aligned}
   \end{equation*}
\end{assumption}

\begin{assumption}\label{as:covariate}
The rows of  $\boldsymbol{X}$ are i.i.d. sub-Gaussian vectors.
\end{assumption}

\begin{assumption}\label{as:positive}
There exists a positive constant $ \lambda^{\ast}$ such that $\lambda_{\min}
(\boldsymbol{\mathcal{J}}) > \lambda^{\ast}$. Here, $\boldsymbol{\mathcal{J}}$ 
is the expected scaled Hessian matrix of the negative log-likelihood function in \eqref{eq:log_lik} at the true
point, that is, \vspace{-20pt}
\[ \boldsymbol{\mathcal{J}} := \left(\begin{array}{ccc}
       \frac{1}{N} \mathbb{E} \left\{\frac{\partial^2  \mathcal{L} \left( \boldsymbol{\gamma}^{\ast},      \delta^{\ast}, \boldsymbol{\theta}^{\ast} \right) }{\partial \boldsymbol{\gamma} 
       \partial \boldsymbol{\gamma}^T}\right\} & \boldsymbol{O} & \boldsymbol{O}\\
       \boldsymbol{O} & \frac{1}{N} \mathbb{E} \left\{\frac{ \partial^2\mathcal{L} \left( \boldsymbol{\gamma}^{\ast},      \delta^{\ast}, \boldsymbol{\theta}^{\ast} \right)}{\partial \delta  \partial
       \boldsymbol{} \delta}\right\} & \frac{1}{m} \mathbb{E} \left\{\frac{\partial^2  \mathcal{L} \left( \boldsymbol{\gamma}^{\ast},      \delta^{\ast}, \boldsymbol{\theta}^{\ast} \right)}{\partial \boldsymbol{} \delta
       \partial \boldsymbol{} \boldsymbol{\theta}}\right\}\\
       \boldsymbol{O} & \frac{1}{m} \mathbb{E} \left\{\frac{\partial^2 \mathcal{L} \left( \boldsymbol{\gamma}^{\ast},      \delta^{\ast}, \boldsymbol{\theta}^{\ast} \right)}{\boldsymbol{} \partial
       \boldsymbol{\theta}  \partial \boldsymbol{} \delta}\right\} & \frac{1}{m}
       \mathbb{E}\left\{ \frac{\partial^2  \mathcal{L} \left( \boldsymbol{\gamma}^{\ast},      \delta^{\ast}, \boldsymbol{\theta}^{\ast} \right)}{\boldsymbol{} \partial
       \boldsymbol{\theta}  \partial \boldsymbol{} \boldsymbol{\theta}^T}\right\}
     \end{array}\right). \]
\end{assumption}

\begin{assumption}\label{as:bound}
For any constant $\xi > 0$, there exist a constant $c > 0$ such that
\vspace{-15pt}
\[ \inf_{\boldsymbol{\theta} : \| \boldsymbol{\theta} - \boldsymbol{\theta}^{\ast}
   \| > \xi} \frac{1}{m} \| \boldsymbol{B} (\boldsymbol{\theta}) \boldsymbol{K} 
   (\boldsymbol{\theta}) \boldsymbol{B}^{\top} (\boldsymbol{\theta}) - \boldsymbol{B}
   (\boldsymbol{\theta}^{\ast}) \boldsymbol{K}  (\boldsymbol{\theta}^{\ast})
   \boldsymbol{B}^{\top} (\boldsymbol{\theta}^{\ast}) \|_F^2 > c. \vspace{-10pt} \]
\end{assumption}
\begin{remark}
The assumptions are broadly aligned with those in \cite{chu2024maximizing}, with modifications and extensions tailored to the low-rank case. Assumption \ref{as:compact} is a standard requirement for parameter inference, as noted in \cite{van2000asymptotic}. The condition of three-times differentiability in Assumption \ref{as:covariance} ensures the existence and continuity of the Hessian, together with Assumption \ref{as:positive}, which is critical for guaranteeing the local convexity of the objective function. The operator norm condition in Assumption \ref{as:covariance} generalizes the corresponding conditions A2,A6 in  \cite{chu2024maximizing}  to accommodate the low-rank setting. The sub-Gaussianity of covariates in Assumption \ref{as:covariate} is imposed to control the tail behavior and to establish the corresponding concentration bounds \citep{wainwright2019high}. Assumption \ref{as:positive} aligns with condition A7 in \cite{chu2024maximizing}, while Assumption \ref{as:bound} again adapts the corresponding condition A5 in \cite{chu2024maximizing} to the low-rank setting.
\end{remark}
\begin{theorem}[Local Convexity]\label{thm:convexity}
  Suppose that Assumptions \ref{as:compact}--\ref{as:positive} hold. For any $\epsilon > 0$, there exist a constant $\xi_{\epsilon} > 0$ and a constant integer $M_{\epsilon}$, such that, if $m>M_{\epsilon}$, the Hessian of the objective function $f (\boldsymbol{\mu}, \boldsymbol{\Sigma}, \boldsymbol{\gamma}, \delta,
  \boldsymbol{\theta}) \boldsymbol{}$  is positive definite and the Hessian of the negative log-likelihood function (defined in \eqref{eq:log_lik}) has the minimum eigenvalue  greater than $m\lambda^\ast/2$ over the region $\{
  (\boldsymbol{\mu}, \boldsymbol{\Sigma}, \boldsymbol{\gamma}, \delta,
  \boldsymbol{\theta}) : \| \boldsymbol{\gamma} - \boldsymbol{\gamma}^{\ast} \|  + \|
  \delta - \delta^{\ast} \| + \| \boldsymbol{\theta} - \boldsymbol{\theta}^{\ast}
  \| \leqslant \xi_{\epsilon}, \lambda_{min}(\boldsymbol{\Sigma})>0 \}$ with probability greater than $1 - \epsilon$ where $\lambda_{min}(\boldsymbol{\Sigma})$ is the minimum eigenvalue of $\boldsymbol{\Sigma}$. 
\end{theorem}

Theorem \ref{thm:convexity} demonstrates that the objective function is locally convex, and the negative log-likelihood function is locally strongly convex in a neighborhood around the true parameters. This local (strong) convexity ensures that the optimization landscape near the true parameters is well-behaved, allowing for the establishment of convergence theory. Importantly, this property guarantees that, starting from initial points sufficiently close to the global optimizer, the algorithm can converge effectively, as supported by subsequent theorems.

\begin{theorem}[Consistency]\label{thm:consistency}
  Under Assumptions \ref{as:compact}--\ref{as:bound}, $\widehat{\boldsymbol{\gamma}} \boldsymbol{}
  \xrightarrow{p} \boldsymbol{\gamma}^{\ast}, \hat{\delta} \xrightarrow{p}
  \delta^{\ast}$ and $\widehat{\boldsymbol{\theta}} \boldsymbol{} \xrightarrow{p}
  \boldsymbol{\theta}^{\ast}$ as $N, m \rightarrow \infty$ with $m=o(N)$.
\end{theorem}

\begin{theorem}[Asymptotic Normality]\label{thm:asy_normality}
 \sloppy  For convenience, denote $\boldsymbol{C}  \left( {\boldsymbol{\theta} } , \delta  \right) :=
  \delta^{-1} \boldsymbol{I} + \boldsymbol{B} (\boldsymbol{\theta})
  \boldsymbol{K}  (\boldsymbol{\theta}) \boldsymbol{B}^{\top} (\boldsymbol{\theta})$
  and let $\boldsymbol{V}_{\boldsymbol{} \boldsymbol{\gamma}} :=  \frac{1}{N}\mathbb{E}\{
  \boldsymbol{X}^{\top} \boldsymbol{C}^{-1}  \left(
  {\boldsymbol{\theta}^{\ast}} , \delta^{\ast} \right) \boldsymbol{X}\}$, $v_{\delta}: =  \frac{1}{2N}  (\delta^{\ast})^{-4} \text{tr} \left[
  \boldsymbol{C}^{- 2} \left( {\boldsymbol{\theta}^{\ast}} , \delta^{\ast} \right)
  \right]$, $\boldsymbol{V}_{\boldsymbol{\theta}} := \left(
  \frac{1}{2m} \text{tr} \left( \boldsymbol{C}^{- 1} (\boldsymbol{\theta}^{\ast},
  \delta^{\ast}) \frac{\partial \boldsymbol{C} (\boldsymbol{\theta}^{\ast},
  \delta^{\ast})}{\partial \boldsymbol{\theta}_k} \boldsymbol{C}^{- 1}
  (\boldsymbol{\theta}^{\ast}, \delta^{\ast}) \frac{\partial \boldsymbol{C}
  (\boldsymbol{\theta}^{\ast}, \delta^{\ast})}{\partial \boldsymbol{\theta}_l }
  \right) \right)_{k l}$.  Under  Assumptions \ref{as:compact}--\ref{as:bound}, as $N, m \rightarrow \infty$ with $m=o(N)$, \vspace{-10pt}
    \begin{equation}
      \sqrt{N} \boldsymbol{V}_{\boldsymbol{} \boldsymbol{\gamma}}^{\frac{1}{2}}
      (\widehat{\boldsymbol{\gamma}} \boldsymbol{}  - \boldsymbol{\gamma}^{\ast})\rightsquigarrow  \mathcal{N}_{p } (\boldsymbol{0}_{p },
    \boldsymbol{I}_{p \times p}), \sqrt{N} v_{\delta}^{\frac{1}{2}} (\hat{\delta} - \delta^{\ast})\rightsquigarrow  \mathcal{N} (0,1), \sqrt{m}\boldsymbol{V}_{\boldsymbol{\theta}}^{\frac{1}{2}} (\widehat{\boldsymbol{\theta}} -
     \boldsymbol{\theta}^{\ast}) \rightsquigarrow \mathcal{N}_q (\boldsymbol{0}_q,
     \boldsymbol{I}_{q \times q}) .
     \vspace{-10pt}
  \end{equation}
\end{theorem}

Theorem~\ref{thm:consistency} establishes the consistency of the proposed estimator under the specified assumptions. It shows that as the total sample size $N$ and the number of local observations $m$ grow unbounded, with $m = o(N)$, the estimators $\widehat{\boldsymbol{\gamma}}$, $\hat{\delta}$, and $\widehat{\boldsymbol{\theta}}$ converge in probability to their true values, $\boldsymbol{\gamma}^\ast$, $\delta^\ast$, and $\boldsymbol{\theta}^\ast$, respectively. Theorem~\ref{thm:asy_normality} further provides the asymptotic normality of the estimators. Specifically, under the same growth conditions and assumptions, the appropriately scaled estimators $\sqrt{N} \boldsymbol{V}_{\boldsymbol{\gamma}}^{\frac{1}{2}} (\widehat{\boldsymbol{\gamma}} - \boldsymbol{\gamma}^\ast)$, $\sqrt{N} v_\delta^{\frac{1}{2}} (\hat{\delta} - \delta^\ast)$, and $\sqrt{m} \boldsymbol{V}_{\boldsymbol{\theta}}^{\frac{1}{2}} (\widehat{\boldsymbol{\theta}} - \boldsymbol{\theta}^\ast)$ converge in distribution to standard normal distributions. Here,  $\boldsymbol{V}_{\boldsymbol{\gamma}}$, $v_\delta$, and $\boldsymbol{V}_{\boldsymbol{\theta}}$ are explicitly defined in terms of the model parameters and the covariance structure of the spatial low-rank model. These matrices can be efficiently estimated in a decentralized manner, and then the confidence intervals can be constructed accordingly, as detailed in the Section~\ref{sec:asy_variance} of the Supplementary Material due to space constraints.

To the best of our knowledge, this is the first theory establishing consistency and asymptotic normality for the estimator under spatial low-rank models. The difficulty lies in adapting the theory to accommodate the structure of the low-rank model. To overcome this difficulty, we employ the generic chaining method \citep{talagrand2005generic,dirksen2015tail}  in combination with singular value decomposition of the low-rank matrix, enabling us to derive a supremum bound for certain stochastic processes, which is crucial to the theory.

\begin{remark}
     We note that the estimation of $\boldsymbol{\theta}$ has lower statistical efficiency due to the low-rank nature of the model. By adopting multi-resolution or Hierarchical approaches, the efficiency can be improved \citep{nychka2015multiresolution,katzfuss2017multi,huang2018hierarchical}. In the Section \ref{sec:extension} of the Supplementary Material, we demonstrate the feasibility of extending our method to the multi-resolution case.
\end{remark}

Some additional regularity assumptions are required for convergence. For clarity, these assumptions are detailed in the Supplementary Material (Assumptions \ref{as:weights}--\ref{as:subproblem} in Section \ref{sec:proof_2}).  In the following, define $\overline{\boldsymbol{\gamma} }^t:=J^{-1}\sum_{j=1}^J\boldsymbol{\gamma}^t_j, \overline{\delta}^t:=J^{-1}\sum_{j=1}^J{\delta}^t_j,
     \overline{\boldsymbol{\theta} }^t:=J^{-1}\sum_{j=1}^J\boldsymbol{\theta}^t_j$ as the average of local estimator and $\widehat{\mathcal{L}}:=\mathcal{L}
     ( \widehat{\boldsymbol{\gamma} }^t, \widehat{\delta}^t,
     \widehat{\boldsymbol{\theta} }^t )$ be the minimum value. Define also $\mathcal{L}_{scaled}:=\mathcal{L}/N$ and $\widehat{\mathcal{L}}_{scaled}:=\widehat{\mathcal{L}}/N$ as the scaled versions.

\begin{theorem}[Convergence]\label{thm:convergence}
  Suppose that Assumptions \ref{as:compact}--\ref{as:bound} ,  Assumptions \ref{as:weights}--\ref{as:subproblem}  in the Supplementary Material (Section \ref{sec:proof_2}) hold and the number of
rounds for multi-consensus $K$ satisfies $K > C \frac{\log N}{\log (1 / \rho
  _{\boldsymbol{W}})}$ for some constant $C$ large enough. Then, for a given small positive number
  $\epsilon$, there exists $0
  < c_{\epsilon} < 1$ independent of $t$, and a constant $C_{\epsilon}$ and  a constant integer $M_{\epsilon}$,
  such that, when $m\ge M_{\epsilon}$, \vspace{-10pt}
  \[ \mathcal{L}_{scaled} \left( \overline{\boldsymbol{\gamma} }^{t + 1},
     \overline{\delta}^{t + 1}, \overline{\boldsymbol{\theta} }^{t + 1} \right)
     - \widehat{\mathcal{L}}_{scaled} \leqslant c_{\epsilon}  \left[ \mathcal{L}_{scaled}
     \left( \overline{\boldsymbol{\gamma} }^t, \overline{\delta}^t,
     \overline{\boldsymbol{\theta} }^t \right) - \widehat{\mathcal{L}}_{scaled} \right] +
     C_{\epsilon} N^{- 1},\vspace{-10pt} \]
with probability at least $1 - \epsilon$, and $ \sum_j ( \| \boldsymbol{\gamma}^t_j - \overline{\boldsymbol{\gamma}
     }^t \boldsymbol{} \| + \| \delta^t_j - \overline{\delta}^t
     \| + \| \boldsymbol{\theta}^t_j - \overline{\boldsymbol{\theta}
     }^t \| ) = C_{\epsilon} N^{- 1}.$
 
\end{theorem}

Theorem~\ref{thm:convergence} establishes the convergence of the proposed decentralized optimization algorithm, demonstrating that it achieves an approximately linear convergence rate. Examining the proof reveals that the linear rate, $c_\epsilon$, is influenced by the total sample size $N$ and the parameter $m$. In simulations, $c_\epsilon$ is observed to be small, indicating rapid convergence. The theorem also specifies that the number of multi-consensus rounds, $K$, should satisfy \( K > C \frac{\log N}{\log(1/\rho_{\boldsymbol{W}})} \), ensuring sufficient consensus among nodes in the network. Notably, in our simulations, a value of \( K = 6 \) is sufficient to achieve this. Under these conditions, the optimization objective $\mathcal{L}$ converges to the global minimizer $\widehat{\mathcal{L}}$ with high probability ($1 -\epsilon$), where the residual error diminishes proportionally to $N^{-1}$. Furthermore, the theorem bounds the aggregate deviation of local parameters from their global averages, which also scales with $N^{-1}$, confirming the robustness and scalability of the algorithm.

\begin{corollary}\label{cor:convergence}
  Under the same assumptions as in Theorem \ref{thm:convergence}, when $T$ is sufficiently large,
  we have $\| \boldsymbol{\beta}^t_j - \widehat{\boldsymbol{\beta}} \boldsymbol{} \|
  + \| \delta^t_j - \hat{\delta} \| + \| \boldsymbol{\theta}^t_j -
  \widehat{\boldsymbol{\theta}} \| = O_{\mathbb{P}} (N^{- 1})$ for any $j$. 
\end{corollary}
This corollary implies that the distributed estimators are asymptotically equivalent to the
  traditional MLE estimator, provided the number of iterations is sufficiently large. This condition is typically satisfied within a few dozen iterations in our simulation experiments.

\section{Simulations}\label{sec:simu}
In this section, we perform numerical simulations to validate our approach. Unless stated otherwise, the simulations are based on the Gaussian predictive process model with a Matérn covariance function~\citep{wang2023parameterization},
\vspace{-15pt}
\[
c(\mathbf{s}, \mathbf{s}') = \sigma^2 \frac{2^{1 - \nu}}{\Gamma(\nu)} \left( \frac{\sqrt{2 \nu} \| \mathbf{s} - \mathbf{s}' \|}{\beta} \right)^\nu \mathcal{K}_\nu \left( \frac{\sqrt{2 \nu} \| \mathbf{s} - \mathbf{s}' \|}{\beta} \right),
\vspace{-10pt}
\]
where \( \sigma^2 \) is the variance parameter, \( \nu > 0 \) is the smoothness parameter, \( \beta > 0 \) is the range parameter, \( \mathcal{K}_\nu \) is the modified Bessel function of the second kind of order \( \nu \), and \( \Gamma(\nu) \) is the gamma function evaluated at \( \nu \). Then, the parameter vector $\boldsymbol{\theta}$ is $\boldsymbol{\theta}=(\sigma,\beta,\nu)^{\top}$. Due to the complexity of  \( \mathcal{K}_\nu \) \citep{takekawa2022fast,abramowitz1948handbook}, calculating the derivative of Matérn covariance function with respect $\nu$ is challenging and lacks readily available efficient implementations. Consequently, we fix the value of $\nu$ in all experiments except one, where alternative methods are discussed. 

In the following experiments, unless stated otherwise, the settings are as follows: The total sample size is \( N = 10{,}000 \), with \( J = 10 \) machines, each containing \( n = 1{,}000 \) data points. The decentralized network is constructed using an Erdős–Rényi (ER) model with a connection probability of \( p = 0.5 \). Spatial locations are initially generated on a grid with a minimum spacing of 0.02. Uniform noise \( U[-0.4, 0.4] \times 0.02 \) is then added to both axes. Knots are randomly selected from these spatial locations. The covariate vector \( \boldsymbol{X} \) has 5 dimensions and follows a standard Gaussian distribution. The coefficient vector is \( \boldsymbol{\gamma} = (-1, 2, 3, -2, 1)^\top \). The smoothness parameter is \( \nu = 0.5 \), the variance parameter is \( \sigma^2 = 1 \), the range parameter is \( \beta = 0.1 \), and the nugget parameter is \( \tau = 2 \). 

We evaluate the performance of the proposed method by two main criterions. The first criterion is the closeness between the decentralized estimators $( {\boldsymbol{\gamma}}^t_j, {\delta}_j^t, {\sigma}_j^t, {\gamma}_j^t),j=1,...,J$ at each iteration $t$ in machine $j$ and the traditional MLE estimator $(\widehat{\boldsymbol{\gamma}},\widehat{\delta},\widehat{\sigma},\widehat{\gamma})$, which is showing the convergence property of the optimization. This is measured by the logarithmic relative error defined as 
\vspace{-10pt}
\begin{equation*}
  \text{logarithmic relative error}_t=\log_{10}\max_j\left\{\left\|\frac{ {\boldsymbol{\gamma}}_j^t-\widehat{\boldsymbol{\gamma}}}{{\boldsymbol{\gamma}}^\ast}\right\|_2+\left|\frac{ {\delta}_j^t-\widehat{\delta}}{{{\delta}}^\ast}\right|+\left|\frac{ {\sigma}_j^t-\widehat{\sigma}}{{{\sigma}}^\ast}\right|+\left|\frac{ {\gamma}_j^t-\widehat{\gamma}}{{{\gamma}}^\ast}\right|\right\},
  \vspace{-10pt}
\end{equation*}
where $\boldsymbol{\gamma}^\ast,{\delta}^\ast, {\sigma}^\ast,{\gamma}^\ast$ are the corresponding true parameters. The second criterion is how closely the decentralized estimator at a fixed iteration number $T=100$ approximates the true values compared to the MLE estimator, which shows the statistical property of the decentralized estimator compared with the MLE estimator. This is illustrated using boxplots.

\subsection{Convergence under Various Scenarios}
First, we investigate the convergence of the method under different parameter settings for the covariance function. The smoothness parameter \( \nu \) is set to \( 0.5 \) and \( 1.5 \), where a larger \( \nu \) indicates a smoother process \citep{stein2012interpolation}. For each fixed \( \nu \), the range parameter \( \beta \) is assigned three different values to represent varying levels of dependence, with the effective range\footnote{The effective range is defined as the distance at which the correlation equals 0.05 \citep{apanasovich2010cross}.} set to 0.1, 0.3, and 0.7. Specifically, the pairs \( (\nu, \beta) \) considered are \( (0.5, 0.033) \), \( (0.5, 0.1) \), \( (0.5, 0.234) \), \( (1.5, 0.021) \), \( (1.5, 0.063) \), and \( (1.5, 0.148) \). The results, presented in Figure \ref{fig:convergence_comparison}, demonstrate that the decentralized method converges quickly and remains robust to variations in the covariance parameters.

\begin{figure}[htbp]
  \centering
  \begin{center}
      \includegraphics[width=5.5cm]{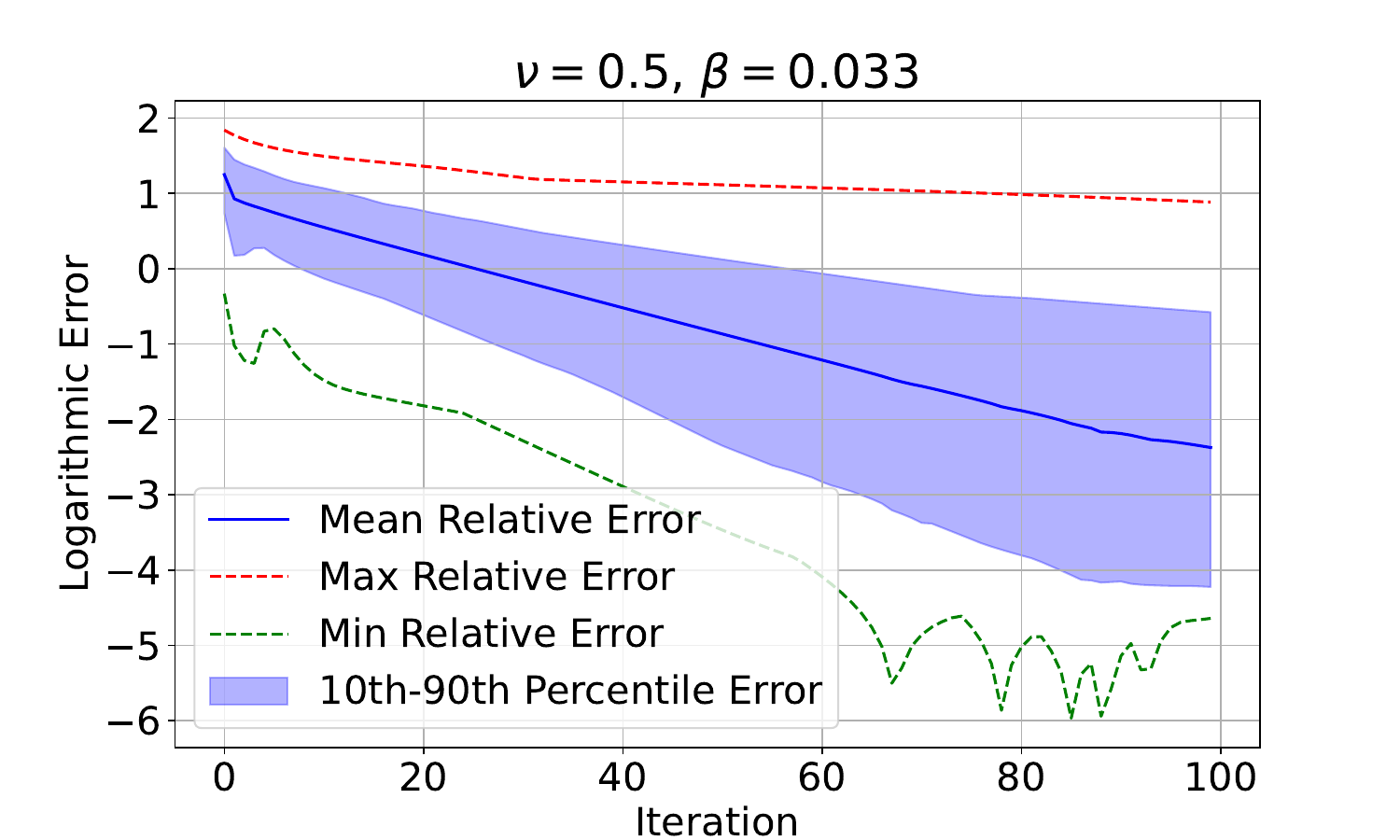}\hspace{-0.65cm}
      \includegraphics[width=5.5cm]{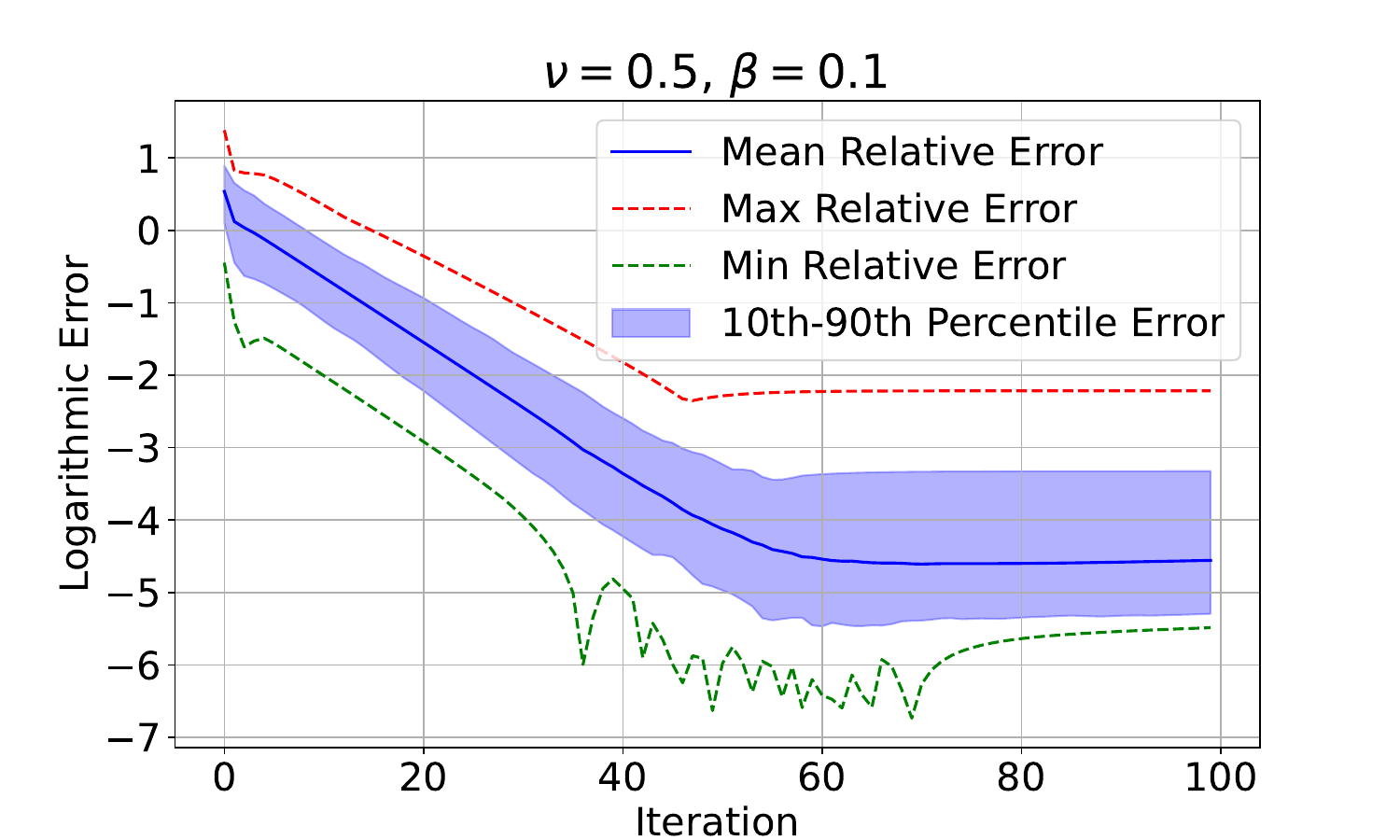}\hspace{-0.65cm}
      \includegraphics[width=5.5cm]{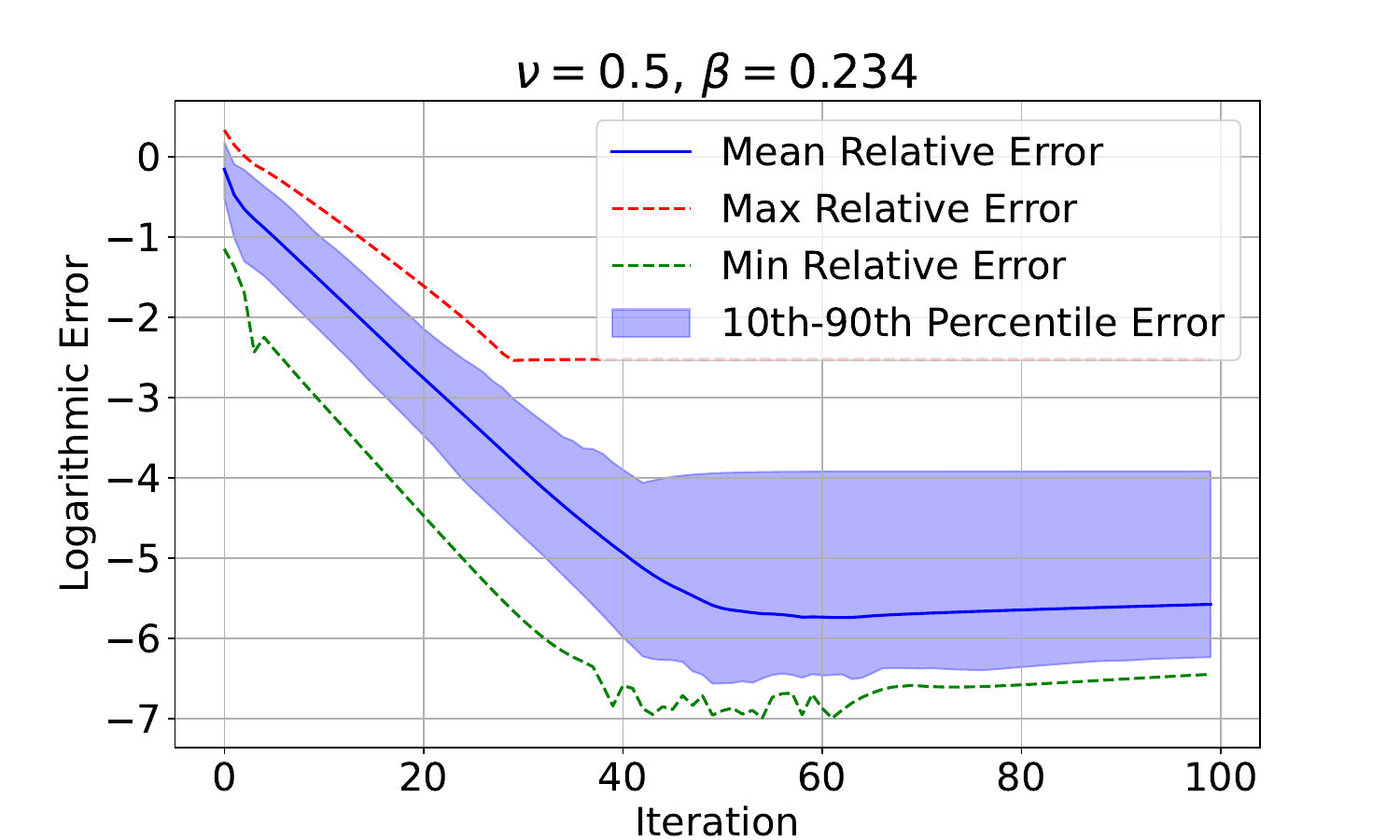}\hspace{-0.65cm}
      \includegraphics[width=5.5cm]{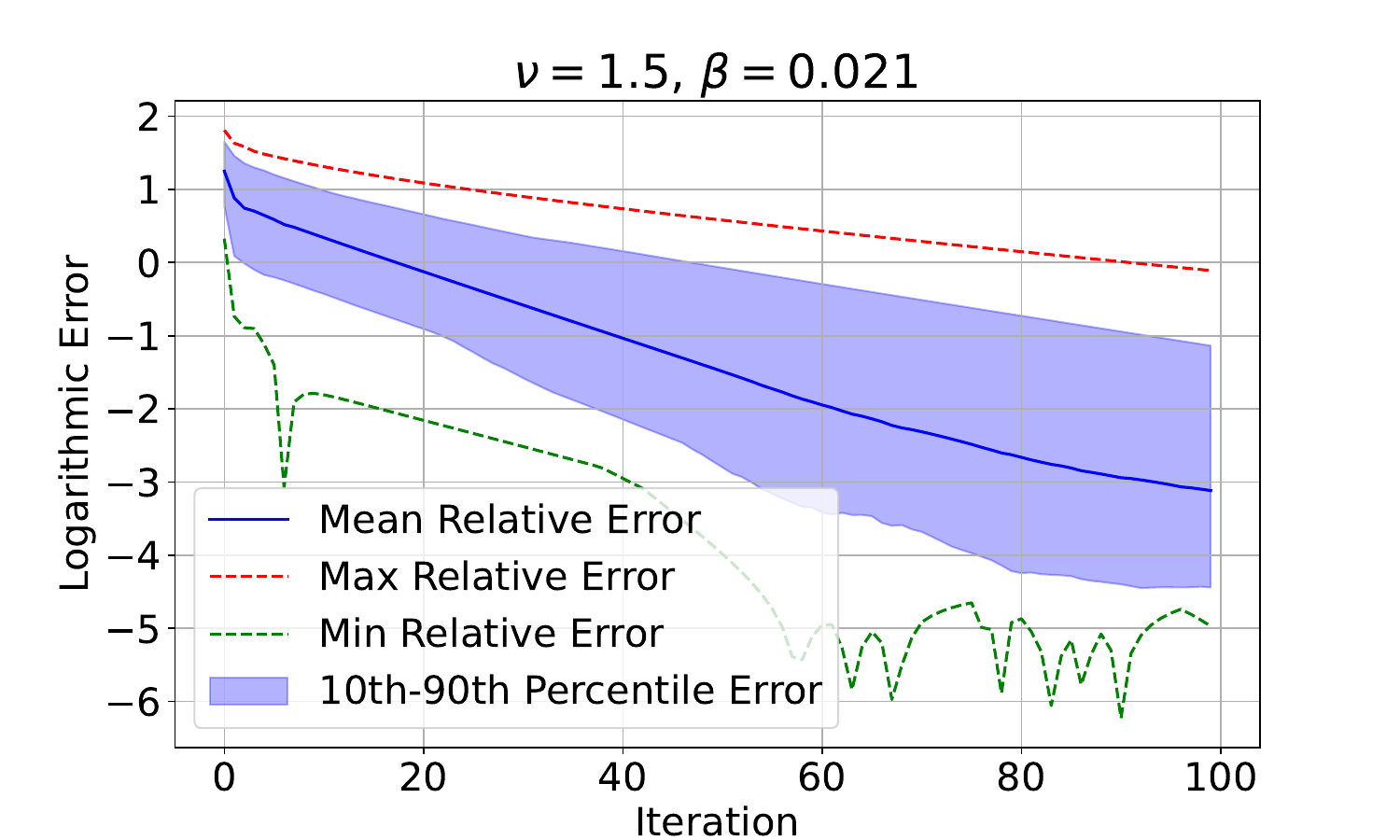}\hspace{-0.65cm}
      \includegraphics[width=5.5cm]{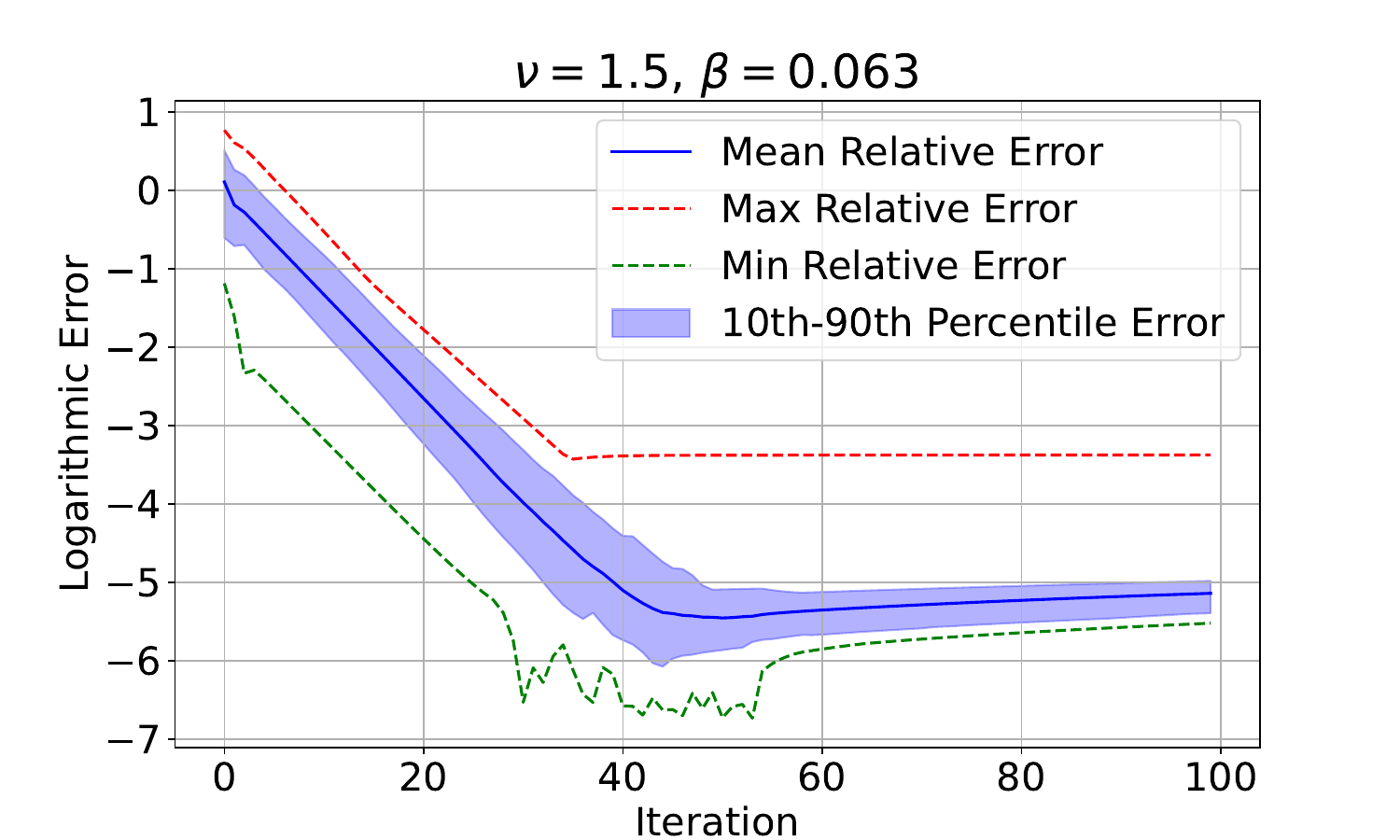}\hspace{-0.65cm}
      \includegraphics[width=5.5cm]{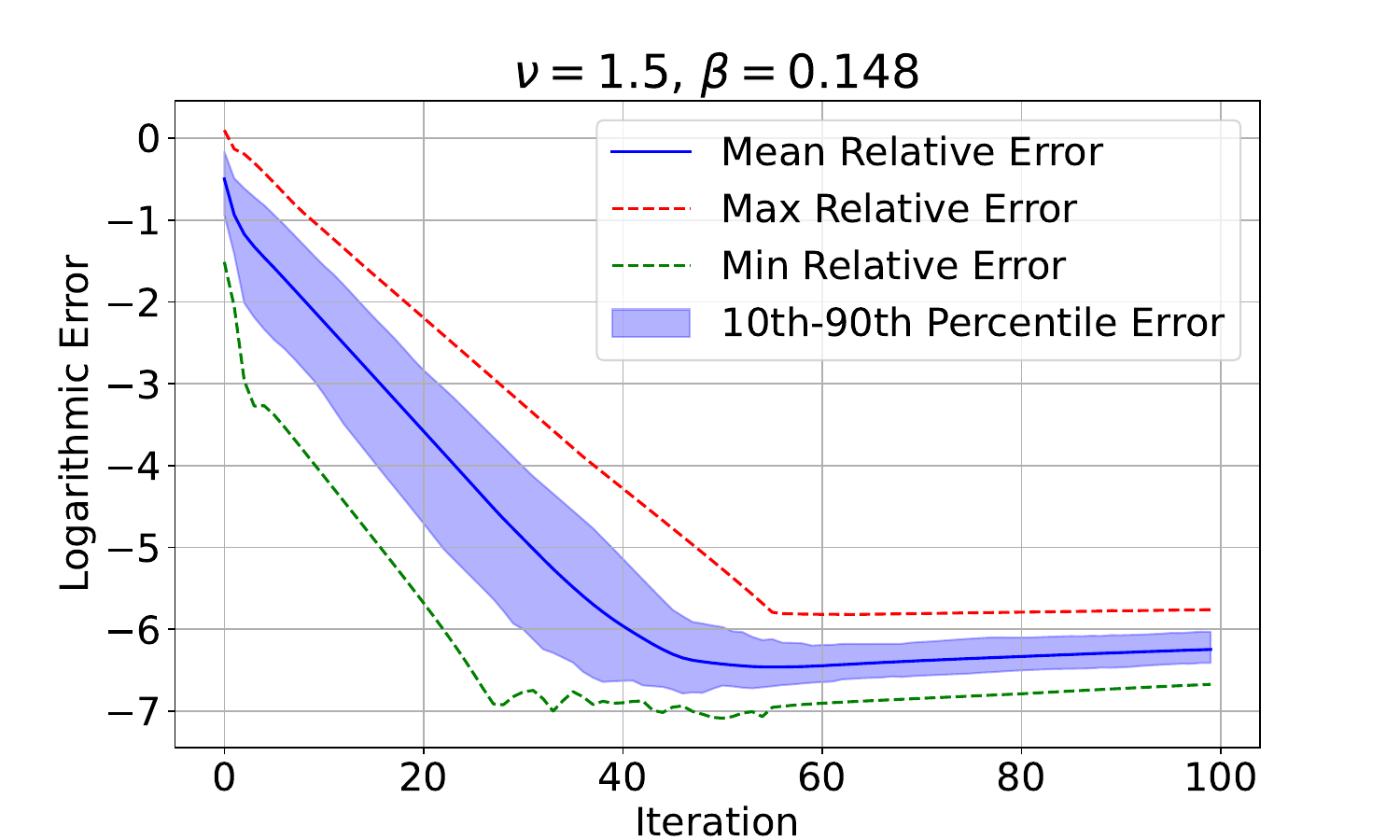}\hspace{-0.65cm}
  \end{center}
  \caption{Convergence of the decentralized method with different parameter settings.}
  \label{fig:convergence_comparison}
\end{figure}

Furthermore, we assess the robustness of the decentralized method to different data partitioning schemes. We consider a total of 9,000 data points distributed across 9 machines, using the following partitioning schemes: Random partitioning, where data points are assigned randomly to each machine without regard to spatial proximity;  Area-based partitioning, where data points are partitioned based on spatial locations; and Random + neighbors, a hybrid partitioning scheme where data points are assigned randomly but with a specified number of spatially neighboring points also included in each subset. We experiment with this hybrid partitioning scheme using $9$, $99$, and $999$ neighboring points per point to explore varying levels of spatial locality. Due to space limitations, the results for 99, and 999 neighboring points per point are included in the Supplementary Material.  Figures \ref{fig:partitioned locations} and \ref{fig:partition} from the Supplementary Material illustrate these partitioning schemes, where each color corresponds to a machine, and present the convergence results, which are similar under different partitioning schemes. This indicates that the method is resilient to a variety of partitioning schemes.

\begin{figure}[htbp]
  \centering
  \begin{center}
    \includegraphics[width=5.5cm]{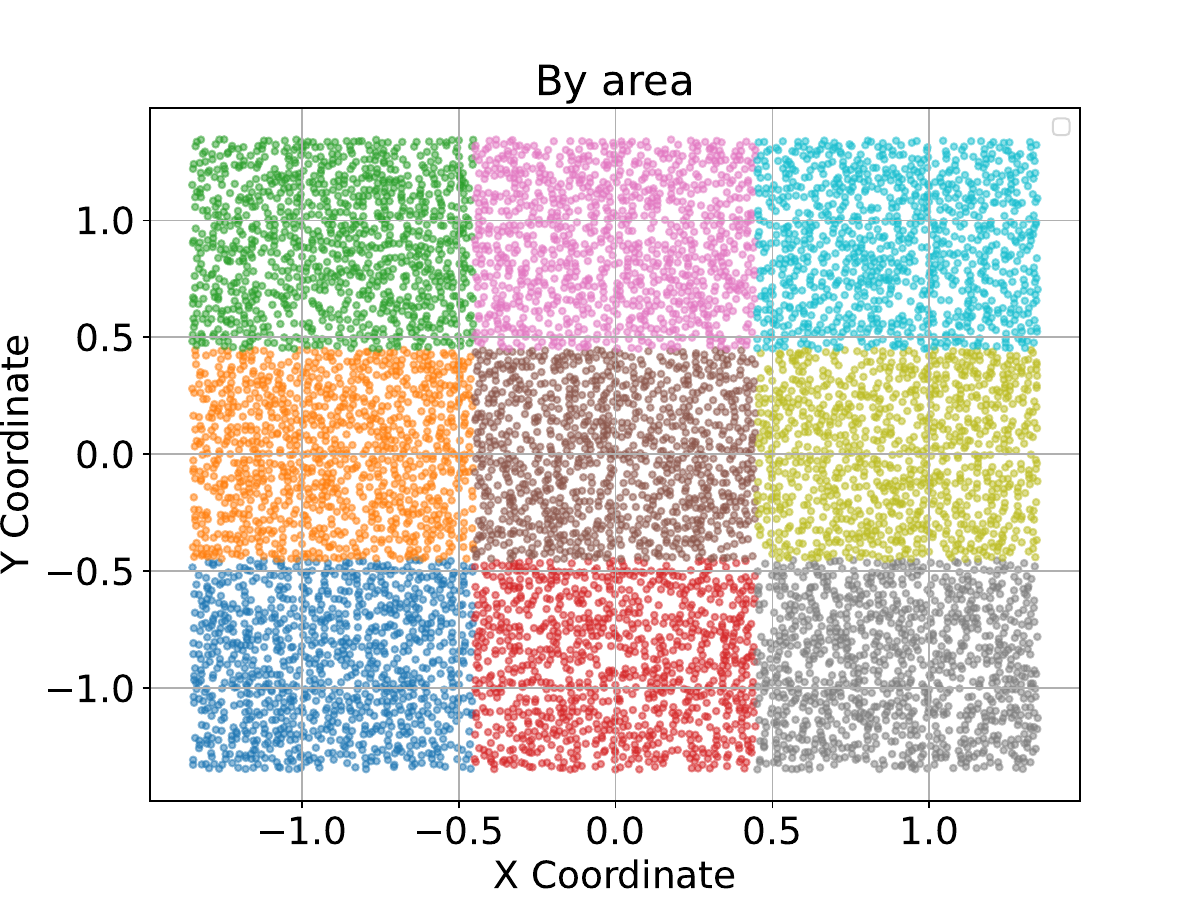}\hspace{-0.65cm}
    \includegraphics[width=5.5cm]{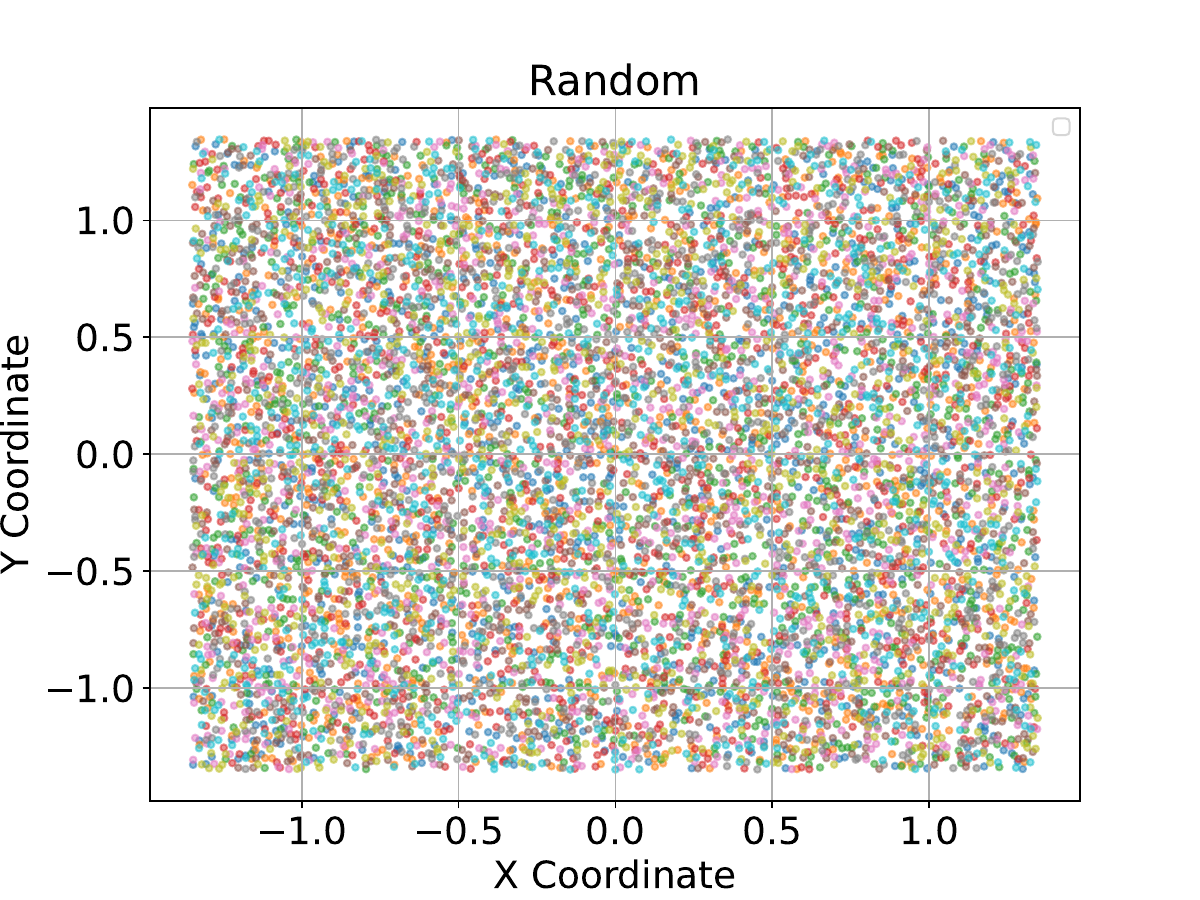}\hspace{-0.65cm}
    \includegraphics[width=5.5cm]{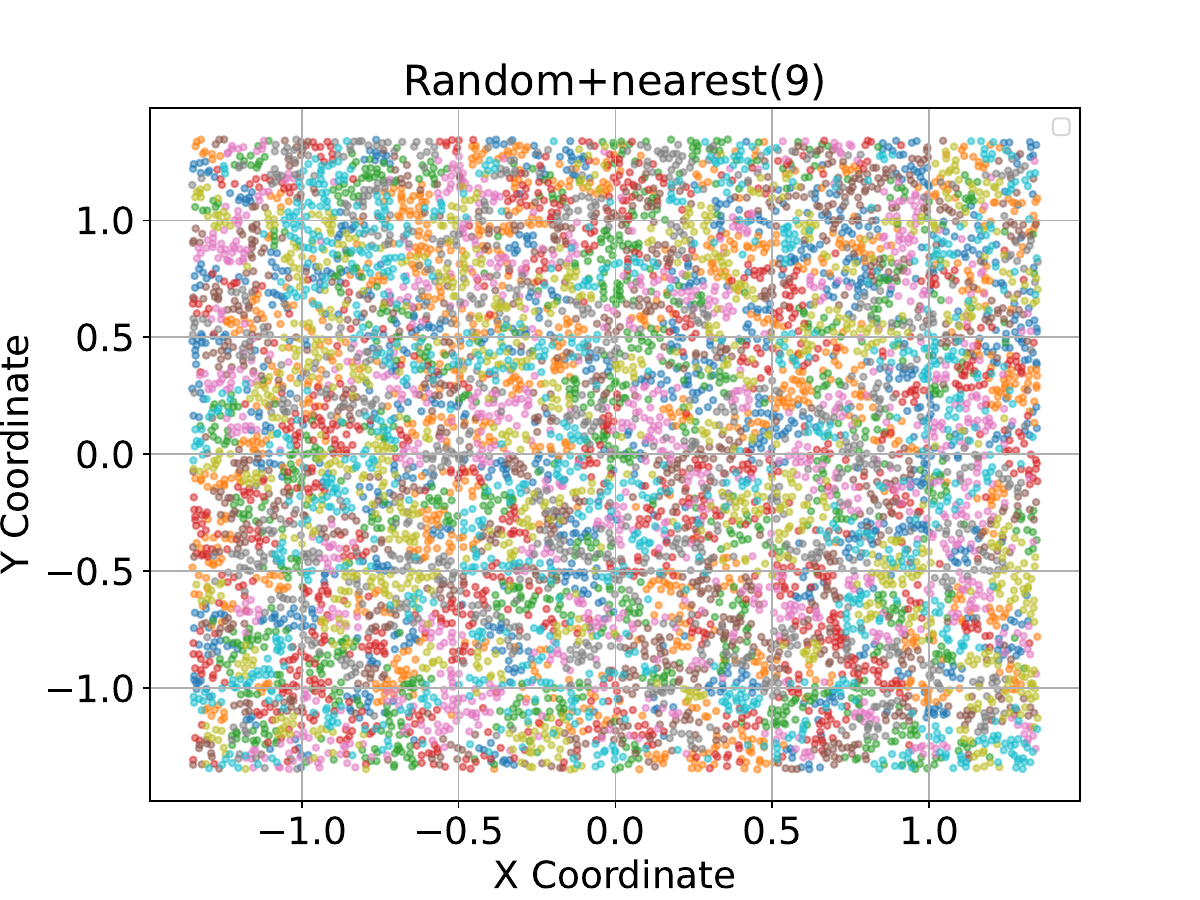}\hspace{-0.65cm}
    \includegraphics[width=5.5cm]{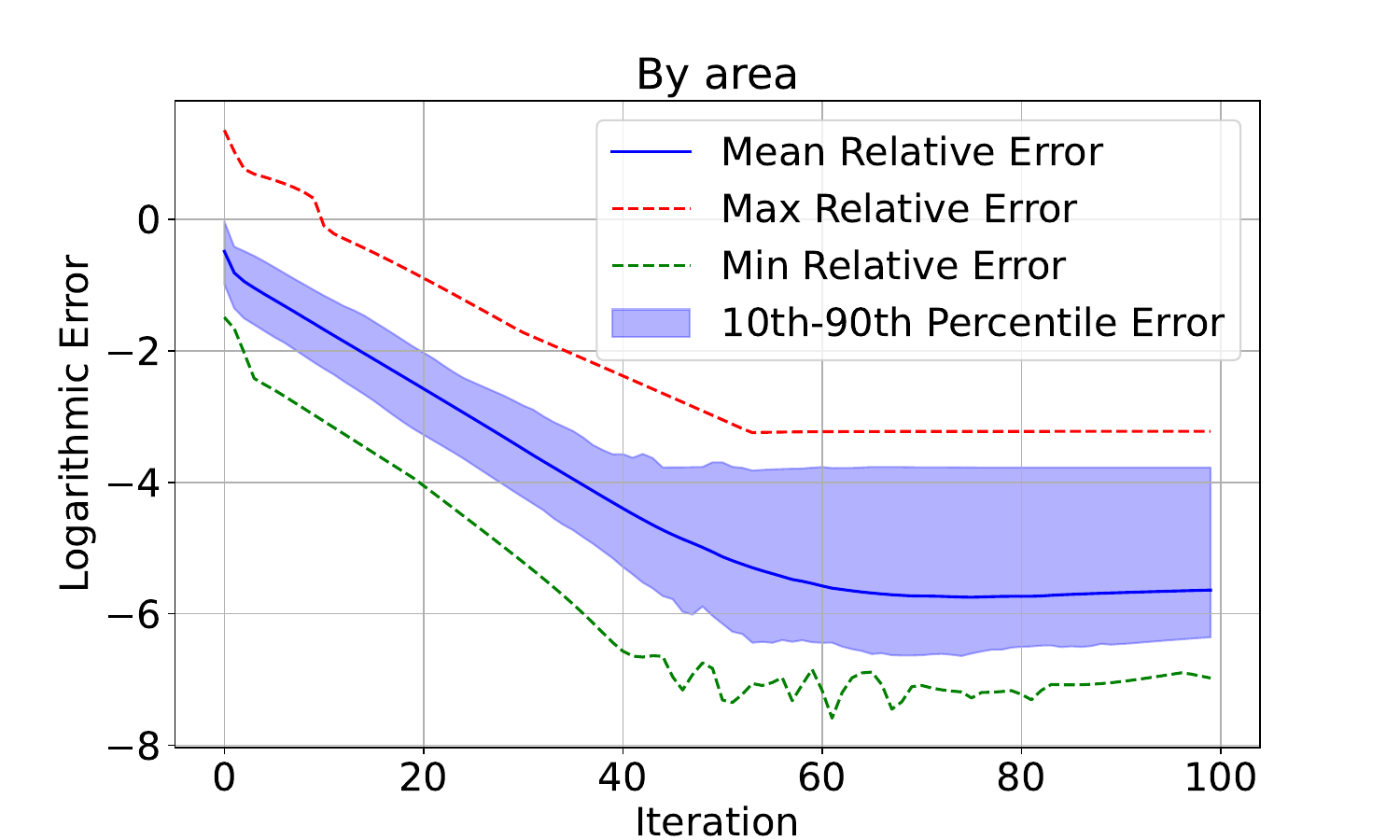}\hspace{-0.65cm}
    \includegraphics[width=5.5cm]{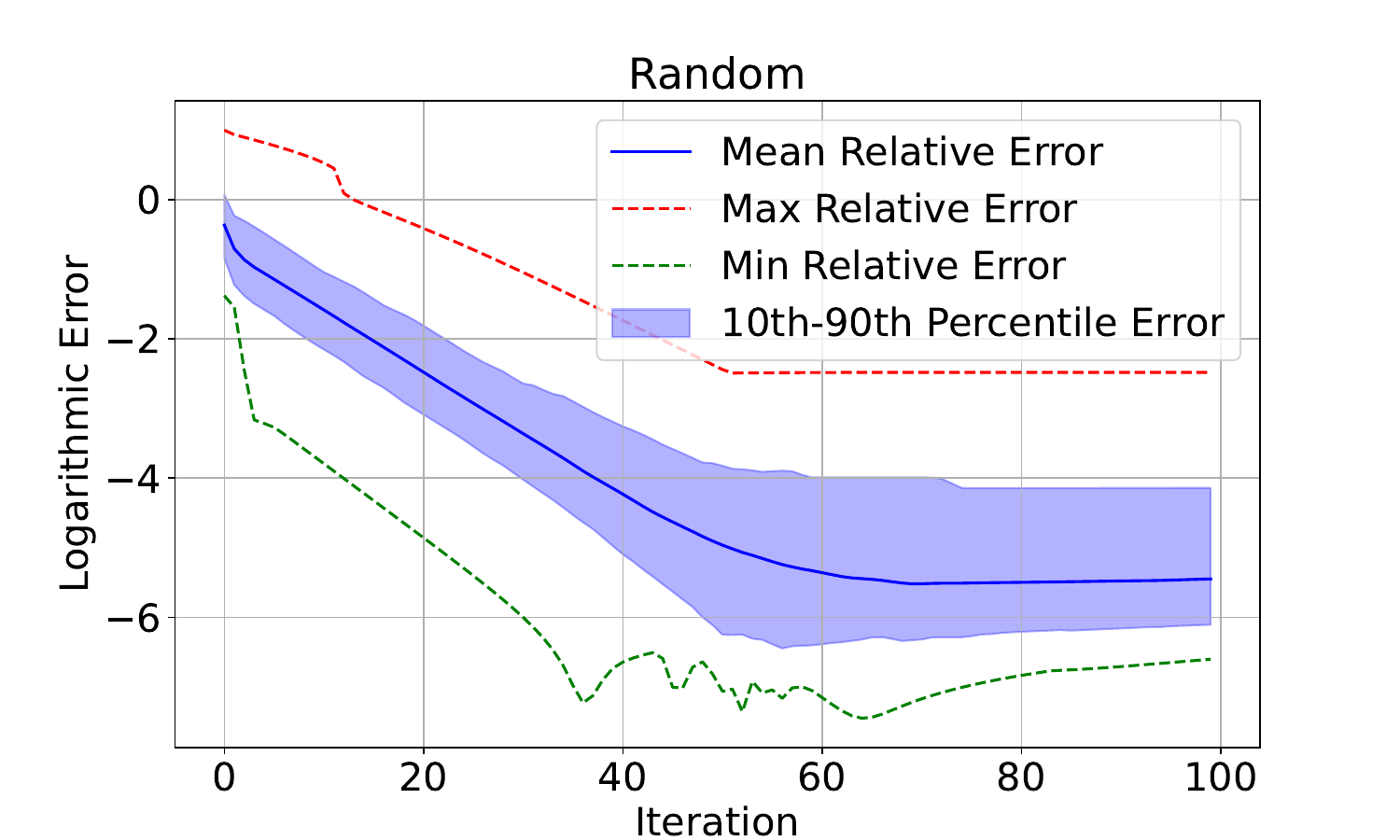}\hspace{-0.65cm}
    \includegraphics[width=5.5cm]{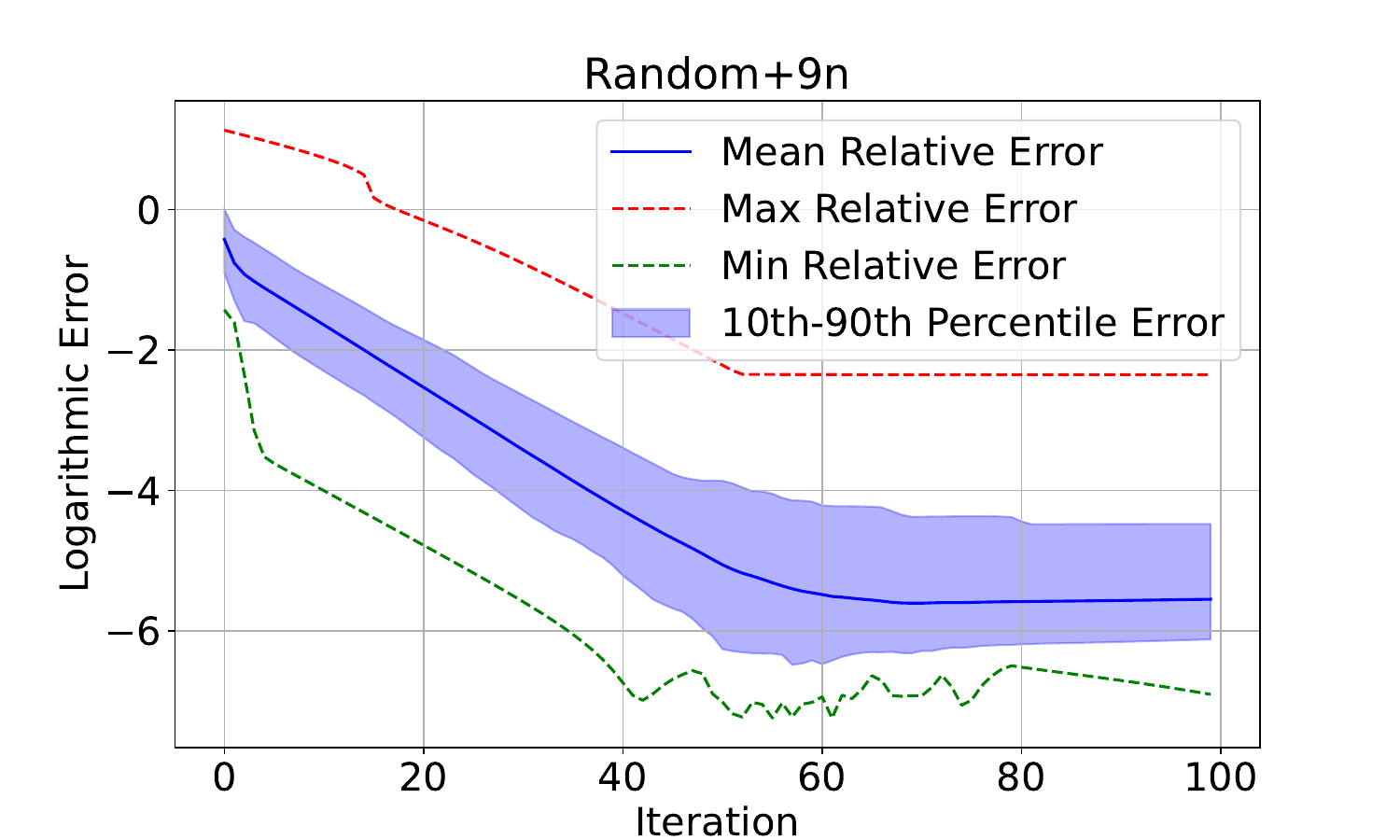}\hspace{-0.65cm}
\end{center}
  \caption{Partitioned locations (upper panel) and convergence of the decentralized method (lower panel) under different data partitioning schemes}
  \label{fig:partitioned locations}
\end{figure}

We further examine the robustness of the decentralized method under varying network connection probabilities and imbalanced sample sizes across machines. We consider two scenarios for network connectivity: a sparse network with a connection probability of $0.3$ and a dense network with a connection probability of $0.8$. Regarding sample sizes, we create an imbalanced setup where five machines each hold $250$ data points, while eight machines each hold $1{,}000$ data points. The results, presented in Figure \ref{fig:merged_results_convergence}, demonstrate the method’s robustness under realistic and diverse conditions in both network structure and sample distribution.

\begin{figure}[htbp]
  \centering
  \begin{subfigure}{0.325\textwidth}
      \centering
       \includegraphics[width=1.1\textwidth]{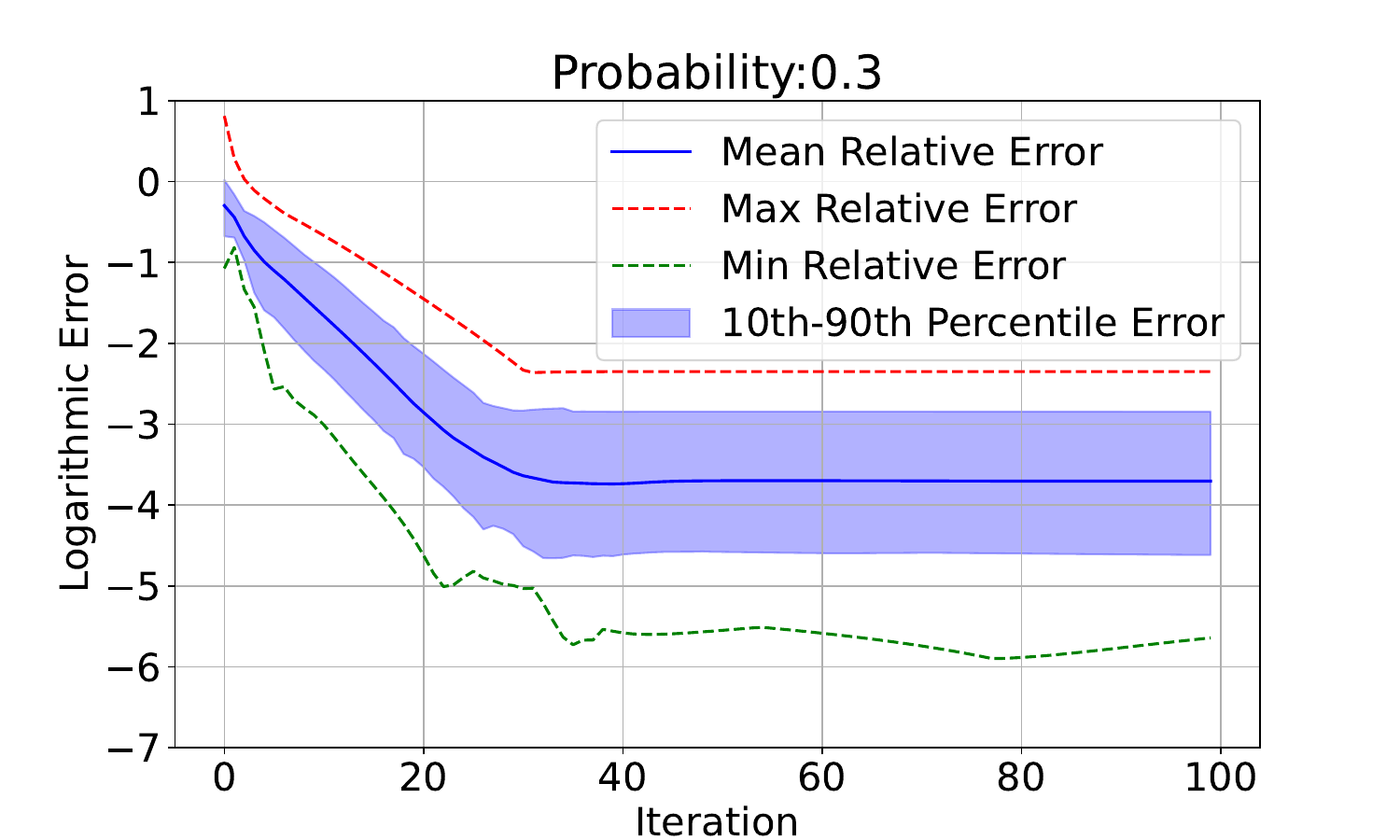}

      \caption{}
     \label{fig:convergence_connection_prob_0.3}
  \end{subfigure}
  \begin{subfigure}{0.325\textwidth}
      \centering
      \includegraphics[width=1.1\textwidth]{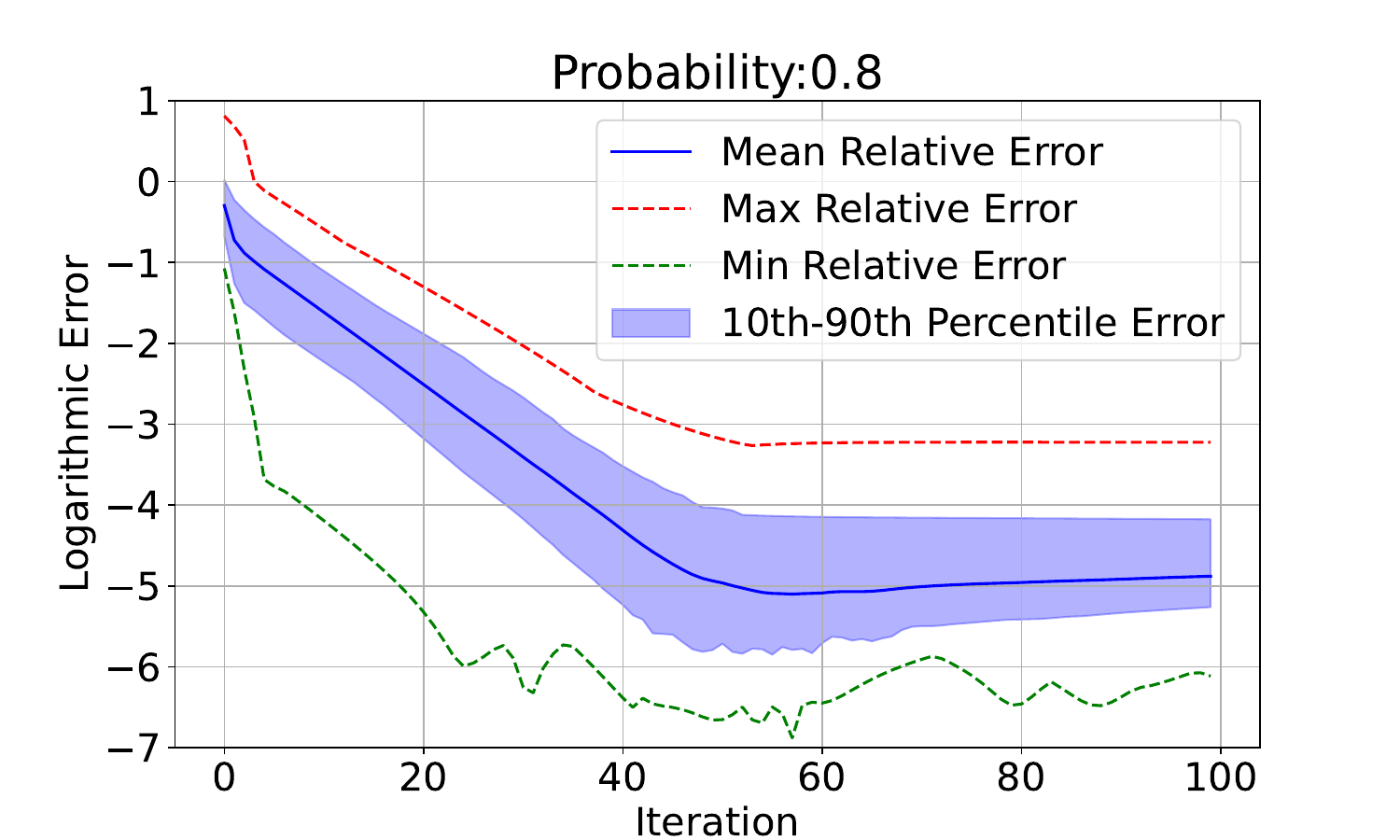}
      \caption{}
     \label{fig:convergence_connection_prob_0.8}
  \end{subfigure}
  \begin{subfigure}{0.325\textwidth}
      \centering
      \includegraphics[width=1.07\textwidth]{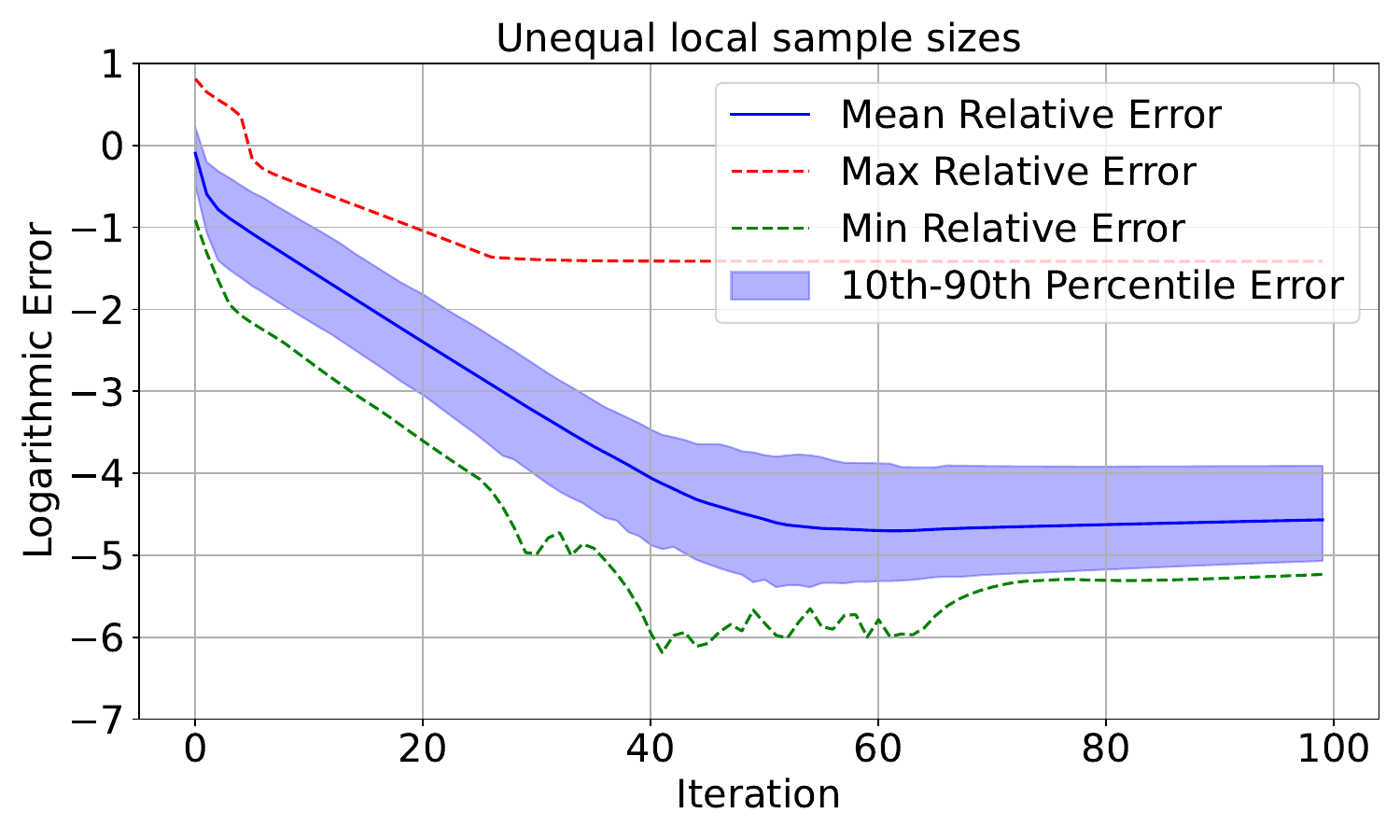}
      \caption{}
       \label{fig:convergence_unequal_sizes}
  \end{subfigure}
  \caption{Convergence of the decentralized method: (a) with a network connection probability of 0.3; (b) with a network connection probability of 0.8; (c) under unequal local sample sizes.}
 \label{fig:merged_results_convergence}
\end{figure}

\vspace{-10pt}
\subsection{Estimation Accuracy}
Next, we investigate the impact of increasing the sample size and rank on the accuracy of the decentralized and MLE estimators. With other factors held constant, the total sample size varies between \( 10{,}000 \), \( 20{,}000 \), and \( 40{,}000 \), while the rank varies among \( 100 \), \( 200 \), and \( 300 \). Since the decentralized estimators across machines are nearly identical, we present their mean. As anticipated from Theorem \ref{thm:asy_normality}, Figure \ref{fig:merged_results_accuracy},  included in the Supplementary Material due to space constraints,  illustrates that the accuracy of the estimates for the linear coefficients and the nugget parameter improves with larger sample sizes, while the accuracy of kernel parameter estimates improves with increasing rank.

We also evaluate the accuracy of the estimated asymptotic variance and the constructed confidence intervals derived according to Theorem \ref{thm:asy_normality}. Due to space constraints, these results are presented in the Supplementary Material. Specifically,  Table  \ref{tab:CI_simu}  demonstrates that the estimated asymptotic variance aligns closely with theoretical expectations, and the constructed confidence intervals achieve empirical coverage rates near the nominal 95\%. These findings corroborate the conclusions of Theorem  \ref{thm:asy_normality}.

Furthermore, we also test the performance of the decentralized method under model misspecification. Specifically, the true model is now
\vspace{-15pt}
\begin{equation}
  z (\mathbf{s}_{j i}) = \boldsymbol{x}_{j i}^{\top} \boldsymbol{\gamma} +
 w (\mathbf{s}_{j i})+
 \varepsilon (\mathbf{s}_{j i}),\ i = 1, \ldots n_j, j = 1, \ldots, J,
 \vspace{-15pt}
\end{equation}
where $\text{Cov}(w (\mathbf{s}),w (\mathbf{s}'))=c(\mathbf{s},\mathbf{s}')$ with $c(\cdot,\cdot)$ being the Matérn covariance function, but we still use the low-rank model in Equation \eqref{eq: model} for parameter estimation. Due to space limitations, the results are included in the Supplementary Material. Specifically, Figure \ref{fig:misspecification} from the Supplementary Material shows that even when the model is misspecified, the decentralized method performs comparably to the MLE method. The performance remains robust, particularly when the field is smooth and the dependency is strong. These results highlight the method's resilience in practical scenarios where model assumptions may not be perfectly satisfied.

\vspace{-15pt}
\subsection{Evaluation of Prediction Accuracy, Computation Efficiency, and Smoothness Parameter Estimation}\label{sec:simu_others}
We further evaluate the prediction accuracy of the decentralized method. In this experiment, \( 10{,}000 \) data points are used for estimation, while \( 50 \) data points are reserved for prediction. The nugget parameter is set to \( \sqrt{0.1} \). Mean squared prediction error (MSPE) serves as the evaluation criterion. As shown in Figure \ref{fig:merged_results_addition}(\subref{fig:prediction_error}), the decentralized method achieves an MSPE nearly identical to that of the MLE method, demonstrating that the decentralized approach performs comparably in terms of prediction accuracy.
\begin{figure}[htbp]
  \centering
  \begin{subfigure}{0.33\textwidth}
      \centering
       \includegraphics[width=\textwidth]{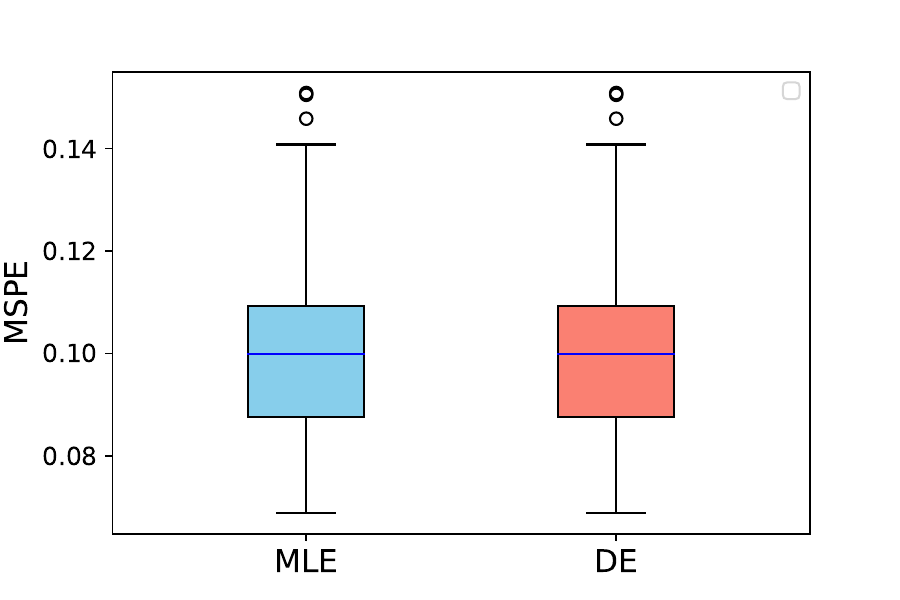}
      \caption{}
      \label{fig:prediction_error}
  \end{subfigure}
  \begin{subfigure}{0.31\textwidth}
      \centering
      \includegraphics[width=\textwidth]{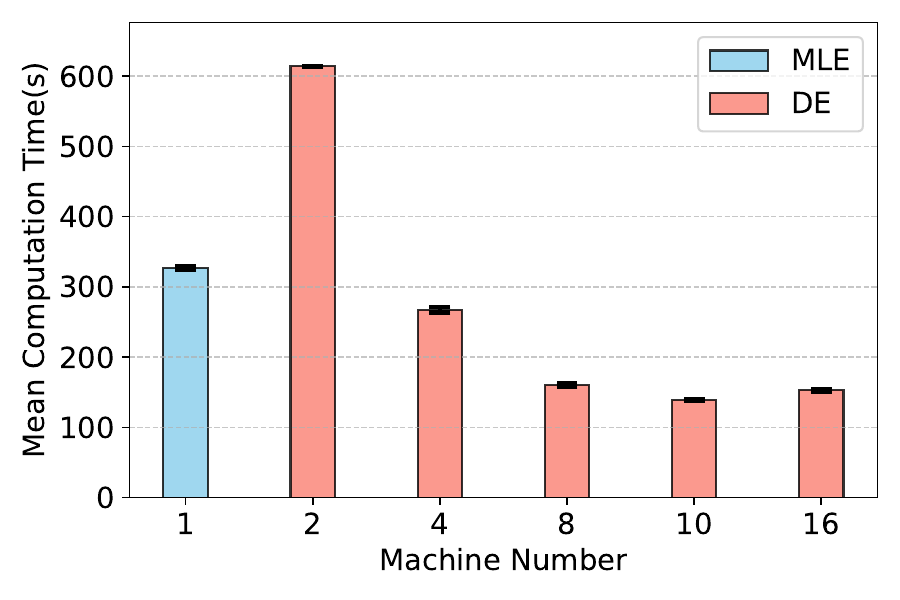}
      \caption{}
      \label{fig:time_comparison}
  \end{subfigure}
   \begin{subfigure}{0.305\textwidth}
      \centering
      \includegraphics[width=\textwidth]{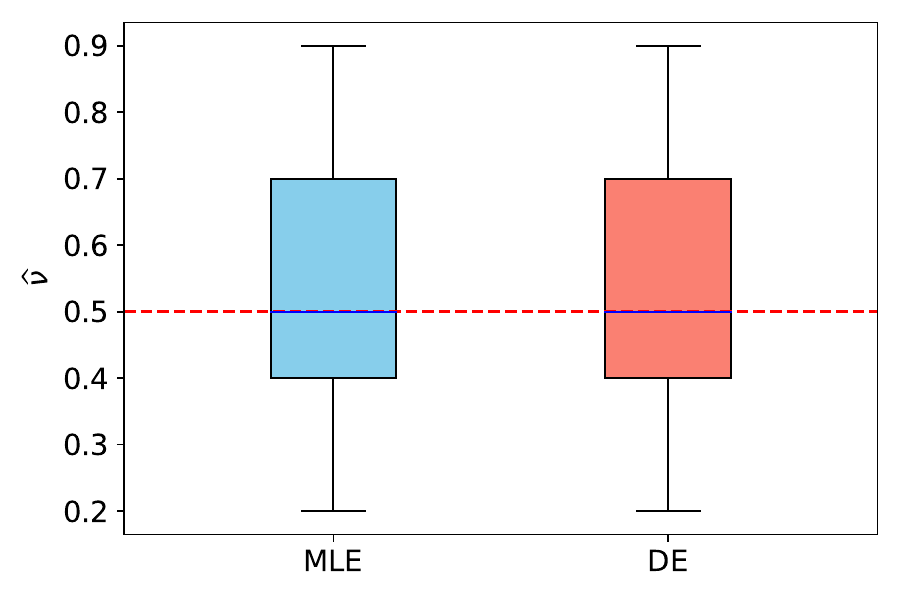}
      \caption{}
      \label{fig:estimate_nu}
  \end{subfigure}
  \caption{(a) Prediction error. (b) Time comparison between distributed and non-distributed method. (c) Estimated value of the smooth parameter.}
  \label{fig:merged_results_addition}
\end{figure}

We also compare the computation time of the distributed method with the traditional non-distributed MLE method, the latter of which employs the BFGS algorithm for optimization. In this experiment, the total sample size is fixed at \( 240,000 \), while the number of machines (simulated using multiple CPU cores, each corresponding to a machine) varies from 2 to 16. The computations were performed on an Ubuntu 22.04.5 LTS system with an Intel Xeon E5-2680 processor (2.50 GHz, 24 cores). The experiment was repeated over 20 replications, and the results are shown in Figure \ref{fig:merged_results_addition}(\subref{fig:time_comparison}). Each bar represents the mean computation time, with error bars indicating the standard deviation. 

The results demonstrate that the decentralized method outperforms the non-distributed method as the number of machines increases. Specifically, the decentralized method requires less time to reach convergence with more machines, making it a scalable and efficient solution for large datasets. These findings highlight the effectiveness of the decentralized approach in addressing large-scale computational challenges.

Finally, we address the estimation of the smoothness parameter \(\nu\). A straightforward approach involves computing the negative log-likelihood for a range of candidate values of \(\nu\) and selecting the value that minimizes the negative log-likelihood as the estimate. In this experiment, the true value of \(\nu\) is set to \( 0.5 \), and the negative log-likelihood is evaluated for candidate values of \(\nu\) ranging from \( 0.2 \) to \( 0.9 \) in increments of \( 0.1 \). The results shown in Figure \ref{fig:merged_results_addition}(\subref{fig:estimate_nu}) demonstrate that this method can accurately identify the true value of \(\nu\).

\vspace{-10pt}
\section{Application to Total Precipitable Water Data }\label{sec:realdata}
In this section, we demonstrate the application of our method using a real-world spatial dataset that records total precipitable water (TPW) values, available online\footnote{The dataset can be accessed at \url{https://www.nesdis.noaa.gov/data-research-services/data-collections}.}. TPW represents the water depth in a column of the atmosphere, assuming all the water vapor is condensed.  Derived from multiple satellites, this dataset is ideal for demonstrating the two scenarios of the distributed method. As discussed in Section \ref{sec:intro}, the first scenario assumes that data cannot be transferred due to privacy concerns or high transfer costs, where each satellite corresponds to a distinct machine in the system. The second scenario assumes data transfer is feasible to improve computational efficiency.

For the first scenario, our analysis focuses on a region\footnote{The region spans latitude $24.40^\circ \mathrm{N}$ to $49.38^\circ \mathrm{N}$ and longitude $125.00^\circ \mathrm{W}$ to $66.93^\circ \mathrm{W}$} covering the contiguous United States during the time between 8 a.m. and 9 a.m. UTC on November 5, 2024. The distribution of data across different satellites and the corresponding connection network are shown in Figure \ref{fig:TPW}  where the numbers 19, 102, 103, 403, 501, 1000 represent satellites numbering.. We then apply the proposed method to the distributed data. The smoothing parameter is fixed at $\nu=1.5$, and no linear coefficient vector 
$\boldsymbol{\gamma}$  is included. Additionally, the mean is subtracted from each observation to ensure zero-centered data. The convergence results shown in Figure \ref{fig:merged_results_real_data}(\subref{fig:result}) demonstrate that the proposed decentralized method converges quickly. Furthermore, the resulting decentralized estimators,  along with confidence intervals presented in Table \ref{tab:CI_real},  closely match the non-distributed MLE estimator and the associated confidence interval.  


\begin{figure}[t!]
  \centering
  \begin{subfigure}{0.5\textwidth}
      \centering
      \includegraphics[width=0.7\textwidth]{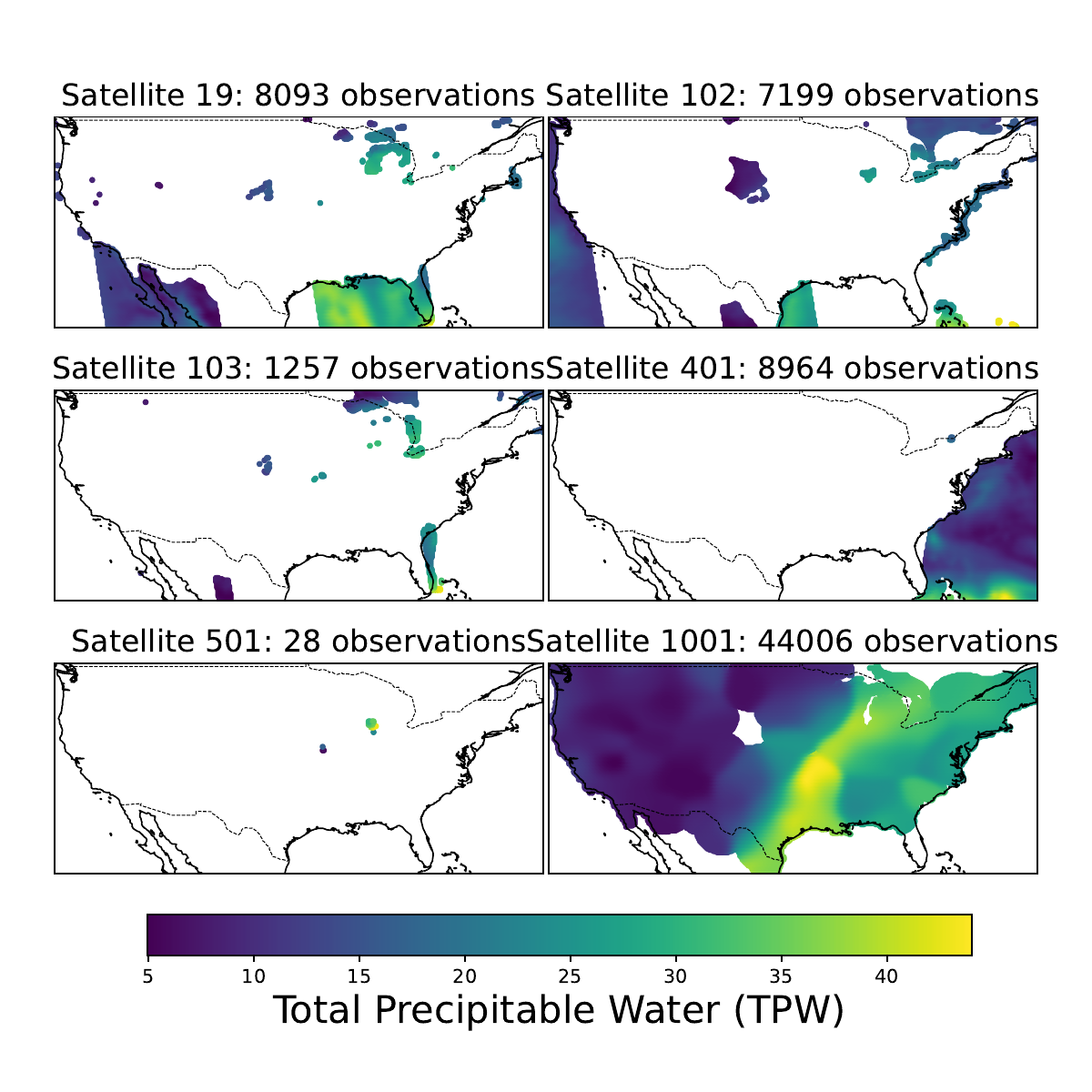}
      \caption{}
  \end{subfigure}
  \begin{subfigure}{0.45\textwidth}
      \centering
      \includegraphics[width=0.7\textwidth]{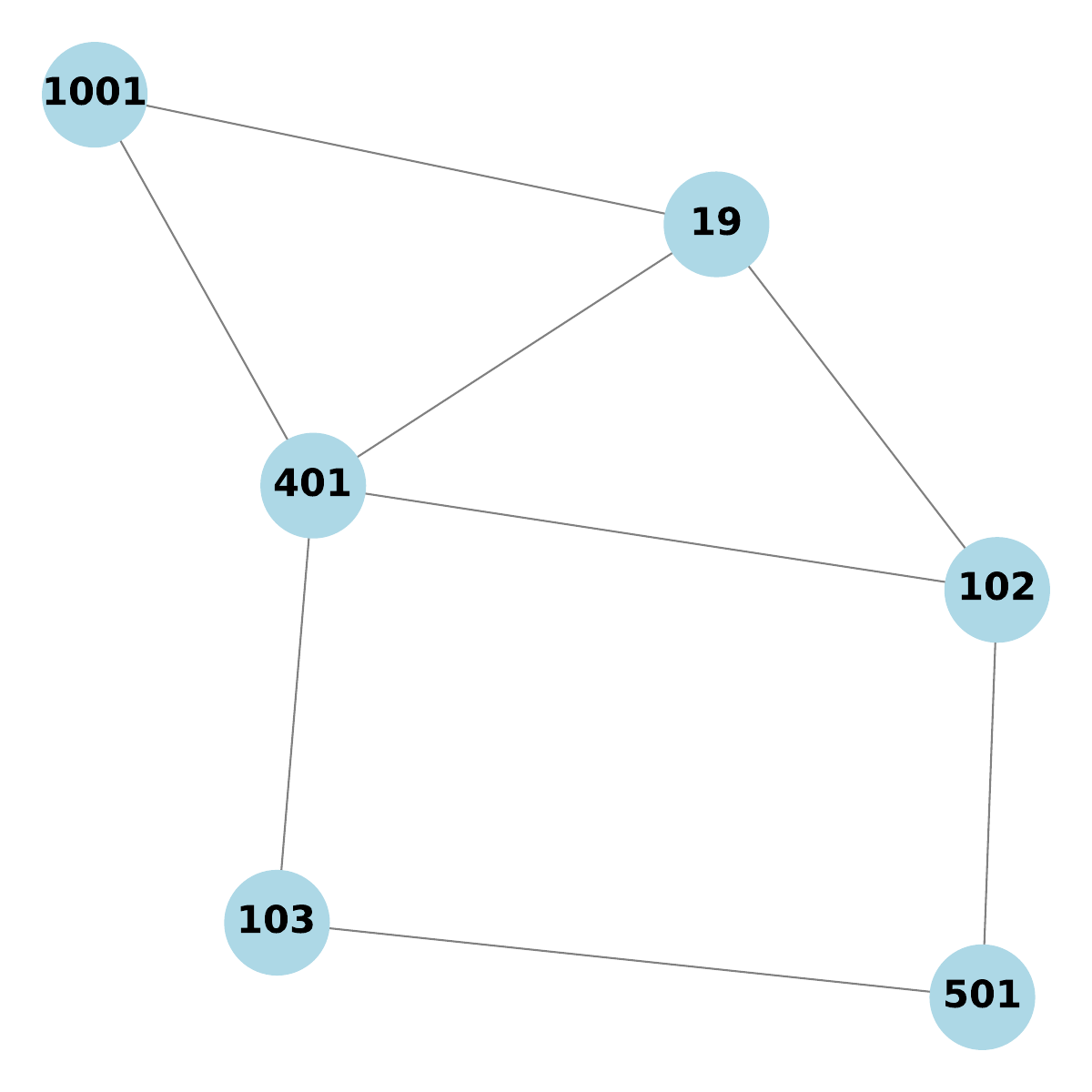}
      \caption{}
  \end{subfigure}
  \caption{(a) TPW observations across different satellites in the first scenario. (b) The connection network in the first scenario.}
  \label{fig:TPW}
\end{figure}

\begin{figure}[t!]
  \centering
  \begin{subfigure}{0.48\textwidth}
      \centering
      \includegraphics[width=0.9\textwidth]{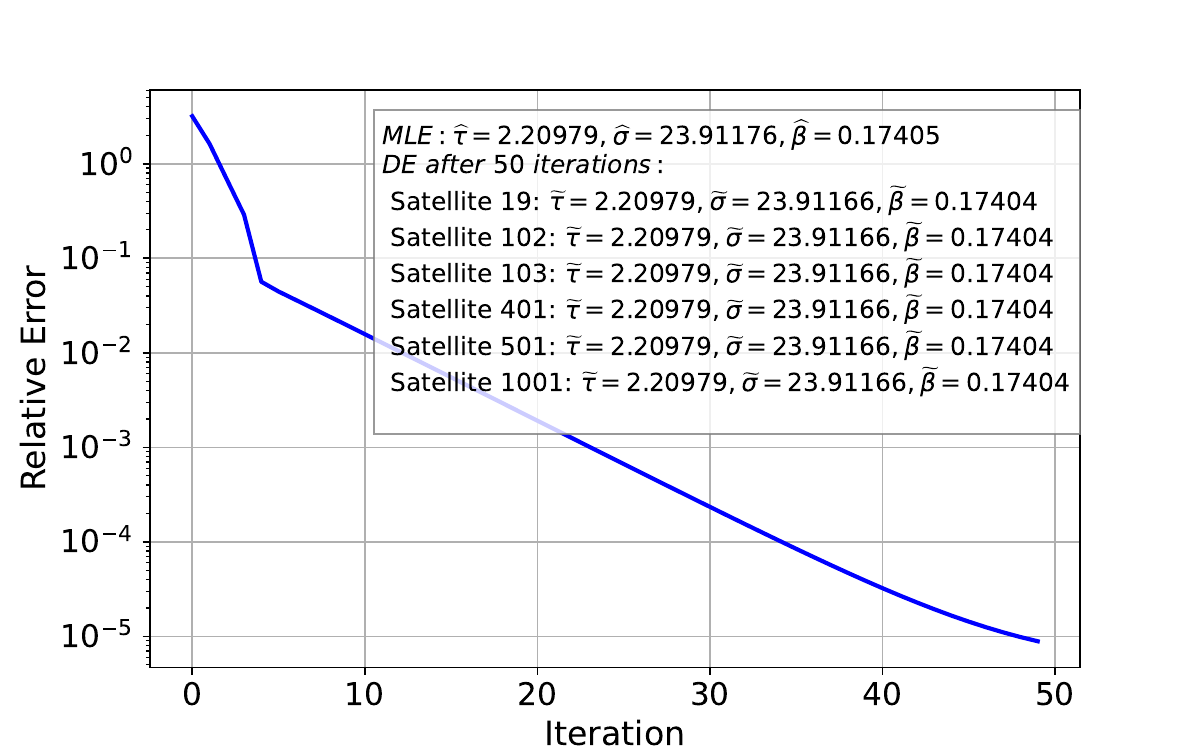}
      \caption{}
      \label{fig:result}
  \end{subfigure}
  \begin{subfigure}{0.4\textwidth}
      \centering
      \includegraphics[width=0.9\textwidth]{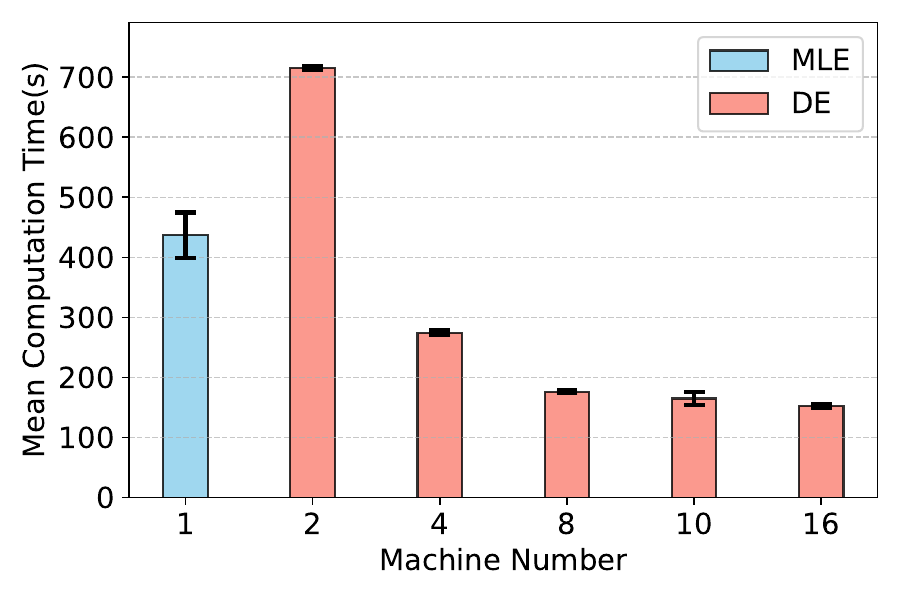}
      \caption{}
      \label{fig:time_comparison_real_data}
  \end{subfigure}
  \caption{(a) Convergence of the decentralized method and the estimated values for TPW data in the first scenario. (b) Time comparison between the distributed and non-distributed method for TPW data in the second scenario (The rank is fixed at 60).}
  \label{fig:merged_results_real_data}
\end{figure}

\begin{table}[h]
    \centering
    \caption{Confidence intervals of parameters for TPW data in the first scenario. Since the values with the given precision are consistent across different machines, we present results from a single machine for simplicity.}
    \begin{tabular}{lcccccc}
        \toprule
        & $\tau_{\text{lower}}$ & $\tau_{\text{upper}}$ & $\sigma_{\text{lower}}$ & $\sigma_{\text{upper}}$ & $\beta_{\text{lower}}$ & $\beta_{\text{upper}}$ \\
        \midrule
        MLE & 2.1983 & 2.2215 & 20.1940 & 27.1247 & 0.1688 & 0.1793 \\
        DE& 2.1983 & 2.2215 & 20.1939 & 27.1246 & 0.1688 & 0.1793 \\
        \bottomrule
    \end{tabular}
    \label{tab:CI_real}
\end{table}

For the second scenario, we use a larger dataset\footnote{This dataset represents TPW measurements taken between 2 a.m. and 3 a.m. UTC on February 1, 2011, over a region encompassing the United States.}($N=271,014$) to demonstrate computational efficiency. Distributed machines are simulated using multiple CPU cores, each corresponding to a machine. The computations were conducted on an Ubuntu 22.04.5 LTS system with a 24-core Intel Xeon E5-2680 processor (2.50 GHz). The dataset is evenly distributed across several machines, with the number of machines varying in {2, 4, 8, 10, 16}. 

We compared the distributed method to the non-distributed method in terms of computation time. The first result, shown in Figure \ref{fig:merged_results_real_data}(\subref{fig:time_comparison_real_data}), indicates that the distributed method significantly outperforms the non-distributed method as the number of machines increases. This observation aligns with the findings presented in Figure \ref{fig:merged_results_addition}(\subref{fig:time_comparison}) in Section \ref{sec:simu_others}. The second result, presented in the left panel of Figure \ref{fig:real_varying_m}, highlights that the advantage of the distributed method becomes more pronounced for larger ranks. In other words,  for a fixed time budget, the distributed method enables the use of higher ranks, which improves predictive performance, as illustrated in the right panel of Figure \ref{fig:real_varying_m}. This effectively mitigates some of the limitations of low-rank models.
\begin{figure}[htbp]
 \begin{center}
    \includegraphics[width=5.3cm]{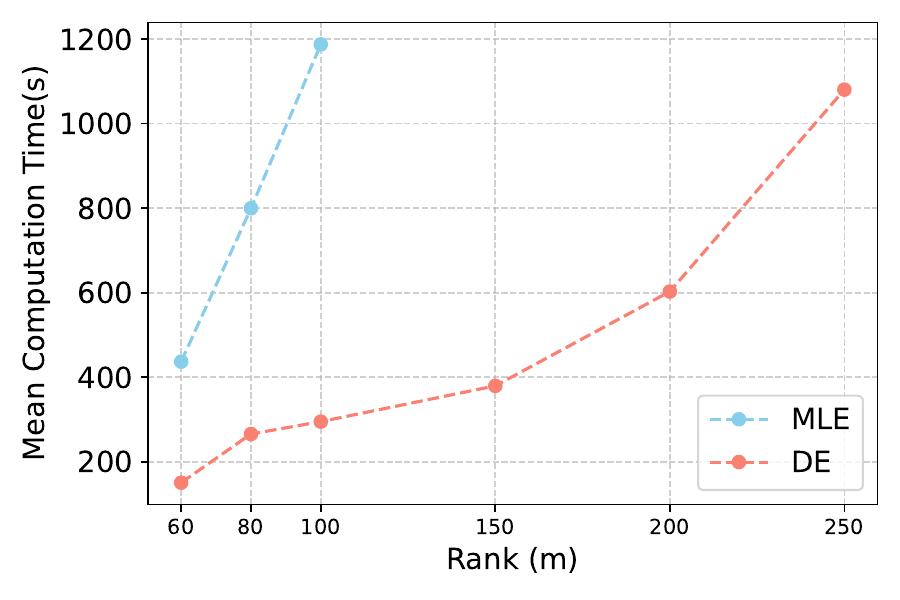}
    \includegraphics[width=5.3cm]{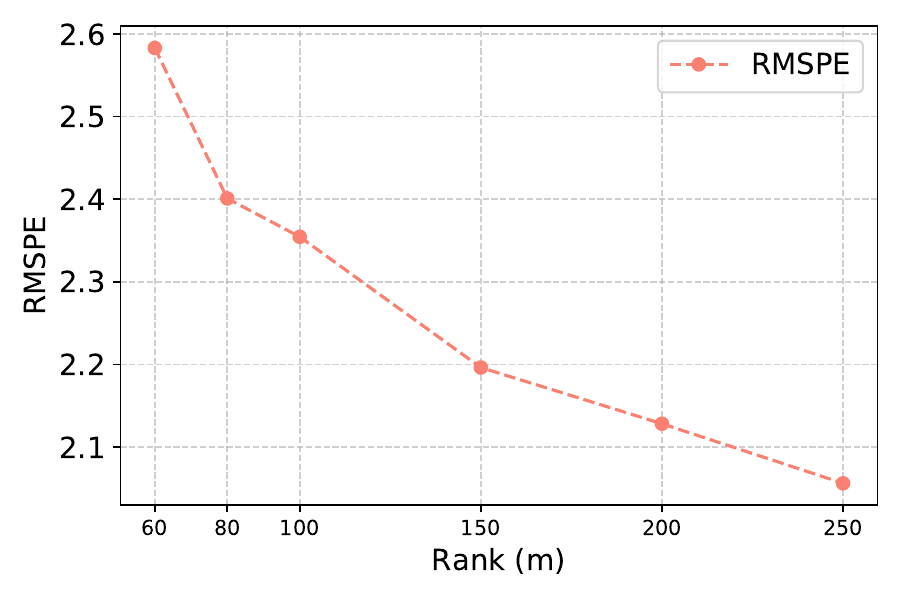}   
\end{center}
  \caption{Comparison of computation time between distributed and non-distributed methods for TPW data in the second scenario across varying ranks (left panel) and root mean squared prediction error (RMSPE) versus rank (right panel). In the left panel, the number of machines is fixed at 16. In the right panel, 1,000 data points are randomly drawn for prediction, while the remaining data points are used for parameter estimation. Note that the distributed and non-distributed methods yield identical RMSPE values, so only a single line is shown in the right panel.}
  \label{fig:real_varying_m}
\end{figure}

\vspace{-1cm}

\section{Conclusion}\label{sec:conclusion}

This paper introduces a decentralized inference framework for spatial low-rank models, addressing key limitations of centralized methods, such as vulnerability to single-machine failures and communication bottlenecks. The proposed approach enables efficient and scalable parameter estimation in decentralized networks by optimizing the ELBO. We provide theoretical guarantees, including the convexity of the ELBO, as well as the consistency and asymptotic normality of the estimators. Extensive simulations validate the effectiveness of the method, demonstrating its scalability, accuracy, and robustness. This work offers a practical and theoretically sound solution for the decentralized analysis of large spatial datasets, marking a significant advancement in distributed computing for spatial modeling.

Several extensions can be explored in future work. First, integrating a multi-resolution approach \citep{katzfuss2017multi} could enhance the model's ability to capture fine- and coarse-scale spatial structures. Second, extending the method to accommodate non-stationary spatial data \citep{kim2005analyzing}, where spatial correlations vary across regions, would increase its applicability to real-world scenarios. Finally, adapting the framework to time-varying networks \citep{saadatniaki2020decentralized}, where the topology of the communication network evolves, could further enhance robustness in dynamic environments with changing infrastructure.

\begingroup
\setstretch{1.8} 

\vspace{-0.5cm}

\bibliographystyle{agsm}

\bibliography{Bibliography}

\endgroup

\newpage
\part*{Supplementary Material}
\spacingset{1.5}

\addtolength{\oddsidemargin}{-0.5in}  
\addtolength{\evensidemargin}{-0.5in}
\addtolength{\textwidth}{-1in} 
\newcommand{\tmop}[1]{\ensuremath{\operatorname{#1}}}

\renewcommand{\thefigure}{S\arabic{figure}}
\setcounter{figure}{0}

\renewcommand{\thetable}{S\arabic{table}}
\setcounter{table}{0}

\renewcommand{\thealgorithm}{S\arabic{algorithm}}
\setcounter{algorithm}{0}

\renewcommand{\thepage}{S\arabic{page}}
\setcounter{page}{1}

\renewcommand{\thetheorem}{S\arabic{theorem}}
\setcounter{theorem}{0}

\renewcommand{\theequation}{S\arabic{equation}}
\setcounter{equation}{0}

\renewcommand{\thesection}{S\arabic{section}}
\setcounter{section}{0}

\section{Analysis of Dynamic Average Consensus}\label{sec:supdac}
In this section, we show a brief analysis demonstrating why dynamic consensus averaging achieves smaller approximation errors than simple weighting. 

Assume that $y^{t-1}_i=\frac{1}{J}\sum_j a_j(x_j^{t-1})$ for each $j$. Then, the error at iteration \(t\) is:  
\[
y_j^t - \frac{1}{J} \sum_i a_i(x_i^t) = \sum_i w_{ij} \left[\frac{1}{J} \sum_i a_i(x_i^{t-1}) - \frac{1}{J} \sum_i a_i(x_i^t) + a_i(x_i^t) - a_i(x_i^{t-1})\right].
\]  
Rewriting this using the derivative \(a_i'(x_i^{t-1})\), we get
\[
y_j^t - \frac{1}{J} \sum_i a_i(x_i^t) = \sum_i w_{ij} \left[\frac{1}{J} \sum_i a_i'(x_i^{t-1}) - a_i'(x_i^{t-1})\right] (x_i^t - x_i^{t-1}).
\]  
The error depends on the product of \(\left| \frac{1}{J} \sum_i a_i'(x_i^{t-1}) - a_i'(x_i^{t-1}) \right|\),  the discrepancy between the mean derivative and individual derivatives, and \(\left| x_i^t - x_i^{t-1} \right|\), the magnitude of the change in \(x_i^t\) between iterations. Since \(\left| x_i^t - x_i^{t-1} \right| \to  0\) as \(t \to \infty\), the approximation error after a single iteration becomes negligible for large \(t\).
In contrast, the error incurred by simple weighting is 
\[
\left| \sum_i w_{ij} a_i(x_i^t) - \frac{1}{J} \sum_i a_i(x_i^t) \right|\text{,}
\]
which is generally bounded below by a positive constant. This demonstrates that dynamic consensus averaging achieves better error reduction, especially as $t$ increases.

\section{Extension to the Multi-resolution Case}\label{sec:extension}

In this section, we explore the possibility of extending our method to the multi-resolution setting.

We begin by briefly reviewing the multi-resolution low-rank model.

Let $\{ \mathcal{D}_{r_1 \ldots r_m}, r_1 = 1, \ldots, R, \ldots, r_m = 1, \ldots, R, m = 1, \ldots, M - 1 \}$ represent recursive partitions of the spatial domain. These partitions satisfy the condition $\bigcup_{r_{m + 1} = 1}^R \mathcal{D}_{r_1 \ldots r_{m + 1}} = \mathcal{D}_{r_1 \ldots r_m}$, meaning that each finer partition $\mathcal{D}_{r_1 \ldots r_{m + 1}}$ is contained within the coarser partition $\mathcal{D}_{r_1 \ldots r_m}$, thus creating a multi-resolution structure. The indices $r_1, \ldots, r_m$ represent the different levels of the partition, and $m$ denotes the depth of the partitioning process.

The multi-resolution low-rank model, as introduced by \cite{katzfuss2017multi}, is given by:
\[
z(\boldsymbol{s}) = \boldsymbol{b}^{\top}(\boldsymbol{s}) \boldsymbol{\eta} + \boldsymbol{b}_{r_1}^{\top}(\boldsymbol{s}) \boldsymbol{\eta}_{r_1} + \cdots + \boldsymbol{b}_{r_1 \ldots r_{M-1}}^{\top}(\boldsymbol{s}) \boldsymbol{\eta}_{r_1 \ldots r_{M-1}} + \varepsilon(\boldsymbol{s}),
\]
where $\boldsymbol{s} \in \mathcal{D}_{r_1 \ldots r_{M-1}}$. In this model, $\boldsymbol{b}(\boldsymbol{s})$ represents basis functions defined over the spatial domain, and $\{ \boldsymbol{\eta}_{r_1 \ldots r_m} \}$ are low-rank random vectors associated with each level of the partition. $\varepsilon(\boldsymbol{s})$ is a white noise process that captures random fluctuations in the data. The model expresses the data $z(\boldsymbol{s})$ as a sum of contributions from multiple resolution levels, with each resolution corresponding to a finer partition of the domain.

The distributed data are represented as $\{ z(\boldsymbol{s}_{ij}), i = 1, \ldots, n_j, j = 1, \ldots, J \}$, where $j$ denotes the machine index and $i$ is the index of the data points on the $j$-th machine. The local data for the $j$-th machine are stored in the vector $\boldsymbol{z}_j = (z(\boldsymbol{s}_{1j}), \ldots, z(\boldsymbol{s}_{n_j j}))^{\top}$, which contains the data for all $n_j$ data points on machine $j$.

Next, we demonstrate that the evidence lower bound (ELBO) can be expressed as a summation of local and common functions. This formulation allows us to extend our method to the multi-resolution setting.

Define $\mathcal{E}_0 := \{ \boldsymbol{\eta} \}$ and $\mathcal{E}_m := \mathcal{E}_{m - 1} \cap \{ \boldsymbol{\eta}_{r_1 \ldots r_m}, r_1 = 1, \ldots, R, \ldots, r_m = 1, \ldots, R \}$. Let $q$ be a density function. The ELBO is defined as:
\[
\text{ELBO}(q) = \mathbb{E}_{q(\mathcal{E}_{M - 1})} \log p(\boldsymbol{z} | \mathcal{E}_{M - 1}) - \text{KL}(q(\mathcal{E}_{M - 1}) \| p(\mathcal{E}_{M - 1})),
\]
where $p(\mathcal{E}_{M - 1})$ is the prior distribution of $\mathcal{E}_{M - 1}$. By conditional independence, we have:
\[
p(\boldsymbol{z} | \mathcal{E}_{M - 1}) = \prod_{ij} p(z(\boldsymbol{s}_{ij}) | \mathcal{E}_{M - 1}),
\]
which implies that:
\[
\mathbb{E}_{q(\mathcal{E}_{M - 1})} \log p(\boldsymbol{z} | \mathcal{E}_{M - 1}) = \sum_k \mathbb{E}_{q(\mathcal{E}_{M - 1})} \log p(\boldsymbol{z}_k | \mathcal{E}_{M - 1}).
\]
Therefore, the ELBO can be written as:
\[
\text{ELBO}(q) = \sum_j \mathbb{E}_{q(\mathcal{E}_{M - 1})} \log p(\boldsymbol{z}_j | \mathcal{E}_{M - 1}) - \text{KL}(q(\mathcal{E}_{M - 1}) \| p(\mathcal{E}_{M - 1})).
\]
This shows that the ELBO can be expressed as the sum of local and common functions, which enables the extension of our method to the multi-resolution case.
{
\section{Estimation of $\boldsymbol{V}_{\boldsymbol{\gamma}},v_{\delta},\boldsymbol{V}_{\boldsymbol{\theta}}$ and Confidence Interval }\label{sec:asy_variance}
In this section, we discuss how to estimate  the terms  $\boldsymbol{V}_{\boldsymbol{\gamma}},v_{\delta},\boldsymbol{V}_{\boldsymbol{\theta}}$ in a decentralized manner.
Recall, from the theory, 
\[
\boldsymbol{V}_{\boldsymbol{\gamma}} :=\frac{1}{N} \mathbb{E}\{
\boldsymbol{X}^{\top} \boldsymbol{C}^{- 1} (\boldsymbol{\theta}^{\ast},
\delta^{\ast}) \boldsymbol{X}\}, \quad 
v_{\delta} := \frac{1}{2 N} (\delta^{\ast})^{- 4} \text{tr} [\boldsymbol{C}^{- 2} (\boldsymbol{\theta}^{\ast}, \delta^{\ast})],
\]
\[
[\boldsymbol{V}_{\boldsymbol{\theta}}]_{k l} := \frac{1}{2 m} \text{tr} \left( \boldsymbol{C}^{- 1} (\boldsymbol{\theta}^{\ast},
\delta^{\ast}) \frac{\partial \boldsymbol{C}(\boldsymbol{\theta}^{\ast}, \delta^{\ast})}{\partial \boldsymbol{\theta}_k} 
\boldsymbol{C}^{- 1} (\boldsymbol{\theta}^{\ast}, \delta^{\ast}) \frac{\partial \boldsymbol{C}(\boldsymbol{\theta}^{\ast}, \delta^{\ast})}{\partial \boldsymbol{\theta}_l} \right).
\]

By the Woodbury formula, we have
$
\boldsymbol{C}^{- 1} (\boldsymbol{\theta}, \delta) = \delta \boldsymbol{I} - \delta^2 \boldsymbol{B} (\boldsymbol{\theta}) 
[\boldsymbol{K}^{- 1} (\boldsymbol{\theta}) + \delta \boldsymbol{B}^{\top} (\boldsymbol{\theta}) \boldsymbol{B} (\boldsymbol{\theta})]^{-1} \boldsymbol{B}^{\top} (\boldsymbol{\theta}),
$
or equivalently,
$\boldsymbol{C}^{- 1} (\boldsymbol{\theta}, \delta) = \delta \boldsymbol{I} - \delta^2 \boldsymbol{B} (\boldsymbol{\theta}) \boldsymbol{A}^{-1} (\boldsymbol{\theta}) \boldsymbol{B}^{\top} (\boldsymbol{\theta}),$
with 
$\boldsymbol{A} (\boldsymbol{\theta}) = \boldsymbol{K}^{-1} (\boldsymbol{\theta}) + \delta \sum_j \boldsymbol{B}^{\top}_j (\boldsymbol{\theta}) \boldsymbol{B}_j (\boldsymbol{\theta}).$ Thus,
\[
\boldsymbol{X}^{\top} \boldsymbol{C}^{- 1} (\boldsymbol{\theta}^{\ast}, \delta^{\ast}) \boldsymbol{X} 
= \delta \sum_j \boldsymbol{X}_j^{\top} \boldsymbol{X}_j - \delta^2 \left[ \sum_j \boldsymbol{X}_j^{\top} \boldsymbol{B}_j (\boldsymbol{\theta}^{\ast}) \right] 
\boldsymbol{A}^{-1} (\boldsymbol{\theta}^{\ast}) \left[ \sum_j \boldsymbol{B}_j^{\top} (\boldsymbol{\theta}^{\ast}) \boldsymbol{X}_j \right].
\]
For the trace term, we have
\begin{align*}
\text{tr} [\boldsymbol{C}^{- 2} (\boldsymbol{\theta}^{\ast}, \delta^{\ast})] 
&= \delta^2 \text{tr} (\boldsymbol{I}) 
- 2 \delta^3 \text{tr} \left\{ \boldsymbol{A}^{-1} (\boldsymbol{\theta}^{\ast}) \left[ \sum_j \boldsymbol{B}_j^{\top} (\boldsymbol{\theta}^{\ast}) \boldsymbol{B}_j (\boldsymbol{\theta}^{\ast}) \right] \right\} \\
&\quad + \delta^4 \text{tr} \left\{ \boldsymbol{A}^{-1} (\boldsymbol{\theta}^{\ast}) \left[ \sum_j \boldsymbol{B}_j^{\top} (\boldsymbol{\theta}^{\ast}) \boldsymbol{B}_j (\boldsymbol{\theta}^{\ast}) \right] 
\boldsymbol{A}^{-1} (\boldsymbol{\theta}^{\ast}) \left[ \sum_j \boldsymbol{B}_j^{\top} (\boldsymbol{\theta}^{\ast}) \boldsymbol{B}_j (\boldsymbol{\theta}^{\ast}) \right] \right\}.
\end{align*}
These expressions show that $\boldsymbol{V}_{\boldsymbol{\gamma}}$ and $v_{\delta}$ depend on terms that are summations of local terms based on local data. Therefore, they can be estimated in a decentralized manner with each summation being approximate using decentralized average. The same approach applies to each $[\boldsymbol{V}_{\boldsymbol{\theta}}]_{k l}$.

Note that the expression for $[\boldsymbol{V}_{\boldsymbol{\theta}}]_{k l}$ is highly complex, making the coding implementation challenging. However, when the connection network is centralized, $[\boldsymbol{V}_{\boldsymbol{\theta}}]_{k l}$ can be estimated using an alternative easier approach, which is detailed as follows. First, notice that
\[
(\boldsymbol{z} - \boldsymbol{X} \boldsymbol{\gamma})^{\top} \boldsymbol{C}^{-1} (\boldsymbol{\theta}, \delta) (\boldsymbol{z} - \boldsymbol{X} \boldsymbol{\gamma}) 
= (\boldsymbol{z} - \boldsymbol{X} \boldsymbol{\gamma})^{\top} (\boldsymbol{z} - \boldsymbol{X} \boldsymbol{\gamma}) 
- \delta^2 (\boldsymbol{z} - \boldsymbol{X} \boldsymbol{\gamma})^{\top} \boldsymbol{B} (\boldsymbol{\theta}) \boldsymbol{A}^{-1} (\boldsymbol{\theta}) \boldsymbol{B}^{\top} (\boldsymbol{\theta}) (\boldsymbol{z} - \boldsymbol{X} \boldsymbol{\gamma}),
\]
where
$(\boldsymbol{z} - \boldsymbol{X} \boldsymbol{\gamma})^{\top} (\boldsymbol{z} - \boldsymbol{X} \boldsymbol{\gamma}) = \sum_j (\boldsymbol{z}_j - \boldsymbol{X}_j \boldsymbol{\gamma})^{\top} (\boldsymbol{z}_j - \boldsymbol{X}_j \boldsymbol{\gamma}),$
$(\boldsymbol{z} - \boldsymbol{X} \boldsymbol{\gamma})^{\top} \boldsymbol{B} (\boldsymbol{\theta}) = \sum_j (\boldsymbol{z}_j - \boldsymbol{X}_j \boldsymbol{\gamma})^{\top} \boldsymbol{B}_j (\boldsymbol{\theta})$ are summations of local terms based on local data.
Second, by the determinant identity, we have
\[
\log | \boldsymbol{C} (\boldsymbol{\theta}, \delta) | = - N \log \delta + \log \det[\boldsymbol{K} (\boldsymbol{\theta})]+ \log \det [\boldsymbol{K}^{-1} (\boldsymbol{\theta}) +  \delta \boldsymbol{B}^{\top} (\boldsymbol{\theta}) \boldsymbol{B} (\boldsymbol{\theta})],
\]
with
$\boldsymbol{B}^{\top} (\boldsymbol{\theta}) \boldsymbol{B} (\boldsymbol{\theta}) = \sum_j \boldsymbol{B}_j^{\top} (\boldsymbol{\theta}) \boldsymbol{B}_j (\boldsymbol{\theta})$ being  summations of local terms based on local data.
Thus, both the log-likelihood and its Hessian can be obtained in a distributed manner. Finally, we estimate $\boldsymbol{V}_{\boldsymbol{\theta}}$ by the Hessian of the log-likelihood, which can be computed using automatic differentiation in a distributed manner. Since the expression for the log-likelihood is significantly simpler than that of $[\boldsymbol{V}_{\boldsymbol{\theta}}]_{k l}$ and does not directly involve differentiation, its coding implementation becomes considerably easier.

 The confidence intervals for each component of \(\boldsymbol{\gamma}^\ast\) and \(\boldsymbol{\theta}^\ast\), as well as for \(\delta^\ast\), can be constructed in each machine after estimating $\boldsymbol{V}_{\boldsymbol{\gamma}},v_{\delta},\boldsymbol{V}_{\boldsymbol{\theta}}$ .  Let  \(\widehat{\boldsymbol{V}}_{\boldsymbol{\gamma},j}\), \(\widehat{v}_{\delta,j}\), and \(\widehat{\boldsymbol{V}}_{\boldsymbol{\theta},j}\)  denote their respective estimates in machine $j$,  and let  $\widehat{\boldsymbol{\gamma}}_{j},\hat{\delta}_j,\widehat{\boldsymbol{\theta}}_{j}$  be the parameter estimates of  \(\boldsymbol{\gamma}^\ast\), \(\delta^\ast\),  \(\boldsymbol{\theta}^\ast\), respectively,  in  machine $j$ . Since
  \begin{equation*}
      \sqrt{N} \boldsymbol{V}_{\boldsymbol{} \boldsymbol{\gamma}}^{\frac{1}{2}}
      (\widehat{\boldsymbol{\gamma}} \boldsymbol{}  - \boldsymbol{\gamma}^{\ast})\rightsquigarrow  \mathcal{N}_{p } (\boldsymbol{0}_{p },
    \boldsymbol{I}_{p \times p}), \sqrt{N} v_{\delta}^{\frac{1}{2}} (\hat{\delta} - \delta^{\ast})\rightsquigarrow  \mathcal{N} (0,1), \sqrt{m}\boldsymbol{V}_{\boldsymbol{\theta}}^{\frac{1}{2}} (\widehat{\boldsymbol{\theta}} -
     \boldsymbol{\theta}^{\ast}) \rightsquigarrow \mathcal{N}_q (\boldsymbol{0}_q,
     \boldsymbol{I}_{q \times q}),
  \end{equation*}
 the confidence intervals are given as:  
\begin{align}
    &\widehat{\gamma}_{ij} \pm z_{\alpha/2} \frac{1}{\sqrt{N}} \nonumber
    \left(\widehat{\boldsymbol{V}}_{\boldsymbol{\gamma},j}\right)_{ii}^{1/2}, \quad i = 1, \ldots, p, \\
    &\hat{\delta}_j \pm z_{\alpha/2} \frac{1}{\sqrt{N}} \widehat{v}_{\delta,j}^{1/2},\nonumber \\
    &\widehat{\theta}_{ij} \pm z_{\alpha/2} \frac{1}{\sqrt{m}} \left(\widehat{\boldsymbol{V}}_{\boldsymbol{\theta},j}\right)_{ii}^{1/2}, \quad i = 1, \ldots, q.\nonumber
\end{align}
Here, \(z_{\alpha/2}\) is the critical value from the standard normal distribution for a given confidence level \(1-\alpha\),  $\widehat{\gamma}_{ij}$ and $\widehat{\theta}_{ij}$ are  $i$-the component $\widehat{\boldsymbol{\gamma}}_{j},\widehat{\boldsymbol{\theta}}_{j}$, respectively.}

\section{Additional Simulation Results}
This section presents additional simulation results that could not be included earlier due to space constraints.
\begin{figure}[htbp]
  \centering
  \begin{center}
    \includegraphics[width=5.3cm]{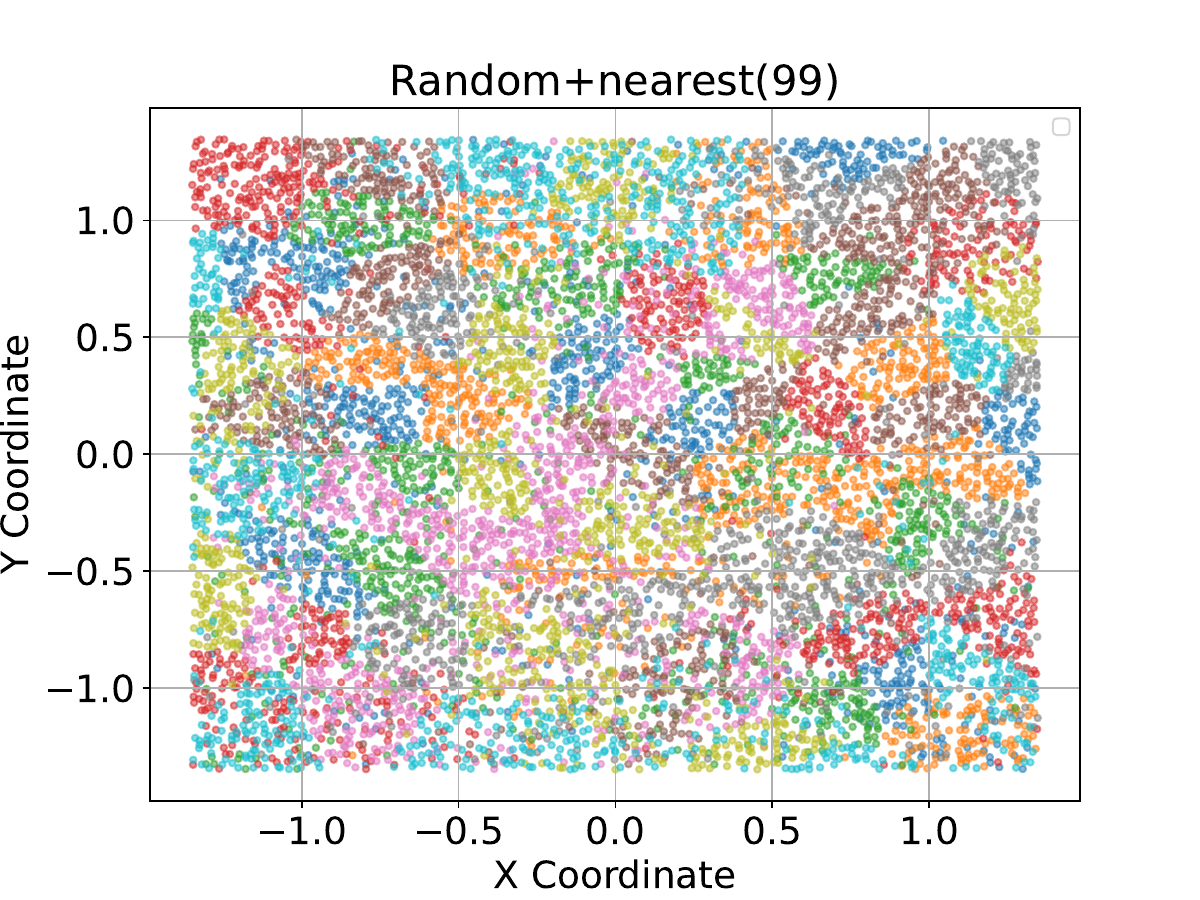}
    \includegraphics[width=5.3cm]{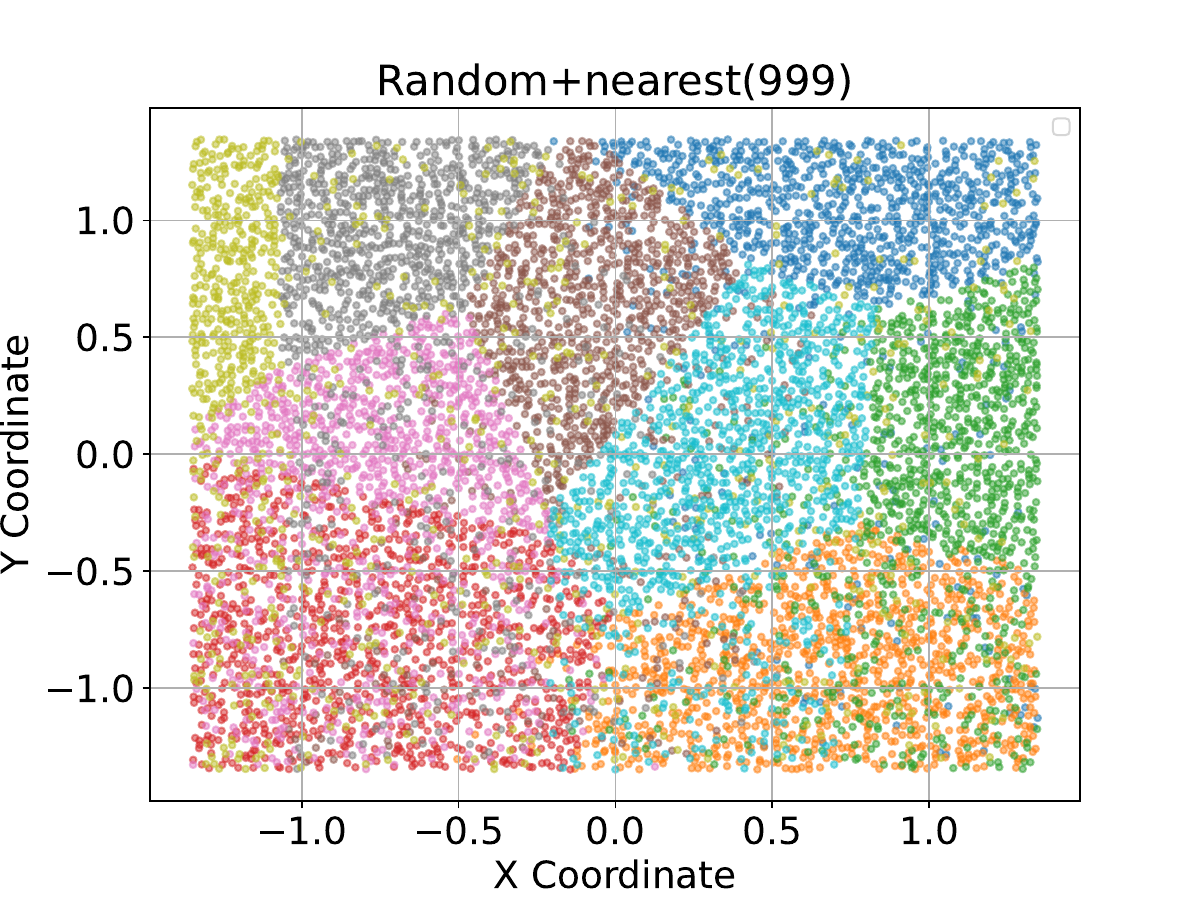}   
      \includegraphics[width=5.3cm]{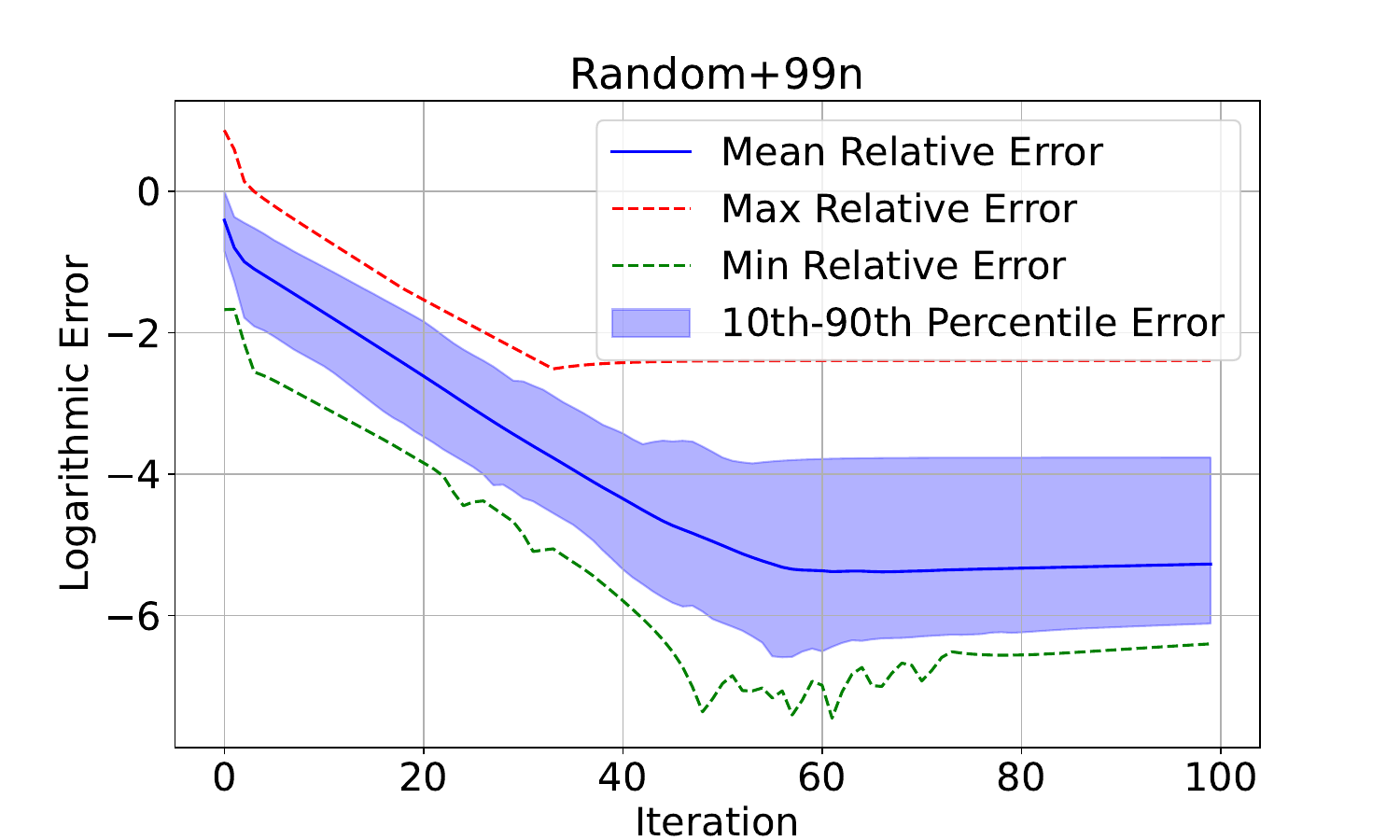}
    \includegraphics[width=5.3cm]{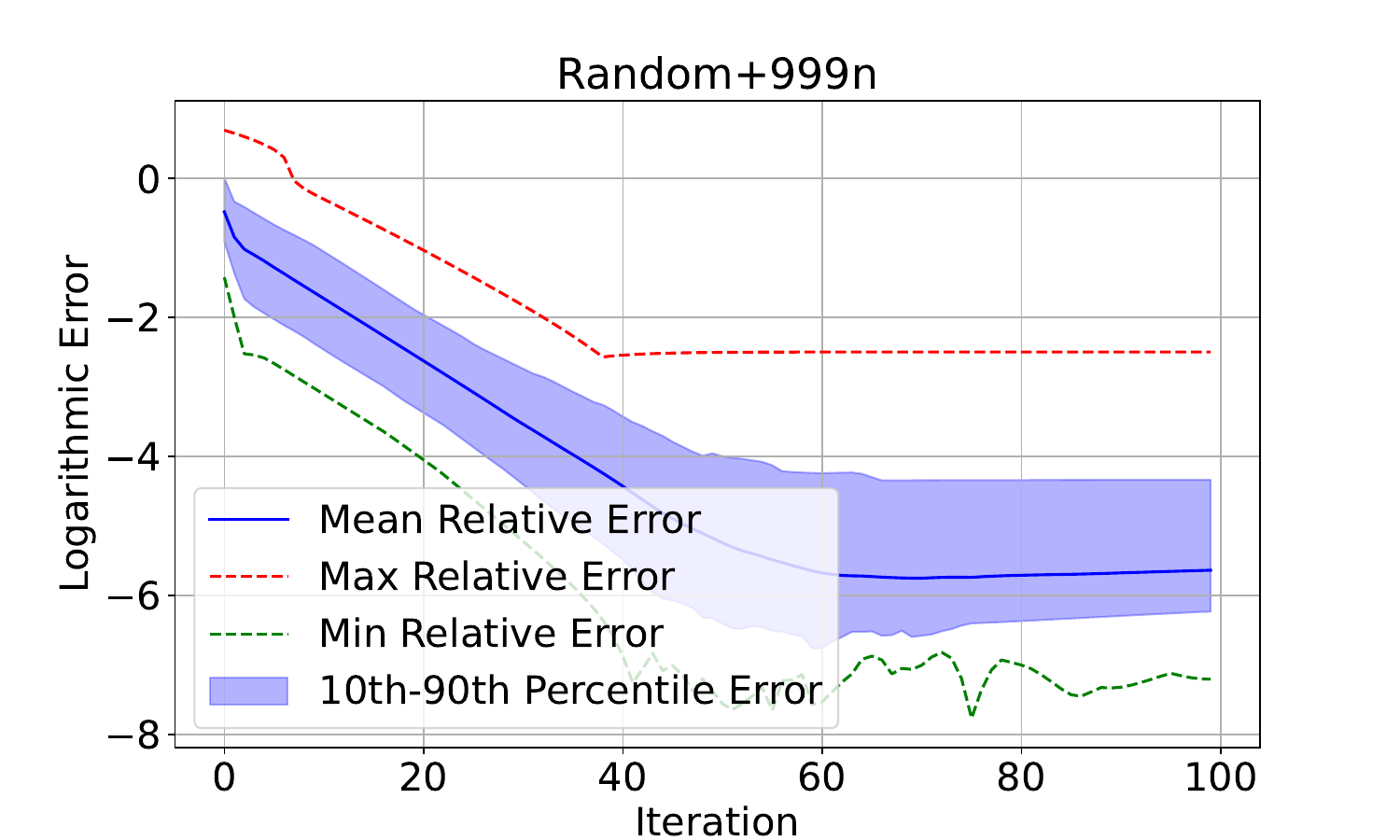} 
\end{center}
  \caption{Partitioned locations under different data partitioning schemes and Convergence of the method under different partitioning schemes.}
  \label{fig:partition}
\end{figure}

\begin{figure}[H]
  \centering
  \begin{subfigure}{0.44\textwidth}
      \centering
       \includegraphics[width=\textwidth]{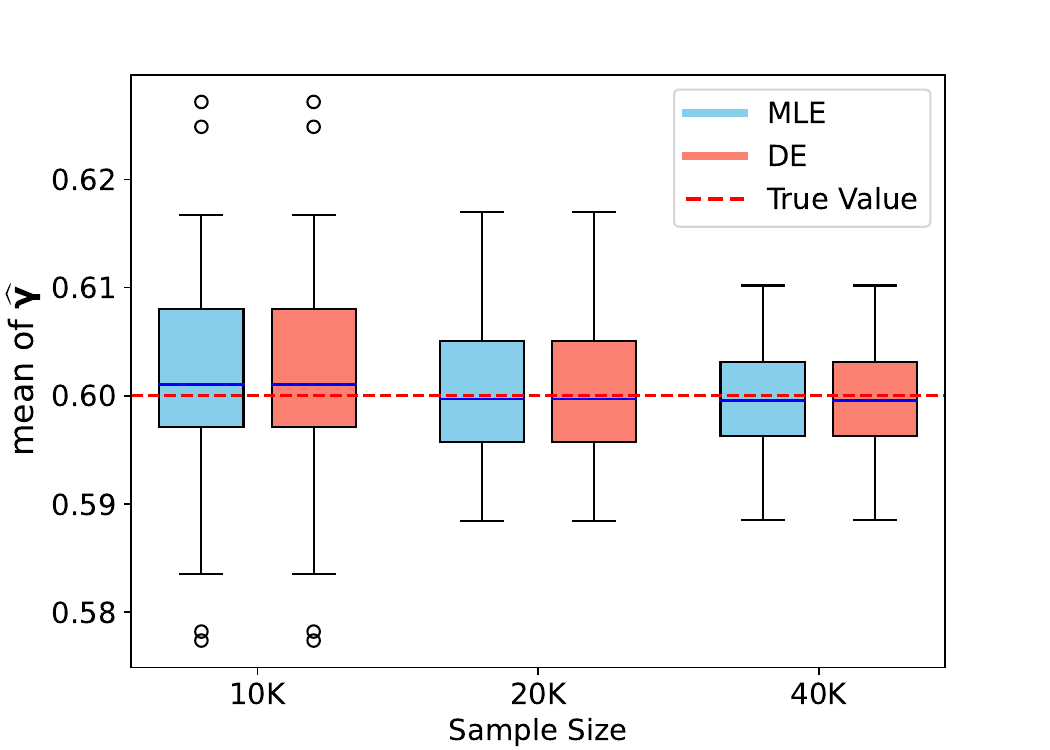}
      \includegraphics[width=\textwidth]{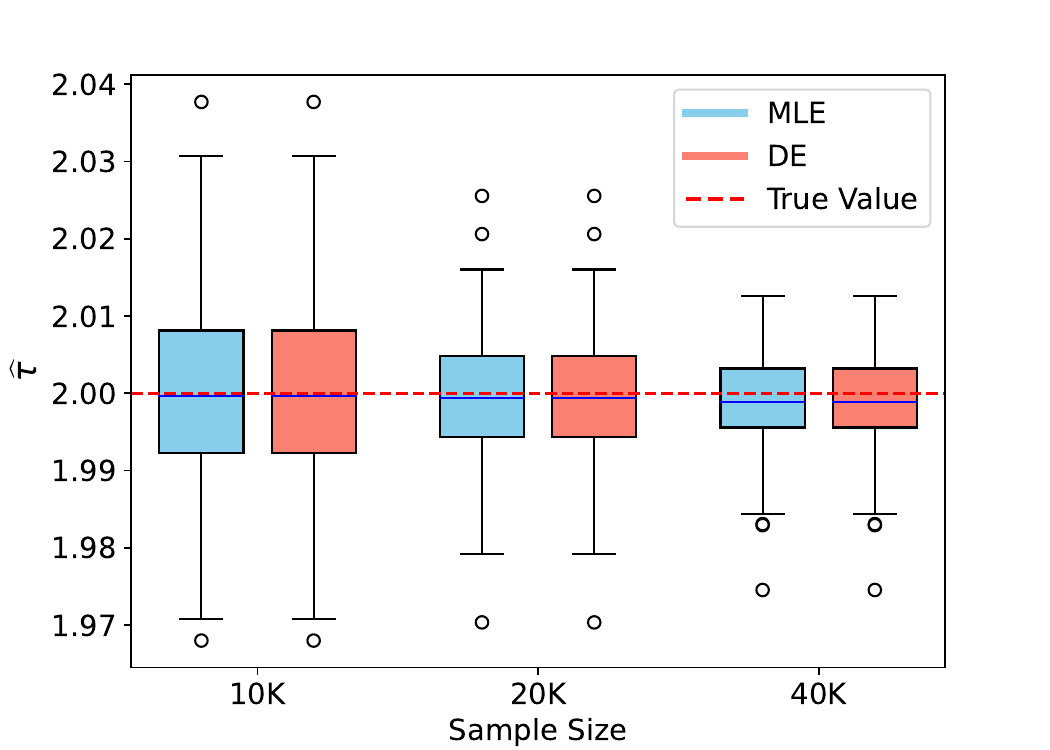}
      \caption{}
      \label{fig:increasing_sample_size}
  \end{subfigure}
  \hspace{-0.2cm}
  \begin{subfigure}{0.44\textwidth}
      \centering
      \includegraphics[width=\textwidth]{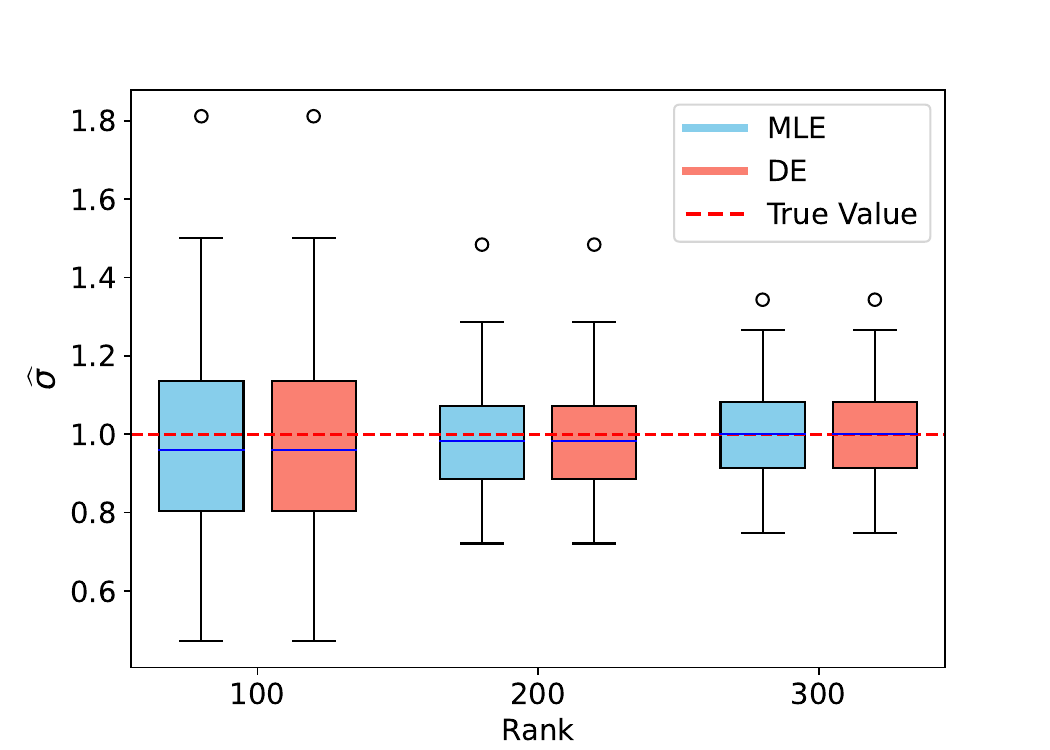}
      \includegraphics[width=\textwidth]{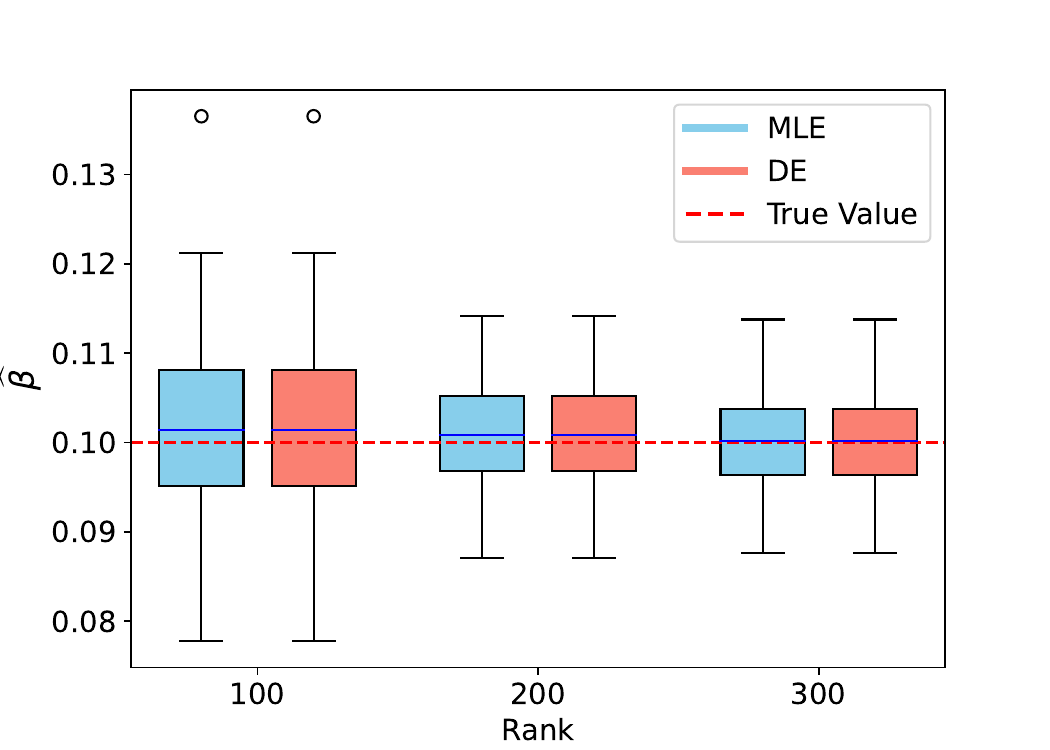}
      \caption{}
      \label{fig:increasing_rank}
  \end{subfigure}
  \caption{(a) Effect of increasing sample size on estimation accuracy of linear coefficients and the nugget parameter (left panel). (b) Effect of increasing rank on estimation accuracy of Matérn covariance parameters (right panel).}
  \label{fig:merged_results_accuracy}
\end{figure}

\begin{figure}[htbp]
  \centering
  \begin{center}
    \includegraphics[width=16cm]{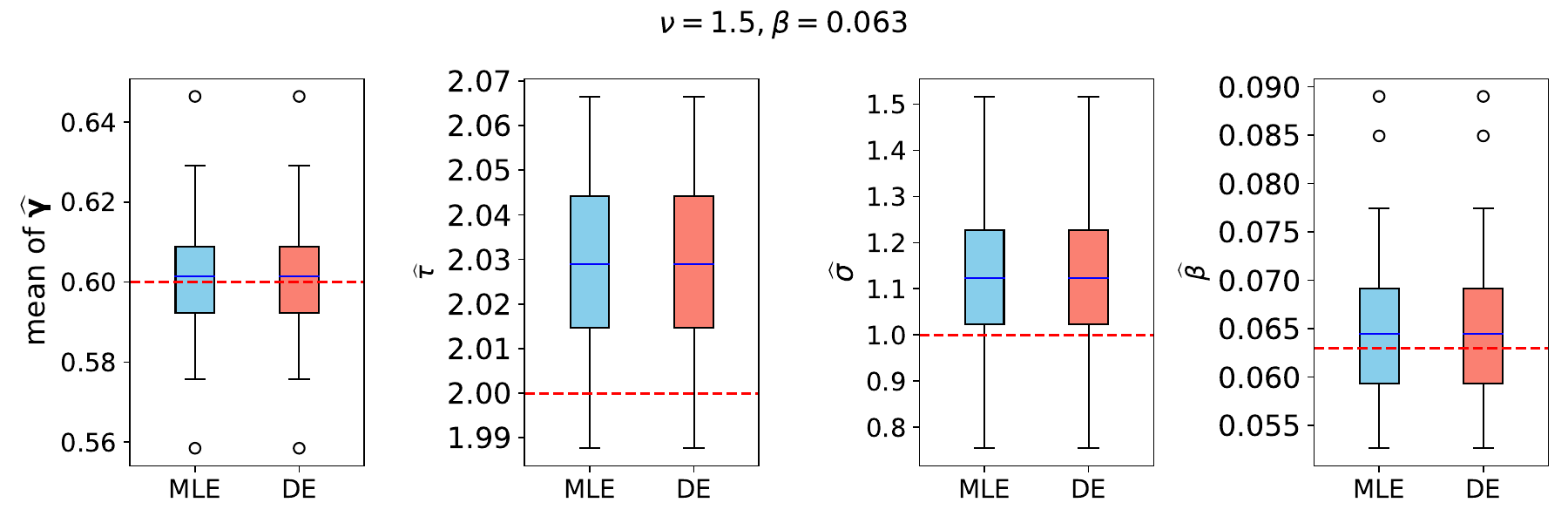}
    \includegraphics[width=16cm]{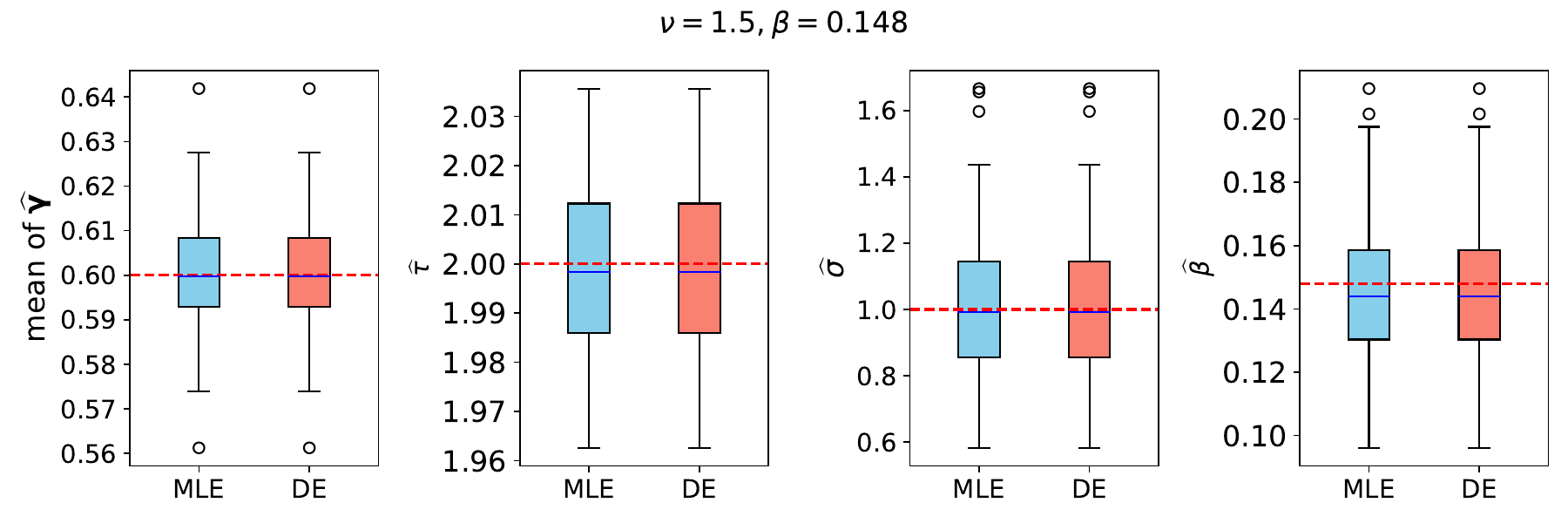}
    \includegraphics[width=16cm]{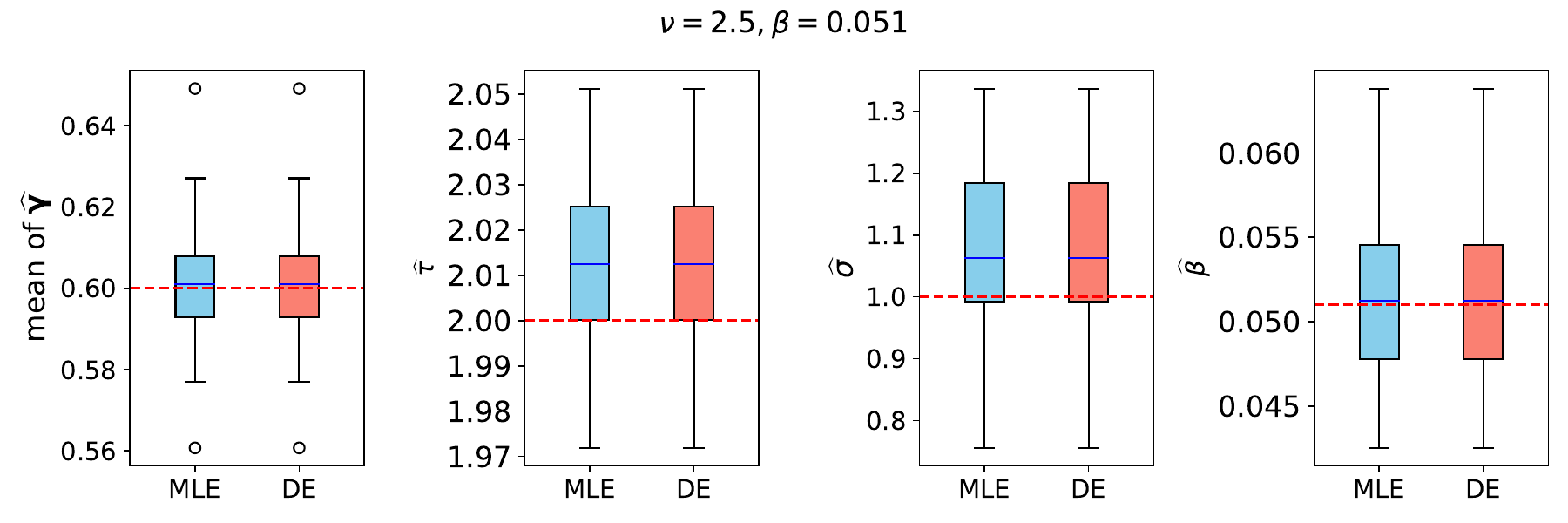}
    \includegraphics[width=16cm]{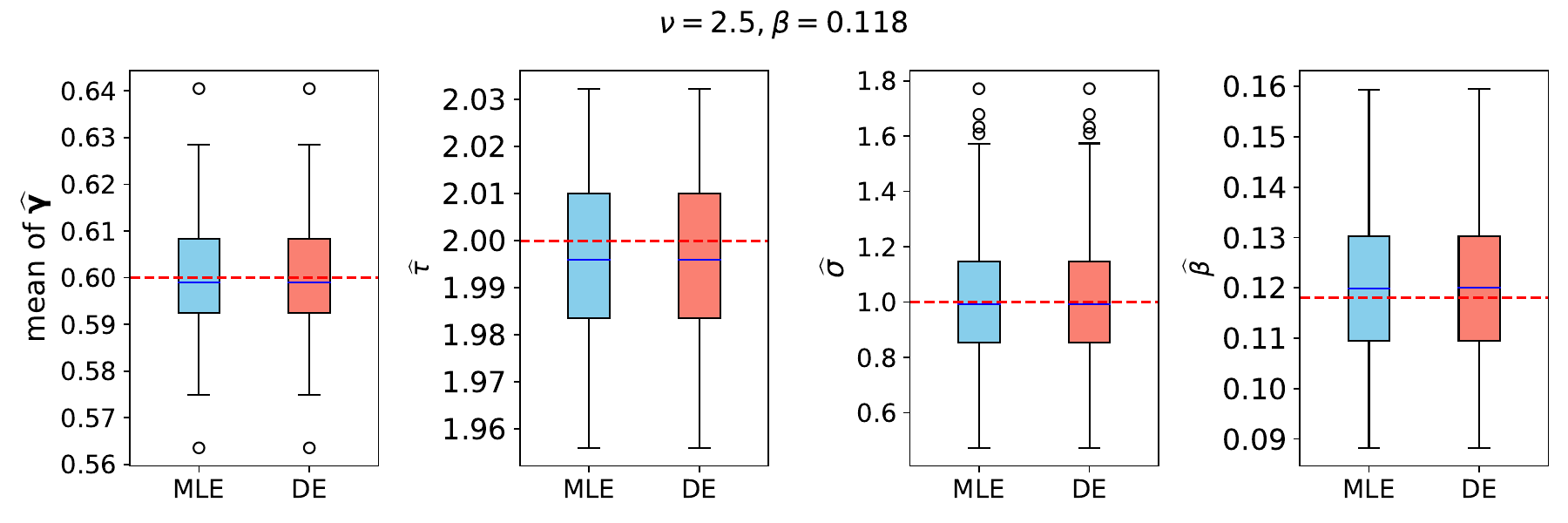}
  \end{center}
  \caption{Estimated value under model misspecification.}
  \label{fig:misspecification}
\end{figure}

\begin{table}[h]
 \caption{{Summary of empirical standard deviations, average of estimated standard deviations, and empirical coverage probabilities based on 100 replications. Let $\boldsymbol{\gamma}=(\gamma_1,\ldots,\gamma_p)^\top$.  The ``\textbf{Empirical Std}" deviation for each parameter ($\gamma_i$,  $\delta$,  $\sigma$ or $\beta$) is the standard deviation of the estimates across the 100 replications. The ``\textbf{Estimated Std}"  is computed as the mean of the standard deviations estimated in each replication using the method described in Section \ref{sec:asy_variance}. The ``\textbf{Coverage Prob}" is the proportion of replications where the true parameter lies within the corresponding confidence interval.  Since the values with the given precision are consistent across different machines, we present results from a single machine for simplicity.}}
    \centering
    \begin{tabular}{ccccccccc}
        \toprule
        & \multicolumn{5}{c}{$\boldsymbol{\gamma}$} &$\delta$ (i.e., $\tau^{-2}$)&$\sigma$ &$\beta$ \\
        \midrule
        \textbf{Empirical Std} & 0.0200 & 0.0205 & 0.0222 & 0.0194 & 0.0222 & 0.0032 & 0.2428 & 0.0113 \\
        \textbf{Estimated Std} & 0.0201 & 0.0201 & 0.0201 & 0.0201 & 0.0201 & 0.0036 & 0.2447 & 0.0113\\
        \textbf{Coverage Prob} & 0.95& 0.97 & 0.90 & 0.94 & 0.92 & 0.95 &0.97& 0.93 \\
        \bottomrule
    \end{tabular}
   
    \label{tab:CI_simu}
\end{table}

\newpage
\section{Summary of the Proposed Algorithm \label{sec:algorithm} }
In this section, we present a concise summary of the proposed algorithm, as detailed in Algorithm \ref{al:dbcd}.

\begin{algorithm}[H]
\small
\setstretch{0.5}
\caption{Decentralized Block Coordinate Descent for Spatial Low-rank Model \label{al:dbcd}}
\begin{algorithmic}[1]
\State \textbf{Input:} Initial parameters $\boldsymbol{\mu}_j^t$, $\boldsymbol{\Sigma}_j^t$, $\boldsymbol{\gamma}_j^t$, $\delta_j^t$, $\boldsymbol{\theta}_j^t$ for each machine $j \in \{1, \ldots, J\}$.
\State \textbf{Output:} Updated parameters $\boldsymbol{\mu}_j^{t+1}$, $\boldsymbol{\Sigma}_j^{t+1}$, $\boldsymbol{\gamma}_j^{t+1}$, $\delta_j^{t+1}$, $\boldsymbol{\theta}_j^{t+1}$.

\For{each machine $j \in \{1, \ldots, J\}$}

    \State \textbf{Step 1: Update} $\boldsymbol{\mu}_j^{t+1}$, $\boldsymbol{\Sigma}_j^{t+1}$:
 \begin{equation*}
\boldsymbol{\mu}_j^{t+1}=\left[\delta^t_jJ\boldsymbol{Y}_{\boldsymbol{\Sigma}, j}^t+\sum_i w_{ij}\boldsymbol{K}^{-1}\left(\boldsymbol{\theta}^t_i\right)\right]^{-1} J\delta^t_j\boldsymbol{y}_{\boldsymbol{\mu}, j}^t,
\boldsymbol{\Sigma}_j^{t+1}=\left[\delta^t_jJ\boldsymbol{Y}_{\boldsymbol{\Sigma}, j}^t+\boldsymbol{K}^{-1}\left(\boldsymbol{\theta}_j^t\right)\right]^{-1}.
\end{equation*}
    \State Here, $\boldsymbol{Y}_{\boldsymbol{\Sigma}, j}^t$ and $\boldsymbol{y}_{\boldsymbol{\mu}, j}^t$ are dynamic approximations of unavailable averages:
  \begin{equation*}
  \begin{split}
    \boldsymbol{Y}_{\boldsymbol{\Sigma}, j}^{t}&=\sum_i w^{[K]}_{i j}\left(\boldsymbol{Y}_{\boldsymbol{\Sigma}, i}^{t-1}+\boldsymbol{B}_i\left(\boldsymbol{\theta}_i^{t}\right)^{\top} \boldsymbol{B}_i\left(\boldsymbol{\theta}_i^{t}\right)-\boldsymbol{B}_i\left(\boldsymbol{\theta}_i^{t-1}\right)^{\top} \boldsymbol{B}_i\left(\boldsymbol{\theta}_i^{t-1}\right)\right),\\
     \boldsymbol{y}_{\boldsymbol{\mu}, j}^{t}&=\sum_i w^{[K]}_{i j}\left(\boldsymbol{y}_{\boldsymbol{\mu}, i}^{t-1}+\boldsymbol{B}_i\left(\boldsymbol{\theta}_i^{t}\right)^{\top}\left(\boldsymbol{z}_i-\boldsymbol{X}_i \boldsymbol{\gamma}_i^{t}\right)-\boldsymbol{B}_i\left(\boldsymbol{\theta}_i^{{t-1}}\right)^{\top}\left(\boldsymbol{z}_i-\boldsymbol{X}_i \boldsymbol{\gamma}_i^{t-1}\right)\right).
  \end{split}
\end{equation*}

    \State \textbf{Step 2: Update} $\boldsymbol{\gamma}_j^{t+1}$:
    \[
    \boldsymbol{\gamma}_j^{t+1} = \left(\boldsymbol{Y}_{\boldsymbol{X}, j}^t\right)^{-1} \boldsymbol{y}_{\boldsymbol{\gamma}, j}^t.
    \]
    \State Here, $\boldsymbol{Y}_{\boldsymbol{X}, j}^t$ and $\boldsymbol{y}_{\boldsymbol{\gamma}, j}^t$ are (dynamic) approximations of unavailable averages:
    \[
    \boldsymbol{Y}_{\boldsymbol{X}, j}^t = \sum_i w_{ij}^{[K]} \boldsymbol{Y}_{\boldsymbol{X}, i}^{t-1}, 
    \]
    \[
    \boldsymbol{y}_{\boldsymbol{\gamma}, j}^t = \sum_i w_{ij}\left(\boldsymbol{y}_{\boldsymbol{\gamma}, i}^t - \boldsymbol{X}_i^{\top} \boldsymbol{B}_i(\boldsymbol{\theta}_i^t) \boldsymbol{\mu}_i^{t+1} + \boldsymbol{X}_i^{\top} \boldsymbol{B}_i(\boldsymbol{\theta}_i^{t-1}) \boldsymbol{\mu}_i^t\right).
    \]

    \State \textbf{Step 3: Update} $\delta_j^{t+1}$:
    \[
    \delta_j^{t+1} = y_{n,j}^t\left(y_{\delta, j}^t\right)^{-1}.
    \]
    \State Here, $y_{\delta, j}^t$ and $y_{n,j}^t$ are (dynamic) approximations of unavailable averages:
    \[
    y_{\delta, j}^t = \sum_i w^{[K]}_{ij}\left(y_{\delta, i}^{t-1} + l_i\left(\boldsymbol{\mu}_i^{t+1}, \boldsymbol{\Sigma}_i^{t+1}, \boldsymbol{\gamma}_i^{t+1}, \boldsymbol{\theta}_i^t\right) - l_i\left(\boldsymbol{\mu}_i^t, \boldsymbol{\Sigma}_i^t, \boldsymbol{\gamma}_i^t, \boldsymbol{\theta}_i^{t-1}\right)\right),
    \]
    \[
    y_{n,j}^t = \sum_i w_{ij}^{[K]} y_{n,j}^{t-1}.
    \]

    \State \textbf{Step 4: Update} $\boldsymbol{\theta}_j^{t+1}$
    \[
    \boldsymbol{\theta}_j^{t+1} = \boldsymbol{\theta}_j^{t, S},
    \]
    where, for $ s=1, \ldots, S$,
\begin{equation*}
    \boldsymbol{\theta}_j^{t, s}=\boldsymbol{\theta}_j^{t, s-1}-\alpha_{t,s}\left\{m d\left[\boldsymbol{Y}_{\boldsymbol{H}_f, j}^{t, s}+\frac{1}{J}\sum w_{ij}^{[K]} \boldsymbol{H}_{h,i}^t\left(\boldsymbol{\theta}_j^{t, s}\right)\right]\right\}^{-1}\left[\boldsymbol{y}_{\boldsymbol{G}_f, j}^{t, s}+\frac{1}{J}\sum w_{ij}^{[K]}\boldsymbol{G}_{h,i}^t\left(\boldsymbol{\theta}_j^{t, s}\right)\right].
\end{equation*}
    Here, $\boldsymbol{Y}_{\boldsymbol{H}_f, j}^{t, s}$ and $\boldsymbol{y}_{\boldsymbol{G}_f, j}^{t, s}$  are (dynamic) approximations of unavailable averages:
    \begin{equation*}
\begin{aligned}
    \boldsymbol{Y}_{\boldsymbol{H}_f, j}^{t, s}=\sum_i w^{[K]}_{i j}\left(\boldsymbol{Y}_{\boldsymbol{H}_f, i}^{t,s-1}+\boldsymbol{H}_{f_i,i}^t\left(\boldsymbol{\theta}_i^{t, s}\right)-\boldsymbol{H}_{f_i,i}^t\left(\boldsymbol{\theta}_i^{t, s-1}\right)\right), \\
\boldsymbol{y}_{\boldsymbol{G}_f, j}^{t, s}=\sum_i w^{[K]}_{i j}\left(\boldsymbol{y}_{\boldsymbol{G}_f, i}^{t,s-1}+\boldsymbol{G}_{f_i,i}^t\left(\boldsymbol{\theta}_i^{t, s}\right)-\boldsymbol{G}_{f_i,i}^t\left(\boldsymbol{\theta}_i^{t, s-1}\right)\right) .
\end{aligned}
\end{equation*}

\EndFor

\State \textbf{Output:} $\boldsymbol{\mu}_j^{t+1}$, $\boldsymbol{\Sigma}_j^{t+1}$, $\boldsymbol{\gamma}_j^{t+1}$, $\delta_j^{t+1}$, $\boldsymbol{\theta}_j^{t+1}$.

\end{algorithmic}
\end{algorithm}

\section{Proofs of Theorems \ref{thm:convexity}--\ref{thm:asy_normality} }\label{sec:proof_1}
In the following, for two sequences $a_N,b_N$, we denote $a_N\lesssim b_N$ or $a_N=O(b_N)$ if $a_N\le c b_N$ for some constant $c>0$, $a_N\gtrsim b_N$ if $a_N\ge c b_N$ for some constant $c>0$. 

\begin{proof}[Proof of Theorem \ref{thm:convexity}]
  Denote
  \[ \nabla^2 f (\boldsymbol{\mu} , \boldsymbol{\Sigma} , \boldsymbol{\gamma } ,
     \delta , \boldsymbol{\theta} ) = \left(\begin{array}{cc}
       \boldsymbol{H}_{11} & \boldsymbol{H}_{12}\\
       \boldsymbol{H}_{21} & \boldsymbol{H}_{22}
     \end{array}\right) \]
  where $\boldsymbol{H}_{11}$ corresponds to the Hessian of the parameters
  $\boldsymbol{\mu}, \boldsymbol{\Sigma}$, and $\boldsymbol{H}_{22}$ corresponds to
  the the Hessian of the parameters $\boldsymbol{\gamma}, \delta,
  \boldsymbol{\theta}$. Note that we have the decomposition that
  \[ \left(\begin{array}{cc}
       \boldsymbol{H}_{11} & \boldsymbol{H}_{12}\\
       \boldsymbol{H}_{21} & \boldsymbol{H}_{22}
     \end{array}\right) = \left(\begin{array}{cc}
       \boldsymbol{I} & \boldsymbol{O}\\
       \boldsymbol{H}_{21} \boldsymbol{H}_{11}^{- 1} & \boldsymbol{I}
     \end{array}\right) \left(\begin{array}{cc}
       \boldsymbol{H}_{11} & \boldsymbol{O}\\
       \boldsymbol{O} & \boldsymbol{H}_{22} - \boldsymbol{H}_{21}
       \boldsymbol{H}_{11}^{- 1} H_{12}
     \end{array}\right) \left(\begin{array}{cc}
       \boldsymbol{I} & \boldsymbol{O}\\
       \boldsymbol{H}_{21} \boldsymbol{H}_{11}^{- 1} & \boldsymbol{I}
     \end{array}\right)^{\top} \]
  whenever $\boldsymbol{H}_{11}$ is positive. Then, it is sufficient to show
  that $H_{11}$ is positive and the Schur complement $\boldsymbol{H}_{22} -
  \boldsymbol{H}_{21} \boldsymbol{H}_{11}^{- 1} \boldsymbol{H}_{12}$ is positive
  when $\boldsymbol{\gamma } , \delta , \boldsymbol{\theta} $ are near
  $\boldsymbol{\gamma }^{\ast}, \delta^{\ast}, \boldsymbol{\theta}^{\ast}$ with
  high probability.

  We first show that $\boldsymbol{H}_{11}$ is positive.
  
  For fixed $\boldsymbol{\gamma}, \delta,
  \boldsymbol{\theta}$ , define $g (\boldsymbol{\mu} , \boldsymbol{\Sigma}) : = f (\boldsymbol{\mu} ,
  \boldsymbol{\Sigma} , \boldsymbol{\gamma } , \delta , \boldsymbol{\theta} )$, and then define $h
  (t) := g (\boldsymbol{\mu}  + t  \boldsymbol{\mu}', \boldsymbol{\Sigma}  +
  t \boldsymbol{\Sigma}')$ for given $\boldsymbol{\mu}'$ and $\boldsymbol{\Sigma}'$,
  then
  \[ h'' (0) = \left(\begin{array}{cc}
       {\boldsymbol{\mu}'}^{\top} & \text{vec} (\boldsymbol{\Sigma}')^{\top}
     \end{array}\right) \boldsymbol{H}_{11} \left(\begin{array}{c}
       {\boldsymbol{\mu}'} \\
       \text{vec} (\boldsymbol{\Sigma}')
     \end{array}\right) . \]
  Thus, it suffices to show that $h'' (0) > 0$ for $\left(\begin{array}{cc}
    {\boldsymbol{\mu}'}^{\top} & \text{vec} (\boldsymbol{\Sigma}')^{\top}
  \end{array}\right)^{\top} \neq \boldsymbol{0} $. By direct computation,
  \begin{eqnarray*}
    h'' (0) & = & \text{tr} \left\{ \left[ \delta \sum_j
    \boldsymbol{B}^{\top}_j (\boldsymbol{\theta} ) \boldsymbol{B}_j
    (\boldsymbol{\theta} ) + \boldsymbol{K} ^{- 1} (\boldsymbol{\theta} ) \right]
    \boldsymbol{\mu}' {\boldsymbol{\mu}'}^{\top} \right\} + \text{tr} \left(
    {\boldsymbol{\Sigma} }^{- 1} \boldsymbol{\Sigma}' {\boldsymbol{\Sigma} }^{- 1}
    \boldsymbol{\Sigma}' \right),
  \end{eqnarray*}
  which is  greater than zero.

  We then show that $\boldsymbol{H}_{22} - \boldsymbol{H}_{21}
  \boldsymbol{H}_{11}^{- 1} \boldsymbol{H}_{12}$ is positive when $\boldsymbol{\gamma
  } , \delta , \boldsymbol{\theta} $ are near $\boldsymbol{\gamma }^{\ast},
  \delta^{\ast}, \boldsymbol{\theta}^{\ast}$ with high probability.
  
  For fixed $\boldsymbol{\gamma } , \delta , \boldsymbol{\theta}$, let
  $\boldsymbol{\mu} (\boldsymbol{\gamma } , \delta , \boldsymbol{\theta})$ and
  $\boldsymbol{\Sigma} (\boldsymbol{\gamma } , \delta , \boldsymbol{\theta})$ be the
  minimizer of $f (\boldsymbol{\mu} , \boldsymbol{\Sigma} , \boldsymbol{\gamma } ,
  \delta , \boldsymbol{\theta})$, then  we have  $: = f (\boldsymbol{\mu} (\boldsymbol{\gamma } , \delta ,
  \boldsymbol{\theta}) , \boldsymbol{\Sigma} (\boldsymbol{\gamma } , \delta ,
  \boldsymbol{\theta}) , \boldsymbol{\gamma } , \delta , \boldsymbol{\theta} )=\mathcal{L}(\boldsymbol{\gamma } , \delta ,
  \boldsymbol{\theta} )$. Recall that
  \[ \mathcal{L}(\boldsymbol{\gamma } , \delta ,   \boldsymbol{\theta} )= \frac{N}{2} \log
     2 \pi + \frac{1}{2} \log |\sigma^{-1} \boldsymbol{I} + \boldsymbol{B}
     (\boldsymbol{\theta}) \boldsymbol{K}  (\boldsymbol{\theta})
     \boldsymbol{B}^{\boldsymbol{\top}} (\boldsymbol{\theta}) | + (\boldsymbol{z} -
     \boldsymbol{X} \boldsymbol{\gamma})^{\top} [\delta^{-1} \boldsymbol{I} + \boldsymbol{B}
     (\boldsymbol{\theta}) \boldsymbol{K}  (\boldsymbol{\theta}) \boldsymbol{B}^{\top}
     (\boldsymbol{\theta})]^{- 1} (\boldsymbol{z} - \boldsymbol{X} \boldsymbol{\gamma}).
  \]
  Thus,
  the Hessian of $\mathcal{L} (\boldsymbol{\gamma } , \delta ,
  \boldsymbol{\theta} )$  is exactly $\boldsymbol{H}_{22} - \boldsymbol{H}_{21}
  \boldsymbol{H}_{11}^{- 1} \boldsymbol{H}_{12}$. 
 Then, it suffices to show
  $\lambda_{\min} [\mathcal{L}  (\boldsymbol{\gamma } , \delta ,
  \boldsymbol{\theta} )] > 0$ when $\boldsymbol{\gamma } , \delta ,
  \boldsymbol{\theta} $ are near $\boldsymbol{\gamma }^{\ast}, \delta^{\ast},
  \boldsymbol{\theta}^{\ast}$ with high probability.
  
  Let $\boldsymbol{C} (\widetilde{\boldsymbol{\theta}}) =\sigma^{-1} \boldsymbol{I} +
  \boldsymbol{B} (\boldsymbol{\theta}) \boldsymbol{K } (\boldsymbol{\theta})
  \boldsymbol{B}^{\top} (\boldsymbol{\theta})$ where
  $\widetilde{\boldsymbol{\theta}} = (\boldsymbol{\theta}^{\top}, \delta
  )^{\top}$. Then, for $k = 1, \ldots, q + 1$,
  \[ \frac{\partial^2\mathcal{L}}{\partial \boldsymbol{\gamma}  \partial
     \boldsymbol{\gamma}^{\top}} = \boldsymbol{X}^{\top} \boldsymbol{C}^{- 1}
     (\widetilde{\boldsymbol{\theta}}) \boldsymbol{X}, \frac{\partial^2
    \mathcal{L}}{\partial \boldsymbol{\gamma}  \partial \boldsymbol{}
     \widetilde{\boldsymbol{\theta}}_k} = - 2 \boldsymbol{X}^{\top}
     \boldsymbol{C}^{- 1} (\widetilde{\boldsymbol{\theta}} ) \frac{\partial
     \boldsymbol{C} (\widetilde{\boldsymbol{\theta}})}{\partial \boldsymbol{}
     \widetilde{\boldsymbol{\theta}}_k} \boldsymbol{C}^{- 1}
     (\widetilde{\boldsymbol{\theta}} ) (\boldsymbol{z} - \boldsymbol{X}
     \boldsymbol{\gamma}) \]
  \[ \frac{\partial^2\mathcal{L}}{\partial \boldsymbol{}
     \widetilde{\boldsymbol{\theta}}_k \partial \boldsymbol{}
     \widetilde{\boldsymbol{\theta}}_l} = \text{tr} \left( \boldsymbol{C}^{- 1}
     (\widetilde{\boldsymbol{\theta}}) \frac{\partial \boldsymbol{C}
     (\widetilde{\boldsymbol{\theta}})}{\partial \boldsymbol{}
     \widetilde{\boldsymbol{\theta}}_k} \boldsymbol{C}^{- 1}
     (\widetilde{\boldsymbol{\theta}}) \frac{\partial \boldsymbol{C}
     (\widetilde{\boldsymbol{\theta}})}{\partial \boldsymbol{}
     \widetilde{\boldsymbol{\theta}}_l} \right) + \text{tr} \left( [\boldsymbol{C}
     (\widetilde{\boldsymbol{\theta}}) - (\boldsymbol{z} - \boldsymbol{X}
     \boldsymbol{\gamma})  (\boldsymbol{z} - \boldsymbol{X} \boldsymbol{\gamma})^{\top}]
     \frac{\partial^2 \boldsymbol{C}^{- 1}
     (\widetilde{\boldsymbol{\theta}})}{\partial \boldsymbol{}
     \widetilde{\boldsymbol{\theta}}_k \partial \boldsymbol{}
     \widetilde{\boldsymbol{\theta}}_l} \right) \]
  Let $\boldsymbol{u} = \boldsymbol{z} - \boldsymbol{X} \boldsymbol{\gamma}^{\ast}$,
  then
  \[ \frac{\partial^2\mathcal{L}}{\partial \boldsymbol{\gamma}  \partial \boldsymbol{}
     \widetilde{\boldsymbol{\theta}}_k} = - 2 \boldsymbol{X}^{\top}
     \boldsymbol{C}^{- 1} (\widetilde{\boldsymbol{\theta}} ) \frac{\partial
     \boldsymbol{C} (\widetilde{\boldsymbol{\theta}})}{\partial \boldsymbol{}
     \widetilde{\boldsymbol{\theta}}_k} \boldsymbol{C}^{- 1}
     (\widetilde{\boldsymbol{\theta}} ) \boldsymbol{u}  - 2 \boldsymbol{X}^{\top}
     \boldsymbol{C}^{- 1} (\widetilde{\boldsymbol{\theta}} ) \frac{\partial
     \boldsymbol{C} (\widetilde{\boldsymbol{\theta}})}{\partial \boldsymbol{}
     \widetilde{\boldsymbol{\theta}}_k} \boldsymbol{C}^{- 1}
     (\widetilde{\boldsymbol{\theta}} ) \boldsymbol{X} (\boldsymbol{\gamma}^{\ast} -
     \boldsymbol{\gamma}) \]
  \begin{eqnarray*}
    \frac{\partial^2\mathcal{L}}{\partial \boldsymbol{} \widetilde{\boldsymbol{\theta}}_k
    \partial \boldsymbol{} \widetilde{\boldsymbol{\theta}}_l} & = & \text{tr}
    \left( \boldsymbol{C}^{- 1} (\widetilde{\boldsymbol{\theta}}) \frac{\partial
    \boldsymbol{C} (\widetilde{\boldsymbol{\theta}})}{\partial \boldsymbol{}
    \widetilde{\boldsymbol{\theta}}_k} \boldsymbol{C}^{- 1}
    (\widetilde{\boldsymbol{\theta}}) \frac{\partial \boldsymbol{C}
    (\widetilde{\boldsymbol{\theta}})}{\partial \boldsymbol{}
    \widetilde{\boldsymbol{\theta}}_l} \right) + \text{tr} \left( [\boldsymbol{C} 
    (\widetilde{\boldsymbol{\theta}}^{\ast}) - \boldsymbol{u} 
    \boldsymbol{u}^{\top}] \frac{\partial^2 \boldsymbol{C}^{- 1}
    (\widetilde{\boldsymbol{\theta}})}{\partial \boldsymbol{}
    \widetilde{\boldsymbol{\theta}}_k \partial \boldsymbol{}
    \widetilde{\boldsymbol{\theta}}_l} \right) +\\
    &  & \text{tr} \left( [\boldsymbol{C}  (\widetilde{\boldsymbol{\theta}}) -
    \boldsymbol{C}  (\widetilde{\boldsymbol{\theta}}^{\ast}) + 2 \boldsymbol{X}
    (\boldsymbol{\gamma} - \boldsymbol{\gamma}^{\ast})
    \boldsymbol{\boldsymbol{u}^{\top}} + \boldsymbol{X} (\boldsymbol{\gamma} -
    \boldsymbol{\gamma}^{\ast}) (\boldsymbol{\gamma} -
    \boldsymbol{\gamma}^{\ast})^{\top} \boldsymbol{X}^{\top}] \frac{\partial^2
    \boldsymbol{C}^{- 1} (\widetilde{\boldsymbol{\theta}})}{\partial \boldsymbol{}
    \widetilde{\boldsymbol{\theta}}_k \partial \boldsymbol{}
    \widetilde{\boldsymbol{\theta}}_l} \right)
  \end{eqnarray*}
  Let $N_k = m$ for $k = 1, \ldots, q$ and $N_{q + 1} = N$, we have
  \begin{equation}\label{eq:sup1}
    \sup_{\boldsymbol{\theta} \in \Theta, \delta \in \mathbb{S}} \left\|
    \frac{1}{N_k} \boldsymbol{X}^{\top} \boldsymbol{C}^{- 1}
    (\widetilde{\boldsymbol{\theta}} ) \frac{\partial \boldsymbol{C}
    (\widetilde{\boldsymbol{\theta}})}{\partial \boldsymbol{}
    \widetilde{\boldsymbol{\theta}}_k} \boldsymbol{C}^{- 1}
    (\widetilde{\boldsymbol{\theta}} ) \boldsymbol{u} \right\|  = o_{\mathbb{P}}
    (1),
  \end{equation}
  \begin{equation}
    \sup_{\boldsymbol{\theta} \in \Theta, \delta \in \mathbb{S}} \left\|
    \frac{1}{N_k} \boldsymbol{X}^{\top} \boldsymbol{C}^{- 1}
    (\widetilde{\boldsymbol{\theta}} ) \frac{\partial \boldsymbol{C}
    (\widetilde{\boldsymbol{\theta}})}{\partial \boldsymbol{}
    \widetilde{\boldsymbol{\theta}}_k} \boldsymbol{C}^{- 1}
    (\widetilde{\boldsymbol{\theta}} ) \boldsymbol{X} (\boldsymbol{\gamma}^{\ast} -
    \boldsymbol{\gamma}) \right\| {= O_{\mathbb{P}}}  (\| \boldsymbol{\gamma}^{\ast}
    - \boldsymbol{\gamma} \|)
  \end{equation}
  \begin{equation}
    \sup_{\boldsymbol{\theta} \in \Theta, \delta \in \mathbb{S}} \left|
    \frac{1}{\min \{ N_k, N_l \}} \text{tr} \left( [\boldsymbol{C} 
    (\widetilde{\boldsymbol{\theta}}^{\ast}) - \boldsymbol{u} 
    \boldsymbol{u}^{\top}] \frac{\partial^2 \boldsymbol{C}^{- 1}
    (\widetilde{\boldsymbol{\theta}})}{\partial \boldsymbol{}
    \widetilde{\boldsymbol{\theta}}_k \partial \boldsymbol{}
    \widetilde{\boldsymbol{\theta}}_l} \right) \right| = o_{\mathbb{P}} (1)
  \end{equation}
  \begin{equation}
    \sup_{\boldsymbol{\theta} \in \Theta, \delta \in \mathbb{S}} \left|
    \frac{1}{\min \{ N_k, N_l \}} \text{tr} \left( [\boldsymbol{C} 
    (\widetilde{\boldsymbol{\theta}}) - \boldsymbol{C} 
    (\widetilde{\boldsymbol{\theta}}^{\ast})] \frac{\partial^2 \boldsymbol{C}^{-
    1} (\widetilde{\boldsymbol{\theta}})}{\partial \boldsymbol{}
    \widetilde{\boldsymbol{\theta}}_k \partial \boldsymbol{}
    \widetilde{\boldsymbol{\theta}}_l} \right) \right| = O  (\|
    \widetilde{\boldsymbol{\theta}} - \widetilde{\boldsymbol{\theta}}^{\ast} \|)
  \end{equation}
  \begin{equation}
    \sup_{\boldsymbol{\theta} \in \Theta, \delta \in \mathbb{S}} \left|
    \frac{1}{\min \{ N_k, N_l \}} \text{tr} \left( \boldsymbol{X}
    (\boldsymbol{\gamma} - \boldsymbol{\gamma}^{\ast})
    \boldsymbol{\boldsymbol{u}^{\top}} \frac{\partial^2 \boldsymbol{C}^{- 1}
    (\widetilde{\boldsymbol{\theta}})}{\partial \boldsymbol{}
    \widetilde{\boldsymbol{\theta}}_k \partial \boldsymbol{}
    \widetilde{\boldsymbol{\theta}}_l} \right) \right| = O_{\mathbb{P}} (\|
    \boldsymbol{\gamma} - \boldsymbol{\gamma}^{\ast} \|)
  \end{equation}
  \begin{equation}
    \sup_{\boldsymbol{\theta} \in \Theta, \delta \in \mathbb{S}} \left|
    \frac{1}{\min \{ N_k, N_l \}} \text{tr} \left( \boldsymbol{X}
    (\boldsymbol{\gamma} - \boldsymbol{\gamma}^{\ast}) (\boldsymbol{\gamma} -
    \boldsymbol{\gamma}^{\ast})^{\top} \boldsymbol{X}^{\top} \frac{\partial^2
    \boldsymbol{C}^{- 1} (\widetilde{\boldsymbol{\theta}})}{\partial \boldsymbol{}
    \widetilde{\boldsymbol{\theta}}_k \partial \boldsymbol{}
    \widetilde{\boldsymbol{\theta}}_l} \right) \right| {= O_{\mathbb{P}}}  (\|
    \boldsymbol{\gamma} - \boldsymbol{\gamma}^{\ast} \|^2)
  \end{equation}
  \begin{equation}
    \sup_{\boldsymbol{\theta} \in \Theta, \delta \in \mathbb{S}} \left\|
    \frac{1}{N} \boldsymbol{X}^{\top} \boldsymbol{C} 
    (\widetilde{\boldsymbol{\theta}}) \boldsymbol{X} - \frac{1}{N} \mathbb{E}
    \boldsymbol{X}^{\top} \boldsymbol{C}  (\widetilde{\boldsymbol{\theta}} )
    \boldsymbol{X} \right\|_{\text{op}} = o_{\mathbb{P}} (1)
  \end{equation}
  \begin{equation}
    \left\| \frac{1}{N} \mathbb{E} \boldsymbol{X}^{\top} \boldsymbol{C} 
    (\widetilde{\boldsymbol{\theta}} ) \boldsymbol{X} - \frac{1}{N} \mathbb{E}
    \boldsymbol{X}^{\top} \boldsymbol{C}  (\widetilde{\boldsymbol{\theta}}^{\ast})
    \boldsymbol{X} \right\|_{\text{op}} = O  (\| \widetilde{\boldsymbol{\theta}} -
    \widetilde{\boldsymbol{\theta}}^{\ast} \|)
  \end{equation}
  \begin{equation}\label{eq:sup2}
    \frac{1}{\min \{ N_k, N_l \}} \text{tr} \left( \boldsymbol{C}^{- 1}
    (\widetilde{\boldsymbol{\theta}}) \frac{\partial \boldsymbol{C}
    (\widetilde{\boldsymbol{\theta}})}{\partial \boldsymbol{}
    \widetilde{\boldsymbol{\theta}}_k} \boldsymbol{C}^{- 1}
    (\widetilde{\boldsymbol{\theta}}) \frac{\partial \boldsymbol{C}
    (\widetilde{\boldsymbol{\theta}})}{\partial \boldsymbol{}
    \widetilde{\boldsymbol{\theta}}_l} - \boldsymbol{C}^{- 1}
    (\widetilde{\boldsymbol{\theta}}^{\ast}) \frac{\partial \boldsymbol{C}
    (\widetilde{\boldsymbol{\theta}}^{\ast})}{\partial \boldsymbol{}
    \widetilde{\boldsymbol{\theta}}_k} \boldsymbol{C}^{- 1}
    (\widetilde{\boldsymbol{\theta}}^{\ast}) \frac{\partial \boldsymbol{C}
    (\widetilde{\boldsymbol{\theta}}^{\ast})}{\partial \boldsymbol{}
    \widetilde{\boldsymbol{\theta}}_l} \right) = O  (\|
    \widetilde{\boldsymbol{\theta}} - \widetilde{\boldsymbol{\theta}}^{\ast} \|)
  \end{equation}
  We will postpone the proof of the above equations until the end to enhance
  readability.
  
  Therefore, the Hessian matrix can be written as
  {\scriptsize
  \[ \left(\begin{array}{ccc}
       \frac{\partial^2\mathcal{L}}{\partial \boldsymbol{\gamma}  \partial
       \boldsymbol{\gamma}^{\top}} & \frac{\partial^2\mathcal{L}}{\partial \boldsymbol{\gamma} 
       \partial \boldsymbol{} \delta} & \frac{\partial^2\mathcal{L}}{\partial
       \boldsymbol{\gamma}  \partial \boldsymbol{} \boldsymbol{\theta}^{\top}}\\
       \frac{\partial^2\mathcal{L}}{\partial \boldsymbol{} \delta \partial
       \boldsymbol{\gamma}} & \frac{\partial^2\mathcal{L}}{\partial \delta  \partial
       \boldsymbol{} \delta} & \frac{\partial^2\mathcal{L}}{\partial \boldsymbol{} \delta
       \partial \boldsymbol{} \boldsymbol{\theta}}\\
       \frac{\partial^2\mathcal{L}}{\boldsymbol{} \partial \boldsymbol{\theta}^{\top} \partial
       \boldsymbol{\gamma}} & \frac{\partial^2\mathcal{L}}{\boldsymbol{} \partial
       \boldsymbol{\theta}^{\top} \partial \boldsymbol{} \delta} & \frac{\partial^2
      \mathcal{L}}{\boldsymbol{} \partial \boldsymbol{\theta}^{\top} \partial \boldsymbol{}
       \delta}
     \end{array}\right) = \left(\begin{array}{ccc}
       N \frac{1}{N} \mathbb{E} \frac{\partial^2\mathcal{L} (\boldsymbol{\gamma}^{\ast},
       \widetilde{\boldsymbol{\theta}}^{\ast})}{\partial \boldsymbol{\gamma} 
       \partial \boldsymbol{\gamma}^{\top}} & \boldsymbol{O} & \boldsymbol{O}\\
       \boldsymbol{O} & N \frac{1}{N} \mathbb{E} \frac{\partial^2\mathcal{L}
       (\boldsymbol{\gamma}^{\ast},
       \widetilde{\boldsymbol{\theta}}^{\ast})}{\partial \delta  \partial
       \boldsymbol{} \delta} & m \frac{1}{m} \mathbb{E} \frac{\partial^2\mathcal{L}
       (\boldsymbol{\gamma}^{\ast},
       \widetilde{\boldsymbol{\theta}}^{\ast})}{\partial \boldsymbol{} \delta
       \partial \boldsymbol{} \boldsymbol{\theta}}\\
       \boldsymbol{O} & m \frac{1}{m} \mathbb{E} \frac{\partial^2\mathcal{L}
       (\boldsymbol{\gamma}^{\ast},
       \widetilde{\boldsymbol{\theta}}^{\ast})}{\boldsymbol{} \partial
       \boldsymbol{\theta}  \partial \boldsymbol{} \delta} & m \frac{1}{m}
       \mathbb{E} \frac{\partial^2\mathcal{L} (\boldsymbol{\gamma}^{\ast},
       \widetilde{\boldsymbol{\theta}}^{\ast})}{\boldsymbol{} \partial
       \boldsymbol{\theta}  \partial \boldsymbol{} \boldsymbol{\theta}^{\top}}
     \end{array}\right) + \left(\begin{array}{ccc}
       N \boldsymbol{R}_{11} & N \boldsymbol{R}_{12} & m \boldsymbol{R}_{13}\\
       N \boldsymbol{R}_{21} & N \boldsymbol{R}_{22} & m \boldsymbol{R}_{23}\\
       m \boldsymbol{R}_{31} & m \boldsymbol{R}_{32} & m \boldsymbol{R}_{33}
     \end{array}\right) \]
     }
  where each element of $\boldsymbol{R}_{i j}$ is $o_{\mathbb{P}} (1) +
  O_{\mathbb{P}} (\| \widetilde{\boldsymbol{\theta}} -
  \widetilde{\boldsymbol{\theta}}^{\ast} \| + \| \boldsymbol{\gamma}^{\ast} -
  \boldsymbol{\gamma} \|)$ for $i, j = 1, 2, 3$. The Hessian matrix can further
  be written as
  \[ \boldsymbol{H}_1 + \boldsymbol{H}_2 \]
  where
  \[ \boldsymbol{H}_1 = m \left(\begin{array}{ccc}
       \frac{1}{N} \mathbb{E} \frac{\partial^2\mathcal{L} (\boldsymbol{\gamma}^{\ast},
       \widetilde{\boldsymbol{\theta}}^{\ast})}{\partial \boldsymbol{\gamma} 
       \partial \boldsymbol{\gamma}^{\top}} & \boldsymbol{O} & \boldsymbol{O}\\
       \boldsymbol{O} & \frac{1}{N} \mathbb{E} \frac{\partial^2\mathcal{L}
       (\boldsymbol{\gamma}^{\ast},
       \widetilde{\boldsymbol{\theta}}^{\ast})}{\partial \delta  \partial
       \boldsymbol{} \delta} & \frac{1}{m} \mathbb{E} \frac{\partial^2\mathcal{L}
       (\boldsymbol{\gamma}^{\ast},
       \widetilde{\boldsymbol{\theta}}^{\ast})}{\partial \boldsymbol{} \delta
       \partial \boldsymbol{} \boldsymbol{\theta}}\\
       \boldsymbol{O} & \frac{1}{m} \mathbb{E} \frac{\partial^2\mathcal{L}
       (\boldsymbol{\gamma}^{\ast},
       \widetilde{\boldsymbol{\theta}}^{\ast})}{\boldsymbol{} \partial
       \boldsymbol{\theta}  \partial \boldsymbol{} \delta} & \frac{1}{m}
       \mathbb{E} \frac{\partial^2\mathcal{L} (\boldsymbol{\gamma}^{\ast},
       \widetilde{\boldsymbol{\theta}}^{\ast})}{\boldsymbol{} \partial
       \boldsymbol{\theta}  \partial \boldsymbol{} \boldsymbol{\theta}^{\top}}
     \end{array}\right) + m \left(\begin{array}{ccc}
       \boldsymbol{R}_{11} &  \boldsymbol{R}_{12} &  \boldsymbol{R}_{13}\\
       \boldsymbol{R}_{21} &  \boldsymbol{R}_{22} &  \boldsymbol{R}_{23}\\
       \boldsymbol{R}_{31} &  \boldsymbol{R}_{32} &  \boldsymbol{R}_{33}
     \end{array}\right), \]
  and
  \[ \boldsymbol{H}_2 = (N - m) \left(\begin{array}{ccc}
       \frac{1}{N} \frac{\partial^2\mathcal{L} (\boldsymbol{\gamma}^{\ast},
       \widetilde{\boldsymbol{\theta}}^{\ast})}{\partial \boldsymbol{\gamma} 
       \partial \boldsymbol{\gamma}^{\top}} & \boldsymbol{O} & \boldsymbol{O} \\
       \boldsymbol{O} & \frac{1}{N} \mathbb{E} \frac{\partial^2\mathcal{L}
       (\boldsymbol{\gamma}^{\ast},
       \widetilde{\boldsymbol{\theta}}^{\ast})}{\partial \delta  \partial
       \boldsymbol{} \delta} & \boldsymbol{O}\\
       \boldsymbol{O} & \boldsymbol{O} & \boldsymbol{O}
     \end{array}\right) + (N - m) \left(\begin{array}{ccc}
       \boldsymbol{R}_{11} &  \boldsymbol{R}_{12} & \boldsymbol{O}\\
       \boldsymbol{R}_{21} &  \boldsymbol{R}_{22} & \boldsymbol{O}\\
       \boldsymbol{O} & \boldsymbol{O} & \boldsymbol{O}
     \end{array}\right) . \]
  Therefore, the minimal eigenvalue is lower bounded by
  \[ m [\lambda^{\ast} + o_{\mathbb{P}} (1) + O_{\mathbb{P}} (\|
     \widetilde{\boldsymbol{\theta}} - \widetilde{\boldsymbol{\theta}}^{\ast} \| +
     \| \boldsymbol{\gamma}^{\ast} - \boldsymbol{\gamma} \|)] + (N - m) \|
     \boldsymbol{x} \|^2 [\lambda^{\ast} + o_{\mathbb{P}} (1) + O_{\mathbb{P}}
     (\| \widetilde{\boldsymbol{\theta}} - \widetilde{\boldsymbol{\theta}}^{\ast}
     \| + \| \boldsymbol{\gamma}^{\ast} - \boldsymbol{\gamma} \|)] \]
  for some vector $\boldsymbol{x} \boldsymbol{}$ with $\| \boldsymbol{x} \|
  \leqslant 1$. Then, for each $\epsilon > 0$, there exist constants
  $\xi_{\epsilon}$ and integer $M_{\epsilon}$, if $\| \widetilde{\boldsymbol{\theta}} -
  \widetilde{\boldsymbol{\theta}}^{\ast} \| + \| \boldsymbol{\gamma}^{\ast} -
  \boldsymbol{\gamma} \| \leqslant \xi$ and $m \geqslant M_{\epsilon}$, then the minimal
  eigenvalue is lower bounded $m \lambda^{\ast} / 2$ with probability greater
  than $1 - \epsilon$.
  
  Now, we turn to the proof of the Equations \eqref{eq:sup1}--\eqref{eq:sup2}. Here, we will prove only \eqref{eq:sup1},
  as the proof for the others is similar.
  
  Let $\boldsymbol{x}_i$ be the $i$-th column of $\boldsymbol{X}$. It suffices to show, for each $i=1,\ldots,p$,
  \[ {\sup_{{\widetilde{\boldsymbol{\theta}}  } }}  \left| h_{i, k} \left(
     {\widetilde{\boldsymbol{\theta}}  }  \right) \right| = O_{\mathbb{P}} (1)
     \text{ with } h_{i, k} (\widetilde{\boldsymbol{\theta}}  ) = m^{-
     \frac{1}{2}} \boldsymbol{x}^{\top}_i \boldsymbol{C}^{- 1}
     (\widetilde{\boldsymbol{\theta}}  ) \frac{\partial \boldsymbol{C}
     (\widetilde{\boldsymbol{\theta}} )}{\partial \boldsymbol{}
     \widetilde{\boldsymbol{\theta}}_k} \boldsymbol{C}^{- 1}
     (\widetilde{\boldsymbol{\theta}} ) \boldsymbol{u}. \]
  We use Dudley's entropy integral as stated in the Definition \ref{def:subexp} and Theorem \ref{thm:subexp} in Section \ref{sec:tech} for the proof.

  Note that
  \[ \frac{\partial \boldsymbol{C} (\widetilde{\boldsymbol{\theta}})}{\partial
     \boldsymbol{} \widetilde{\boldsymbol{\theta}}_k} = \boldsymbol{B}_k
     (\boldsymbol{\theta}) \boldsymbol{K}  (\boldsymbol{\theta})
     \boldsymbol{B}^{\boldsymbol{\top}} (\boldsymbol{\theta}) + \boldsymbol{B}_k
     (\boldsymbol{\theta}) \boldsymbol{K} (\boldsymbol{\theta})
     \boldsymbol{B}_k^{\boldsymbol{\top}}(\boldsymbol{\theta}) +
     \boldsymbol{B}  (\boldsymbol{\theta}) \boldsymbol{K}_k (\boldsymbol{\theta})
     \boldsymbol{B}^{\boldsymbol{\top}}  (\boldsymbol{\theta}) . \]
  Then, according to Lemma \ref{le:tech1} in Section \ref{sec:tech} and Assumption  \ref{thm:convexity}, \ref{as:covariance}, it has at most $m$
  nonzero bounded singular values. Let $\boldsymbol{M}_{1 k}
  (\widetilde{\boldsymbol{\theta}} ) : = \boldsymbol{C}^{- 1}
  (\widetilde{\boldsymbol{\theta}} ) \frac{\partial \boldsymbol{C}
  (\widetilde{\boldsymbol{\theta}})}{\partial \boldsymbol{}
  \widetilde{\boldsymbol{\theta}}_k} \boldsymbol{C}^{- 1}
  (\widetilde{\boldsymbol{\theta}} ) \boldsymbol{C}^{\frac{1}{2}}
  (\widetilde{\boldsymbol{\theta}}^{\ast})$, then by singular value
  decompostion, we have
  \[ \boldsymbol{M}_{1 k} (\widetilde{\boldsymbol{\theta}} ) - \boldsymbol{M}_{1 k}
     \left( {\widetilde{\boldsymbol{\theta}} }' \right) = \boldsymbol{U}
     (\widetilde{\boldsymbol{\theta}}, \widetilde{\boldsymbol{\theta}}') \boldsymbol{\Sigma}
     (\widetilde{\boldsymbol{\theta}}, \widetilde{\boldsymbol{\theta}}') \boldsymbol{V}^{\top}
     (\widetilde{\boldsymbol{\theta}}, \widetilde{\boldsymbol{\theta}}'), \]
  where $\boldsymbol{U} (\widetilde{\boldsymbol{\theta}},
  \widetilde{\boldsymbol{\theta}}') = \left( \boldsymbol{u}_1 \left(
  \widetilde{\boldsymbol{\theta}} , {\widetilde{\boldsymbol{\theta}} }' \right),
  \ldots, \boldsymbol{u}_N \left( \widetilde{\boldsymbol{\theta}} ,
  {\widetilde{\boldsymbol{\theta}} }' \right) \right)$ and $\boldsymbol{V }
  (\widetilde{\boldsymbol{\theta}}, \widetilde{\boldsymbol{\theta}}') = \left(
  \boldsymbol{v}_1 \left( \widetilde{\boldsymbol{\theta}} ,
  {\widetilde{\boldsymbol{\theta}} }' \right), \ldots, \boldsymbol{v}_N \left(
  \widetilde{\boldsymbol{\theta}} , {\widetilde{\boldsymbol{\theta}} }' \right)
  \right)$ are two orthogonal matrices and $\boldsymbol{\Sigma} (\widetilde{\boldsymbol{\theta}}, \widetilde{\boldsymbol{\theta}}')
  = \text{diag} \left( \lambda \boldsymbol{}_1 \left(
  \widetilde{\boldsymbol{\theta}} , {\widetilde{\boldsymbol{\theta}} }' \right),
  \ldots, \lambda \boldsymbol{}_m \left( \widetilde{\boldsymbol{\theta}} ,
  {\widetilde{\boldsymbol{\theta}} }' \right), 0, \ldots, 0 \right)$. Thus,
  \begin{eqnarray}\label{eq:sup1_1}
    &  & \mathbb{E} \left[ \exp \left( \lambda (h_{i, k} \left(
    {\widetilde{\boldsymbol{\theta}}  }  \right) - h_{i, k} \left(
    {{\widetilde{\boldsymbol{\theta}}  } }' \right)) \right) \right] \nonumber\\
    & = & \mathbb{E} \left[ \exp \left( \lambda m^{- \frac{1}{2}}
    \boldsymbol{x}^{\top}_i \left[ \boldsymbol{M}_{1 k}
    (\widetilde{\boldsymbol{\theta}} ) - \boldsymbol{M}_{1 k} \left(
    {\widetilde{\boldsymbol{\theta}} }' \right) \right] \boldsymbol{e} \right)
    \right] \nonumber\\
    &= & \mathbb{E} \left[ \exp \left( \lambda m^{- \frac{1}{2}}
    \sum_{j = 1}^m \boldsymbol{x}^{\top}_i \boldsymbol{u}_j \left(
    \widetilde{\boldsymbol{\theta}} , {\widetilde{\boldsymbol{\theta}} }' \right)
    \lambda \boldsymbol{}_j \left( \widetilde{\boldsymbol{\theta}} ,
    {\widetilde{\boldsymbol{\theta}} }' \right) \boldsymbol{v}_j^{\top} \left(
    \widetilde{\boldsymbol{\theta}} , {\widetilde{\boldsymbol{\theta}} }' \right)
    \boldsymbol{e} \right) \right] . 
  \end{eqnarray}
  Since $\boldsymbol{v}_1^{\top} \left( \widetilde{\boldsymbol{\theta}} ,
  {\widetilde{\boldsymbol{\theta}} }' \right) \boldsymbol{e}, \boldsymbol{v}_2^{\top}
  \left( \widetilde{\boldsymbol{\theta}} , {\widetilde{\boldsymbol{\theta}} }'
  \right) \boldsymbol{e}, \ldots, \boldsymbol{v}_m^{\top} \left(
  \widetilde{\boldsymbol{\theta}} , {\widetilde{\boldsymbol{\theta}} }' \right)
  \boldsymbol{e}$ are standard Gaussian random vector conditional on
  $\boldsymbol{X}$, above term further equals to
  \begin{eqnarray}\label{eq:sup1_2}
    &  & \mathbb{E} \prod_{j = 1}^m \mathbb{E} \left\{ \exp \left[ \lambda
    m^{- \frac{1}{2}} \boldsymbol{x}^{\top}_i \boldsymbol{u}_j \left(
    \widetilde{\boldsymbol{\theta}} , {\widetilde{\boldsymbol{\theta}} }' \right)
    \lambda \boldsymbol{}_j \left( \widetilde{\boldsymbol{\theta}} ,
    {\widetilde{\boldsymbol{\theta}} }' \right) \boldsymbol{v}_j^{\top} \left(
    \widetilde{\boldsymbol{\theta}} , {\widetilde{\boldsymbol{\theta}} }' \right)
    \boldsymbol{e} \right] | \boldsymbol{X} \right\} \nonumber\\
    & \leqslant & \mathbb{E} \exp \left[ \lambda^2 m^{- 1} \sum_{j = 1}^m
    \left( \boldsymbol{x}^{\top}_i \boldsymbol{u}_j \left(
    \widetilde{\boldsymbol{\theta}} , {\widetilde{\boldsymbol{\theta}} }' \right)
    \right)^2 \lambda \boldsymbol{}_j^2 \left( \widetilde{\boldsymbol{\theta}} ,
    {\widetilde{\boldsymbol{\theta}} }' \right) \right] . 
  \end{eqnarray}
  Since, for any $p > 0$,
  \[ \left\| m^{- 1} \sum_{j = 1}^m \left( \boldsymbol{x}^{\top}_i
     \boldsymbol{u}_j \left( \widetilde{\boldsymbol{\theta}} ,
     {\widetilde{\boldsymbol{\theta}} }' \right) \right)^2 \lambda
     \boldsymbol{}_j^2 \left( \widetilde{\boldsymbol{\theta}} ,
     {\widetilde{\boldsymbol{\theta}} }' \right) \right\|_{L ^p} \leqslant m^{-
     1} \sum_{j = 1}^m \left\| \left( \boldsymbol{x}^{\top}_i \boldsymbol{u}_j
     \left( \widetilde{\boldsymbol{\theta}} , {\widetilde{\boldsymbol{\theta}} }'
     \right) \right)^2 \right\|_{L ^p} \lambda \boldsymbol{}_j^2 \left(
     \widetilde{\boldsymbol{\theta}} , {\widetilde{\boldsymbol{\theta}} }'
     \right), \]
  and $\boldsymbol{x} _i$ is sub-Gaussian, $m^{- 1} \sum_{j = 1}^m \left(
  \boldsymbol{x}^{\top}_i \boldsymbol{u}_j \left( \widetilde{\boldsymbol{\theta}} ,
  {\widetilde{\boldsymbol{\theta}} }' \right) \right)^2 \lambda \boldsymbol{}_j^2
  \left( \widetilde{\boldsymbol{\theta}} , {\widetilde{\boldsymbol{\theta}} }'
  \right)$ is sub-exponential with parameter proportional to $m^{- 1} \sum_{j
  = 1}^m \lambda \boldsymbol{}_j^2 \left( \widetilde{\boldsymbol{\theta}} ,
  {\widetilde{\boldsymbol{\theta}} }' \right)$. Thus, for some constant $c > 0$
  and $\lambda \leqslant c \left( m^{- 1} \sum_{j = 1}^m \lambda
  \boldsymbol{}_j^2 \left( \widetilde{\boldsymbol{\theta}} ,
  {\widetilde{\boldsymbol{\theta}} }' \right) \right)^{- 1 / 2}$,
  \begin{equation}\label{eq:sup1_3}
    \mathbb{E} \exp \left[ \lambda^2 m^{- 1} \sum_{j = 1}^m \left(
    \boldsymbol{x}^{\top}_i \boldsymbol{u}_j \left( \widetilde{\boldsymbol{\theta}}
    , {\widetilde{\boldsymbol{\theta}} }' \right) \right)^2 \lambda
    \boldsymbol{}_j^2 \left( \widetilde{\boldsymbol{\theta}} ,
    {\widetilde{\boldsymbol{\theta}} }' \right) \right] \leqslant \exp \left\{
    \lambda^2 c^2 m^{- 1} \sum_{j = 1}^m \lambda \boldsymbol{}_j^2 \left(
    \widetilde{\boldsymbol{\theta}} , {\widetilde{\boldsymbol{\theta}} }' \right)
    \right\} .
  \end{equation}

  Combining Equations \eqref{eq:sup1_1},\eqref{eq:sup1_2},\eqref{eq:sup1_3}, we have,
  \begin{equation}\label{eq:sup1_4}
    \mathbb{E} \left[ \exp \left( \lambda \frac{h_{i, k} \left(
    {\widetilde{\boldsymbol{\theta}}  }  \right) - h_{i, k} \left(
    {{\widetilde{\boldsymbol{\theta}}  } }' \right)}{c} \right) \right]
    \leqslant \exp \left( \lambda^2 d_k^2 \left( \widetilde{\boldsymbol{\theta}}
    , {\widetilde{\boldsymbol{\theta}} }' \right) \right)  \text{ for } \lambda
    \leqslant \frac{1}{d_k  \left( \widetilde{\boldsymbol{\theta}} ,
    {\widetilde{\boldsymbol{\theta}} }' \right)},
  \end{equation}
  where
  \[ d_k \left( \widetilde{\boldsymbol{\theta}} , {\widetilde{\boldsymbol{\theta}}
     }' \right) : = \sqrt{m^{- 1} \sum_{j = 1}^m \lambda \boldsymbol{}_j^2 \left(
     \widetilde{\boldsymbol{\theta}} , {\widetilde{\boldsymbol{\theta}} }'
     \right)} = m^{- \frac{1}{2}} \left\| \boldsymbol{M}_{1 k}
     (\widetilde{\boldsymbol{\theta}} ) - \boldsymbol{M}_{1 k} \left(
     {\widetilde{\boldsymbol{\theta}} }' \right) \right\|_F . \]
  For fixed $\widetilde{\boldsymbol{\theta}}'$, let
  \[ \bar{h} (\widetilde{\boldsymbol{\theta}} ) = \frac{1}{m} \text{tr} \left\{
     \left[ \boldsymbol{M}_{1 k} (\widetilde{\boldsymbol{\theta}} ) -
     \boldsymbol{M}_{1 k} \left( {\widetilde{\boldsymbol{\theta}} }' \right)
     \right]^2 \right\}, \]
  then, by Taylor expansion that $\bar{h}(\widetilde{\boldsymbol{\theta}} ) = h
  \left( {\widetilde{\boldsymbol{\theta}} }' \right) + \left(
  \widetilde{\boldsymbol{\theta}}  - {\widetilde{\boldsymbol{\theta}} }'
  \right)^{\top} \nabla \bar{h}\left( {\widetilde{\boldsymbol{\theta}} }' \right) +
  \left( \widetilde{\boldsymbol{\theta}}  - {\widetilde{\boldsymbol{\theta}} }'
  \right)^{\top} \nabla^2 \bar{h}\left( \iota \widetilde{\boldsymbol{\theta}}  + (1 -
  \iota) {\widetilde{\boldsymbol{\theta}} }' \right) \left(
  \widetilde{\boldsymbol{\theta}}  - {\widetilde{\boldsymbol{\theta}} }' \right)$
  for some $0 < \iota < 1$ and $\nabla \bar{h}\left( {\widetilde{\boldsymbol{\theta}}
  }' \right) = 0$,
  \begin{equation}
    \bar{h}(\widetilde{\boldsymbol{\theta}} ) = \left( \widetilde{\boldsymbol{\theta}} 
    - {\widetilde{\boldsymbol{\theta}} }' \right)^{\top} \nabla^2 \bar{h}\left( \iota
    \widetilde{\boldsymbol{\theta}}  + (1 - \iota) {\widetilde{\boldsymbol{\theta}}
    }' \right) \left( \widetilde{\boldsymbol{\theta}}  -
    {\widetilde{\boldsymbol{\theta}} }' \right) = O \left( \left\|
    \widetilde{\boldsymbol{\theta}}  - {\widetilde{\boldsymbol{\theta}} }'
    \right\|^2 \right) .
  \end{equation}
  From \eqref{eq:sup1_4}, we finally have, for some constant $c > 0$,
  \begin{equation}
    \mathbb{E} \left[ \exp \left( \lambda \frac{h_{i, k} \left(
    {\widetilde{\boldsymbol{\theta}}  }  \right) - h_{i, k} \left(
    {{\widetilde{\boldsymbol{\theta}}  } }' \right)}{c} \right) \right]
    \leqslant \exp \left( \lambda^2  \left\| \widetilde{\boldsymbol{\theta}}  -
    {\widetilde{\boldsymbol{\theta}} }' \right\|^2 \right)  \text{ for } \lambda
    \leqslant \frac{1}{\left\| \widetilde{\boldsymbol{\theta}}  -
    {\widetilde{\boldsymbol{\theta}} }' \right\| } .
  \end{equation}
  The covering number $N (\epsilon, \Theta \times \mathbb{S}, \| \cdot \|)$
  with respect to the Euclidean distance $\| \cdot \|$ is proportional to
  $\left( 1 + \frac{2 D}{\epsilon} \right)^{q + 1}$. Therefore, according to Theorem \ref{thm:subexp},
  \[ \mathbb{E} \left[ \sup_{{\widetilde{\boldsymbol{\theta}}  }  \in \Theta
     \times \mathbb{S}}  \left| h_{i, k} \left( {\widetilde{\boldsymbol{\theta}}
     }  \right) - h_{i, k} \left( {{\widetilde{\boldsymbol{\theta}}  } }'
     \right) \right| \right] \leq 8 \int_0^D \log (1 + N (\epsilon, \Theta
     \times \mathbb{S}, \| \cdot \|)) d \epsilon = O (1) \]
  Note that, by singular value decomposition again, $\boldsymbol{M}_{1 k} \left(
  {{\widetilde{\boldsymbol{\theta}}  } }  \right) = \boldsymbol{U} \left(
  {{\widetilde{\boldsymbol{\theta}}  } }  \right) \Sigma \left(
  {{\widetilde{\boldsymbol{\theta}}  } }  \right) V^{\top} \left(
  {{\widetilde{\boldsymbol{\theta}}  } }  \right),$where $\boldsymbol{U}
  (\widetilde{\boldsymbol{\theta}}) = (\boldsymbol{u}_1
  (\widetilde{\boldsymbol{\theta}} ), \ldots, \boldsymbol{u}_N
  (\widetilde{\boldsymbol{\theta}} ))$ and $V  (\widetilde{\boldsymbol{\theta}}) =
  (\boldsymbol{v}_1 (\widetilde{\boldsymbol{\theta}} ), \ldots, \boldsymbol{v}_N
  (\widetilde{\boldsymbol{\theta}} ))$ are two orthogonal matrices, and $\Sigma
  (\widetilde{\boldsymbol{\theta}}, \widetilde{\boldsymbol{\theta}}') =
  \text{diag} (\lambda \boldsymbol{}_1 (\widetilde{\boldsymbol{\theta}} ,), \ldots,
  \lambda \boldsymbol{}_m (\widetilde{\boldsymbol{\theta}} ), 0, \ldots, 0)$. Thus,
  \[ h_{i, k} \left( {\widetilde{\boldsymbol{\theta}}  }^{\ast} \right) = m^{-
     \frac{1}{2}} \sum_{j = 1}^m \boldsymbol{x}^{\top}_i \boldsymbol{u}_j \left(
     {\widetilde{\boldsymbol{\theta}}  }^{\ast} \right) \lambda \boldsymbol{}_j
     \boldsymbol{v}_j^{\top} \boldsymbol{e} \]
  has mean zero and finite variance, which means that
  \[ \mathbb{E} \left| h_{i, k} \left( {\widetilde{\boldsymbol{\theta}} 
     }^{\ast} \right) \right| = O (1) . \]
  Finally,
  \[ \mathbb{E} \left[ \sup_{{\widetilde{\boldsymbol{\theta}}  }  \in \Theta
     \times \mathbb{S}}  \left| h_{i, k} \left( {\widetilde{\boldsymbol{\theta}}
     }  \right) \right| \right] \leqslant \mathbb{E} \left[
     \sup_{{\widetilde{\boldsymbol{\theta}}  }  \in \Theta \times \mathbb{S}} 
     \left| h_{i, k} \left( {\widetilde{\boldsymbol{\theta}}  }  \right) - h_{i,
     k} \left( {{\widetilde{\boldsymbol{\theta}}  } }^{\ast} \right) \right| \right]
     +\mathbb{E} \left| h_{i, k} \left( {\widetilde{\boldsymbol{\theta}} 
     }^{\ast} \right) \right| = O (1) \]
  
\end{proof}

\begin{proof}[Proof of Theorem \ref{thm:consistency}]
  Without causing ambiguity, we also define
 $\mathcal{L} (\widetilde{\boldsymbol{\theta}} , \boldsymbol{\gamma}): = \log |
  \boldsymbol{C} (\widetilde{\boldsymbol{\theta}}) | + (\boldsymbol{z} -
  \boldsymbol{X} \boldsymbol{\gamma})^{\top} \boldsymbol{C}^{- 1}
  (\widetilde{\boldsymbol{\theta}}) (\boldsymbol{z} - \boldsymbol{X}
  \boldsymbol{\gamma}) $ where $\boldsymbol{C} (\widetilde{\boldsymbol{\theta}}) =
  \delta^{-1} \boldsymbol{I} + \boldsymbol{B} (\boldsymbol{\theta}) \boldsymbol{K} 
  (\boldsymbol{\theta}) \boldsymbol{B}^{\top} (\boldsymbol{\theta})$ and
  $\widetilde{\boldsymbol{\theta}} = (\boldsymbol{\theta}^{\top}, \delta
  )^{\top}$. It is sufficient to prove that: for any given constant $\xi >
  0$, if ${\| \widetilde{\boldsymbol{\theta}}   - \widetilde{\boldsymbol{\theta} } ^{\ast} \|^2}  + \| \boldsymbol{\gamma }  - \boldsymbol{\gamma}^{\ast} \|^2 >
  \xi^2$, then there exists some constant $\zeta > 0$ such that
  \[ m ^{- 1} \left[ \mathcal{L} (\widetilde{\boldsymbol{\theta}}  , \boldsymbol{\gamma } )
     - \mathcal{L} \left( {\widetilde{\boldsymbol{\theta}} }^{\ast},
     \boldsymbol{\gamma}^{\ast} \right) \right] \geqslant \zeta + o_{\mathbb{P}}
     (1) \text{ or }  N ^{- 1} \left[ \mathcal{L} (\widetilde{\boldsymbol{\theta}}  ,
     \boldsymbol{\gamma } ) - \mathcal{L} \left( {\widetilde{\boldsymbol{\theta}} }^{\ast},
     \boldsymbol{\gamma}^{\ast} \right) \right] \geqslant \zeta + o_{\mathbb{P}}
     (1) \]
  Let $\mathbb{L} (\widetilde{\boldsymbol{\theta}}  , \boldsymbol{\gamma } )
  =\mathbb{E}\mathcal{L} (\widetilde{\boldsymbol{\theta}} , \boldsymbol{\gamma})$, and note
  that
  \begin{eqnarray*}
    \mathcal{L} (\widetilde{\boldsymbol{\theta}}  , \boldsymbol{\gamma } ) - \mathcal{L} \left(
    {\widetilde{\boldsymbol{\theta}} }^{\ast}, \boldsymbol{\gamma}^{\ast} \right) &
    = & \mathbb{L} (\widetilde{\boldsymbol{\theta}}  , \boldsymbol{\gamma } )
    {-\mathbb{L}}  \left( {\widetilde{\boldsymbol{\theta}} }^{\ast},
    \boldsymbol{\gamma}^{\ast} \right)\\
    &  & + \mathcal{L} (\widetilde{\boldsymbol{\theta}}  , \boldsymbol{\gamma } ) - \mathcal{L}
    \left( {\widetilde{\boldsymbol{\theta}} }^{\ast}, \boldsymbol{\gamma}^{\ast}
    \right) - \left[ \mathbb{L} (\widetilde{\boldsymbol{\theta}}  ,
    \boldsymbol{\gamma } ) {-\mathbb{L}}  \left( {\widetilde{\boldsymbol{\theta}}
    }^{\ast}, \boldsymbol{\gamma}^{\ast} \right) \right]
  \end{eqnarray*}
  we then bound the two terms $\mathbb{L} (\widetilde{\boldsymbol{\theta}}  ,
  \boldsymbol{\gamma } ) {-\mathbb{L}}  \left( {\widetilde{\boldsymbol{\theta}}
  }^{\ast}, \boldsymbol{\gamma}^{\ast} \right)$ and $\mathcal{L}
  (\widetilde{\boldsymbol{\theta}}  , \boldsymbol{\gamma } ) - \mathcal{L} \left(
  {\widetilde{\boldsymbol{\theta}} }^{\ast}, \boldsymbol{\gamma}^{\ast} \right) -
  \left[ \mathbb{L} (\widetilde{\boldsymbol{\theta}}  , \boldsymbol{\gamma } )
  {-\mathbb{L}}  \left( {\widetilde{\boldsymbol{\theta}} }^{\ast},
  \boldsymbol{\gamma}^{\ast} \right) \right]$.
  
  Since
  \[ \lambda_{\min} \left[ \boldsymbol{C}^{\frac{1}{2}}
     (\widetilde{\boldsymbol{\theta} } ^{\ast}) \boldsymbol{C}^{- 1}
     (\widetilde{\boldsymbol{\theta}}) \boldsymbol{C}^{\frac{1}{2}}
     (\widetilde{\boldsymbol{\theta} } ^{\ast}) \right] > c, \lambda_{\max}
     \left[ \boldsymbol{C}^{\frac{1}{2}} (\widetilde{\boldsymbol{\theta} } ^{\ast})
     \boldsymbol{C}^{- 1} (\widetilde{\boldsymbol{\theta}})
     \boldsymbol{C}^{\frac{1}{2}} (\widetilde{\boldsymbol{\theta} } ^{\ast})
     \right] =O(1),\lambda_{\min} [N^{- 1} \mathbb{E} \boldsymbol{X}^{\top} \boldsymbol{C}^{- 1}
     (\widetilde{\boldsymbol{\theta}}) \boldsymbol{X}] > c  ,\]
  for some positive constant $c$, we have,
  \begin{eqnarray*}
    &  & \mathbb{L} (\widetilde{\boldsymbol{\theta}}  , \boldsymbol{\gamma } )
    {-\mathbb{L}}  \left( {\widetilde{\boldsymbol{\theta}} }^{\ast},
    \boldsymbol{\gamma}^{\ast} \right)\\
    & = & \log | \boldsymbol{C} (\widetilde{\boldsymbol{\theta}}) | - \log |
    \boldsymbol{C} (\widetilde{\boldsymbol{\theta} } ^{\ast}) | + \text{tr}
    (\boldsymbol{C}^{- 1} (\widetilde{\boldsymbol{\theta}}) \boldsymbol{C} 
    (\widetilde{\boldsymbol{\theta} } ^{\ast})) - N + (\boldsymbol{\gamma} -
    \boldsymbol{\gamma}^{\ast})^{\top} \mathbb{E} \boldsymbol{X}^{\top}
    \boldsymbol{C}^{- 1} (\widetilde{\boldsymbol{\theta}}) \boldsymbol{X}
    (\boldsymbol{\gamma} - \boldsymbol{\gamma}^{\ast})\\
    & \gtrsim & \| \boldsymbol{C}  (\widetilde{\boldsymbol{\theta}}) -
    \boldsymbol{C} (\widetilde{\boldsymbol{\theta} } ^{\ast}) \|_F^2 + N \|
    \boldsymbol{\gamma} - \boldsymbol{\gamma}^{\ast} \|^2\\
    & = & | \delta^{-1} -(\delta^{\ast})^{-1} |^2 N + [\delta^{-1} -
   (\delta^{\ast})^{-1}] \text{tr} (\boldsymbol{B} (\boldsymbol{\theta})
    \boldsymbol{K}  (\boldsymbol{\theta}) \boldsymbol{B}^{\top} (\boldsymbol{\theta})
    - \boldsymbol{B} (\boldsymbol{\theta}^{\ast}) \boldsymbol{K} 
    (\boldsymbol{\theta}^{\ast}) \boldsymbol{B}^{\top}
    (\boldsymbol{\theta}^{\ast}))\\
    &  & + \| \boldsymbol{B} (\boldsymbol{\theta}) \boldsymbol{K} 
    (\boldsymbol{\theta}) \boldsymbol{B}^{\top} (\boldsymbol{\theta}) - \boldsymbol{B}
    (\boldsymbol{\theta}^{\ast}) \boldsymbol{K}  (\boldsymbol{\theta}^{\ast})
    \boldsymbol{B}^{\top} (\boldsymbol{\theta}^{\ast}) \|_F^2 + N \|
    \boldsymbol{\gamma} - \boldsymbol{\gamma}^{\ast} \|^2\\
    & \gtrsim & (\delta-\delta^{\ast})^2 N + 1 (\| \boldsymbol{\theta} -
    \boldsymbol{\theta}^{\ast} \| > \xi/2) m + N \| \boldsymbol{\gamma} -
    \boldsymbol{\gamma}^{\ast} \|^2 - A| (\delta-\delta^{\ast})|m .
  \end{eqnarray*}
 where $A$ is some constant.  Let $\xi_1 =\min\{\sqrt{3}\xi/2,(2A)^{-1}\}$ , then $(\delta-\delta^{\ast})^2 + \| \boldsymbol{\gamma} - \boldsymbol{\gamma}^{\ast} \|^2 <
  \xi_1^2$ implies that $\| \boldsymbol{\theta} - \boldsymbol{\theta}^{\ast} \|
  > \xi/2$ and $A | \delta-\delta^{\ast} | < \frac{1}{2}$. Then, when
  $(\delta-\delta^{\ast})^2 + \| \boldsymbol{\gamma} - \boldsymbol{\gamma}^{\ast}
  \|^2 < \delta_1^2$, we have
  \begin{equation}
    \mathbb{L} (\widetilde{\boldsymbol{\theta}}  , \boldsymbol{\gamma } )
    {-\mathbb{L}}  \left( {\widetilde{\boldsymbol{\theta}} }^{\ast},
    \boldsymbol{\gamma}^{\ast} \right) \gtrsim [(\delta-\delta^{\ast})^2 + \|
    \boldsymbol{\gamma} - \boldsymbol{\gamma}^{\ast} \|^2] N + m.
  \end{equation}
  On the contrary, when $(\delta-\delta^{\ast})^2 + \| \boldsymbol{\gamma} -
  \boldsymbol{\gamma}^{\ast} \|^2 \geqslant \xi_1^2$, we have
  \begin{equation}
    \mathbb{L} (\widetilde{\boldsymbol{\theta}}  , \boldsymbol{\gamma } )
    {-\mathbb{L}}  \left( {\widetilde{\boldsymbol{\theta}} }^{\ast},
    \boldsymbol{\gamma}^{\ast} \right) \gtrsim N - A m .
  \end{equation}
  Let $N_k = m$ for $k = 1, \ldots, q$ and $N_{q + 1} = N$, with similar
  technique the the proof of Theorem \ref{thm:convexity}, it can be proved that, $k = 1, \ldots, q + 1, $ 
  \[ \sup_{\boldsymbol{\theta} \in \Theta, \delta \in \mathbb{S}}  \left| N_k^{-
     \frac{1}{2}} \text{tr} \left( \boldsymbol{C}^{- 1}
     (\widetilde{\boldsymbol{\theta}} ) \frac{\partial \boldsymbol{C} 
     (\widetilde{\boldsymbol{\theta}} )}{\partial \widetilde{\boldsymbol{\theta}}
     _k} \boldsymbol{} \boldsymbol{C}^{- 1} (\widetilde{\boldsymbol{\theta}} )
     [\boldsymbol{u} \boldsymbol{u}^{\top} - \boldsymbol{C}
     (\widetilde{\boldsymbol{\theta} } ^{\ast})] \right) \right| = o_{\mathbb{P}}
     (1), \]
  \[ \sup_{\boldsymbol{\theta} \in \Theta, \delta \in \mathbb{S},
     \boldsymbol{\alpha} \in \mathbb{R}^p : \| \boldsymbol{\alpha} \| = 1}  \left|
     N_k^{- \frac{1}{2}} \boldsymbol{u}^{\top} \left. \boldsymbol{C}^{- 1}
     (\widetilde{\boldsymbol{\theta}} ) \frac{\partial \boldsymbol{C} 
     (\widetilde{\boldsymbol{\theta}} )}{\partial \widetilde{\boldsymbol{\theta}}
     _k} \boldsymbol{} \boldsymbol{C}^{- 1} (\widetilde{\boldsymbol{\theta}} )
     \boldsymbol{X} \boldsymbol{\alpha} \right) \right| = o_{\mathbb{P}} (1), \]
  \[ \sup_{\boldsymbol{\theta} \in \Theta, \delta \in \mathbb{S},
     \boldsymbol{\alpha} \in \mathbb{R}^p : \| \boldsymbol{\alpha} \| = 1}  \left|
     N_k^{- \frac{1}{2}} \left( \boldsymbol{\alpha}^{\top} \boldsymbol{X}^{\top}
     \left. \boldsymbol{C}^{- 1} (\widetilde{\boldsymbol{\theta}} ) \frac{\partial
     \boldsymbol{C}  (\widetilde{\boldsymbol{\theta}} )}{\partial
     \widetilde{\boldsymbol{\theta}} _k} \boldsymbol{} \boldsymbol{C}^{- 1}
     (\widetilde{\boldsymbol{\theta}} ) \boldsymbol{X} \boldsymbol{\alpha}
     -\mathbb{E} \boldsymbol{\alpha}^{\top} \boldsymbol{X}^{\top}  \boldsymbol{C}^{-
     1} (\widetilde{\boldsymbol{\theta}} ) \frac{\partial \boldsymbol{C} 
     (\widetilde{\boldsymbol{\theta}} )}{\partial \widetilde{\boldsymbol{\theta}}
     _k} \boldsymbol{} \boldsymbol{C}^{- 1} (\widetilde{\boldsymbol{\theta}} )
     \boldsymbol{X} \boldsymbol{\alpha} \right) \right. \right| = o_{\mathbb{P}}
     (1), \]
  \[ \sup_{\boldsymbol{\theta} \in \Theta, \delta \in \mathbb{S},
     \boldsymbol{\alpha} \in \mathbb{R}^p : \| \boldsymbol{\alpha} \| = 1} \left|
     N^{- \frac{1}{2}} \boldsymbol{u}^{\top}  \boldsymbol{C}^{- 1}
     (\widetilde{\boldsymbol{\theta}}) \boldsymbol{X} \boldsymbol{\alpha} \right| =
     O_{\mathbb{P}} (1).\]
  Then, by the mean value theorem, there exists $\iota\in(0,1)$ with $\widetilde{\boldsymbol{\theta}}^{\iota}=\iota\widetilde{\boldsymbol{\theta}}+(1-\iota)\widetilde{\boldsymbol{\theta}}^\ast$,$\boldsymbol{\gamma}^{\iota}=\iota\boldsymbol{\gamma}+(1-\iota)\boldsymbol{\gamma}^\ast$ such that 
  \begin{eqnarray*}
    &  & \mathcal{L} (\widetilde{\boldsymbol{\theta}}, \boldsymbol{\gamma}) - \mathcal{L}
    \left( {\widetilde{\boldsymbol{\theta}} }^{\ast}, \boldsymbol{\gamma}^{\ast} \right) 
    - \left[ \mathbb{L} (\widetilde{\boldsymbol{\theta}}, \boldsymbol{\gamma}) 
    - \mathbb{L} \left( {\widetilde{\boldsymbol{\theta}} }^{\ast}, \boldsymbol{\gamma}^{\ast} \right) \right] \\
    & = & \sum_k \left[ \frac{\partial \mathcal{L} (\widetilde{\boldsymbol{\theta}}^{\iota}, \boldsymbol{\gamma}^{\iota})}{\partial \widetilde{\boldsymbol{\theta}}_k} 
    - \frac{\partial \mathbb{L} (\widetilde{\boldsymbol{\theta}}^{\iota}, \boldsymbol{\gamma}^{\iota})}{\partial \widetilde{\boldsymbol{\theta}}_k} \right] 
    (\widetilde{\boldsymbol{\theta}}_k - \widetilde{\boldsymbol{\theta}}^{\ast}_k) + \left[ \frac{\partial \mathcal{L} (\widetilde{\boldsymbol{\theta}}^{\iota}, \boldsymbol{\gamma}^{\iota})}{\partial \boldsymbol{\gamma}} 
    - \frac{\partial \mathbb{L} (\widetilde{\boldsymbol{\theta}}^{\iota}, \boldsymbol{\gamma}^{\iota})}{\partial \boldsymbol{\gamma}} \right]^{\top} 
    (\boldsymbol{\gamma} - \boldsymbol{\gamma}^{\ast}) \\
    & = & \sum_k \text{tr} \left( \boldsymbol{C}^{- 1} (\widetilde{\boldsymbol{\theta}}^{\iota}) 
    \frac{\partial \boldsymbol{C} (\widetilde{\boldsymbol{\theta}}^{\iota})}{\partial \widetilde{\boldsymbol{\theta}}_k} 
    \boldsymbol{C}^{- 1} (\widetilde{\boldsymbol{\theta}}^{\iota}) \left[ \boldsymbol{u} \boldsymbol{u}^{\top} 
    - \boldsymbol{C} (\widetilde{\boldsymbol{\theta}}^{\ast}) 
    + 2 \iota \boldsymbol{X} (\boldsymbol{\gamma}^{\ast} - \boldsymbol{\gamma}) \boldsymbol{u}^{\top} \right. \right. \\
    &  & \left. \left. + \iota \boldsymbol{X} (\boldsymbol{\gamma}^{\ast} - \boldsymbol{\gamma}) (\boldsymbol{\gamma}^{\ast} - \boldsymbol{\gamma})^{\top} \boldsymbol{X}^{\top} 
    - \iota \mathbb{E} \boldsymbol{X} (\boldsymbol{\gamma}^{\ast} - \boldsymbol{\gamma}) (\boldsymbol{\gamma}^{\ast} - \boldsymbol{\gamma})^{\top} \boldsymbol{X}^{\top} \right] \right) 
    (\widetilde{\boldsymbol{\theta}}_k - \widetilde{\boldsymbol{\theta}}^{\ast}_k) + \boldsymbol{u}^{\top} \boldsymbol{C}^{- 1} (\widetilde{\boldsymbol{\theta}}^{\iota}) \boldsymbol{X} (\boldsymbol{\gamma} - \boldsymbol{\gamma}^{\ast}) \\
    & = & \sqrt{N} (\delta - \delta^{\ast}) O_{\mathbb{P}} (1) 
    + m \sum_k (\boldsymbol{\theta}_k - \boldsymbol{\theta}_k^{\ast}) o_{\mathbb{P}} (1) 
    + \sqrt{N} \| \boldsymbol{\gamma} - \boldsymbol{\gamma}^{\ast} \| O_{\mathbb{P}} (1).
\end{eqnarray*}

  Combine the above results, if  $(\delta-\delta^{\ast})^2 + \|
  \boldsymbol{\gamma} - \boldsymbol{\gamma}^{\ast} \|^2 < \xi_1^2$,
  \begin{eqnarray*}
    &  & \mathcal{L} (\widetilde{\boldsymbol{\theta}}  , \boldsymbol{\gamma } ) - \mathcal{L}
    \left( {\widetilde{\boldsymbol{\theta}} }^{\ast}, \boldsymbol{\gamma}^{\ast}
    \right)\\
    & \gtrsim & [(\delta-\delta^{\ast})^2 + \| \boldsymbol{\gamma} -
    \boldsymbol{\gamma}^{\ast} \|^2] N + [|\delta-\delta^{\ast}| + \|
    \boldsymbol{\gamma} - \boldsymbol{\gamma}^{\ast} \|] O_{\mathbb{P}} \left(
    \sqrt{N} \right) + m (1 + o_{\mathbb{P}} (1))\\
    & \gtrsim & m (1 + o_{\mathbb{P}} (1))
  \end{eqnarray*}
  and if $(\delta-\delta^{\ast})^2 + \| \boldsymbol{\gamma} -
  \boldsymbol{\gamma}^{\ast} \|^2 \geqslant \xi_1^2$,
  \begin{eqnarray*}
    &  & \mathcal{L} (\widetilde{\boldsymbol{\theta}}  , \boldsymbol{\gamma } ) - \mathcal{L}
    \left( {\widetilde{\boldsymbol{\theta}} }^{\ast}, \boldsymbol{\gamma}^{\ast}
    \right)\\
    & \gtrsim & [(\delta-\delta^{\ast})^2 + \| \boldsymbol{\gamma} -
    \boldsymbol{\gamma}^{\ast} \|^2] N + [|\delta-\delta^{\ast}| + \|
    \boldsymbol{\gamma} - \boldsymbol{\gamma}^{\ast} \|] O_{\mathbb{P}} \left(
    \sqrt{N} \right) + m (1 + o_{\mathbb{P}} (1))\\
    & \gtrsim & N - O_{\mathbb{P}} \left( \sqrt{N} \right) - O (m)\\
    & = & N (1 + o_{\mathbb{P}} (1)) .
  \end{eqnarray*}
  Then, the proof is done.
\end{proof}

\begin{proof}[Proof of Theorem \ref{thm:asy_normality}]
  Let $\mathcal{L} (\boldsymbol{\theta} , \delta, \boldsymbol{\gamma}) = \log |
  \boldsymbol{C} (\boldsymbol{\theta} , \delta) | + (\boldsymbol{z} - \boldsymbol{X}
  \boldsymbol{\gamma})^{\top} \boldsymbol{C }^{- 1} (\boldsymbol{\theta} , \delta)
  (\boldsymbol{z} - \boldsymbol{X} \boldsymbol{\gamma}) $ where $\boldsymbol{C}
  (\boldsymbol{\theta} , \delta) = \delta^{-1} \boldsymbol{I} + \boldsymbol{B}
  (\boldsymbol{\theta}) \boldsymbol{K}  (\boldsymbol{\theta}) \boldsymbol{B}^{\top}
  (\boldsymbol{\theta})$.
  
  We first prove the asymptotic results of $\widehat{\boldsymbol{\gamma}}$ and
  $\widehat{\delta}$. According to the mean value theorem, there exists some
  $\iota > 0$ such that
  \begin{equation}\label{eq: mean_value}
    \boldsymbol{} \left(\begin{array}{c}
      \frac{\partial \mathcal{L} (\widehat{\boldsymbol{\theta}} \boldsymbol{} ,
      \widehat{\delta}, \widehat{\boldsymbol{\gamma}})}{\partial \boldsymbol{\gamma}}\\
      \frac{\partial \mathcal{L} (\widehat{\boldsymbol{\theta}} \boldsymbol{} ,
      \widehat{\delta}, \widehat{\boldsymbol{\gamma}})}{\partial \delta}
    \end{array}\right) = \left(\begin{array}{c}
      \frac{\partial \mathcal{L} (\widehat{\boldsymbol{\theta}}, \delta^{\ast},
      \boldsymbol{\gamma}^{\ast})}{\partial \boldsymbol{\gamma}}\\
      \frac{\partial \mathcal{L} (\widehat{\boldsymbol{\theta}}, \delta^{\ast},
      \boldsymbol{\gamma}^{\ast})}{\partial \delta}
    \end{array}\right) + \left(\begin{array}{cc}
      \frac{\partial^2 \mathcal{L} (\widehat{\boldsymbol{\theta}}, \delta^{\iota}  ,
      \boldsymbol\gamma^{\iota} )}{\partial \boldsymbol{\gamma} \partial
      \boldsymbol{\gamma}^{\top}} & \frac{\partial^2 \mathcal{L}
      (\widehat{\boldsymbol{\theta}}, \delta^{\iota}  ,
      \boldsymbol\gamma^{\iota} )}{\partial \boldsymbol{\gamma} \partial \delta}\\
      \frac{\partial^2 \mathcal{L} (\widehat{\boldsymbol{\theta}}, \delta^{\iota}  ,
      \boldsymbol\gamma^{\iota} )}{\partial \delta \partial
      \boldsymbol{\gamma}^{\top}} & \frac{\partial^2 \mathcal{L}
      (\widehat{\boldsymbol{\theta}}, \delta^{\iota}  ,
      \boldsymbol\gamma^{\iota} )}{\partial \delta \partial \delta ^{\top}}
    \end{array}\right) \left(\begin{array}{c}
      \widehat{\boldsymbol{\gamma}} \boldsymbol{}  - \boldsymbol{\gamma}^{\ast}\\
      \widehat{\delta} - \delta^{\ast}
    \end{array}\right)
  \end{equation}
  where $\boldsymbol\gamma^{\iota}  = \iota \widehat{\boldsymbol{\gamma}} \boldsymbol{}
  + (1 - \iota) \boldsymbol{\gamma}^{\ast}$ and $\delta^{\iota}   = \iota
  \widehat{\delta} + (1 - \iota)\delta^{\ast}$. By the consistency of
  $\widehat{\boldsymbol{\theta}} \boldsymbol{} $, $\widehat{\boldsymbol{\gamma}}$,
  $\widehat{\delta}$ and the proof of Theorem \ref{thm:convexity},
  \begin{equation}\label{eq:thm3_1}
    N^{- \frac{1}{2}} \left(\begin{array}{c}
      \frac{\partial \mathcal{L} (\widehat{\boldsymbol{\theta}}, \delta^{\ast},
      \boldsymbol{\gamma}^{\ast})}{\partial \boldsymbol{\gamma}}\\
      \frac{\partial \mathcal{L} (\widehat{\boldsymbol{\theta}}, \delta^{\ast},
      \boldsymbol{\gamma}^{\ast})}{\partial \delta}
    \end{array}\right) = N^{- \frac{1}{2}} \left(\begin{array}{c}
      \frac{\partial \mathcal{L} \left( {\boldsymbol{\theta}^{\ast}} , \delta^{\ast},
      \boldsymbol{\gamma}^{\ast} \right)}{\partial \boldsymbol{\gamma}}\\
      \frac{\partial \mathcal{L} \left( {\boldsymbol{\theta}^{\ast}} , \delta^{\ast},
      \boldsymbol{\gamma}^{\ast} \right)}{\partial \delta}
    \end{array}\right) + \boldsymbol{r}
  \end{equation}
  
  \begin{equation}\label{eq:thm3_2}
    \frac{1}{N} \left(\begin{array}{cc}
      \frac{\partial^2 \mathcal{L} (\widehat{\boldsymbol{\theta}}, \delta^{\iota}  ,
      \boldsymbol\gamma^{\iota} )}{\partial \boldsymbol{\gamma} \partial
      \boldsymbol{\gamma}^{\top}} & \frac{\partial^2 \mathcal{L}
      (\widehat{\boldsymbol{\theta}}, \delta^{\iota}  ,
      \boldsymbol\gamma^{\iota} )}{\partial \boldsymbol{\gamma} \partial \delta}\\
      \frac{\partial^2 \mathcal{L} ({\widehat{\boldsymbol{\theta}}},
      \delta^{\iota}  , \boldsymbol\gamma^{\iota} )}{\partial \delta \partial
      \boldsymbol{\gamma}^{\top}} & \frac{\partial^2 \mathcal{L}
      (\widehat{\boldsymbol{\theta}}, \delta^{\iota}  ,
      \boldsymbol\gamma^{\iota} )}{\partial \delta \partial \delta ^{\top}}
    \end{array}\right) = \left(\begin{array}{cc}
      \frac{2}{N} \mathbb{E} \boldsymbol{X}^{\top} \boldsymbol{C}^{-1}  \left(
      {\boldsymbol{\theta}^{\ast}} , \delta^{\ast} \right) \boldsymbol{X} &
      \boldsymbol{0}_p\\
      \boldsymbol{0}_p^{\top} & \frac{1}{N} (\delta^{\ast})^{-4} \text{tr} \left[
      \boldsymbol{C}^{- 1} \left( {\boldsymbol{\theta}^{\ast}} , \delta^{\ast}
      \right) \boldsymbol{C}^{- 1} \left( {\boldsymbol{\theta}^{\ast}} ,
      \delta^{\ast} \right) \right]
    \end{array}\right) + \boldsymbol{R}
  \end{equation}
  where each element of vector $\boldsymbol{r}$ and matrix $\boldsymbol{R}$ is
  $o_{\mathbb{P}} (1)$. Note that $\boldsymbol{z} - \boldsymbol{X}
  \boldsymbol{\gamma}^{\ast} = \boldsymbol{C }^{\frac{1}{2}}
  (\boldsymbol{\theta}^{\ast}, \delta^{\ast}) \boldsymbol{e}$ where $\boldsymbol{e}$
  is a standard Gaussian vector, by Equations \eqref{eq: mean_value}, \eqref{eq:thm3_1}, \eqref{eq:thm3_2}, we have
  \begin{equation}
      \sqrt{N} \boldsymbol{V}_{\boldsymbol{} \boldsymbol{\gamma}}^{\frac{1}{2}}
      (\widehat{\boldsymbol{\gamma}} \boldsymbol{}  - \boldsymbol{\gamma}^{\ast})=N^{- \frac{1}{2}} \boldsymbol{V}_{\boldsymbol{} \boldsymbol{\gamma}}^{- \frac{1}{2}}
      \boldsymbol{X}^{\top} \boldsymbol{C }^{- \frac{1}{2}}
      (\boldsymbol{\theta}^{\ast}, \delta^{\ast}) \boldsymbol{e}+o_{\mathbb{P}} (1) \rightsquigarrow  \mathcal{N}_{p } (\boldsymbol{0}_{p },
    \boldsymbol{I}_{p \times p}),
  \end{equation}
  where the asymptotic normality is derived through the conditional central limit theorem \citep{bulinski2017conditional}, and 
  \begin{equation}
      \sqrt{N} v_{\delta}^{\frac{1}{2}} (\widehat{\delta} - \delta^{\ast})=N^{- \frac{1}{2}} v_{\delta}^{-\frac{1}{2}} (\delta^{\ast})^{-2} \text{tr}
      [\boldsymbol{C }^{- 1} (\boldsymbol{\theta}^{\ast}, \delta^{\ast})
      (\boldsymbol{e} \boldsymbol{e}^{\top} - \boldsymbol{I}) ]
     + o_{\mathbb{P}} (1) \rightsquigarrow  \mathcal{N} (0,1).
  \end{equation}
  We then prove the asymptotic results of $\widehat{\boldsymbol{\theta}}$.
  Similarly, we have
  \begin{equation}\label{eq:thm3_3}
    \frac{\partial \mathcal{L} (\widehat{\boldsymbol{\theta}} \boldsymbol{} ,
    \widehat{\delta}, \widehat{\boldsymbol{\gamma}})}{\partial \boldsymbol{\theta}}
    \boldsymbol{} = \frac{\partial \mathcal{L} (\boldsymbol{}  \boldsymbol{\theta}^{\ast},
    \widehat{\delta}, \widehat{\boldsymbol{\gamma}})}{\partial \boldsymbol{\theta}} +
    \frac{\partial \mathcal{L} (\boldsymbol{}  \boldsymbol{\theta}^{\ast}, \widehat{\delta},
    \widehat{\boldsymbol{\gamma}})}{\partial \boldsymbol{\theta} \partial
    \boldsymbol{\theta}^{\top}} (\widehat{\boldsymbol{\theta}} \boldsymbol{}  -
    \boldsymbol{}  \boldsymbol{\theta}^{\ast})
  \end{equation}
  \begin{equation}\label{eq:thm3_4}
    m^{- 1 / 2} \frac{\partial \mathcal{L} (\boldsymbol{}  \boldsymbol{\theta}^{\ast},
    \widehat{\delta}, \widehat{\boldsymbol{\gamma}})}{\partial \boldsymbol{\theta}} =
    m^{- 1 / 2} \frac{\partial \mathcal{L} \left( {\boldsymbol{\theta}^{\ast}} ,
    \delta^{\ast}, \boldsymbol{\gamma}^{\ast} \right)}{\partial
    \boldsymbol{\theta}} + o_{\mathbb{P}} (1)
  \end{equation}
  \begin{equation}\label{eq:thm3_5}
    \frac{1}{m} \frac{\partial^2 \mathcal{L} (\boldsymbol{}  \boldsymbol{\theta}^{\ast},
    \widehat{\delta}, \widehat{\boldsymbol{\gamma}})}{\partial \boldsymbol{\theta}
    \partial \boldsymbol{\theta}^{\top}} = \left( \frac{1}{m} \text{tr} \left(
    \boldsymbol{C}^{- 1} (\boldsymbol{\theta}^{\ast}, \delta^{\ast})
    \frac{\partial \boldsymbol{C} (\boldsymbol{\theta}^{\ast},
    \delta^{\ast})}{\partial \boldsymbol{\theta}_k} \boldsymbol{C}^{- 1}
    (\boldsymbol{\theta}^{\ast}, \delta^{\ast}) \frac{\partial \boldsymbol{C}
    (\boldsymbol{\theta}^{\ast}, \delta^{\ast})}{\partial \boldsymbol{\theta}_l }
    \right)+o_{\mathbb{P}} (1) \right)_{k l}
  \end{equation}
  Combining \eqref{eq:thm3_3},\eqref{eq:thm3_4},\eqref{eq:thm3_5}, we have
  \[ \sqrt{m} \boldsymbol{V}_{\boldsymbol{\theta}}^{\frac{1}{2}}
     (\widehat{\boldsymbol{\theta}} \boldsymbol{}  - \boldsymbol{} 
     \boldsymbol{\theta}^{\ast}) \rightsquigarrow \mathcal{N}_q (\boldsymbol{0}_q,
     \boldsymbol{I}_q) . \]
  
\end{proof}

\section{Addition Conditions for Theorem \ref{thm:asy_normality} and Its Proof}\label{sec:proof_2}
Without loss of generality, we assume equal local sample sizes, such that $n=N/J$. Our assumptions and proof, however, can be readily extended to cases with unequal sample sizes.
\begin{assumption}\label{as:weights}
The weights $w_{i j}$ are positive and bounded.
\end{assumption}

\begin{assumption}\label{as:rate}
$m=o(n), \log J = o (n^{1 / 6})$.
\end{assumption}

\begin{assumption}\label{as:localB}
$\| \boldsymbol{B}_{j, [k_1]} (\boldsymbol{\theta})
\|_{\tmop{op}} = O (1)$ for all $j=1,\ldots, J, k_1=1,\ldots,q$ where $\boldsymbol{B}_{j, [k_1]}$ represents the partial derivative with respect to $k_1$-th element of  $\boldsymbol{\theta}$.  There exist some positive constants $\underline{\lambda}_{\boldsymbol{B}}$ and $\underline{\lambda}_{\boldsymbol{K}}$, such that $\lambda_{\min} (\boldsymbol{B} ^{\top} (\boldsymbol{\theta}) \boldsymbol{B}  
(\boldsymbol{\theta}))\ge \underline{\lambda}_{\boldsymbol{B}}$, $\lambda_{\min} ( \boldsymbol{K} (\boldsymbol{\theta}))\ge \underline{\lambda}_{\boldsymbol{K}}$. 
\end{assumption}

For $t \geqslant 0$, recall that $\overline{\boldsymbol{\Sigma} }^t$,
$\overline{\boldsymbol{\mu} }^t$, $\overline{\boldsymbol{\gamma} }^t$,
$\overline{\delta}^t$, $\overline{\boldsymbol{\theta}} \boldsymbol{}^t$ are
the averages. For convenience, let $\overline{\boldsymbol{\mu} }^{t + 1, \ast}$,
$\overline{\boldsymbol{\Sigma} }^{t + 1, \ast}$ be the optimizer of $f \left(
\boldsymbol{\mu}, \boldsymbol{\Sigma}, \overline{\boldsymbol{\gamma} }^t,
\overline{\delta}^t, \overline{\boldsymbol{\theta} }^t \right)$,
$\overline{\boldsymbol{\gamma} }^{t + 1, \ast}$ be the optimizer of $f \left(
\overline{\boldsymbol{\mu} }^{t + 1, \ast}, \overline{\boldsymbol{\Sigma} }^{t +
1, \ast}, \boldsymbol{\gamma}, \overline{\delta}^t, \overline{\boldsymbol{\theta}
}^t \right)$,$\overline{\delta}^{t + 1, \ast}$ be the optimizer of $f \left(
\overline{\boldsymbol{\mu} }^{t + 1, \ast}, \overline{\boldsymbol{\Sigma} }^{t +
1, \ast}, \overline{\boldsymbol{\gamma} }^{t + 1, \ast}, \delta,
\overline{\boldsymbol{\theta} }^t \right)$, $\overline{\boldsymbol{\theta} }^{t +
1, \ast}$ be the optimizer of $f \left( \overline{\boldsymbol{\mu} }^{t + 1,
\ast}, \overline{\boldsymbol{\Sigma} }^{t + 1, \ast}, \overline{\boldsymbol{\gamma}
}^{t + 1, \ast}, \overline{\delta}^{t + 1, \ast}, \boldsymbol{\theta} \right)$,
$\overline{\boldsymbol{\mu} }^{t + 1, \star}$, $\overline{\boldsymbol{\Sigma} }^{t
+ 1, \star}$ be the optimizer of $f \left( \boldsymbol{\mu}, \boldsymbol{\Sigma},
\overline{\boldsymbol{\gamma} }^{t + 1, \ast}, \overline{\delta}^{t + 1, \ast},
\overline{\boldsymbol{\theta} }^{t + 1, \ast} \right)$.

\begin{assumption}\label{as:neighbour}
The parameters sequence $\boldsymbol{\gamma}^t_j, \delta^t_j,
\boldsymbol{\theta}^t_j$, $\overline{\boldsymbol{\gamma} }^{t + 1, \ast}$,
$\overline{\delta}^{t + 1, \ast}, \overline{\boldsymbol{\theta} }^{t + 1, \ast}$
satisfies, for any given positive $\epsilon$,  the following inequality holds with probability larger than $1 -
\epsilon'$ with $\epsilon'=\epsilon/4$ : $\| \boldsymbol{\gamma}^t_j -
\boldsymbol{\gamma}^{\ast} \| \leqslant \xi_{\epsilon'} / 3, \| \delta^t_j -
\delta^{\ast} \| \leqslant \xi_{\epsilon'} / 3, \| \boldsymbol{\theta}^t_j -
\boldsymbol{\theta}^{\ast} \| \leqslant \xi_{\epsilon'} / 3$, $\left\|
\overline{\boldsymbol{\gamma} }^{t + 1, \ast} - \boldsymbol{\gamma}^{\ast} \right\|
\leqslant \xi_{\epsilon'} / 3, \left\| \overline{\delta}^{t + 1, \ast} -
\delta^{\ast} \right\| \leqslant \xi_{\epsilon'} / 3$, $\left\|
\overline{\boldsymbol{\theta} }^{t + 1, \ast} - \boldsymbol{\theta}^{\ast}
\right\| \leqslant \xi_{\epsilon'} / 3$. 
\end{assumption}
\begin{assumption}\label{as:subproblem}
The sub-problem about $\boldsymbol{\theta}$ is optimized within the precision
$O_{\mathbb{P}} (N^{- 1})$.
\end{assumption}

\begin{remark}  
Assumptions \ref{as:weights}–\ref{as:localB} are relatively mild. Since the initial estimators, being  local minimizers, are consistent, it is reasonable to assume that Assumption \ref{as:neighbour} holds. Since the function with respect to $\boldsymbol{\theta}$ is locally strongly convex, the Newton-Raphson algorithm converges. Consequently, the precision required for Assumption \ref{as:subproblem} can be achieved within a few iterations.

\end{remark}

\begin{proof}[Proof of Theorem \ref{thm:asy_normality}]
The framework of the proof is outlined as follows. Let  
\[
\mathcal{L} \left( \overline{\boldsymbol{\gamma}}^t, \overline{\delta}^t, \overline{\boldsymbol{\theta}}^t \right) - \mathcal{L} \left( \overline{\boldsymbol{\gamma}}^{t+1}, \overline{\delta}^{t+1}, \overline{\boldsymbol{\theta}}^{t+1} \right) = \tmop{gap}_t - \tmop{consensus}_t,
\]
where the terms \(\tmop{gap}_t\) and \(\tmop{consensus}_t\) are defined as follows:
\[
\tmop{gap}_t := f \left( \overline{\boldsymbol{\mu}}^{t+1, \ast}, \overline{\boldsymbol{\Sigma}}^{t+1, \ast}, \overline{\boldsymbol{\gamma}}^t, \overline{\delta}^t, \overline{\boldsymbol{\theta}}^t \right) - f \left( \overline{\boldsymbol{\mu}}^{t+1, \star}, \overline{\boldsymbol{\Sigma}}^{t+1, \star}, \overline{\boldsymbol{\gamma}}^{t+1, \ast}, \overline{\delta}^{t+1, \ast}, \overline{\boldsymbol{\theta}}^{t+1, \ast} \right),
\]
\[
\tmop{consensus}_t := f \left( \overline{\boldsymbol{\mu}}^{t+1, \star}, \overline{\boldsymbol{\Sigma}}^{t+1, \star}, \overline{\boldsymbol{\gamma}}^{t+1, \ast}, \overline{\delta}^{t+1, \ast}, \overline{\boldsymbol{\theta}}^{t+1, \ast} \right) - f \left( \overline{\boldsymbol{\mu}}^{t+2, \ast}, \overline{\boldsymbol{\Sigma}}^{t+2, \ast}, \overline{\boldsymbol{\gamma}}^{t+1}, \overline{\delta}^{t+1}, \overline{\boldsymbol{\theta}}^{t+1} \right).
\]

We demonstrate that for any small positive number $\epsilon$, there exists some constant \(0 < \mu_{\epsilon} < 1\) independent of \(t\),  such that
\[
\tmop{gap}_t > \mu_{\epsilon} \left[ \mathcal{L} \left( \overline{\boldsymbol{\gamma}}^t, \overline{\delta}^t, \overline{\boldsymbol{\theta}}^t \right) - \widehat{\mathcal{L}} \right],
\]
with probability larger than $\epsilon$. Here, \(\widehat{\mathcal{L}}\) is the minimal value of the function. 

Additionally, if \(K\) is sufficiently large, then
\[
\tmop{consensus}_t \leqslant O_{\mathbb{P}} \left( \frac{1}{N} \right),
\]
where the bound in probability is uniformly for $t$.

Finally, the results above establish the validity of the conclusion in Theorem \ref{thm:asy_normality}.

We bound $\tmop{gap}_t$ and $\tmop{consensus}_t$ in the following.
\subsection*{Part 1: Bound $\tmop{gap}_t$}
In the following, we let 
\[
f^t = f \left( \overline{\boldsymbol{\mu}}^{t + 1, \ast}, \overline{\boldsymbol{\Sigma}}^{t + 1, \ast}, \overline{\boldsymbol{\gamma}}^t, \overline{\delta}^t, \overline{\boldsymbol{\theta}}^t \right),
\]
\[
f^t_{\delta} = f \left( \overline{\boldsymbol{\mu}}^{t + 1, \ast}, \overline{\boldsymbol{\Sigma}}^{t + 1, \ast}, \overline{\boldsymbol{\gamma}}^{t + 1, \ast}, \overline{\delta}^t, \overline{\boldsymbol{\theta}}^t \right),
\]
\[
f^t_{\boldsymbol{\theta}} = f \left( \overline{\boldsymbol{\mu}}^{t + 1, \ast}, \overline{\boldsymbol{\Sigma}}^{t + 1, \ast}, \overline{\boldsymbol{\gamma}}^{t + 1, \ast}, \overline{\delta}^{t + 1, \ast}, \overline{\boldsymbol{\theta}}^t \right),
\]
\[
f^t_{\boldsymbol{\mu}, \boldsymbol{\Sigma}} = f \left( \overline{\boldsymbol{\mu}}^{t + 1, \ast}, \overline{\boldsymbol{\Sigma}}^{t + 1, \ast}, \overline{\boldsymbol{\gamma}}^{t + 1, \ast}, \overline{\delta}^{t + 1, \ast}, \overline{\boldsymbol{\theta}}^{t + 1, \ast} \right),
\]
and let \(\nabla_{\boldsymbol{\gamma}} f^t\), \(\nabla_{\delta} f^t_{\delta}\), \(\nabla_{\boldsymbol{\theta}} f^t_{\boldsymbol{\theta}}\), \(\nabla_{\boldsymbol{\mu}, \boldsymbol{\Sigma}} f^t_{\boldsymbol{\mu}, \boldsymbol{\Sigma}}\) denote the gradients at the corresponding points. According to the block descent lemmas, e.g., Lemma \ref{le:block1} and Lemma \ref{le:block2} , we have
\begin{equation}\label{eq:gap1}
  \tmop{gap}_t \geqslant \frac{1}{L_{\boldsymbol{\gamma}}} \| \nabla_{\boldsymbol{\gamma}} f^t \|^2 + \frac{1}{\kappa_{\delta} L_{\delta}} \| \nabla_{\delta} f^t_{\delta} \|^2 + \frac{1}{\kappa_{\boldsymbol{\theta}}L_{\boldsymbol{\theta}}} \| \nabla_{\boldsymbol{\theta}} f^t_{\boldsymbol{\theta}} \|^2 \geqslant
  \frac{1}{L_{\max}} \left( \| \nabla_{\boldsymbol{\gamma}} f^t \|^2 + \| \nabla_{\delta} f^t_{\delta} \|^2 + \| \nabla_{\boldsymbol{\theta}} f^t_{\boldsymbol{\theta}} \|^2 \right),
\end{equation}
where \(L_{\boldsymbol{\gamma}}\), \(L_{\delta}\), \(L_{\boldsymbol{\theta}}\) are the corresponding Lipschitz constants, \(\kappa_{\delta}\) is the corresponding condition number, and \(L_{\max} := \max \{ L_{\boldsymbol{\gamma}}, \kappa_{\delta}L_{\delta}, \kappa_{\boldsymbol{\theta}}L_{\boldsymbol{\theta}} \}\).

The bound for \(L_{\boldsymbol{\gamma}}\), \(L_{\delta}\), \(L_{\boldsymbol{\theta}}\), \(\kappa_{\delta}\), $ \kappa_{\boldsymbol{\theta}}$ are given as follows. 

First  of all, it can be easily shown that
\begin{equation}\label{eq:L_mu_gamma_delta}
    \begin{aligned}
  L_{\boldsymbol{\gamma}} \leqslant \left( \sup_{\delta \in \mathbb{S}} \delta \right) \lambda_{\max} (\boldsymbol{X}^{\top} \boldsymbol{X}),\quad & \mu_{\boldsymbol{\gamma}} \geqslant \left( \inf_{\delta \in \mathbb{S}} \delta \right) \lambda_{\min} (\boldsymbol{X}^{\top} \boldsymbol{X}), \\
  L_{\delta} \leqslant \frac{N}{2} \left( \inf_{\delta \in \mathbb{S}} \delta \right)^{-2},\quad & \mu_{\delta} \geqslant \frac{N}{2} \left( \sup_{\delta \in \mathbb{S}} \delta \right)^{-2}.
\end{aligned}
\end{equation}
where \( \mu_{\boldsymbol{\gamma}}, \mu_{\delta} \) are the corresponding strongly convex parameters. 

Second, the bound for \( L_{\boldsymbol{\theta}} \) and \( \mu_{\boldsymbol{\theta}} \) can be derived as follows.

Let $\boldsymbol{C}^t  (\boldsymbol{\theta}) := \overline{\delta}^{t + 1, \ast}
\boldsymbol{B}^\top (\boldsymbol{\theta}) \boldsymbol{B} (\boldsymbol{\theta}) +
\boldsymbol{K}^{-1} (\boldsymbol{\theta})$ and 
$\boldsymbol{C}^t_{[k_1 k_2]} (\boldsymbol{\theta}) := \frac{\partial^2 \boldsymbol{C}^t (\boldsymbol{\theta})}{\partial \boldsymbol{\theta}_{k_1} \partial \boldsymbol{\theta}_{k_2}}$, 
$\boldsymbol{B}_{[k_1 k_2]} (\boldsymbol{\theta}) := \frac{\partial^2 \boldsymbol{B} (\boldsymbol{\theta})}{\partial \boldsymbol{\theta}_{k_1} \partial \boldsymbol{\theta}_{k_2}}$, 
$\boldsymbol{H}^t (\boldsymbol{\theta}) = (h^t_{k_1 k_2} (\boldsymbol{\theta}))$ with
\[
h^t_{k_1 k_2} (\boldsymbol{\theta}) = \frac{1}{2} \tmop{tr} \left\{
   \boldsymbol{C}^t_{[k_1 k_2]} (\boldsymbol{\theta}) \overline{\boldsymbol{\Sigma}}^{t + 1, \ast} \right\} 
   - \overline{\delta}^{t + 1, \ast} \left( \boldsymbol{z} - \boldsymbol{X} \overline{\boldsymbol{\gamma}}^{t + 1, \ast} \right)^{\top} {\boldsymbol{B}_{[k_1 k_2]}} (\boldsymbol{\theta})
   \overline{\boldsymbol{\mu}}^{t + 1, \ast} + \frac{1}{2} \ln \det \left[ \boldsymbol{K} (\boldsymbol{\theta}) \right],
\]
then
\[
L_{\boldsymbol{\theta}} = \sup_{t,\boldsymbol{\theta}} \| \boldsymbol{H}^t(\boldsymbol{\theta}) \|_{\tmop{op}}.
\]
Note that
\[
\overline{\boldsymbol{\Sigma}}^{t + 1, \ast} = \left[ \overline{\delta}^t \boldsymbol{B}^\top \left( \overline{\boldsymbol{\theta}}^t \right)
\boldsymbol{B} \left( \overline{\boldsymbol{\theta}}^t \right) + \boldsymbol{K}^{-1} \left( \overline{\boldsymbol{\theta}}^t \right) \right]^{-1}.
\]
and
\[
\overline{\boldsymbol{\mu}}^{t + 1, \ast} = \overline{\boldsymbol{\Sigma}}^{t + 1, \ast} \overline{\delta}^t \left( \boldsymbol{z} - \boldsymbol{X} \overline{\boldsymbol{\gamma}}^t \right)^{\top}.
\]
Then, with the similar technique as the proof of Theorem \ref{thm:convexity}, we have
\begin{equation}\label{eq:L_theta}
   L_{\boldsymbol{\theta}} = O_{\mathbb{P}} (m).
\end{equation}

As for $\mu_{\boldsymbol{\theta}}$, since $\mu_{\boldsymbol{\theta}} \geqslant \mu_{\mathcal{L}}$ and Assumption \ref{as:neighbour}, then,  for any small positive $\epsilon$, 
\begin{equation}\label{eq:mu_theta}
    \mu_{\boldsymbol{\theta}} \geqslant \mu_{\mathcal{L}} > \frac{1}{2} m \lambda^{\ast},
\end{equation}
with probability larger than $1-\epsilon/2$.

By Lemma \ref{le:gradbound},
\[ \| \nabla_{\boldsymbol{}} f^t_{\delta} - \nabla_{\boldsymbol{}} f^t  \| 
   {\leqslant L_{\boldsymbol{\gamma}}}  \left\| \overline{\boldsymbol{\gamma} }^{t +
   1, \ast} - \overline{\boldsymbol{\gamma} }^t \right\|  {\leqslant
   \kappa_{\boldsymbol{\gamma}}}  \left\| {\nabla_{\boldsymbol{\gamma}}}  f^t
   \right\|. \]
Similarly,
\[ \| \nabla_{\boldsymbol{}} f^t_{\boldsymbol{\theta}} - \nabla_{\boldsymbol{}}
   f^t_{\delta} \| {\leqslant \kappa_{\delta}}  \left\| {\nabla_{\delta}} 
   f^t_{\delta} \right\|, \| \nabla_{\boldsymbol{\mu}, \boldsymbol{\Sigma}} f^t -
   \nabla_{\boldsymbol{\theta}} f^t \| {\leqslant \kappa_{\boldsymbol{\theta}}} 
   \left\| {\nabla_{\boldsymbol{\theta}}}  f^t_{\boldsymbol{\theta}} \right\|. \]
Thus,
\[ \| [\nabla_{\boldsymbol{}} f^t ]_{\delta} \|^2 \leqslant 2 \|
   [\nabla_{\boldsymbol{}} f^t - \nabla_{\boldsymbol{}} f^t_{\delta}]_{\delta}
   \|^2 + 2 \| [\nabla_{\boldsymbol{}} f^t_{\delta}]_{\delta} \|^2 {\leqslant 2
   \kappa_{\boldsymbol{\gamma}}}  \left\| {\nabla_{\boldsymbol{\gamma}}}  f^t
   \right\| + 2 \| \nabla_{\delta} f^t_{\delta} \|, \]
\begin{equation*}
    \begin{aligned}
  \| [\nabla_{\boldsymbol{\theta}} f^t ]_{\boldsymbol{\theta}} \|^2 
  &\leqslant 3 \| [\nabla_{\boldsymbol{\theta}} f^t - \nabla_{\boldsymbol{\theta}} f^t_{\delta}]_{\boldsymbol{\theta}} \|^2 
+ 3 \| [\nabla_{\boldsymbol{\theta}} f^t_{\delta} - \nabla_{\boldsymbol{\theta}} f^t_{\boldsymbol{\theta}}]_{\boldsymbol{\theta}} \|^2 + 3 \left\| \left[ \nabla_{\boldsymbol{\theta}} f^t_{\boldsymbol{\theta}} \right]_{\boldsymbol{\theta}} \right\|^2 \\
  &\leqslant 3 \left( \kappa_{\boldsymbol{\gamma}} \| \nabla_{\boldsymbol{\gamma}} f^t \|^2 + \kappa_{\delta} \| \nabla_{\delta} f^t_{\delta} \|^2 \right. \left. + \| \nabla_{\boldsymbol{\theta} \boldsymbol{\theta}} f^t_{\boldsymbol{\theta}} \|^2 \right),
\end{aligned}
\end{equation*}
which implies that (note that $[\nabla_{\boldsymbol{}} f^t ]_{\boldsymbol{\mu},
\boldsymbol{\Sigma}} = \boldsymbol{0}$), with $\kappa_{\max} = \max \left\{
{\kappa_{\boldsymbol{\theta}}} , \kappa_{\boldsymbol{\gamma}} {{, \kappa_{\sigma}}
}  \right\}$
\begin{equation}\label{eq:gap2}
  \| \nabla_{\boldsymbol{}} f^t  \|^2 \leqslant 6 \kappa_{\max} (\|
  \nabla_{\boldsymbol{\gamma}} f^t \|^2 + \| \nabla_{\delta} f^t_{\delta} \|^2 +
  \| \nabla_{\boldsymbol{\theta}} f^t_{\boldsymbol{\theta}} \|^2) .
\end{equation}
Combining Equations \eqref{eq:gap1} and \eqref{eq:gap2}, we have
\[ \tmop{Gap}_t \geqslant \frac{1}{6 \kappa_{\max} L_{\max}} \|
   \nabla_{\boldsymbol{}} f^t  \|^2  \text{\ recall $L_{\max}:=\max \{
   L_{\boldsymbol{\gamma}}, \kappa_{\delta}L_{\delta}, \kappa_{\boldsymbol{\theta}}L_{\boldsymbol{\theta}} \}$} . \]
Since ${{\nabla_{\boldsymbol{\mu}, \boldsymbol{\Sigma}}} }_{\boldsymbol{}} f^t  =
0$, $\| \nabla_{\boldsymbol{}} f^t  \| = \left\| \nabla \mathcal{L} \left(
\overline{\boldsymbol{\gamma} }^t, \overline{\delta}^t,
\overline{\boldsymbol{\theta} }^t \right) \right\|$. By the strong convexity of
$\mathcal{L}$, we have
\[ \left\| \nabla \mathcal{L} \left( \overline{\boldsymbol{\gamma} }^t,
   \overline{\delta}^t, \overline{\boldsymbol{\theta} }^t \right) \right\|^2 >
   \mu_{\mathcal{L}} \left[ \mathcal{L} \left( \overline{\boldsymbol{\gamma} }^t,
   \overline{\delta}^t, \overline{\boldsymbol{\theta} }^t \right) -
   \mathcal{L}^{\ast} \right] \]
Therefore,
\[ \tmop{gap}_t \geqslant \frac{\mu_{\mathcal{L}}}{6 \kappa_{\max} L_{\max}}
   \left( \mathcal{L} \left( \overline{\boldsymbol{\gamma} }^t,
   \overline{\delta}^t, \overline{\boldsymbol{\theta} }^t \right) -
   \mathcal{L}^{\ast} \right)  \]
which, in combination with Equations \eqref{eq:L_mu_gamma_delta}, \eqref{eq:L_theta}, \eqref{eq:mu_theta}, implies that for any small positive number $\epsilon$, there exists some constant \(0 < \mu_{\epsilon} < 1\) independent of \(t\) and constant integer $M_{\epsilon}$,  such that for $m>M_{\epsilon}$,
\[
\tmop{gap}_t > \mu_{\epsilon} \left[ \mathcal{L} \left( \overline{\boldsymbol{\gamma}}^t, \overline{\delta}^t, \overline{\boldsymbol{\theta}}^t \right) - \widehat{\mathcal{L}} \right],
\]
with probability larger than $\epsilon$.

\subsection*{Part 2: Bound $\tmop{consensus}_t$}

In the following,  $\mathbb{C}$ is a random variable that satistfies $\mathbb{C}=
O_{\mathbb{P}} (1)$.  Recall that, $\tmop{consensus}_t = f \left( \overline{\boldsymbol{\mu} }^{t + 1,
\star}, \overline{\boldsymbol{\Sigma} }^{t + 1, \star},
\overline{\boldsymbol{\gamma} }^{t + 1, \ast}, \overline{\delta}^{t + 1, \ast},
\overline{\boldsymbol{\theta} }^{t + 1, \ast} \right) - f \left(
\overline{\boldsymbol{\mu} }^{t + 2, \ast}, \overline{\boldsymbol{\Sigma} }^{t +
2, \ast}, \overline{\boldsymbol{\gamma} }^{t + 1}, \overline{\delta}^{t + 1},
\overline{\boldsymbol{\theta} }^{t + 1} \right)$. Also note that $\tmop{consensus}_t$ is also equals to $\left| \mathcal{L} \left(
\overline{\boldsymbol{\gamma} }^{t + 1, \ast}, \overline{\delta}^{t + 1, \ast},
\overline{\boldsymbol{\theta} }^{t + 1, \ast} \right) - \mathcal{L} \left(
\overline{\boldsymbol{\gamma} }^{t + 1}, \overline{\delta}^{t + 1},
\overline{\boldsymbol{\theta} }^{t + 1} \right) \right|$, then, according to
Lemma \ref{le:funbound},
\[ | \tmop{consensus}_t | \leqslant \left[ L_{\mathcal{L}} E_{t + 1}  +
   \left\| \nabla \mathcal{L} \left( \overline{\boldsymbol{\gamma} }^{t + 1},
   \overline{\delta}^{t + 1}, \overline{\boldsymbol{\theta} }^{t + 1} \right)
   \right\| \right] E_{t + 1}  \]
where $L_{\mathcal{L}}$ is the Lipschitz constant of the gradient of
$\mathcal{L}$ and $E_{t + 1} $ is defined by $E_{t + 1}^2 = \left\|
\overline{\boldsymbol{\gamma} }^{t + 1, \ast} - \overline{\boldsymbol{\gamma} }^{t +
1} \right\|^2 + \left\| \overline{\delta}^{t + 1, \ast} - \overline{\delta}^{t
+ 1} \right\|^2 + \left\| \overline{\boldsymbol{\theta} }^{t + 1, \ast} -
\overline{\boldsymbol{\theta} }^{t + 1} \right\|^2$. With the similar proof technique as the proof of Theorem \ref{thm:convexity}, we have $L_{\mathcal{L}}=O_{\mathbb{P}}(N)$. Thus, we only need to bound the terms in $E_{t+1}$, which is given in the following.

\subsubsection*{Part 2.1: Bound $\| \overline{\boldsymbol{\gamma} }^{t + 1, \ast} -
\overline{\boldsymbol{\gamma} }^{t + 1} \| $.}

Since 
\[
\overline{\boldsymbol{\gamma}}^{t + 1, \ast} = \left( \frac{1}{J} \sum_i \boldsymbol{X}_i^{\top} \boldsymbol{X}_i \right)^{-1} 
\left( \frac{1}{J} \sum_i \left( \boldsymbol{X}_i^{\top} \boldsymbol{z}_i - \boldsymbol{X}_i^{\top} 
\boldsymbol{B}_i \left( \overline{\boldsymbol{\theta}}^t \right) \overline{\boldsymbol{\mu}}^{t + 1, \ast} \right) \right),
\]
\[
\overline{\boldsymbol{\gamma}}^{t + 1} = \frac{1}{J} \sum_i (\boldsymbol{y}_{\boldsymbol{X}, i}^t)^{-1} \boldsymbol{y}_{\boldsymbol{\gamma}, i}^t,
\]
where 
\[
\overline{\boldsymbol{\Sigma}}^{t + 1, \ast} = \left[ \overline{\delta}^t J J^{-1} \sum_{i = 1}^J \boldsymbol{B}_i \left( \overline{\boldsymbol{\theta}}^t \right)^{\top} \boldsymbol{B}_i \left( \overline{\boldsymbol{\theta}}^t \right) + \boldsymbol{K}^{-1} \left( \overline{\boldsymbol{\theta}}^t \right) \right]^{-1},
\]
\[
\overline{\boldsymbol{\mu}}^{t + 1, \ast} = \overline{\boldsymbol{\Sigma}}^{t + 1, \ast} \left( J \overline{\delta}^t J^{-1} \sum_{i = 1}^J \boldsymbol{B}_i \left( \overline{\boldsymbol{\theta}}^t \right)^{\top} \left( \boldsymbol{z}_i - \boldsymbol{X}_i \overline{\boldsymbol{\gamma}}^t \right) \right),
\]
we have
\begin{eqnarray}
  \overline{\boldsymbol{\gamma}}^{t + 1} - \overline{\boldsymbol{\gamma}}^{t + 1, \ast} 
  & = & \frac{1}{J} \sum_i \left[ (\boldsymbol{y}_{\boldsymbol{X}, i}^t)^{-1} 
  - \left( \overline{\boldsymbol{y}}_{\boldsymbol{X}} \right)^{-1} \right]
  \boldsymbol{y}_{\boldsymbol{\gamma}, i}^t \nonumber \\
  & & + \left( \overline{\boldsymbol{y}}_{\boldsymbol{X}} \right)^{-1} \left[
  \overline{\boldsymbol{y}}_{\boldsymbol{\gamma}}^t 
  - \frac{1}{J} \sum_i \left( \boldsymbol{X}_i^{\top} \boldsymbol{z}_i 
  - \boldsymbol{X}_i^{\top} \boldsymbol{B}_i \left( \overline{\boldsymbol{\theta}}^t \right)  
  \overline{\boldsymbol{\mu}}^{t + 1, \ast} \right) \right] \label{eq:gamma}
\end{eqnarray}
where $\overline{\boldsymbol{y}}_{\boldsymbol{X}} = \frac{1}{J} \sum_i
\boldsymbol{y}_{\boldsymbol{X}, i}^t = \frac{1}{J} \sum_i \boldsymbol{X}_i^{\top}
\boldsymbol{X}_i$ and $\overline{\boldsymbol{y}}_{\boldsymbol{\gamma}}^t =
\frac{1}{J} \sum_i \boldsymbol{y}_{\boldsymbol{\gamma}, i}^t = \frac{1}{J} \sum_i
\left( \boldsymbol{X}_i^{\top} \boldsymbol{z}_i - \boldsymbol{X}_i^{\top} \boldsymbol{B}_i
(\boldsymbol{\theta}_i^t) \boldsymbol{\mu}^{t + 1, \ast} \right)$.

We now bound the terms in Equation \eqref{eq:supgamma} in the following.
\paragraph{Part 2.1.1: Bound $\frac{1}{J}  \sum_i \left[ (\boldsymbol{y}_{\boldsymbol{X},
i}^t)^{- 1} - \left( \overline{\boldsymbol{y}}_{\boldsymbol{X}}  \right)^{- 1}
\right] \boldsymbol{y}_{\boldsymbol{\gamma}, i}^t$}
\

According to Lemma \ref{le:xx_concentration},  we have
\[ \sup_i \| n^{- 1} (\boldsymbol{X}_i^{\top} \boldsymbol{X}_i
   -\mathbb{E}\boldsymbol{X}_i^{\top} \boldsymbol{X}_i) \|_{\tmop{op}} =
   o_{\mathbb{P}} (1),  \] 
which implies that 
\[ n^{- 1} \boldsymbol{X}_i^{\top} \boldsymbol{X}_i = n^{- 1}
   \mathbb{E}\boldsymbol{X}_i^{\top} \boldsymbol{X}_i + n^{- 1}
   (\boldsymbol{X}_i^{\top} \boldsymbol{X}_i -\mathbb{E}\boldsymbol{X}_i^{\top}
   \boldsymbol{X}_i)\succ \underline{\lambda}_{\boldsymbol{X}}\boldsymbol{I}+o_{\mathbb{P}} (1),  \]
Thus, 
\begin{equation}
  \sum_i \left\| (n^{- 1} \boldsymbol{y}_{\boldsymbol{X}, i}^t)^{- 1} - \left(
  n^{- 1} \overline{\boldsymbol{y}}_{\boldsymbol{X}}  \right)^{- 1}
  \right\|_{\tmop{op}}^2 \leqslant \left( \underline{\lambda}_{\boldsymbol{X}} +
  o_{\mathbb{P}} (1) \right)^{- 4} \sum_i n^{- 2} \left\|
  \boldsymbol{y}_{\boldsymbol{X}, i}^t - \overline{\boldsymbol{y}}_{\boldsymbol{X}} 
  \right\|_{\tmop{op}}^2 .
\end{equation}
By Cauchy-Swartz inequality,
\begin{eqnarray*}
  &&\left\| \frac{1}{J} \sum_i \left[ (\boldsymbol{y}_{\boldsymbol{X}, i}^t)^{- 1} -
  \left( \overline{\boldsymbol{y}}_{\boldsymbol{X}}  \right)^{- 1} \right]
  \boldsymbol{y}_{\boldsymbol{\gamma}, i}^t \right\|_2 \\& \leqslant & \frac{1}{J}
  \sqrt{\sum_i \left\| (n^{- 1} \boldsymbol{y}_{\boldsymbol{X}, i}^t)^{- 1} -
  \left( n^{- 1} \overline{\boldsymbol{y}}_{\boldsymbol{X}}  \right)^{- 1}
  \right\|_{\tmop{op}}^2}  \sqrt{\sum_i \| n^{- 1}
  \boldsymbol{y}_{\boldsymbol{\gamma}, i}^t \|^2_2}\\
  & \leqslant & \left( \underline{\lambda}_{\boldsymbol{X}} + o_{\mathbb{P}}
  (1) \right)^{- 2} \sqrt{\frac{1}{J} \sum_i n^{- 2} \left\|
  \boldsymbol{y}_{\boldsymbol{X}, i}^t - \overline{\boldsymbol{y}}_{\boldsymbol{X}} 
  \right\|_{\tmop{op}}^2}   \sqrt{\frac{1}{J} \sum_i \| n^{- 1}
  \boldsymbol{y}_{\boldsymbol{\gamma}, i}^t \|^2_2}.
\end{eqnarray*}
Let $\boldsymbol{Y}_{\boldsymbol{X}}^t, \overline{\boldsymbol{Y}}_{\boldsymbol{X}} \in \mathbb{R}^{p \times p \times J}$ be the stacks of $\boldsymbol{y}_{\boldsymbol{X}, 1}^t, \ldots, \boldsymbol{y}_{\boldsymbol{X}, J}^t$ and $\overline{\boldsymbol{y}}_{\boldsymbol{X}}^t, \ldots, \overline{\boldsymbol{y}}_{\boldsymbol{X}}^t$, respectively, such that
\[
\boldsymbol{Y}_{\boldsymbol{X}}^t[:, :, i] = \boldsymbol{y}_{\boldsymbol{X}, i}^t, \quad \overline{\boldsymbol{Y}}_{\boldsymbol{X}}[:, :, i] = \overline{\boldsymbol{y}}_{\boldsymbol{X}}.
\]
Also, let $\widetilde{\boldsymbol{W}} = \boldsymbol{W} - \frac{1}{J} \boldsymbol{1} \boldsymbol{1}^{\top} \in \mathbb{R}^{J \times J}$, then
\[
\boldsymbol{Y}_{\boldsymbol{X}}^t - \overline{\boldsymbol{Y}}_{\boldsymbol{X}} = \left( \boldsymbol{Y}_{\boldsymbol{X}}^{t - 1} - \overline{\boldsymbol{Y}}_{\boldsymbol{X}} \right) \boldsymbol{\widetilde{\boldsymbol{W}}}^K = \left( \boldsymbol{Y}_{\boldsymbol{X}}^0 - \overline{\boldsymbol{Y}}_{\boldsymbol{X}} \right) \boldsymbol{\widetilde{\boldsymbol{W}}}^{Kt},
\]
which implies that
\[
\left\| \boldsymbol{Y}_{\boldsymbol{X}}^t - \overline{\boldsymbol{Y}}_{\boldsymbol{X}} \right\|_{\tmop{op}, 2} \leqslant \left\| \boldsymbol{Y}_{\boldsymbol{X}}^0 - \overline{\boldsymbol{Y}}_{\boldsymbol{X}} \right\|_{\tmop{op}, 2} \| \boldsymbol{\widetilde{\boldsymbol{W}}}^{Kt} \|_{2, 2} \leqslant \sqrt{J} \left\| \boldsymbol{Y}_{\boldsymbol{X}}^0 - \overline{\boldsymbol{Y}}_{\boldsymbol{X}} \right\|_2 \rho_{\boldsymbol{W}}^{Kt},
\]
where the second inequality follows from Lemma \ref{le:tensor_op}. In addition,
\[
\left\| \boldsymbol{Y}_{\boldsymbol{X}}^0 - \overline{\boldsymbol{Y}}_{\boldsymbol{X}} \right\|_{\tmop{op}, 2} \leqslant \sqrt{nJ \log J } \mathbb{C}.
\]
Therefore, we have
\begin{equation}
  \left\| \frac{1}{J} \sum_i \left[ (\boldsymbol{y}_{\boldsymbol{X}, i}^t)^{- 1} - \left( \overline{\boldsymbol{y}}_{\boldsymbol{X}}^t \right)^{- 1} \right] \boldsymbol{y}_{\boldsymbol{\gamma}, i}^t \right\|_2 \leqslant 2 \left( \underline{\lambda}_{\boldsymbol{X}} + o_{\mathbb{P}}(1) \right)^{- 2} \sqrt{\frac{\log J}{n}} \sqrt{\frac{1}{J} \sum_i \| n^{-1} \boldsymbol{y}_{\boldsymbol{\gamma}, i}^t \|_2^2} \rho_{\boldsymbol{W}}^{Kt} \mathbb{C}.
  \label{yx-mean}
\end{equation}

We then bound $\sqrt{J^{- 1} \sum_i \| n^{- 1} \boldsymbol{y}_{\boldsymbol{\gamma},
i}^t \|^2_2}$.

Let $\boldsymbol{Y}_{\boldsymbol{\gamma}}^t = [\boldsymbol{y}_{\boldsymbol{\gamma}, 1}^t, \ldots, \boldsymbol{y}_{\boldsymbol{\gamma}, J}^t]$, 
$\overline{\boldsymbol{Y}}^t_{\boldsymbol{\gamma}} = \left[\overline{\boldsymbol{y}}_{\boldsymbol{\gamma}}^t, \ldots, \overline{\boldsymbol{y}}_{\boldsymbol{\gamma}}^t \right] \in \mathbb{R}^{p \times J}$ 
with 
\[
\overline{\boldsymbol{y}}_{\boldsymbol{\gamma}}^t = \frac{1}{J} \sum_i \boldsymbol{y}_{\boldsymbol{\gamma}, i}^t = \frac{1}{J} \sum_i \boldsymbol{X}_i^{\top} \boldsymbol{z}_i - \boldsymbol{X}_i^{\top} \boldsymbol{B}_i (\boldsymbol{\theta}^t_i) {\boldsymbol{\mu}_i^{t + 1}}.
\]
Then, by the triangle inequality, we have
\[
\sqrt{\sum_i \| n^{-1} \boldsymbol{y}_{\boldsymbol{\gamma}, i}^t \|_2^2} = n^{-1} \| \boldsymbol{Y}_{\boldsymbol{\gamma}}^t \|_F \leqslant n^{-1} \left\| \overline{\boldsymbol{Y}}^t_{\boldsymbol{\gamma}} \right\|_F + n^{-1} \left\| \boldsymbol{Y}_{\boldsymbol{\gamma}}^t - \overline{\boldsymbol{Y}}^t_{\boldsymbol{\gamma}} \right\|_F,
\]
where $\| \cdot \|_F$ represents the Frobenius norm.

Further let $\boldsymbol{E}_{\boldsymbol{\gamma}}^t = [- \boldsymbol{X}_1^{\top}
\boldsymbol{B}_1 (\boldsymbol{\theta}^t_1)   \boldsymbol{\mu}_1^{t + 1}, \ldots, -
\boldsymbol{X}_J^{\top} \boldsymbol{B}_J (\boldsymbol{\theta}^t_J)  
\boldsymbol{\mu}_J^{t + 1}]$ and recall $\widetilde{\boldsymbol{W}} = \boldsymbol{W}
- \frac{1}{J} \boldsymbol{1} \boldsymbol{1}^{\top} \in \mathbb{R}^{J \times J}$,
then
\[ \boldsymbol{Y}_{\boldsymbol{\gamma}}^t -
   \overline{\boldsymbol{Y}}^t_{\boldsymbol{\gamma}} = \left[
   \boldsymbol{Y}_{\boldsymbol{\gamma}}^{t - 1} - \overline{\boldsymbol{Y}}^{t -
   1}_{\boldsymbol{\gamma}} \right] \widetilde{\boldsymbol{W}}^K +
   [\boldsymbol{E}_{\boldsymbol{\gamma}}^t - \boldsymbol{E}_{\boldsymbol{\gamma}}^{t -
   1}] \widetilde{\boldsymbol{W}}^K, \]
which implies that
\[ \boldsymbol{Y}_{\boldsymbol{\gamma}}^t -
   \overline{\boldsymbol{Y}}^t_{\boldsymbol{\gamma}} = \left[
   \boldsymbol{Y}_{\boldsymbol{\gamma}}^0 -
   \overline{\boldsymbol{Y}}^0_{\boldsymbol{\gamma}} \right]
   \widetilde{\boldsymbol{W}}^{Kt} + \sum_{s = 0}^{t - 1}
   [\boldsymbol{E}_{\boldsymbol{\gamma}}^{t - s} -
   \boldsymbol{E}_{\boldsymbol{\gamma}}^{t - s - 1}]
   \widetilde{\boldsymbol{W}}^{K(s + 1)}, \]
Therefore,
\[ \left\| \boldsymbol{Y}_{\boldsymbol{\gamma}}^t -
   \overline{\boldsymbol{Y}}^t_{\boldsymbol{\gamma}} \right\|_F \leqslant \left\|
   \boldsymbol{Y}_{\boldsymbol{\gamma}}^0 -
   \overline{\boldsymbol{Y}}^0_{\boldsymbol{\gamma}} \right\|_F
   \rho^{Kt}_{\boldsymbol{W}} + \sum_{s = 0}^{t - 1} \|
   \boldsymbol{E}_{\boldsymbol{\gamma}}^{t - s} -
   \boldsymbol{E}_{\boldsymbol{\gamma}}^{t - s - 1} \|_F \rho^{K(s +
   1)}_{\boldsymbol{W}} . \]
According to Lemma \ref{le:xbepsilon_bound}, we have
\[
\| \boldsymbol{X}_j^{\top} \boldsymbol{B}_j (\boldsymbol{\theta}^t_j) \boldsymbol{\mu}_j^{t + 1} \|_2 \leqslant \mathbb{C}_0,
\]
where 
\[
\mathbb{C}_0 := \left[ m J \log (J) + m J^{7 / 2} \log (J) \left( \rho^{Kt}_{\boldsymbol{W}} + \frac{\rho_{\boldsymbol{W}}^K}{1 - \rho_{\boldsymbol{W}}^K} \right) \right] \mathbb{C}.
\]
Then, we have the following bounds:
\[
\| \boldsymbol{E}_{\boldsymbol{\gamma}}^t \|_F \leqslant \sqrt{J} \mathbb{C}_0, \quad
\| \boldsymbol{Y}_{\boldsymbol{\gamma}}^0 \|_F \leqslant \sqrt{J} \mathbb{C}_0, \quad
\left\| \overline{\boldsymbol{Y}}^0_{\boldsymbol{\gamma}} \right\|_F \leqslant \sqrt{J} \mathbb{C}_0,
\]
and it follows that
\[
\left\| \boldsymbol{Y}_{\boldsymbol{\gamma}}^t - \overline{\boldsymbol{Y}}^t_{\boldsymbol{\gamma}} \right\|_F \leqslant \sqrt{J} \mathbb{C}_0 \left( 2 \rho^{Kt}_{\boldsymbol{W}} + \frac{\rho_{\boldsymbol{W}}^K}{1 - \rho_{\boldsymbol{W}}^K} \right).
\]
Moreover, we have
\[
\left\| \overline{\boldsymbol{Y}}^t_{\boldsymbol{\gamma}} \right\|_F \leqslant \sqrt{J} \mathbb{C}_0.
\]
Thus,
\[ \| \boldsymbol{Y}_{\boldsymbol{\gamma}}^t \|_F \leqslant \sqrt{J} \mathbb{C}_0
   \left( 2 \rho^{Kt}_{\boldsymbol{W}} + \frac{\rho^K_{\boldsymbol{W}}}{1 - \rho^K_{\boldsymbol{W}}} \right) + \sqrt{J} \mathbb{C}_0 . \]
This means that
\begin{equation}
  \sqrt{J^{- 1} \sum_i \| n^{- 1} \boldsymbol{y}_{\boldsymbol{\gamma}, i}^t \|^2_2}
  \leqslant n^{- 1} \mathbb{C}_0 \left( 2 \rho^{Kt}_{\boldsymbol{W}} + \frac{\rho^K_{\boldsymbol{W}}}{1 - \rho^K_{\boldsymbol{W}}} \right) + n^{- 1} \mathbb{C}_0.
  \label{eq:mean_y_gamma}
\end{equation}

\paragraph*{Part 2.1.1:  $\overline{\boldsymbol{y}}_{\boldsymbol{\gamma}}^t -
\frac{1}{J}  \sum_i \left( \boldsymbol{X}_i^{\top} \boldsymbol{z}_i
-\boldsymbol{X}_i^{\top} \boldsymbol{B}_i \left( \overline{\boldsymbol{\theta} }^t
\right)  \overline{\boldsymbol{\mu} }^{t + 1, \ast} \right)$}

\

Now we focus on 
\[
\overline{\boldsymbol{y}}_{\boldsymbol{\gamma}}^t - \frac{1}{J} 
\sum_i \left( \boldsymbol{X}_i^{\top} \boldsymbol{z}_i - \boldsymbol{X}_i^{\top}
\boldsymbol{B}_i \left( \overline{\boldsymbol{\theta} }^t \right) 
\overline{\boldsymbol{\mu} }^{t + 1, \ast} \right),
\]
since it equals
\[
\frac{1}{J}  \sum_i \left( \boldsymbol{X}_i^{\top} \boldsymbol{B}_i
(\boldsymbol{\theta}^t_i) \boldsymbol{\mu}_j^{t + 1} - \boldsymbol{X}_i^{\top}
\boldsymbol{B}_i \left( \overline{\boldsymbol{\theta} }^t \right) 
\overline{\boldsymbol{\mu} }^{t + 1, \ast} \right),
\]
whose norm is bounded by
\begin{equation}
  \left\| \frac{1}{J}  \sum_i \boldsymbol{X}_i^{\top} \left[ \boldsymbol{B}_i
  (\boldsymbol{\theta}^t_i)  - \boldsymbol{B}_i \left( \overline{\boldsymbol{\theta}
  }^t \right) \right] \boldsymbol{\mu}_i^{t + 1} \right\|_2 + \left\|
  \frac{1}{J}  \sum \boldsymbol{X}_i^{\top} \boldsymbol{B}_i \left(
  \overline{\boldsymbol{\theta} }^t \right) \left( \boldsymbol{\mu}_i^{t + 1} -
  \overline{\boldsymbol{\mu} }^{t + 1, \ast} \right) \right\|_2.
  \label{eq:y_gamma-star}
\end{equation}
For the first term in Equation \eqref{eq:y_gamma-star}, we have
\begin{eqnarray*}
  &  & \left\| \frac{1}{J}  \sum_i \boldsymbol{X}_i^{\top} \left[
  \boldsymbol{B}_i (\boldsymbol{\theta}^t_i)  -\boldsymbol{B}_i \left(
  \overline{\boldsymbol{\theta} }^t \right) \right] \boldsymbol{\mu}_i^{t + 1}
  \right\|_2\\
  & \leqslant^{(i)} & \frac{1}{J}  \sum_i \left\| \boldsymbol{X}_i^{\top}
  \left[ \boldsymbol{B}_{i k} (\boldsymbol{\theta}^t_i)  -\boldsymbol{B}_i \left(
  \overline{\boldsymbol{\theta} }^t \right) \right] \boldsymbol{\mu}_i^{t + 1}
  \right\|_2\\
  & =^{(i i)} & \frac{1}{J}  \sum_i \left\| \sum_{k_1} \boldsymbol{X}_i^{\top}
  \boldsymbol{B}_{i, k_1} (\widetilde{\boldsymbol{\theta}}^t_i)  
  \boldsymbol{\mu}_i^{t + 1} \left[ \boldsymbol{\theta}^t_i -
  \overline{\boldsymbol{\theta} }^t \right]_{k_1} \right\|_2\\
  & \leqslant^{(i i i)} & \frac{1}{J}  \sum_i \left\| \boldsymbol{\theta}^t_i -
  \overline{\boldsymbol{\theta} }^t \right\|_2 \left( \sum_{k_1} \|
  \boldsymbol{X}_i^{\top} \boldsymbol{B}_{i, k_1}
  (\widetilde{\boldsymbol{\theta}}^t_i)   \boldsymbol{\mu}_i^{t + 1} \|_2 \right),
\end{eqnarray*}
where $(i)$ is due to the triangle inequality, $(i i)$ is due to the
mean value theorem and $\widetilde{\boldsymbol{\theta}}^t_i = \iota
\boldsymbol{\theta}^t_i + (1 - \iota) \overline{\boldsymbol{\theta} }^t$ for some
$\iota \in (0, 1)$, $(i i i)$ is due to the Cauchy-Swartchz inequality.
According to {{Lemma \ref{le:xbepsilon_bound} }} that $\sum_{k_1} \| \boldsymbol{X}_i^{\top}
\boldsymbol{B}_{i, k_1} (\widetilde{\boldsymbol{\theta}}^t_i)  
\boldsymbol{\mu}_i^{t + 1} \|_2 \leqslant \mathbb{C}_0$ , we
furthe\textbf{}r have
\begin{equation}
  \left\| \frac{1}{J}  \sum_i \boldsymbol{X}_i^{\top} \left[ \boldsymbol{B}_i
  (\boldsymbol{\theta}^t_i)  -\boldsymbol{B}_i \left( \overline{\boldsymbol{\theta}
  }^t \right) \right] \boldsymbol{\mu}_i^{t + 1} \right\|_2 \leqslant
  \frac{1}{J}  \sum_i \left\| \boldsymbol{\theta}^t_i -
  \overline{\boldsymbol{\theta} }^t \right\|_2 \times q\mathbb{C}_0 .
  \label{eq:Btheta-Bthetamean}
\end{equation}
For the second term in \eqref{eq:y_gamma-star}, we have {{Lemma \ref{le:xb_bound}}} that 
\[
\sup_{j, \boldsymbol{\theta}} \| \boldsymbol{X}_j^{\top} \boldsymbol{B}_j (\boldsymbol{\theta}) \|_F = O_{\mathbb{P}} \left( \sqrt{m \log (J)} \right),
\]
so we only need to bound 
\[
\left\| \boldsymbol{\mu}_j^{t + 1} - \overline{\boldsymbol{\mu}}^{t + 1, \ast} \right\|.
\]
Recall that 
\[
\boldsymbol{\mu}_j^{t + 1} = \boldsymbol{\Sigma}_j^{t + 1} J \delta_j^t \boldsymbol{y}_{\boldsymbol{\mu}, j}^t, \quad \overline{\boldsymbol{\mu}}^{t + 1, \ast} = \overline{\boldsymbol{\Sigma}}^{t + 1, \ast} \left( J \overline{\delta}^t J^{-1} \sum_{i = 1}^J \boldsymbol{B}_i \left( \overline{\boldsymbol{\theta}}^t \right)^{\top} \left( \boldsymbol{z}_i - \boldsymbol{X}_i \overline{\boldsymbol{\gamma}}^t \right) \right),
\]
with 
\[
\boldsymbol{\Sigma}_j^{t + 1} = \left[ \delta_j^t J [\boldsymbol{y}_{\boldsymbol{\Sigma}, j}^t]_+ + \sum_i w_{ij} \boldsymbol{K}^{-1} (\boldsymbol{\theta}_i^t) \right]^{-1}, \quad \overline{\boldsymbol{\Sigma}}^{t + 1, \ast} = \left[ \overline{\delta}^t J J^{-1} \sum_{i = 1}^J \boldsymbol{B}_i \left( \overline{\boldsymbol{\theta}}^t \right)^{\top} \boldsymbol{B}_i \left( \overline{\boldsymbol{\theta}}^t \right) + \boldsymbol{K}^{-1} \left( \overline{\boldsymbol{\theta}}^t \right) \right]^{-1}.
\]
Let
\[
\widetilde{\boldsymbol{\mu}}_j^{t + 1} = \left[ \delta_j^t J \overline{\boldsymbol{y}}_{\boldsymbol{\Sigma}}^t + \sum_i w_{ij} \boldsymbol{K}^{-1} (\boldsymbol{\theta}_i^t) \right]^{-1} J \delta_j^t \overline{\boldsymbol{y}}_{\boldsymbol{\mu}}^t,
\]
then
\[ \left\| \boldsymbol{\mu}_j^{t + 1} - \overline{\boldsymbol{\mu} }^{t + 1,
   \ast} \right\| \leqslant \| \boldsymbol{\mu}_j^{t + 1} -
   \widetilde{\boldsymbol{\mu} }_j^{t + 1} \| + \left\|
   \widetilde{\boldsymbol{\mu} }_j^{t + 1} - \overline{\boldsymbol{\mu} }^{t + 1,
   \ast} \right\| .\]
According to the proof of {{Lemma \ref{le:xbepsilon_bound} }},
\[ \| \boldsymbol{\mu}_j^{t + 1} - \widetilde{\boldsymbol{\mu} }_j^{t + 1} \|
   \leqslant \mathbb{C}_1 \]
with some random variable $\mathbb{C}_1 = \overline{\lambda}_{\boldsymbol{K}} m
J^{7 / 2} \log (J) \left( \rho^{Kt}_{\boldsymbol{W}} + \frac{\rho^K_{\boldsymbol{W}}}{1 - \rho^K_{\boldsymbol{W}}} \right) \mathbb{C}$. By
mean value theorem,
\[ \left\| \widetilde{\boldsymbol{\mu} }_j^{t + 1} - \overline{\boldsymbol{\mu}
   }^{t + 1, \ast} \right\| \leqslant \left( \left| \delta^t_j -
   \overline{\delta }^t \right| + \frac{1}{J}  \sum_i \left\|
   \boldsymbol{\theta}^t_i - \overline{\boldsymbol{\theta} }^t \right\|_2 +
   \frac{1}{J}  \sum_i \left\| \boldsymbol{\gamma}^t_i -
   \overline{\boldsymbol{\gamma} }^t \right\|_2 \right) \sqrt{m \log (J)}
   \mathbb{C}.\]
Combining the above three equations, we have
\begin{equation}
  \left\| \boldsymbol{\mu}_j^{t + 1} - \overline{\boldsymbol{\mu} }^{t + 1, \ast}
  \right\| \leqslant \mathbb{C}_1 + \left( \left| \delta^t_j -
  \overline{\delta }^t \right| + \frac{1}{J}  \sum_i \left\|
  \boldsymbol{\theta}^t_i - \overline{\boldsymbol{\theta} }^t \right\|_2 +
  \frac{1}{J}  \sum_i \left\| \boldsymbol{\gamma}^t_i -
  \overline{\boldsymbol{\gamma} }^t \right\|_2 \right) \sqrt{m \log (J)}
  \mathbb{C}.\label{eq:mu-mu_start}
\end{equation}

\paragraph*{Part 2.1.3: Combining results}
\

Combining Equations \eqref{yx-mean}, \eqref{eq:mean_y_gamma},
\eqref{eq:y_gamma-star}, \eqref{eq:Btheta-Bthetamean}, \eqref{eq:mu-mu_start},
we have
\begin{equation*}
    \begin{aligned}
        \left\| \overline{\boldsymbol{\gamma} }^{t + 1, \ast} -
   \overline{\boldsymbol{\gamma} }^{t + 1} \right\| \leqslant &\sqrt{\frac{\log
   J}{n^3}} \left( 2 \rho^{Kt}_{\boldsymbol{W}} + \frac{\rho^K_{\boldsymbol{W}}}{1 -
   \rho^K_{\boldsymbol{W}}} + 1 \right) \mathbb{C}_0 \rho_{\boldsymbol{W}}^{Kt} +
   \frac{\mathbb{C}_0}{n} \frac{1}{J}  \sum_i \left\| \boldsymbol{\theta}^t_i -
   \overline{\boldsymbol{\theta} }^t \right\|_2 + \frac{\mathbb{C}_1}{n} \sqrt{m
   \log (J)}\\
   &+ \frac{1}{J}  \sum_i \left( \left| \delta^t_i - \overline{\delta
   }^t \right| +  \left\| \boldsymbol{\theta}^t_i -
   \overline{\boldsymbol{\theta} }^t \right\|_2 +  \left\|
   \boldsymbol{\gamma}^t_i - \overline{\boldsymbol{\gamma} }^t \right\|_2 \right)
   \frac{m \log (J) \mathbb{C}}{n}.
    \end{aligned}
\end{equation*}
Therefore, there exists some positive constant $C$ such that if $K \geqslant C (\log N ) / \log (1 / \rho_{\boldsymbol{W}})$, then
\begin{equation}
  \left\| \overline{\boldsymbol{\gamma} }^{t + 1, \ast} -
  \overline{\boldsymbol{\gamma} }^{t + 1} \right\| \leqslant \frac{1}{J}  \sum_i
  \left( \left\| \boldsymbol{\gamma}^t_i - \overline{\boldsymbol{\gamma} }^t
  \right\|_2 + \left| \delta^t_j - \overline{\delta }^t \right| + \left\|
  \boldsymbol{\theta}^t_i - \overline{\boldsymbol{\theta} }^t \right\|_2 \right) \mathbb{C}
  + \frac{1}{N} \mathbb{C} .
  \label{eq:supgamma}
\end{equation}

\subsubsection*{Part 2.2: Bound $\left| \overline{\delta }^{t + 1, \ast} -
\overline{\delta}^{t + 1} \right|$}

Recall that
\begin{equation*}
    \overline{\delta }^{t + 1, \ast} = n \left[ J^{- 1} \sum_i l_i
\left( \overline{\boldsymbol{\mu} }^{t + 1, \ast}, \overline{\boldsymbol{\Sigma}
}^{t + 1, \ast}, \overline{\boldsymbol{\gamma} }^{t + 1, \ast} \boldsymbol{},
\overline{\boldsymbol{\theta} }^t \right) \right]^{- 1},
\end{equation*}
and 
\begin{equation*}
    \overline{\delta}^{t + 1} = J^{- 1} \sum_j \delta^{t + 1}_j \text{\ with\ } \delta^{t +
1}_j = y_{n, j}^t (y_{\delta, j}^t)^{- 1} = n (y_{\delta, j}^t)^{- 1}.
\end{equation*}
We let
$\delta^{t + 1}  = n (\overline{y}^t_{\delta})^{- 1}$ with
$\overline{y}^t_{\delta} = J^{- 1} \sum_i y_{\delta, j}^t = J^{- 1} \sum_i l_i
(\boldsymbol{\mu}_i^{t + 1}, \boldsymbol{\Sigma}_i^{t + 1}, \boldsymbol{\gamma}_i^{t
+ 1},   \boldsymbol{\theta}_i^t)$. Then,
\[ \left| \overline{\delta}^{t + 1} - \overline{\delta }^{t + 1, \ast}
   \right| < \left| \delta^{t + 1}  - \overline{\delta }^{t + 1, \ast} \right|
   + \left| \overline{\delta}^{t + 1} - \delta^{t + 1}  \right| . \]
In this following, we bound the terms in the above equation.
\paragraph*{Part 2.2.1: Bound $\left| \delta^{t + 1}  - \overline{\delta }^{t + 1,
\ast} \right|$}

\

 We use the following decomposition to give the bound:
\begin{eqnarray*}
  &  & \left| \overline{\delta}^{t + 1} - \delta^{t + 1}  \right|\\
  & = & \left[ n^{- 1} J^{- 1} \sum_i l_i \left( \overline{\boldsymbol{\mu}
  }^{t + 1, \ast}, \overline{\boldsymbol{\Sigma} }^{t + 1, \ast},
  \overline{\boldsymbol{\gamma} }^{t + 1, \ast} \boldsymbol{},
  \overline{\boldsymbol{\theta} }^t \right) - l_i (\boldsymbol{\mu}_i^{t + 1},
  \boldsymbol{\Sigma}_i^{t + 1}, \boldsymbol{\gamma}_i^{t + 1}, 
   \boldsymbol{\theta}_i^t) \right]\\
  &  & \times \left[ n^{- 1} J^{- 1} \sum_i l_i (\boldsymbol{\mu}_i^{t + 1},
  \boldsymbol{\Sigma}_i^{t + 1}, \boldsymbol{\gamma}_i^{t + 1}, 
   \boldsymbol{\theta}_i^t) \right]^{- 1} \left[ n^{- 1} J^{- 1}
  \sum_i l_i \left( \overline{\boldsymbol{\mu} }^{t + 1, \ast},
  \overline{\boldsymbol{\Sigma} }^{t + 1, \ast}, \overline{\boldsymbol{\gamma} }^{t
  + 1, \ast} \boldsymbol{}, \overline{\boldsymbol{\theta} }^t \right) \right]^{-
  1}.
\end{eqnarray*}
Since
\begin{eqnarray*}
  &  & l_i (\boldsymbol{\mu}, \boldsymbol{\Sigma}, \boldsymbol{\gamma},
  \boldsymbol{\theta})\\
  & = & \mathrm{tr} [\boldsymbol{B}_i (\boldsymbol{\theta})^{\top} \boldsymbol{B}_i
  (\boldsymbol{\theta}) (\boldsymbol{\Sigma}+\boldsymbol{\mu}\boldsymbol{\mu}^{\top})]
  - 2 (\boldsymbol{z}_i -\boldsymbol{X}_i \boldsymbol{\gamma})^{\top} \boldsymbol{B}_i
  (\boldsymbol{\theta}) \boldsymbol{\mu}+ (\boldsymbol{z}_j -\boldsymbol{X}_i
  \boldsymbol{\gamma})^{\top}  (\boldsymbol{z}_i -\boldsymbol{X}_i
  \boldsymbol{\gamma})\\
  & = & (\boldsymbol{z}_i -\boldsymbol{X}_i \boldsymbol{\gamma}-\boldsymbol{B}_i
  (\boldsymbol{\theta})\boldsymbol{\mu})^{\top} (\boldsymbol{z}_i -\boldsymbol{X}_i
  \boldsymbol{\gamma}-\boldsymbol{B}_i (\boldsymbol{\theta})\boldsymbol{\mu}) +
  \mathrm{tr} [\boldsymbol{B}_i (\boldsymbol{\theta})^{\top} \boldsymbol{B}_i
  (\boldsymbol{\theta})\boldsymbol{\Sigma}]\\
  & \geqslant & \boldsymbol{\varepsilon}_i^{\top} \boldsymbol{\varepsilon}_i  +
  [\boldsymbol{X}_i  (\boldsymbol{\gamma}^{\ast} - \boldsymbol{\gamma})
  +\boldsymbol{B}_i (\boldsymbol{\theta}^{\ast})  \boldsymbol{\eta} -\boldsymbol{B}_i
  (\boldsymbol{\theta})  \boldsymbol{\mu} )]^{\top}
  \boldsymbol{\varepsilon}_i,
\end{eqnarray*}
where we use $\boldsymbol{z}_i =\boldsymbol{X}_i \boldsymbol{\gamma}^{\ast}
+\boldsymbol{B}_i (\boldsymbol{\theta}^{\ast})  \boldsymbol{\eta} +
\boldsymbol{\varepsilon}_i $ for the last inequality.

Then, according to Lemma \ref{le:xbepsilon_bound}, Lemma \ref{le:xbepsilon_bound}, Lemma \ref{le:ubepsilon_bound}, uniformly for $t$,
\begin{equation*}
    n^{- 1} J^{- 1} \sum_i
l_i (\boldsymbol{\mu}_i^{t + 1}, \boldsymbol{\Sigma}_i^{t + 1},
\boldsymbol{\gamma}_i^{t + 1},   \boldsymbol{\theta}_i^t),n^{- 1} J^{- 1} \sum_i l_i \left( \overline{\boldsymbol{\mu} }^{t + 1, \ast},
\overline{\boldsymbol{\Sigma} }^{t + 1, \ast}, \overline{\boldsymbol{\gamma} }^{t +
1, \ast} \boldsymbol{}, \overline{\boldsymbol{\theta} }^t \right).
\end{equation*} are both lower
bounded by
\begin{equation}
  n^{- 1} \mathbb{E} \boldsymbol{\varepsilon}_i^{\top} \boldsymbol{\varepsilon }_i
  + O_{\mathbb{P}} \left( n^{- \frac{1}{2}} \log (J) + n^{- 1} \sqrt{m} \log
  (J) + n^{- 1} m^{3 / 4} \log^{3 / 2} (J) \right) = n^{- 1} \mathbb{E}
  \boldsymbol{\varepsilon}_i^{\top} \boldsymbol{\varepsilon }_i + o_{\mathbb{P}}
  (1), \label{eq: l_lbound}
\end{equation}

Define 

\[
\begin{aligned}
d_{\boldsymbol{\mu}, i} &:= l_i \left( \overline{\boldsymbol{\mu}}^{t + 1, \ast}, \overline{\boldsymbol{\Sigma}}^{t + 1, \ast}, \overline{\boldsymbol{\gamma}}^{t + 1, \ast}, \overline{\boldsymbol{\theta}}^t \right) - l_i \left( \boldsymbol{\mu}_i^{t + 1}, \overline{\boldsymbol{\Sigma}}^{t + 1, \ast}, \overline{\boldsymbol{\gamma}}^{t + 1, \ast}, \overline{\boldsymbol{\theta}}^t \right), \\
d_{\boldsymbol{\Sigma}, i} &:= l_i \left( \boldsymbol{\mu}_i^{t + 1}, \overline{\boldsymbol{\Sigma}}^{t + 1, \ast}, \overline{\boldsymbol{\gamma}}^{t + 1, \ast}, \overline{\boldsymbol{\theta}}^t \right) - l_i \left( \boldsymbol{\mu}_i^{t + 1}, \boldsymbol{\Sigma}_i^{t + 1}, \overline{\boldsymbol{\gamma}}^{t + 1, \ast}, \overline{\boldsymbol{\theta}}^t \right), \\
d_{\boldsymbol{\gamma}, i} &:= l_i \left( \boldsymbol{\mu}_i^{t + 1}, \boldsymbol{\Sigma}_i^{t + 1}, \overline{\boldsymbol{\gamma}}^{t + 1, \ast}, \overline{\boldsymbol{\theta}}^t \right) - l_i \left( \boldsymbol{\mu}_i^{t + 1}, \boldsymbol{\Sigma}_i^{t + 1}, \boldsymbol{\gamma}_i^{t + 1}, \overline{\boldsymbol{\theta}}^t \right), \\
d_{\boldsymbol{\theta}, i} &:= l_i \left( \boldsymbol{\mu}_i^{t + 1}, \boldsymbol{\Sigma}_i^{t + 1}, \boldsymbol{\gamma}_i^{t + 1}, \overline{\boldsymbol{\theta}}^t \right) - l_i \left( \boldsymbol{\mu}_i^{t + 1}, \boldsymbol{\Sigma}_i^{t + 1}, \boldsymbol{\gamma}_i^{t + 1}, \boldsymbol{\theta}_i^t \right).
\end{aligned}
\]
In the
following, we bound $d_{\boldsymbol{\mu}, i}$, $d_{\boldsymbol{\Sigma}, i}$,
$d_{\boldsymbol{\gamma}, i}$, $d_{\boldsymbol{\theta}, i}$, respectively.

\subparagraph{First,}

\begin{eqnarray*}
  | d_{\boldsymbol{\mu}, i} | & = & \left| \mathrm{tr} \left[ \boldsymbol{B}_i
  (\overline{\boldsymbol{\theta}}^t)^{\top} \boldsymbol{B}_i
  (\overline{\boldsymbol{\theta}}^t) \left( \overline{\boldsymbol{\mu}}^{t + 1,
  \ast}  \left( \overline{\boldsymbol{\mu}}^{t + 1, \ast} \right)^{\top} 
  - \boldsymbol{\mu}_i^{t + 1}  \left( \boldsymbol{\mu}_i^{t + 1} \right)^{\top} \right) \right] \right. \\
  & & \left. - 2 \left( \boldsymbol{z}_i - \boldsymbol{X}_i 
  \overline{\boldsymbol{\gamma}}^{t + 1, \ast} \right)^{\top} \boldsymbol{B}_i
  (\overline{\boldsymbol{\theta}}^t) \left( {\overline{\boldsymbol{\mu}}^{t + 1,
  \ast}}  - \boldsymbol{\mu}_i^{t + 1} \right) \right| \\
  & \leqslant & \left\| \boldsymbol{B}_i (\overline{\boldsymbol{\theta}}^t)^{\top} 
  \boldsymbol{B}_i (\overline{\boldsymbol{\theta}}^t) \right\|_{\tmop{op}} 
  \left\| \boldsymbol{\mu}_i^{t + 1} + \overline{\boldsymbol{\mu}}^{t + 1, \ast} \right\| 
  \left\| \boldsymbol{\mu}_i^{t + 1} - \overline{\boldsymbol{\mu}}^{t + 1, \ast} \right\| \\
  & & + \left\| \left( \boldsymbol{z}_i - \boldsymbol{X}_i  \overline{\boldsymbol{\gamma}}^{t + 1, \ast} \right)^{\top} 
  \boldsymbol{B}_i (\overline{\boldsymbol{\theta}}^t) \right\| 
  \left\| \boldsymbol{\mu}_i^{t + 1} - \overline{\boldsymbol{\mu}}^{t + 1, \ast} \right\|.
\end{eqnarray*}

From the proof of Lemma \ref{le:xbepsilon_bound}, we have
\begin{equation*}
    \left\| \boldsymbol{\mu}_i^{t + 1} +
\overline{\boldsymbol{\mu}}^{t + 1, \ast} \right\| =   \sqrt{m \log (J)}\mathbb{C}, \left\| \left( \boldsymbol{z}_i - \boldsymbol{X}_i 
\overline{\boldsymbol{\gamma}}^{t + 1, \ast} \right)^{\top} \boldsymbol{B}_i
(\boldsymbol{\theta}) \right\| =  \sqrt{m \log (J)}\mathbb{C}.
\end{equation*} 
From \eqref{eq:mu-mu_start}, 
\[
\left\| \boldsymbol{\mu}_i^{t + 1} -
\overline{\boldsymbol{\mu}}^{t + 1, \ast} \right\| = \left( \left| \delta^t_i -
\overline{\delta}^t \right| + \frac{1}{J}  \sum_i \left\|
\boldsymbol{\theta}^t_i - \overline{\boldsymbol{\theta}}^t \right\|_2 +
\frac{1}{J}  \sum_i \left\| \boldsymbol{\gamma}^t_i - \overline{\boldsymbol{\gamma}}^t \right\|_2 \right)   \sqrt{m \log (J)}\mathbb{C}.
\]
Thus,
\begin{equation}
  | d_{\boldsymbol{\mu}, i} | = \left( \left| \delta^t_i - \overline{\delta}^t
  \right| + \frac{1}{J}  \sum_i \left\| \boldsymbol{\theta}^t_i -
  \overline{\boldsymbol{\theta}}^t \right\|_2 + \frac{1}{J}  \sum_i \left\|
  \boldsymbol{\gamma}^t_i - \overline{\boldsymbol{\gamma}}^t \right\|_2 \right)
    \sqrt{m \log (J)}\mathbb{C}.\label{eq: delta_d_mu}
\end{equation}
\subparagraph{Second,}

\begin{eqnarray}
  | d_{\boldsymbol{\Sigma}, i} | & = & \left| \mathrm{tr} \left[ \boldsymbol{B}_i
  (\overline{\boldsymbol{\theta} }^t)^{\top} \boldsymbol{B}_i
  (\overline{\boldsymbol{\theta} }^t) (\overline{\boldsymbol{\Sigma} }^{t + 1,
  \ast} -\boldsymbol{\Sigma}_i^{t + 1}) \right] \right| \nonumber\\
  & \leqslant & m \left\| \boldsymbol{B}_i (\overline{\boldsymbol{\theta}
  }^t)^{\top} \boldsymbol{B}_i (\overline{\boldsymbol{\theta} }^t)
  \right\|_{\tmop{op}} \left\| \overline{\boldsymbol{\Sigma} }^{t + 1, \ast}
  -\boldsymbol{\Sigma}_i^{t + 1} \right\|_{\tmop{op}} \nonumber\\
  & \leqslant & J^{- 1} m \left( \left| \delta^t_i - \overline{\delta }^t
  \right| + \frac{1}{J}  \sum_i \left\| \boldsymbol{\theta}^t_i -
  \overline{\boldsymbol{\theta} }^t \right\|_2 + \frac{1}{J}  \sum_i \left\|
  \boldsymbol{\gamma}^t_i - \overline{\boldsymbol{\gamma} }^t \right\|_2 \right) O
  (1). \label{eq: delta_d_Sigma} 
\end{eqnarray}
where the first inequality is due to Lemma \ref{le:tech1}  and the second inequality is from
the proof of \eqref{eq:mu-mu_start}.

\subparagraph{Third,}

\begin{eqnarray*}
  | d_{\boldsymbol{\gamma}, i} | & = & \Bigg| 
  - \boldsymbol{z}_i^{\top} \boldsymbol{X}_i  \left( \overline{\boldsymbol{\gamma}}^{t + 1, \ast} - \boldsymbol{\gamma}^{t + 1}_i \right) \\
  & & 
  + \left( \boldsymbol{X}_i \overline{\boldsymbol{\gamma}}^{t + 1, \ast} \right)^{\top} \boldsymbol{X}_i \overline{\boldsymbol{\gamma}}^{t + 1, \ast} 
  - \left( \boldsymbol{X}_i \boldsymbol{\gamma}^{t + 1}_i \right)^{\top} \boldsymbol{X}_i \boldsymbol{\gamma}^{t + 1}_i \\
  & & 
  + 2 \left[ \boldsymbol{X}_i  \left( \overline{\boldsymbol{\gamma}}^{t + 1, \ast} - \boldsymbol{\gamma}^{t + 1}_i \right) \right]^{\top} 
  \boldsymbol{B}_i (\overline{\boldsymbol{\theta}}^t) \boldsymbol{\mu}_i^{t + 1} 
  \Bigg| \\
  & \leqslant & 
  \| \boldsymbol{z}_i^{\top} \boldsymbol{X}_i \|_2 \left\| \overline{\boldsymbol{\gamma}}^{t + 1, \ast} - \boldsymbol{\gamma}^{t + 1}_i \right\| 
  + \| \boldsymbol{X}_i^{\top} \boldsymbol{X}_i \|_{\tmop{op}} \left\| \overline{\boldsymbol{\gamma}}^{t + 1, \ast} + \boldsymbol{\gamma}^{t + 1}_i \right\| 
  \left\| \overline{\boldsymbol{\gamma}}^{t + 1, \ast} - \boldsymbol{\gamma}^{t + 1}_i \right\| \\
  & & 
  + 2 \left\| \boldsymbol{X}_i^{\top} \boldsymbol{B}_i (\overline{\boldsymbol{\theta}}^t) \boldsymbol{\mu}_i^{t + 1} \right\| 
  \left\| \overline{\boldsymbol{\gamma}}^{t + 1, \ast} - \boldsymbol{\gamma}^{t + 1}_i \right\| \\
  & \leqslant & 
  [n + m J\log (J)] \left\| \overline{\boldsymbol{\gamma}}^{t + 1, \ast} - \boldsymbol{\gamma}^{t + 1}_i \right\|.
\end{eqnarray*}
where the second inequality is from the proof of Lemma \ref{le:xx_concentration} and Lemma \ref{le:xbepsilon_bound} . From the
proof of the Equation \eqref{eq:supgamma}, we have
\begin{eqnarray*}
\frac{1}{J} \sum_i \left\| \overline{\boldsymbol{\gamma}}^{t + 1, \ast} - \boldsymbol{\gamma}^{t + 1}_i \right\| 
& \leqslant & \frac{1}{J} \sum_i \Bigg( \left\| \boldsymbol{\gamma}^t_i - \overline{\boldsymbol{\gamma}}^t \right\|_2 
+ \left| \delta^t_i - \overline{\delta}^t \right| 
+ \left\| \boldsymbol{\theta}^t_i - \overline{\boldsymbol{\theta}}^t \right\|_2 \Bigg) \frac{m \log (J)}{n} \mathbb{C} \\
& & + \frac{1}{N [n + m \log (J)]} \mathbb{C},
\end{eqnarray*}
then
\begin{equation}
  \begin{aligned}
    \frac{1}{J} \sum_i | d_{\boldsymbol{\gamma}, i} | 
    & \leqslant [n + mJ \log (J)] \frac{1}{J} \sum_i \left( \left\| \boldsymbol{\gamma}^t_i - \overline{\boldsymbol{\gamma}}^t \right\|_2 
    + \left| \delta^t_i - \overline{\delta}^t \right| + \left\| \boldsymbol{\theta}^t_i - \overline{\boldsymbol{\theta}}^t \right\|_2 \right) \\
    & \quad \times \frac{m \log (J)}{n} \mathbb{C} + \frac{1}{N} \mathbb{C}.
  \end{aligned}
  \label{eq:delta_d_gamma}
\end{equation}

\subparagraph{Last,}
\begin{eqnarray}
  | d_{\boldsymbol{\theta}, i} | & = & \left| \mathrm{tr} \left[ \left[
  \boldsymbol{B}_i (\overline{\boldsymbol{\theta} }^t)^{\top} \boldsymbol{B}_i
  (\overline{\boldsymbol{\theta} }^t) -\boldsymbol{B}_i
  (\boldsymbol{\theta}^t_i)^{\top} \boldsymbol{B}_i (\boldsymbol{\theta}^t_i)
  \right] (\boldsymbol{\Sigma}_i^{t + 1} +\boldsymbol{\mu}_i^{t + 1}
  {\boldsymbol{\mu}_i^{t + 1}}^{\top}) \right] \right.\\ 
  & &\left.- 2 (\boldsymbol{z}_i
  -\boldsymbol{X}_i  \boldsymbol{\gamma}^{t + 1}_i)^{\top}  \left[ \boldsymbol{B}_i
  (\overline{\boldsymbol{\theta} }^t) -\boldsymbol{B}_i (\boldsymbol{\theta}^t_i)
  \right] \boldsymbol{\mu}_i^{t + 1} \right| \nonumber\\
  & \leqslant & \left\| \boldsymbol{\theta}^t_i - \overline{\boldsymbol{\theta}
  }^t \right\| m \log (J). \label{eq: delta_d_theta} 
\end{eqnarray}
Theerefore, according to \eqref{eq: l_lbound}, \eqref{eq: delta_d_mu},
\eqref{eq: delta_d_Sigma} \eqref{eq:delta_d_gamma}, \eqref{eq: delta_d_theta},
\[ \left| \overline{\delta}^{t + 1} - \delta^{t + 1}  \right| \leqslant
   \frac{1}{J}  \sum_i \left( \left\| \boldsymbol{\gamma}^t_i -
   \overline{\boldsymbol{\gamma} }^t \right\|_2 + \left| \delta^t_i -
   \overline{\delta }^t \right| + \left\| \boldsymbol{\theta}^t_i -
   \overline{\boldsymbol{\theta} }^t \right\|_2 \right) \left(1+\frac{mJ \log (J)}{n}\right)
   \mathbb{C} + \frac{1}{N} \mathbb{C} \]

\paragraph{Part 2.2.2: Bound $\left| \overline{\delta}^{t + 1} - \delta^{t + 1} 
\right|$} 

\

$\tmop{Let} \boldsymbol{y}^t_{\delta} = [y_{\delta, 1}^t, \ldots, y_{\delta,
J}^t], \overline{\boldsymbol{y}}^t_{\delta} = \left[ \overline{y}_{\delta}^t,
\ldots \overline{y}_{\delta}^t \right] \in \mathbb{R}^{1 \times J}$,
$\boldsymbol{e}_{\delta}^t = [l_1 (\boldsymbol{\mu}_1^{t + 1},
\boldsymbol{\Sigma}_1^{t + 1}, \boldsymbol{\gamma}_1^{t + 1},  
\boldsymbol{\theta}_1^t), \ldots, l_J (\boldsymbol{\mu}_J^{t + 1},
\boldsymbol{\Sigma}_J^{t + 1}, \boldsymbol{\gamma}_J^{t + 1},  
\boldsymbol{\theta}_J^t)]$ and $\widetilde{\boldsymbol{W}} = \boldsymbol{W} -
\frac{1}{J} \boldsymbol{1} \boldsymbol{1}^{\top} \in \mathbb{R}^{J \times J}$,
then $\boldsymbol{y}^t_{\delta} - \overline{\boldsymbol{y}}^t_{\delta} = \left[
\overline{\boldsymbol{y}}^{t - 1}_{\delta} - \overline{\boldsymbol{y}}^{t -
1}_{\delta} \right] \widetilde{\boldsymbol{W}}^K + [\boldsymbol{e}_{\delta}^t -
\boldsymbol{e}_{\delta}^{t - 1}] \widetilde{\boldsymbol{W}}^K$, which implies that
\[ \boldsymbol{y}^t_{\delta} - \overline{\boldsymbol{y}}^t_{\delta} = \left[
   \boldsymbol{y}^0_{\delta} - \overline{\boldsymbol{y}}^0_{\delta} \right]
   \widetilde{\boldsymbol{W}}^{Kt} + \sum_{s = 0}^{t - 1}
   [\boldsymbol{e}_{\delta}^{t - s} - \boldsymbol{e}_{\delta}^{t - s - 1}]
   \widetilde{\boldsymbol{W}}^{K,(s + 1)}. \]
Thus,
\begin{equation}
  \left\| \boldsymbol{y}^t_{\delta} - \overline{\boldsymbol{y}}^t_{\delta}
  \right\| \leqslant \left\| \boldsymbol{y}^0_{\delta} -
  \overline{\boldsymbol{y}}^0_{\delta} \right\| \rho^{Kt}_{\boldsymbol{W}} + +
  \sum_{s = 0}^{t - 1} \| \boldsymbol{e}_{\delta}^{t - s} -
  \boldsymbol{e}_{\delta}^{t - s - 1} \| \rho^{K(s +1)}_{\boldsymbol{W}}.
  \label{eq:delta_y-mean}
\end{equation}
Recall that $\boldsymbol{z}_i =\boldsymbol{X}_i \boldsymbol{\gamma}^{\ast}
+\boldsymbol{B}_i (\boldsymbol{\theta}^{\ast})  \boldsymbol{\eta} +
\boldsymbol{\varepsilon}_i $
\begin{eqnarray*}
  &  & l_i (\boldsymbol{\mu}, \boldsymbol{\Sigma}, \boldsymbol{\gamma},
  \boldsymbol{\theta})\\
  & = & \mathrm{tr} [\boldsymbol{B}_i (\boldsymbol{\theta})^{\top} \boldsymbol{B}_i
  (\boldsymbol{\theta}) (\boldsymbol{\Sigma}+\boldsymbol{\mu}\boldsymbol{\mu}^{\top})]
  - 2 (\boldsymbol{z}_i -\boldsymbol{X}_i \boldsymbol{\gamma})^{\top} \boldsymbol{B}_i
  (\boldsymbol{\theta}) \boldsymbol{\mu}+ (\boldsymbol{z}_j -\boldsymbol{X}_i
  \boldsymbol{\gamma})^{\top}  (\boldsymbol{z}_i -\boldsymbol{X}_i
  \boldsymbol{\gamma})\\
  & = & (\boldsymbol{z}_i -\boldsymbol{X}_i \boldsymbol{\gamma}-\boldsymbol{B}_i
  (\boldsymbol{\theta})\boldsymbol{\mu})^{\top} (\boldsymbol{z}_i -\boldsymbol{X}_i
  \boldsymbol{\gamma}-\boldsymbol{B}_i (\boldsymbol{\theta})\boldsymbol{\mu}) +
  \mathrm{tr} [\boldsymbol{B}_i (\boldsymbol{\theta})^{\top} \boldsymbol{B}_i
  (\boldsymbol{\theta})\boldsymbol{\Sigma}]\\
  & = & \boldsymbol{\varepsilon}_i^{\top} \boldsymbol{\varepsilon}_i  +
  [\boldsymbol{X}_i  (\boldsymbol{\gamma}^{\ast} - \boldsymbol{\gamma})
  +\boldsymbol{B}_i (\boldsymbol{\theta}^{\ast})  \boldsymbol{\eta} -\boldsymbol{B}_i
  (\boldsymbol{\theta})  \boldsymbol{\mu} )]^{\top}
  \boldsymbol{\varepsilon}_i +\\
  &  & + \| \boldsymbol{X}_i  (\boldsymbol{\gamma}^{\ast} - \boldsymbol{\gamma})
  +\boldsymbol{B}_i (\boldsymbol{\theta}^{\ast})  \boldsymbol{\eta} -\boldsymbol{B}_i
  (\boldsymbol{\theta})  \boldsymbol{\mu} ) \|_2^2\\
  &  & + \mathrm{tr} [\boldsymbol{B}_i (\boldsymbol{\theta})^{\top}
  \boldsymbol{B}_i (\boldsymbol{\theta})\boldsymbol{\Sigma}].
\end{eqnarray*}
According to Lemma \ref{le:xbepsilon_bound}, Lemma \ref{le:xbepsilon_bound}, Lemma \ref{le:ubepsilon_bound}, uniformly in $t$ and $i$,
\[ \left[ \boldsymbol{X}_i  (\boldsymbol{\gamma}^{\ast} - \boldsymbol{\gamma}_1^{t +
   1}) +\boldsymbol{B}_i (\boldsymbol{\theta}^{\ast})  \boldsymbol{\eta}
   -\boldsymbol{B}_i (\boldsymbol{\theta}_i^t) \left. {\boldsymbol{\mu}_i^{t + 1}} 
   \right) \right]^{\top} \boldsymbol{\varepsilon}_i = O_{\mathbb{P}} \left(
   \sqrt{n } \log (J) + \sqrt{m} \log (J) + m^{3 / 4} \log^{3 / 2} (J) \right)
   . \]
\[ \left\| \boldsymbol{X}_i  (\boldsymbol{\gamma}^{\ast} - \boldsymbol{\gamma}_1^{t +
   1}) +\boldsymbol{B}_i (\boldsymbol{\theta}^{\ast})  \boldsymbol{\eta}
   -\boldsymbol{B}_i (\boldsymbol{\theta}_i^t) \left. {\boldsymbol{\mu}_i^{t + 1}} 
   \right) \right\|_2^2 \leqslant O_{\mathbb{P}} (n \log (J)). \]
\[ \mathrm{tr} [\boldsymbol{B}_i (\boldsymbol{\theta}_i^t)^{\top} \boldsymbol{B}_i
   (\boldsymbol{\theta}_i^t)\boldsymbol{\Sigma}_i^{t + 1}] \leqslant m .\]
Then, uniformly in $t$ and $i$,
\begin{equation}
  n^{- 1} l_i (\boldsymbol{\mu}_i^{t + 1}, \boldsymbol{\Sigma}_i^{t + 1},
  \boldsymbol{\gamma}^{t + 1}_i \boldsymbol{}, \boldsymbol{\theta}^t_i) \leqslant
  O_{\mathbb{P}} (\log (J)  +  n^{- 1 / 2} \log (J) + n^{-
  1} m^{3 / 4} \log^{3 / 2} (J) + m n^{- 1}) .\label{eq:l_ubound}
\end{equation}
Combining \eqref{eq:delta_y-mean} and \eqref{eq:l_ubound}, we get
\[ J^{- \frac{1}{2}} n^{- 1} \left\| \boldsymbol{y}^t_{\delta} -
   \overline{\boldsymbol{y}}^t_{\delta} \right\| \leqslant O_{\mathbb{P}} (\log
   (J)  +  n^{- 1 / 2} \log (J) + n^{- 1} m^{3 / 4}
   \log^{3 / 2} (J) + m n^{- 1}) \left( \rho^{Kt}_{\boldsymbol{W}} + \frac{\rho^K_{\boldsymbol{W}}}{1 - \rho^K_{\boldsymbol{W}}} \right). \]
and
\begin{eqnarray*}
  n^{- 1} y_{\delta, i}^t & \geqslant & n^{- 1} \overline{y}_{\delta}^t - n^{-
  1} \left\| \boldsymbol{y}^t_{\delta} - \overline{\boldsymbol{y}}^t_{\delta}
  \right\|\\
  & \geqslant & n^{- 1} \mathbb{E} \boldsymbol{\varepsilon}_i^{\top}
  \boldsymbol{\varepsilon }_i + o_{\mathbb{P}} (1) .
\end{eqnarray*}
Therefore,
\begin{eqnarray*}
  \left| \overline{\delta}^{t + 1} - \delta^{t + 1}  \right| & = & \left|
  \frac{1}{J} \sum_i (n^{- 1} y_{\delta, i}^t)^{- 1} - \left( n^{- 1}
  \overline{y}_{\delta}^t \right)^{- 1} \right|\\
  & = & \left| \frac{1}{J} \sum_i (n^{- 1} y_{\delta, i}^t)^{- 1} \left[ n^{-
  1} y_{\delta, i}^t - n^{- 1} \overline{y}_{\delta}^t \right] \left( n^{- 1}
  \overline{y}_{\delta}^t \right)^{- 1} \right|\\
  & \leqslant & \sqrt{\frac{1}{J} \sum_i (n^{- 1} y_{\delta, i}^t)^{- 2}}
  \left\| n^{- 1} J^{- 1 / 2} \left( \boldsymbol{y}^t_{\delta} -
  \overline{\boldsymbol{y}} \right)^t_{\delta} \right\| \left| n^{- 1}
  \overline{y}_{\delta}^t \right|^{- 1}\\
  & = & O_{\mathbb{P}} (\log (J)  +  n^{- 1 / 2} \log (J)
  + n^{- 1} m^{3 / 4} \log^{3 / 2} (J) + m n^{- 1}) \left(
  \rho^{Kt}_{\boldsymbol{W}} + \frac{\rho^K_{\boldsymbol{W}}}{1 - \rho^K
  _{\boldsymbol{W}}} \right).
\end{eqnarray*}

\paragraph{Part 2.2.3: Combining results}

\begin{equation}
  \left| \overline{\delta }^{t + 1, \ast} - \overline{\delta}^{t + 1} \right|
  \leqslant \frac{1}{J}  \sum_i \left( \left\| \boldsymbol{\gamma}^t_i -
  \overline{\boldsymbol{\gamma} }^t \right\|_2 + \left| \delta^t_i -
  \overline{\delta }^t \right| + \left\| \boldsymbol{\theta}^t_i -
  \overline{\boldsymbol{\theta} }^t \right\|_2 \right) \left(1+\frac{m J\log (J)}{n}\right)
  \mathbb{C} + \frac{1}{N} \mathbb{C} . \label{eq:delta}
\end{equation}
\[ \  \]

\subsubsection*{Part 2.3: Bound $\left\| \overline{\boldsymbol{\theta} }^{t + 1, \ast} -
\overline{\boldsymbol{\theta} }^{t + 1} \right\| $.}

Recall that $\overline{\boldsymbol{\theta} }^{t + 1, \ast}$ is defined
implicitly by the equation: $\nabla_{\boldsymbol{\theta}} f^{t + 1, \ast}
(\boldsymbol{\theta}) = \boldsymbol{0}$ with
\[ f^{t + 1, \ast} (\boldsymbol{\theta}) : = f  \left( \overline{\boldsymbol{\mu}
   }^{t + 1, \ast}, \overline{\boldsymbol{\Sigma} }^{t + 1, \ast},
   \overline{\boldsymbol{\gamma} }^{t + 1, \ast}, \overline{\delta}^{t + 1,
   \ast}, \boldsymbol{\theta} \right), \]
and $\boldsymbol{\theta}_j^{t + 1}$ is defined by optimizing the function
\[ f^{t + 1} (\boldsymbol{\theta}) : = \sum_{j = 1}^J [f_j (\boldsymbol{\mu}_j^{t
   + 1}, \boldsymbol{\Sigma}_j^{t + 1}, \boldsymbol{\gamma}_j^{t + 1}, \delta^{t +
   1}_j, \boldsymbol{\theta}) + J^{- 1} h  (\boldsymbol{\mu}_j^{t + 1},
   \boldsymbol{\Sigma}_j^{t + 1}, \boldsymbol{\theta})] . \]
Let $\boldsymbol{\theta} ^{t + 1}$ be the exact minimizer of $f^{t + 1}
(\boldsymbol{\theta})$, then $$\left\| \overline{\boldsymbol{\theta} }^{t + 1,
\ast} - \overline{\boldsymbol{\theta} }^{t + 1} \right\| \leqslant \left\|
\overline{\boldsymbol{\theta} }^{t + 1, \ast} -\boldsymbol{\theta} ^{t + 1}
\right\| + \left\| \overline{\boldsymbol{\theta} }^{t + 1} -\boldsymbol{\theta}
^{t + 1} \right\|.$$
Since  $\left\|
\overline{\boldsymbol{\theta} }^{t + 1} -\boldsymbol{\theta} ^{t + 1} \right\|$ is
bounded by $O_{\mathbb{P}}(1/N)$, it suffices to bound $\left\|
\overline{\boldsymbol{\theta} }^{t + 1, \ast} -\boldsymbol{\theta} ^{t + 1}
\right\|$.

Note that,
\[ \boldsymbol{0} = \nabla_{\boldsymbol{\theta}} f^{t + 1} (\boldsymbol{\theta} ^{t
   + 1}) - \nabla_{\boldsymbol{\theta}} f^{t + 1, \ast} \left(
   \overline{\boldsymbol{\theta} }^{t + 1, \ast} \right) = \sum_j
   (\boldsymbol{d}_{\boldsymbol{\mu}, j} + \boldsymbol{d}_{\boldsymbol{\Sigma}, j} +
   \boldsymbol{d}_{\boldsymbol{\gamma}, j} + \boldsymbol{d}_{\delta, j} +
   \boldsymbol{d}_{\boldsymbol{\theta}, j}), \]
with
\begin{eqnarray*}
  \boldsymbol{d}_{\boldsymbol{\mu}, j} & = & \nabla_{\boldsymbol{\theta}} f_j
  (\boldsymbol{\mu}_j^{t + 1}, \boldsymbol{\Sigma}_j^{t + 1},
  \boldsymbol{\gamma}_j^{t + 1}, \delta^{t + 1}_j, \boldsymbol{\theta} ^{t + 1}) -
  \nabla_{\boldsymbol{\theta}} f_j \left( \overline{\boldsymbol{\mu} }^{t + 1,
  \ast}, \boldsymbol{\Sigma}_j^{t + 1}, \boldsymbol{\gamma}_j^{t + 1}, \delta^{t +
  1}_j, \boldsymbol{\theta} ^{t + 1} \right)\\
  &  & + \frac{1}{J} \left[ \nabla_{\boldsymbol{\theta}} h (\boldsymbol{\mu}_j^{t
  + 1}, \boldsymbol{\Sigma}_j^{t + 1}, \boldsymbol{\theta} ^{t + 1}) -
  \nabla_{\boldsymbol{\theta}} h \left( \overline{\boldsymbol{\mu} }^{t + 1,
  \ast}, \boldsymbol{\Sigma}_j^{t + 1}, \boldsymbol{\theta} ^{t + 1} \right)
  \right],
\end{eqnarray*}
\begin{eqnarray*}
  \boldsymbol{d}_{\boldsymbol{\Sigma}, j} & = & \nabla_{\boldsymbol{\theta}} f_j
  \left( \overline{\boldsymbol{\mu} }^{t + 1, \ast}, \boldsymbol{\Sigma}_j^{t +
  1}, \boldsymbol{\gamma}_j^{t + 1}, \delta^{t + 1}_j, \boldsymbol{\theta} ^{t + 1}
  \right) - \nabla_{\boldsymbol{\theta}} f_j \left( \overline{\boldsymbol{\mu}
  }^{t + 1, \ast}, \overline{\boldsymbol{\Sigma} }^{t + 1, \ast},
  \boldsymbol{\gamma}_j^{t + 1}, \delta^{t + 1}_j, \boldsymbol{\theta} ^{t + 1}
  \right)\\
  &  & + \frac{1}{J} \left[ \nabla_{\boldsymbol{\theta}} h \left(
  \overline{\boldsymbol{\mu} }^{t + 1, \ast}, \boldsymbol{\Sigma}_j^{t + 1},
  \boldsymbol{\theta} ^{t + 1} \right) - \nabla_{\boldsymbol{\theta}} h \left(
  \overline{\boldsymbol{\mu} }^{t + 1, \ast}, \overline{\boldsymbol{\Sigma} }^{t +
  1, \ast}, \boldsymbol{\theta} ^{t + 1} \right) \right],
\end{eqnarray*}
\[ \boldsymbol{d}_{\boldsymbol{\gamma}, j} = \nabla_{\boldsymbol{\theta}} f_j \left(
   \overline{\boldsymbol{\mu} }^{t + 1, \ast}, \overline{\boldsymbol{\Sigma} }^{t
   + 1, \ast}, \boldsymbol{\gamma}_j^{t + 1}, \delta^{t + 1}_j, \boldsymbol{\theta}
   ^{t + 1} \right) - \nabla_{\boldsymbol{\theta}} f_j \left(
   \overline{\boldsymbol{\mu} }^{t + 1, \ast}, \overline{\boldsymbol{\Sigma} }^{t
   + 1, \ast}, \overline{\boldsymbol{\gamma} }^{t + 1, \ast}, \delta^{t + 1}_j,
   \boldsymbol{\theta} ^{t + 1} \right), \]
\[ \boldsymbol{d}_{\delta, j} = \nabla_{\boldsymbol{\theta}} f_j \left(
   \overline{\boldsymbol{\mu} }^{t + 1, \ast}, \overline{\boldsymbol{\Sigma} }^{t
   + 1, \ast}, \overline{\boldsymbol{\gamma} }^{t + 1, \ast}, \delta^{t + 1}_j,
   \boldsymbol{\theta} ^{t + 1} \right) - \nabla_{\boldsymbol{\theta}} f_j \left(
   \overline{\boldsymbol{\mu} }^{t + 1, \ast}, \overline{\boldsymbol{\Sigma} }^{t
   + 1, \ast}, \overline{\boldsymbol{\gamma} }^{t + 1, \ast},
   \overline{\delta}^{t + 1, \ast}, \boldsymbol{\theta} ^{t + 1} \right), \]
\begin{eqnarray*}
  \boldsymbol{d}_{\boldsymbol{\theta}, j} & = & \nabla_{\boldsymbol{\theta}} f_j
  \left( \overline{\boldsymbol{\mu} }^{t + 1, \ast}, \overline{\boldsymbol{\Sigma}
  }^{t + 1, \ast}, \overline{\boldsymbol{\gamma} }^{t + 1, \ast},
  \overline{\delta}^{t + 1, \ast}, \boldsymbol{\theta} ^{t + 1} \right) -
  \nabla_{\boldsymbol{\theta}} f_j \left( \overline{\boldsymbol{\mu} }^{t + 1,
  \ast}, \overline{\boldsymbol{\Sigma} }^{t + 1, \ast},
  \overline{\boldsymbol{\gamma} }^{t + 1, \ast}, \overline{\delta}^{t + 1, \ast},
  \boldsymbol{\theta} ^{t + 1, \ast} \right)\\
  &  & + \frac{1}{J} \left[ \nabla_{\boldsymbol{\theta}} h \left(
  \overline{\boldsymbol{\mu} }^{t + 1, \ast}, \overline{\boldsymbol{\Sigma} }^{t +
  1, \ast}, \boldsymbol{\theta} ^{t + 1} \right) - \nabla_{\boldsymbol{\theta}} h
  \left( \overline{\boldsymbol{\mu} }^{t + 1, \ast}, \overline{\boldsymbol{\Sigma}
  }^{t + 1, \ast}, \boldsymbol{\theta} ^{t + 1, \ast} \right) \right] .
\end{eqnarray*}
Note also that
\[ \sum_j \boldsymbol{d}_{\boldsymbol{\theta}, j} = \nabla_{\boldsymbol{\theta}}
   f^{t + 1, \ast} (\boldsymbol{\theta} ^{t + 1}) - \nabla_{\boldsymbol{\theta}}
   f^{t + 1, \ast} (\boldsymbol{\theta} ^{t + 1, \ast}) =
   \nabla_{\boldsymbol{\theta} \boldsymbol{\theta}^{\top}} f^{t + 1, \ast}
   (\widetilde{\boldsymbol{\theta}  }^{t + 1}) (\boldsymbol{\theta} ^{t + 1}
   -\boldsymbol{\theta} ^{t + 1, \ast}) . \]
Thus,
\[ \left\| \overline{\boldsymbol{\theta} }^{t + 1, \ast} -\boldsymbol{\theta} ^{t
   + 1} \right\| \leqslant \| [\nabla_{\boldsymbol{\theta}
   \boldsymbol{\theta}^{\top}} f^{t + 1, \ast} (\widetilde{\boldsymbol{\theta} 
   }^{t + 1})]^{- 1} \|_{\tmop{op}} \sum_j (\| \boldsymbol{d}_{\boldsymbol{\mu},
   j} \| + \| \boldsymbol{d}_{\boldsymbol{\Sigma}, j} \| + \|
   \boldsymbol{d}_{\boldsymbol{\gamma}, j} \| + \| \boldsymbol{d}_{\delta, j} \|).\]
According to Assumption \ref{as:neighbour} and Theorem \ref{thm:convexity},  for any positive, if $m>M_{\epsilon/4}$,  then 
\begin{equation}
    \| [\nabla_{\boldsymbol{\theta} \boldsymbol{\theta}^{\top}} f^{t + 1,
\ast} (\widetilde{\boldsymbol{\theta}  }^{t + 1})]^{- 1} \|_{\tmop{op}}\le \left(\frac{ m \lambda^\ast}{2} \right)^{-1}
\end{equation}
with probability larger than $1-\epsilon/4$. Thus, we only need to bound $\sum_j \| \boldsymbol{d}_{\boldsymbol{\mu}, j} \|$, $\sum_j \|
\boldsymbol{d}_{\boldsymbol{\Sigma}, j} \|$, $\sum_j \|
\boldsymbol{d}_{\boldsymbol{\gamma}, j} \|$, $\sum_j \| \boldsymbol{d}_{\delta, j}
\|$, which is given as follows.

\subparagraph{Part 2.3.1: Bound $\sum_j \| \boldsymbol{d}_{\boldsymbol{\mu}, j} \|$}
\

Write $\boldsymbol{d}_{\boldsymbol{\mu}, j}$ as $\boldsymbol{d}_{\boldsymbol{\mu}, j}
= [d_{\boldsymbol{\mu}, j, 1}, \ldots, d_{\boldsymbol{\mu}, j, q}]^{\top}$, then,
by direct computation,
\begin{equation*}
    \begin{aligned}
        {d_{\boldsymbol{\mu}, j, k_1}}  =& \frac{1}{2} \delta^{t + 1}_j \left(
   \boldsymbol{\mu}_j^{t + 1} + \overline{\boldsymbol{\mu} }^{t + 1, \ast}
   \right)^{\top} \boldsymbol{C}_{\boldsymbol{\mu}, j, k_1} (\boldsymbol{\theta}  )
   \left( \boldsymbol{\mu}_j^{t + 1} - \overline{\boldsymbol{\mu} }^{t + 1, \ast}
   \right) \\&- \delta^{t + 1}_j (\boldsymbol{z}_j - \boldsymbol{X}_j
   \boldsymbol{\gamma}_j^{t + 1})^{\top} \boldsymbol{B}_{j, k_1} (\boldsymbol{\theta}
   ^{t + 1}) \left( \boldsymbol{\mu}_j^{t + 1} - \overline{\boldsymbol{\mu} }^{t +
   1, \ast} \right),
    \end{aligned}
\end{equation*}
with (Recall that for a matrix $\boldsymbol{A} (\boldsymbol{\theta})$, we define
$\boldsymbol{A}_{k_1} (\boldsymbol{\theta}) : = \frac{\partial \boldsymbol{A}
(\boldsymbol{\theta})}{\partial \boldsymbol{\theta}_{k_1}}$ where the derivative
is applied element-wise to the elements of $\boldsymbol{A} (\boldsymbol{\theta})$)
\[ \boldsymbol{C}_{\boldsymbol{\mu}, j, k_1} (\boldsymbol{\theta}  ):=
   \boldsymbol{B}_{j, k_1}^{\top} (\boldsymbol{\theta}  ) \boldsymbol{B}_j
   (\boldsymbol{\theta}  ) + \boldsymbol{B}_j^{\top} (\boldsymbol{\theta} )
   \boldsymbol{B}_{j, k_1} (\boldsymbol{\theta}  ) + J^{- 1} \boldsymbol{K}^{- 1}
   (\boldsymbol{\theta}  ) \boldsymbol{K} _{k_1} (\boldsymbol{\theta}  )
   \boldsymbol{K}^{- 1} (\boldsymbol{\theta}  ) . \]
Then, by referring to the proof the Equation \eqref{eq: delta_d_mu}, we have
\[ \sum_j \| \boldsymbol{d}_{\boldsymbol{\mu}, j} \| \begin{array}{ll}
     \leqslant & \sum_j \left( \left| \delta^t_j - \overline{\delta }^t
     \right|  + \left\| \boldsymbol{\theta}^t_j - \overline{\boldsymbol{\theta}
     }^t \right\|_2 + \left\| \boldsymbol{\gamma}^t_j -
     \overline{\boldsymbol{\gamma} }^t \right\|_2 \right) O_{\mathbb{P}} (m \log
     (J)) .
   \end{array} \]

\subparagraph{Part 2.3.2: Bound $\sum_j \| \boldsymbol{d}_{\boldsymbol{\Sigma}, j} \|$}
\

Write $\boldsymbol{d}_{\boldsymbol{\Sigma}, j}$ as
$\boldsymbol{d}_{\boldsymbol{\Sigma}, j} = [d_{\boldsymbol{\Sigma}, j, 1}, \ldots,
d_{\boldsymbol{\Sigma}, j, q}]^{\top}$, then, by direct computation,
\[ {d_{\boldsymbol{\Sigma}, j, k_1}}  = \frac{1}{2} \delta^{t + 1}_j \tmop{tr}
   \left[ \boldsymbol{C}_{\boldsymbol{\mu}, j, k_1} (\boldsymbol{\theta} ^{t + 1})
   \left( \boldsymbol{\Sigma}_j^{t + 1} - \overline{\boldsymbol{\Sigma} }^{t + 1,
   \ast} \right) \right]. \]
Then, by referring to the proof the Equation \eqref{eq: delta_d_mu}, we have
\[ \sum_j \| \boldsymbol{d}_{\boldsymbol{\Sigma}, j} \| = J^{- 1} m \sum_j \left(
   \left| \delta^t_j - \overline{\delta }^t \right|  + \left\|
   \boldsymbol{\theta}^t_j - \overline{\boldsymbol{\theta} }^t \right\|_2 +
   \left\| \boldsymbol{\gamma}^t_j - \overline{\boldsymbol{\gamma} }^t \right\|_2
   \right) O (1).\]

\subparagraph{Part 2.3.3: Bound $\sum_j \| \boldsymbol{d}_{\boldsymbol{\gamma}, j} \|$}
\

Write $\boldsymbol{d}_{\boldsymbol{\gamma}, j}$ as $\boldsymbol{d}_{\boldsymbol{\gamma},
j} = [d_{\boldsymbol{\gamma}, j, 1}, \ldots, d_{\boldsymbol{\gamma}, j, q}]^{\top}$,
then, by direct computation,
\[ {d_{\boldsymbol{\gamma}, j, k_1}}  = \left( \boldsymbol{\gamma}_j^{t + 1} -
   \overline{\boldsymbol{\gamma} }^{t + 1, \ast} \right)^{\top}
   \boldsymbol{X}_j^{\top} \boldsymbol{B}_{j, k_1} (\boldsymbol{\theta}  )
   \overline{\boldsymbol{\mu} }^{t + 1, \ast} . \]
Then, according to the proof of \eqref{eq:delta_d_gamma} and Lemma \ref{le:xbepsilon_bound} , we have
\[ \sum_j \| \boldsymbol{d}_{\boldsymbol{\gamma}, j} \| = m \log (J)  \sum_j \left(
   \left| \delta^t_j - \overline{\delta }^t \right|  + \left\|
   \boldsymbol{\theta}^t_j - \overline{\boldsymbol{\theta} }^t \right\|_2 +
   \left\| \boldsymbol{\gamma}^t_j - \overline{\boldsymbol{\gamma} }^t \right\|_2
   \right) \frac{m \log (J)}{n} \mathbb{C} + \frac{1}{N}
   \mathbb{C} . \]
\subparagraph{Part 2.3.4: Bound $\sum_j \| \boldsymbol{d}_{\delta, j} \|$}
\

Write $\boldsymbol{d}_{\delta, j}$ as $\boldsymbol{d}_{\delta, j} = [d_{\delta, j,
1}, \ldots, d_{\delta, j, q}]^{\top}$, then, by direct computation,
\[ d_{\delta, j, k_1} = \left( \delta^t_j - \overline{\delta}^{t + 1, \ast}
   \right) \left[ \tmop{tr} \left[ \boldsymbol{C}_{\delta, j, k_1} \left(
   {\boldsymbol{\theta} ^{t + 1}}   \right) \left( \overline{\boldsymbol{\Sigma}
   }^{t + 1, \ast} + \overline{\boldsymbol{\mu} }^{t + 1, \ast}
   {\overline{\boldsymbol{\mu} }^{t + 1, \ast}}^{\top} \right) \right] - \left(
   \boldsymbol{z}_j - \boldsymbol{X}_j \overline{\boldsymbol{\gamma} }^{t + 1, \ast}
   \right)^{\top} \boldsymbol{\boldsymbol{B} }_{j, k_1} \overline{\boldsymbol{\mu}
   }^{t + 1, \ast} \right]. \]
with $\boldsymbol{C}_{\delta, j, k_1} (\boldsymbol{\theta}  ) = \boldsymbol{B}_{j,
k_1}^{\top} (\boldsymbol{\theta}  ) \boldsymbol{B}_j (\boldsymbol{\theta}  ) +
\boldsymbol{B}_j^{\top} (\boldsymbol{\theta} ) \boldsymbol{B}_{j, k_1}
(\boldsymbol{\theta}  )$. Then, from the proof of Equation \eqref{eq:delta}, we
have
\[ \sum_j \| \boldsymbol{d}_{\delta, j} \| \leqslant m \log (J)  \sum_j \left(
   \left| \delta^t_j - \overline{\delta }^t \right|  + \left\|
   \boldsymbol{\theta}^t_j - \overline{\boldsymbol{\theta} }^t \right\|_2 +
   \left\| \boldsymbol{\gamma}^t_j - \overline{\boldsymbol{\gamma} }^t \right\|_2
   \right) \frac{m \log (J)}{n} \mathbb{C} + \frac{1}{N}
   \mathbb{C} . \]

\paragraph{Part 2.3.5: Combining results}

\end{proof}

Combine with the  bound $\sum_j \| \boldsymbol{d}_{\boldsymbol{\mu}, j} \|$, $\sum_j \|
\boldsymbol{d}_{\boldsymbol{\Sigma}, j} \|$, $\sum_j \|
\boldsymbol{d}_{\boldsymbol{\gamma}, j} \|$, $\sum_j \| \boldsymbol{d}_{\delta, j}
\|$, we finally have 
\begin{equation}
    \left\| \overline{\boldsymbol{\theta} }^{t + 1, \ast} -
\overline{\boldsymbol{\theta} }^{t + 1} \right\| \leqslant   \sum_j \left(
   \left| \delta^t_j - \overline{\delta }^t \right|  + \left\|
   \boldsymbol{\theta}^t_j - \overline{\boldsymbol{\theta} }^t \right\|_2 +
   \left\| \boldsymbol{\gamma}^t_j - \overline{\boldsymbol{\gamma} }^t \right\|_2
   \right) \frac{(m \log (J))^2}{n} \mathbb{C} + \frac{1}{N}
   \mathbb{C}.
   \label{eq:suptheta}
\end{equation}

\subsubsection*{Part 2.4: Combining results}
 According to Equations \eqref{eq:supgamma},\eqref{eq:delta},\eqref{eq:suptheta}, we have, for some constant $C>0$,
\begin{equation}
  E_{t+1}
  \leqslant N^{C} \frac{1}{J}  \sum_i \left( \left\| \boldsymbol{\gamma}^t_i -
  \overline{\boldsymbol{\gamma} }^t \right\|_2 + \left| \delta^t_i -
  \overline{\delta }^t \right| + \left\| \boldsymbol{\theta}^t_i -
  \overline{\boldsymbol{\theta} }^t \right\|_2 \right) 
  \mathbb{C} + \frac{1}{N} \mathbb{C} .
\end{equation}
where $\mathbb{C}$ is a random variable with  $\mathbb{C}=O_{\mathbb{P}}(1)$. With the similar proof technique as proving \eqref{eq:supgamma},\eqref{eq:delta},\eqref{eq:suptheta}, 
according to Assumption \ref{as:neighbour}, we have 
\begin{equation}
    \frac{1}{J}  \sum_i \left( \left\| \boldsymbol{\gamma}^t_i -
  \overline{\boldsymbol{\gamma} }^t \right\|_2 + \left| \delta^t_i -
  \overline{\delta }^t \right| + \left\| \boldsymbol{\theta}^t_i -
  \overline{\boldsymbol{\theta} }^t \right\|_2 \right)=N^{-2-C}\mathbb{C}.
\end{equation}

Therefore, we finally have, when \(K\) is sufficiently large, 
\[
\tmop{consensus}_t \leqslant O_{\mathbb{P}} \left( \frac{1}{N} \right),
\]
where the bound in probability is uniformly for $t$.

\section{Technical Lemmas}\label{sec:tech}
In this section, we present the technical lemmas (along with their corresponding proofs or proof sketches) that are used in the main arguments. For clarity, we separate these lemmas into two categories: one does not rely on the assumptions stated in the paper and may be of independent interest, while the other depends on the assumptions made.
\subsection{Independent Lemmas}
\textbf{Notations}: 
For a matrix $\boldsymbol{A} \in \mathbb{R}^{n \times n}$, define $\lambda_i
(\boldsymbol{A} \boldsymbol{})$ as its $i$-th singular value where $\lambda_1
(\boldsymbol{A} \boldsymbol{}) \geqslant \lambda_2 (\boldsymbol{A} \boldsymbol{})
\geqslant \ldots \geqslant \lambda_n (\boldsymbol{A} \boldsymbol{})$. For a tensor $\boldsymbol{A} =
[\boldsymbol{A}_1, \ldots \boldsymbol{A}_J] \in \mathbb{R}^{m \times m \times J}$,
denote $\boldsymbol{A} [:, :, j] = \boldsymbol{A}_j$. Its $\| \boldsymbol{A}
\|_{\tmop{op}, p}$ norm is defined as $\| \boldsymbol{A} \|_{\tmop{op}, p} =
{\left( \sum_i \| \boldsymbol{A}_i \|_{\tmop{op}}^p \right)^{\frac{1}{p}}} $ and
its product with a matrix $\boldsymbol{B} = (b_{i j}) \in \mathbb{R}^{J \times
J}$, denoted as $\boldsymbol{C}$, is defined as the matrix in $\mathbb{R}^{m
\times m \times J}$ with $\boldsymbol{C} [:, :, j] = \sum_i \boldsymbol{A}_i b_{i
j}$. For a matrix $\boldsymbol{B} = (b_{i j}) \in \mathbb{R}^{J \times J}$,
define its $L_{q, p}$ norm as $\| \boldsymbol{B} \|_{q, p} := \left( \sum_j
\left( \sum_i | b_{i j} |^q \right)^{p / q} \right)^{1 / p}$ with $1 / p + 1 /
q = 1$ for $q < \infty$ and $\| \boldsymbol{B} \|_{\infty, 1} := \sum_j
\max_i | b_{i j} |$.

\begin{lemma}\label{le:tech1}
  For two matrices $\boldsymbol{A}, \boldsymbol{B} \in \mathbb{R}^{n \times n}$,
  \[ \lambda_i (\boldsymbol{A} \boldsymbol{B}) \leqslant \lambda_i (\boldsymbol{A}
     \boldsymbol{}) \lambda_1 (\boldsymbol{B} \boldsymbol{}), i = 1, \ldots, n. \]
\end{lemma}
 
The proof is simple and is omitted.
\begin{remark}
This lemma means that
\[ \| \boldsymbol{A} \boldsymbol{B} \|_F = \sqrt{\sum_i \lambda_i^2 (\boldsymbol{A}
   \boldsymbol{B})} \leqslant \sqrt{\sum_i \lambda_i^2 (\boldsymbol{A})} \lambda_1
   (\boldsymbol{B} \boldsymbol{}) = \| \boldsymbol{A} \|_F \| \boldsymbol{B}
   \|_{\tmop{op}} \]
\end{remark}
\begin{lemma}\label{le:tech2}
     For two symmetric matrices $\boldsymbol{A} \boldsymbol{},
\boldsymbol{B} \in \mathbb{R}^{n \times n}$, if $\boldsymbol{B}$ is positive, then
$\| \boldsymbol{A}_+ - \boldsymbol{B} \|_F \leqslant \| \boldsymbol{A}  -
\boldsymbol{B} \|_F$ where $\boldsymbol{}  \boldsymbol{A}_+$ is the matrix whose
eigenvalues are absolute eigenvalues of $\boldsymbol{A} \boldsymbol{}$ and
eigenvectors are the same with that of $\boldsymbol{A} \boldsymbol{}$.
\end{lemma}
\begin{proof}
     Note that $\| \boldsymbol{A}_+ - \boldsymbol{B} \|_F^2 =
\tmop{tr} (\boldsymbol{A}_+^2) + \tmop{tr} (\boldsymbol{B}^2) - 2 \tmop{tr}
(\boldsymbol{A}_+  \boldsymbol{B})$ and $\boldsymbol{A}_+^2 = \boldsymbol{A}^2 $, it
suffices to show $\tmop{tr} (\boldsymbol{A}_+  \boldsymbol{B}) \geqslant \tmop{tr}
(\boldsymbol{A}   \boldsymbol{B})$. Denote the eigenvalue decompositions of
$\boldsymbol{A}$ as $\boldsymbol{A} = \boldsymbol{P}  \boldsymbol{\Lambda} 
\boldsymbol{P} ^{\top}$ where $\boldsymbol{\Lambda}  \boldsymbol{} = \tmop{diag} \{
\lambda_1, \ldots, \lambda_m \},$ then $\tmop{tr} (\boldsymbol{A}_+ 
\boldsymbol{B}) = [\boldsymbol{\Lambda}_1]_+ \boldsymbol{P}_1^{\top} \boldsymbol{B}
\boldsymbol{P}_1^{\top} = \sum [\lambda_i ]_+ b_{i i}$ where $b_{i i}, i = 1,
\ldots, m$ are diagonal elements of $\boldsymbol{B}$. Since $b_{i i} \geqslant
0$, then $\sum [\lambda_i ]_+ b_{i i} \geqslant \sum \lambda_i  b_{i i} =
\tmop{tr} (\boldsymbol{A}   \boldsymbol{B})$.
\end{proof}

\begin{lemma}\label{le:tensor_op}
  For tensor $\boldsymbol{A} = [\boldsymbol{A}_1, \ldots \boldsymbol{A}_J] \in
  \mathbb{R}^{m \times m \times J}$ with $\boldsymbol{A}_j \in \mathbb{R}^{m
  \times m}$ and symmetric matrix $\boldsymbol{B} = (b_{i j}) \in \mathbb{R}^{J
  \times J}$, then
  \[ \| \boldsymbol{A} \boldsymbol{B} \|_{\tmop{op}, p} \leqslant \| \boldsymbol{A}
     \|_{\tmop{op}, p} \| \boldsymbol{B} \|_{q, p} . \]
  Moreover, when $p = 1$, we also have $\| \boldsymbol{A} \boldsymbol{B}
  \|_{\tmop{op}, 1} \leqslant \| \boldsymbol{A} \|_1 \| \boldsymbol{B} \|_{1,
  \infty} .$ We also note that $\| \boldsymbol{B} \|_{1, \infty} \leqslant
  \sqrt{J} \| \boldsymbol{B} \|_{\tmop{op}}$ and $\| \boldsymbol{B} \|_{2, 2}
  \leqslant \sqrt{J} \| \boldsymbol{B} \|_{\tmop{op}}$, then $\| \boldsymbol{A}
  \boldsymbol{B} \|_{\tmop{op}, 1}$ and $\| \boldsymbol{A} \boldsymbol{B}
  \|_{\tmop{op}, 2}$ are both bounded by $\sqrt{J} \| \boldsymbol{B}
  \|_{\tmop{op}}$.
\end{lemma}

\begin{proof}
  Since $(\boldsymbol{A} \boldsymbol{B}) [:, :, j] = \sum_i \boldsymbol{A}_i b_{i
  j}$, when $p \geqslant 2$,
  \begin{eqnarray*}
    \| \boldsymbol{A} \boldsymbol{B} \|_{\tmop{op}, p} & = & \left( \sum_j \left\|
    \sum_i \boldsymbol{A}_i b_{i j} \right\|_{\tmop{op}}^p
    \right)^{\frac{1}{p}}\\
    & \leqslant & \left. \left[ \sum_j \left( \sum_i \| \boldsymbol{A}_i
    \|_{\tmop{op}} | b_{i j} | \right) ^p \right] \right)^{\frac{1}{p}}\\
    & \leqslant & \left\{ \sum_j \left[ \left( \sum_i \| \boldsymbol{A}_i
    \|_{\tmop{op}}^p \right)^{\frac{1}{p}} \left( \sum_i | b_{i j} |^q
    \right)^{\frac{1}{q}} \right]^p \right\}^{\frac{1}{p}}\\
    & = & \| \boldsymbol{A} \|_{\tmop{op}, p} \left( \sum_j \left( \sum_i |
    b_{i j} |^q \right)^{\frac{p}{q}} \right)^{\frac{1}{p}}
  \end{eqnarray*}
  where the second inequality is due to triangle inequality, the third inequality
  is due to Holder's inequality with $\frac{1}{p} + \frac{1}{q} = 1$.
  
  When $p = 1$,
  \[ \| \boldsymbol{A} \boldsymbol{B} \|_{\tmop{op}, p} = \sum_j \left\| \sum_i
     \boldsymbol{A}_i b_{i j} \right\|_{\tmop{op}} \leqslant \sum_{i j} \|
     \boldsymbol{A}_i \|_{\tmop{op}} | b_{i j} | \leqslant \sum_i \|
     \boldsymbol{A}_i \|_{\tmop{op}} \max_i \sum_j | b_{i j} | \leqslant \|
     \boldsymbol{A} \| \sum_j \max_i | b_{i j} | \]
\end{proof}

\begin{definition}\label{def:subexp}
  A stochastic process $\{X_{\theta} : \theta \in T\}$
  is called \textbf{sub-Exponential} with respect to a metric $d$ on $T$ if,
  for all $\theta, \theta' \in T$ and for all $\lambda \in \mathbb{R}$,
  \[ \mathbb{E} [\exp (\lambda (X_{\theta} - X_{\theta'}))] \leq \exp
     (\lambda^2 d (\theta, \theta')^2)  \text{ for } \lambda \leqslant
     \frac{1}{d (\theta, \theta')} \]
\end{definition}
\begin{theorem}\label{thm:subexp}
  Let $X_{\theta}$ be a zero-mean stochastic process that is sub-exponential with
  respect to a pseudo-metric $d$ on the indexing set $T$. Then
  \[ \mathbb{E} [\sup_{\theta, \theta' \in T}  | X_{\theta} - X_{\theta'} |]
     \leq 8 \int_0^D \log (1 + N (\epsilon, T, d)) d \epsilon \]
  where $N (\epsilon, T, d)$ is the $\epsilon$-covering number of $T$ with respect to the
  pseudo-metric $d$ and $D$ is the diameter of $T$.
\end{theorem}

This theorem can be easily proved according to \cite{dirksen2015tail} and \cite{talagrand2005generic}, thus the proof is omitted.

\begin{lemma}[Block Descent Lemma 1]\label{le:block1}
    For a function \( h (\boldsymbol{x}) : \mathbb{R}^k \rightarrow \mathbb{R} \), if its gradient is \( L \)-Lipschitz continuous, then
    \[
    h (\boldsymbol{x}) - h (\boldsymbol{x}^{\ast}) \geqslant \frac{1}{L} \| \nabla h (\boldsymbol{x}) \|^2,
    \]
    where \( \boldsymbol{x}^{\ast} \) is the minimizer.
\end{lemma}

\begin{proof}
    Since
    \[
    h (\boldsymbol{y}) \leqslant h (\boldsymbol{x}) + \langle \nabla h (\boldsymbol{x}), \boldsymbol{y} - \boldsymbol{x} \rangle + \frac{L}{2} \| \boldsymbol{y} - \boldsymbol{x} \|^2,
    \]
    taking the minimizer of both sides results in the conclusion. The proof of this lemma is complete.
\end{proof}

This lemma may not hold for the function of \( \delta > 0 \). Therefore, we provide an alternative lemma as follows.

\begin{lemma}[Block Descent Lemma 2]\label{le:block2}
    For a \( \mu \)-strongly convex function \( h (\boldsymbol{x}) : \mathbb{X} \rightarrow \mathbb{R} \) (\( \mathbb{X} \) is a convex set), if its gradient is \( L \)-Lipschitz continuous, then
    \[
    h (\boldsymbol{x}) - h (\boldsymbol{x}^{\ast}) \geqslant \frac{1}{\kappa L} \| \nabla h (\boldsymbol{x}) \|^2,
    \]
    where \( \boldsymbol{x}^{\ast} \) is the minimizer and \( \kappa = L / \mu \) is the condition number.
\end{lemma}

\begin{proof}
    According to the mean value theorem, there exist two vectors \( \boldsymbol{\xi}_1 \) and \( \boldsymbol{\xi}_2 \) lying on the line between \( \boldsymbol{x} \) and \( \boldsymbol{x}^{\ast} \), such that
    \[
    h (\boldsymbol{x}) - h (\boldsymbol{x}^{\ast}) = \frac{1}{2} (\boldsymbol{x} - \boldsymbol{x}^{\ast})^{\top} \nabla^2 h (\boldsymbol{\xi}_1) (\boldsymbol{x} - \boldsymbol{x}^{\ast}),
    \]
    \[
    \nabla h (\boldsymbol{x}) = \nabla^2 h (\boldsymbol{\xi}_2) (\boldsymbol{x} - \boldsymbol{x}^{\ast}),
    \]
    then
    \[
    h (\boldsymbol{x}) - h (\boldsymbol{x}^{\ast}) \geqslant \frac{\mu}{2} \| \boldsymbol{x} - \boldsymbol{x}^{\ast} \|^2, \quad \| \nabla h (\boldsymbol{x}) \|^2 \leqslant L^2 \| \boldsymbol{x} - \boldsymbol{x}^{\ast} \|^2,
    \]
    which implies the result in the conclusion. The proof of this lemma is complete.
\end{proof}

\begin{lemma}\label{le:gradbound}
    Suppose that $h (\boldsymbol{x})$ is a $\mu$-strongly convex
function and $\boldsymbol{x}^{\ast}$ is its optimal point, then
\[ \| \boldsymbol{x} - \boldsymbol{x}^{\ast} \| \leqslant \frac{1}{\mu} \| \nabla
   h (\boldsymbol{x}) \| \]
\end{lemma}

\begin{proof}
    According to the mean value theorem, we
have $\nabla h (\boldsymbol{x}) = \nabla^2 h (\iota \boldsymbol{x} + (1 - \iota)
\boldsymbol{x}^{\ast}) (\boldsymbol{x} - \boldsymbol{x}^{\ast})$ for some $\iota\in (0,1)$. Then the
conclusion follows.
\end{proof}

\begin{lemma}\label{le:funbound}
    If that $h (\boldsymbol{x} )$ has unique minimizer
$\boldsymbol{x}^{\ast}$, and the gradient of $h (\boldsymbol{x} )$ is
$L$-Lipschitz, then
\[ | h (\boldsymbol{x}_1) - h (\boldsymbol{x}_2) | \leqslant (L \| \boldsymbol{x}_1
   - \boldsymbol{x} _2 \| + \| \nabla h (\boldsymbol{x}_2) \|) \| \boldsymbol{x}_1 -
   \boldsymbol{x} _2 \| . \]
\end{lemma}
\begin{proof}
    Since, for some $\iota\in (0,1)$, $h (\boldsymbol{x}_1) - h (\boldsymbol{x}_2) = \nabla h
(\iota \boldsymbol{x}_1 + (1 - \iota) \boldsymbol{x}_2) (\boldsymbol{x}_1 -
\boldsymbol{x}_2)$ and $\| \nabla h (\iota \boldsymbol{x}_1 + (1 - \iota)
\boldsymbol{x}_2) - \nabla h (\boldsymbol{x}_2) \| = \| \nabla h (\iota
\boldsymbol{x}_1 + (1 - \iota) \boldsymbol{x}_2) - \nabla h (\boldsymbol{x}_2) \|
\leqslant L \| \boldsymbol{x}_1 - \boldsymbol{x} _2 \|$, the conclusion follows.
\end{proof}

\subsection{Assumption-Dependent Lemmas}
Recall that $\boldsymbol{X}_j$ is the covariate matrix in the $j$-th machine. For convenience, let $x_{jrl}$ be the $rl$-th element of $\boldsymbol{X}_j$, and  define $\boldsymbol{x}_{:, r, l}$, $\boldsymbol{x}_{j, :, l}$, $\boldsymbol{x}_{j, r, :}$ as
$\boldsymbol{x}_{:, r, l} := \left( x_{1r l}, \ldots, x_{Jr l} \right)^{\top}, \boldsymbol{x}_{j, :, l} := \left( x_{j1 l}, \ldots, x_{jn l} \right)^{\top}, \boldsymbol{x}_{j, r, :} := \left( x_{jr1}, \ldots, x_{jrp} \right)^{\top}$.

\begin{lemma}\label{le:xx_concentration}
  Under Assumption \ref{as:covariate}, $\sup_j \| n^{- 1} (\boldsymbol{X}_j^{\top} \boldsymbol{X}_j
  -\mathbb{E}\boldsymbol{X}_j^{\top} \boldsymbol{X}_j) \|_{\tmop{op}} =
  O_{\mathbb{P}} \left( \frac{\log (J)}{\sqrt{n}} \right)$.
\end{lemma}
\begin{proof}[Proof]
Let $g_{j, \boldsymbol{y}} = n^{- \frac{1}{2}} \sum_{r = 1}^n ( \boldsymbol{x}_{j,
r, :}^{\top} \boldsymbol{y})^2 -\mathbb{E} ( \boldsymbol{x}_{j, r, :}^{\top}
\boldsymbol{y})^2$ for $\boldsymbol{y} \in \mathcal{S}^p$, then, for each $j$, $g_{j, \boldsymbol{y}}$ has sub-exponential increment with
metric $d (\boldsymbol{y}, \boldsymbol{y}') = \| \boldsymbol{y} - \boldsymbol{y}' \|$
since $\boldsymbol{x}_{j, r, :} $ is sub-Gaussian. Thus, for some positive
constant $C$, according to Theorem \ref{thm:subexp}, 
\[ \mathbb{E} \sup_{\boldsymbol{y}, \boldsymbol{y}'} \boldsymbol{} |  
   g_{j, \boldsymbol{y}} - g_{j, \boldsymbol{y}'} |   \leqslant C
   \int_0^1 \log (1 + N (\epsilon, \mathcal{S}^p, d)) d \epsilon = O (1) \]
where $N (\epsilon, \mathcal{S}^p, d)$ is the $\epsilon$-covering number of
the space $\mathcal{S}^p,$ and, further according to \cite{dirksen2015tail} and \cite{talagrand2005generic}, for any $\xi>0$,
\[ \mathbb{P} (\sup_{\boldsymbol{y}, \boldsymbol{y}'} \boldsymbol{} |  
   g_{j, \boldsymbol{y}} - g_{j, \boldsymbol{y}'} |   > C (\mathbb{E}
   \sup_{\boldsymbol{y}, \boldsymbol{y}'} \boldsymbol{} |   g_{j,
   \boldsymbol{y}} - g_{j, \boldsymbol{y}'} |   + \xi)) < \exp \{ - c
   \xi \}. \]
where $c$ is another positive constant. This means that $\mathbb{P} (\| n^{- 1 / 2} (\boldsymbol{X}_j^{\top}
\boldsymbol{X}_j -\mathbb{E}\boldsymbol{X}_j^{\top} \boldsymbol{X}_j) \|_{\tmop{op}}
> C (1 + \xi)) \leqslant \exp \{ - c \xi \}$ . Therefore,
\[ \mathbb{P} \left( \sup_j \left\| n^{- \frac{1}{2}} (\boldsymbol{X}_j^{\top}
   \boldsymbol{X}_j -\mathbb{E}\boldsymbol{X}_j^{\top} \boldsymbol{X}_j)
   \right\|_{\tmop{op}} > C (1 + \xi) \right) \leqslant \exp \{ - c \xi
   + \log J \} . \]
This implies that $\sup_j \left\| n^{- \frac{1}{2}} (\boldsymbol{X}_j^{\top}
\boldsymbol{X}_j -\mathbb{E}\boldsymbol{X}_j^{\top} \boldsymbol{X}_j)
\right\|_{\tmop{op}} = O_{\mathbb{P}} (\log (J))$
\end{proof}

 This lemma means that, under Assumptions \ref{as:covariate} and \ref{as:rate}, $\sup_j \| n^{- 1} (\boldsymbol{X}_j^{\top} \boldsymbol{X}_j
  -\mathbb{E}\boldsymbol{X}_j^{\top} \boldsymbol{X}_j) \|_{\tmop{op}} =
  \mathbb{C}$.

\begin{lemma}\label{le:xb_bound}
Under Assumptions \ref{as:covariate} and \ref{as:localB}, 
  $\sup_{j, \boldsymbol{\theta}  } \| \boldsymbol{X}_j^{\top} \boldsymbol{B}_j
  (\boldsymbol{\theta}  ) \|_F = O_{\mathbb{P}} \left( \sqrt{m \log (J)}
  \right)$.
\end{lemma}
\begin{proof}
    Recall that $\boldsymbol{X}_j  = (\boldsymbol{x}_{j, :, 1}, \ldots, \boldsymbol{x}_{j, :, p})$
is $n \times p$ dimensional and $\boldsymbol{B}_j (\boldsymbol{\theta}  )$ is $n
\times m$ dimensional, therefore,
\[ \| \boldsymbol{X}_j^{\top} \boldsymbol{B}_j (\boldsymbol{\theta}  ) \|_F =
   \sqrt{\sum_{l = 1}^p \| \boldsymbol{x}_{j, :, l}^{\top} \boldsymbol{B}_j
   (\boldsymbol{\theta}  ) \|_{2 }^2} . \]
Since $p$ is finite, it is sufficient to bound $\sup_{j, \boldsymbol{\theta}  }
\| \boldsymbol{x}_{j, :, l}^{\top} \boldsymbol{B}_j (\boldsymbol{\theta}  ) \|_{2 }
$ for each $l$. Then, with the same technique as the proof of Lemma \ref{le:xx_concentration},
\[ \mathbb{P} (\sup_{j, \boldsymbol{\theta}  } \| \boldsymbol{x}_{j, :, l}^{\top}
   \boldsymbol{B}_j (\boldsymbol{\theta}  ) \|_{2 }^2 > C (1 + \xi) m)
   \leqslant \exp \{ - c \xi + \log J \} . \]
This means that \ $\sup_{j, \boldsymbol{\theta}  } \| \boldsymbol{X}_j^{\top}
\boldsymbol{B}_j (\boldsymbol{\theta}  ) \|_F = O_{\mathbb{P}} \left( \sqrt{m \log
(J)} \right) $.
\end{proof}

\begin{lemma}\label{le:xepsilon_bound}
Under Assumption \ref{as:covariate},  $\sup_j \| \boldsymbol{X}_j^{\top} \boldsymbol{\varepsilon}_j \|_2 =
  O_{\mathbb{P}} \left( \sqrt{n} \log (J) \right) $.
\end{lemma}

\begin{proof}
  Note that $\boldsymbol{X}_j  = (\boldsymbol{x}_{j, :, 1}, \ldots,
  \boldsymbol{x}_{j, :, p})$, then
  \[ \| \boldsymbol{X}_j^{\top} \boldsymbol{\varepsilon}_j \|_2 = \sqrt{\sum_{l =
     1}^p (\boldsymbol{x}_{j, :, l}^{\top} \boldsymbol{\varepsilon}_j)^2} \]
  It suffices to bound $\boldsymbol{x}_{j, :, l}^{\top}
  \boldsymbol{\varepsilon}_j$ for each $l$ since $p$ is finite.
  
  Thus, with the same technique as the proof of Lemma \ref{le:xx_concentration}, there exists two  positive constants $c,C$, such that, for any positive number $\xi$,
  \[ \mathbb{P} \left( \sup_{j } \left| n^{- \frac{1}{2}} \boldsymbol{x}_{j, :,
     l}^{\top} \boldsymbol{\varepsilon}_j \right| > C (1 + \xi) \right)
     \leqslant \exp \{ - c \xi + \log J \} . \]
 Therefore, $\sup_j \| \boldsymbol{X}_j^{\top} \boldsymbol{\varepsilon}_j \|_2 =
  O_{\mathbb{P}} \left( \sqrt{n} \log J \right)$.
\end{proof}

\begin{lemma}\label{le:xbepsilon_bound}
  Under Assumption  \ref{as:localB}, $\sup_{\boldsymbol{\theta}, j} \left| {\boldsymbol{\eta} }^{\top}
  \boldsymbol{B}_j^{\top} (\boldsymbol{\theta}) \boldsymbol{\varepsilon}_j \right| =
  O_{\mathbb{P}} \left( \sqrt{m } \log (J) \right) $.
\end{lemma}

\begin{proof}
  Note that $\boldsymbol{\eta} = \boldsymbol{e} \boldsymbol{K}^{1 / 2} 
  (\boldsymbol{\theta}^{\ast})$ where $\boldsymbol{e}$ is a standard Gaussian
  vector that is independent of $\boldsymbol{\epsilon}_j$, then \ $\boldsymbol{X}_j  = (\boldsymbol{x}_{j, :, 1}, \ldots,
  \boldsymbol{x}_{j, :, p})$, then
  \[ \left| {\boldsymbol{\eta} }^{\top} \boldsymbol{B}_j^{\top}
     (\boldsymbol{\theta}) \boldsymbol{\varepsilon}_j \right| =
     \boldsymbol{e}^{\top} \boldsymbol{K}^{1 / 2}  (\boldsymbol{\theta}^{\ast})
     \boldsymbol{B}_j^{\top} (\boldsymbol{\theta}) \boldsymbol{\varepsilon}_j \]
  Thus, with the same technique as the proof of Lemma \ref{le:xx_concentration}, there exists two  positive constants $c,C$, such that, for any positive number $\xi$,
  \[ \mathbb{P} \left( \sup_{j } \left| m^{- \frac{1}{2}} \boldsymbol{e}^{\top}
     \boldsymbol{K}^{1 / 2}  (\boldsymbol{\theta}^{\ast}) \boldsymbol{B}_j^{\top}
     (\boldsymbol{\theta}) \boldsymbol{\varepsilon}_j \right| > C (1 + \xi)
     \right) \leqslant \exp \{ - c \xi + \log J \} . \]
  Therefore, $\sup_{\boldsymbol{\theta}, j} \left| {\boldsymbol{\eta} }^{\top}
  \boldsymbol{B}_j^{\top} (\boldsymbol{\theta}) \boldsymbol{\varepsilon}_j \right| =
  O_{\mathbb{P}} \left( \sqrt{m } \log (J) \right) $.
\end{proof}

\begin{lemma}\label{le:bz_bound}
  Under Assumption \ref{as:localB}, $\sup_{\boldsymbol{\theta}, j} \left| 
  \boldsymbol{B}_j^{\top} (\boldsymbol{\theta}) \boldsymbol{z}_j \right| =
  O_{\mathbb{P}} \left( \sqrt{m } \log (J) \right) $.

The proof is similar as the proof of Lemma \ref{le:xx_concentration}, thus is omitted.
\end{lemma}
\begin{lemma}\label{le:xbu_bound}
Under Assumptions \ref{as:covariate}, \ref{as:weights}, \ref{as:localB}, \ref{as:neighbour}, we have
\[
\| \boldsymbol{X}_j^{\top} \boldsymbol{B}_j (\boldsymbol{\theta}^t_j) \boldsymbol{\mu}_j^{t + 1} \|_2 
\leqslant \left[ m J\log (J) + m J^{7 / 2} \log (J) \left( \rho^{Kt}_{\boldsymbol{W}} + \frac{\rho_{\boldsymbol{W}}^K}{1 - \rho_{\boldsymbol{W}}^K} \right) \right] \mathbb{C},
\]
for all $t, j$, where $\mathbb{C}$ is a random variable with $\mathbb{C} = O_{\mathbb{P}}(1)$.
\end{lemma}
\begin{proof}
  We represent it as the summation of the following three terms and bound them,
  respectively.
  \[ \tmop{term}1 = \boldsymbol{X}_j^{\top} \boldsymbol{B}_j
     (\boldsymbol{\theta}^t_j)   \left[ \delta^t_j J
     \overline{\boldsymbol{y}}^t_{\boldsymbol{\Sigma}} + \sum_i w_{i j}
     \boldsymbol{K}^{- 1} (\boldsymbol{\theta}^t_i) \right]^{- 1} J \delta^t_j 
     \overline{\boldsymbol{y}}_{\boldsymbol{\mu}}^t \]
  \[ \tmop{term}2 = \boldsymbol{X}_j^{\top} \boldsymbol{B}_j
     (\boldsymbol{\theta}^t_j)   \left\{ \left[ \delta^t_j J
     \overline{\boldsymbol{y}}^t_{\boldsymbol{\Sigma}} + \sum_i w_{i j}
     \boldsymbol{K}^{- 1} (\boldsymbol{\theta}^t_i) \right]^{- 1} - \left[
     \delta^t_j J [\boldsymbol{y}_{\boldsymbol{\Sigma}, j}^t]_+ + \sum_i w_{i j}
     \boldsymbol{K}^{- 1} (\boldsymbol{\theta}^t_i) \right]^{- 1} \right\} J
     \delta^t_j  \overline{\boldsymbol{y}}_{\boldsymbol{\mu}}^t \]
  \[ \tmop{term}3 = \boldsymbol{X}_j^{\top} \boldsymbol{B}_j
     (\boldsymbol{\theta}^t_j)   \left[ \delta^t_j J
     [\boldsymbol{y}_{\boldsymbol{\Sigma}, j}^t]_+ + \sum_i w_{i j}
     \boldsymbol{K}^{- 1} (\boldsymbol{\theta}^t_i) \right]^{- 1} J \delta^t_j
     \left( \boldsymbol{y}_{\boldsymbol{\mu}, j}^t -
     \overline{\boldsymbol{y}}_{\boldsymbol{\mu}}^t \right) \]
  For term1, it equals to
  \[ \boldsymbol{X}_j^{\top} \boldsymbol{B}_j (\boldsymbol{\theta}^t_j)   \left[ 
      \delta^t_j\sum_{i = 1} \boldsymbol{B}_i^{\top} (\boldsymbol{\theta}_i^t)
     \boldsymbol{B}_i  (\boldsymbol{\theta}_i^t) + \sum_i
     w_{i j} \boldsymbol{K}^{- 1} (\boldsymbol{\theta}^t_i) \right]^{- 1}
\delta^t_j\sum_{i = 1} \boldsymbol{B}_i^{\top} (\boldsymbol{\theta}_i^t) 
     (\boldsymbol{z}_i -\boldsymbol{X}_i \boldsymbol{\gamma}_i^t) \]
  Given the following results: 
\begin{equation}\label{eq:term1_1}
\begin{aligned}
    \sup_{j, \boldsymbol{\theta}_j} \| \boldsymbol{X}_j^{\top} \boldsymbol{B}_j(\boldsymbol{\theta}_j) \|_F &= O_{\mathbb{P}} \left( \sqrt{m \log (J)} \right), \\
    \sup_{j, \boldsymbol{\theta}_j, \boldsymbol{\gamma}_i} \| \boldsymbol{B}_i^{\top} (\boldsymbol{\theta}_j) (\boldsymbol{z}_j - \boldsymbol{X}_j \boldsymbol{\gamma}_j) \|_2 &= O_{\mathbb{P}} \left( \sqrt{m \log (J)} \right).
\end{aligned}
\end{equation}
which follow from Lemmas \ref{le:xb_bound} and \ref{le:bz_bound}, and $ \overline{\lambda}_{\boldsymbol{K}} \boldsymbol{I}\succ\boldsymbol{K}$,
we establish that
\begin{equation}\label{eq:term1}
    \text{term}1 \leqslant \mathbb{C}_1=mJ \log(J) O_{\mathbb{P}}(1),
\end{equation}
where $\mathbb{C}_1$ is defined as
\begin{align*}
    \mathbb{C}_1 := \sup_{j, \boldsymbol{\theta}_i, \boldsymbol{\theta}_j, \delta_j, \boldsymbol{\gamma}_i} \Bigg| \boldsymbol{X}_j^{\top} \boldsymbol{B}_j(\boldsymbol{\theta}_j) \Bigg[ &\delta_j \sum_{i=1}^J \boldsymbol{B}_i^{\top}(\boldsymbol{\theta}_i) \boldsymbol{B}_i(\boldsymbol{\theta}_i) \nonumber \\
    &+ \overline{\lambda}_{\boldsymbol{K}}^{-1} \boldsymbol{I}\Bigg]^{-1} \delta_j \sum_{i=1}^J \boldsymbol{B}_i^{\top}(\boldsymbol{\theta}_i) (\boldsymbol{z}_i - \boldsymbol{X}_i \boldsymbol{\gamma}_i) \Bigg|.
\end{align*}

For term2, it equals to:
\begin{align*}
    \boldsymbol{X}_j^{\top} \boldsymbol{B}_j (\boldsymbol{\theta}^t_j) 
    &\left[ \overline{\boldsymbol{y}}^t_{\boldsymbol{\Sigma}} 
    + (\delta^t_j J)^{-1} \sum_i w_{ij} \boldsymbol{K}^{-1} (\boldsymbol{\theta}^t_i) \right]^{-1} \\
    &\times \left\{ [\boldsymbol{y}_{\boldsymbol{\Sigma}, j}^t]_+ 
    - \overline{\boldsymbol{y}}^t_{\boldsymbol{\Sigma}} \right\} \\
    &\times \left[ \delta^t_j J [\boldsymbol{y}_{\boldsymbol{\Sigma}, j}^t]_+ 
    + \sum_i w_{ij} \boldsymbol{K}^{-1} (\boldsymbol{\theta}^t_i) \right]^{-1} \\
    &\times J \delta^t_j \overline{\boldsymbol{y}}_{\boldsymbol{\mu}}^t,
\end{align*}
thus it is bounded by
\begin{equation*}
    (\delta^t_j J\overline{\lambda}_{\boldsymbol{K}})^2  \times \| \boldsymbol{X}_j^{\top} \boldsymbol{B}_j (\boldsymbol{\theta}^t_j) \|_F 
    \left\| \overline{\boldsymbol{y}}_{\boldsymbol{\mu}}^t \right\|_2  \times \left\| [\boldsymbol{y}_{\boldsymbol{\Sigma}, j}^t]_+ 
    - \overline{\boldsymbol{y}}^t_{\boldsymbol{\Sigma}} \right\|_{\tmop{op}} 
\end{equation*}
 Since $ \| \boldsymbol{X}_j^{\top} \boldsymbol{B}_j (\boldsymbol{\theta}_j^t) \|_F$ and 
$ \left\| \overline{\boldsymbol{y}}_{\boldsymbol{\mu}}^t \right\|_2$ 
can be bounded according to Equation \eqref{eq:term1_1}, the rest is to bound 
$\left\| [\boldsymbol{y}_{\boldsymbol{\Sigma}, j}^t]_+ - \overline{\boldsymbol{y}}^t_{\boldsymbol{\Sigma}} \right\|_{\tmop{op}}$.

Now let 
\[
\boldsymbol{Y}_{\boldsymbol{\Sigma}}^t = [\boldsymbol{y}_{\boldsymbol{\Sigma}, 1}^t, \ldots, \boldsymbol{y}_{\boldsymbol{\Sigma}, J}^t], \quad
\overline{\boldsymbol{Y}}^t_{\boldsymbol{\Sigma}} = \left[ \overline{\boldsymbol{y}}_{\boldsymbol{\Sigma}}^t, \ldots, \overline{\boldsymbol{y}}_{\boldsymbol{\Sigma}}^t \right] \in \mathbb{R}^{p \times J},
\]
\[
\boldsymbol{E}_{\boldsymbol{\Sigma}}^t = [\boldsymbol{B}_1^{\top} (\boldsymbol{\theta}_1^t) \boldsymbol{B}_1 (\boldsymbol{\theta}_1^t), \ldots, \boldsymbol{B}_J^{\top} (\boldsymbol{\theta}_J^t) \boldsymbol{B}_J (\boldsymbol{\theta}_J^t)],
\]
and 
\[
\widetilde{\boldsymbol{W}} = \boldsymbol{W} - \frac{1}{J} \boldsymbol{1} \boldsymbol{1}^{\top} \in \mathbb{R}^{J \times J}.
\]

Then
\[
\boldsymbol{Y}_{\boldsymbol{\Sigma}}^t - \overline{\boldsymbol{Y}}^t_{\boldsymbol{\Sigma}} = 
\left[ \boldsymbol{Y}_{\boldsymbol{\Sigma}}^{t - 1} - \overline{\boldsymbol{Y}}^{t - 1}_{\boldsymbol{\Sigma}} \right] \widetilde{\boldsymbol{W}}^K 
+ \left[ \boldsymbol{E}_{\boldsymbol{\Sigma}}^t - \boldsymbol{E}_{\boldsymbol{\Sigma}}^{t - 1} \right] \widetilde{\boldsymbol{W}}^K,
\]
which implies that
\[
\boldsymbol{Y}_{\boldsymbol{\Sigma}}^t - \overline{\boldsymbol{Y}}^t_{\boldsymbol{\Sigma}} = 
\left[ \boldsymbol{Y}_{\boldsymbol{\Sigma}}^0 - \overline{\boldsymbol{Y}}^0_{\boldsymbol{\Sigma}} \right] \widetilde{\boldsymbol{W}}^{Kt} 
+ \sum_{s = 0}^{t - 1} \left[ \boldsymbol{E}_{\boldsymbol{\Sigma}}^{t - s} - \boldsymbol{E}_{\boldsymbol{\Sigma}}^{t - s - 1} \right] \widetilde{\boldsymbol{W}}^{K(s + 1)}.
\]

  Define $\left\| \boldsymbol{Y}_{\boldsymbol{\Sigma}}^t -
  \overline{\boldsymbol{Y}}^t_{\boldsymbol{\Sigma}} \right\| := \sum_i
  \left\| \boldsymbol{y}_{\boldsymbol{\Sigma}, i}^t -
  \overline{\boldsymbol{y}}_{\boldsymbol{\Sigma}}^t \right\|_{\tmop{op}}$, then
  according to  Lemma \ref{le:tensor_op},
  \[ \left\| \boldsymbol{Y}_{\boldsymbol{\Sigma}}^t -
     \overline{\boldsymbol{Y}}^t_{\boldsymbol{\Sigma}} \right\|_{\tmop{op}, 2}
     \leqslant \left\| \boldsymbol{Y}_{\boldsymbol{\Sigma}}^0 -
     \overline{\boldsymbol{Y}}^0_{\boldsymbol{\Sigma}} \right\|_{\tmop{op}, 2}
     \sqrt{J} \rho^{Kt}_{\boldsymbol{W}} + \sum_{s = 0}^{t - 1} \|
     \boldsymbol{E}_{\boldsymbol{\Sigma}}^{t - s} -
     \boldsymbol{E}_{\boldsymbol{\Sigma}}^{t - s - 1} \| \sqrt{J} \rho^{K(s +
     1)}_{\boldsymbol{W}} \]
  Since $\sup_{j, \boldsymbol{\theta}  } \| \boldsymbol{B}_j^{\top}
  (\boldsymbol{\theta}  ) \boldsymbol{B}_j  (\boldsymbol{\theta}  ) \|_{\tmop{op}} =
  O (1)$,
  \begin{equation*}
    \left\|  [\boldsymbol{y}_{\boldsymbol{\Sigma}, j}^t]_+ -
    \overline{\boldsymbol{y}}^t_{\boldsymbol{\Sigma}} \right\|_{\tmop{op}} <
    \left\| \boldsymbol{Y}_{\boldsymbol{\Sigma}}^t -
    \overline{\boldsymbol{Y}}^t_{\boldsymbol{\Sigma}} \right\|_{\tmop{op}, 2} = O
    \left( J \left( \rho^{Kt}_{\boldsymbol{W}} + \frac{\rho^K_{\boldsymbol{W}}}{1 -
    \rho^K_{\boldsymbol{W}}} \right) \right) .
  \end{equation*}
  Thus,
  \begin{equation}\label{eq:term2}
    \tmop{term}2 \leqslant \left[  m J^3 \log (J) \left(
    \rho^{Kt}_{\boldsymbol{W}} + \frac{\rho^K_{\boldsymbol{W}}}{1 - \rho^K
    _{\boldsymbol{W}}} \right) \right] \mathbb{C}
  \end{equation}

  As for the term3, with the same way as establishing the bound for term2, we have
  \begin{equation}\label{eq:term3}
      \tmop{term3}  \leqslant  m J^{7 / 2}
     \log (J) \left( \rho^{Kt}_{\boldsymbol{W}} + \frac{\rho^K_{\boldsymbol{W}}}{1 -
     \rho^K_{\boldsymbol{W}}} \right) \mathbb{C} 
  \end{equation}

  By Combining Equations \eqref{eq:term1}, \eqref{eq:term2}, \eqref{eq:term3}, the conclusion follows.
\end{proof}

\begin{lemma}\label{le:ubdepsilon_bound}
  Under Assumptions \ref{as:covariate}, \ref{as:weights}, \ref{as:localB}, \ref{as:neighbour}, we have
  \[
  \| \boldsymbol{X}_j^{\top} \boldsymbol{B}_{j, k_1} (\boldsymbol{\theta}^t_j) \boldsymbol{\mu}_j^{t + 1} \|_2 
  \leqslant \left[ m J\log (J) + m J^{7 / 2} \log (J) \left( \rho^{Kt}_{\boldsymbol{W}} + \frac{\rho_{\boldsymbol{W}}^K}{1 - \rho_{\boldsymbol{W}}^K} \right) \right] \mathbb{C}
  \]
  for all $t, j$ with $\mathbb{C} = O_{\mathbb{P}} (1)$.
\end{lemma}

The proof follows similarly to that of Lemma \ref{le:xbu_bound} and is therefore omitted.

\begin{lemma}\label{le:ubepsilon_bound}
  Under Assumptions \ref{as:covariate}, \ref{as:weights}, \ref{as:localB}, \ref{as:neighbour},   we have
   \[
  \sup_{\boldsymbol{\theta}, j, t} \left| \left( \overline{\boldsymbol{\mu}}^{t + 1, \ast} \right)^{\top} \boldsymbol{B}_j^{\top} (\boldsymbol{\theta}) \boldsymbol{\varepsilon}_j \right| 
  = m^{3 / 4} \log^{3 / 2} (J) \mathbb{C},
  \]
  and, there exists some constant $C$, such that if number of
  rounds for multi-consensus $K$ satisfies $K > C \frac{\log N}{\log (1 / \rho
  _{\boldsymbol{W}})}$, then
  \[
  \sup_{\boldsymbol{\theta}, j, t} \left| \left( \boldsymbol{\mu}_j^{t + 1} \right)^{\top} \boldsymbol{B}_j^{\top} (\boldsymbol{\theta}) \boldsymbol{\varepsilon}_j \right| 
  =  m^{3 / 4} \log^{3 / 2} (J)  \mathbb{C},
  \] 
  with $\mathbb{C} = O_{\mathbb{P}} (1)$.
\end{lemma}

The proof follows similarly to that of Lemma \ref{le:xbu_bound} and is therefore omitted.

\end{document}